\documentclass[aos]{imsart}

\RequirePackage{amsthm,amsmath,amsfonts,amssymb}
\RequirePackage[numbers,sort&compress]{natbib}
\RequirePackage[colorlinks,citecolor=blue,urlcolor=blue]{hyperref}
\RequirePackage{graphicx}

\usepackage{dsfont} 	                    % For \mathds{1} blackboard bold 
\usepackage{bm}                             % For roman (upright) bold latin
\usepackage{mathtools}                      % Better math
\usepackage{mathrsfs}         
\usepackage{accents}
\usepackage{xr-hyper}
\usepackage[scientific-notation=true]{siunitx}
\usepackage{blkarray}
\usepackage{enumerate}
\usepackage{subfig}
\usepackage{centernot}
\usepackage{paralist}
\usepackage{stackrel}
\usepackage[table]{xcolor}
\usepackage{pagenote}
\usepackage{booktabs}
\usepackage{tikz}
\usepackage[english]{babel}
\usepackage{caption}
\usepackage{stackengine}
\usepackage{scalerel}
\makepagenote

\renewcommand*{\pagenotesubhead}[2]{}
\let\footnote\pagenote

\makeatletter
\let\std@thebibliography\thebibliography
\makeatother
\usepackage{bibunits}
\defaultbibliographystyle{imsart-number}
\defaultbibliography{FaIn}
\usepackage{algorithm}
\usepackage{algpseudocodex}

\usepackage{listings}
\usepackage{color}

\definecolor{lightgrey}{rgb}{0.9,0.9,0.9}
\definecolor{darkgreen}{rgb}{0,0.6,0}
\startlocaldefs
\theoremstyle{plain}

\newtheorem{theorem}{Theorem}
\newtheorem{lemma}{Lemma}
\newtheorem{proposition}{Proposition}
\newtheorem{corollary}{Corollary}
\theoremstyle{remark}
\newtheorem{remark}{Remark}
\newtheorem{definition}{Definition}
\newtheorem{example}{Example}

\theoremstyle{plain}
\newtheorem{IState}{Indeterminacy statement}%[section]

\endlocaldefs
\newcommand*{\QEDE}{\hfill\ensuremath{\blacksquare}}
\def\ci{\perp\!\!\!\perp}

\DeclareMathOperator*{\argmax}{arg\,max}    % maxmizing argument
\DeclareMathOperator*{\argmin}{arg\,min}    % minimizing argument
\DeclareMathOperator{\tr}{tr}               % trace
\DeclareMathOperator{\proj}{proj}           % projection
\DeclareMathOperator\erf{erf}               % Gaussian error function
\newcommand{\LA}{\mathbf{\Lambda}}
\newcommand{\LAt}{\mathbf{\Lambda}^{\top}}
\newcommand{\PS}{\mathbf{\Psi}}
\newcommand{\PSi}{\mathbf{\Psi}^{-1}}
\newcommand{\PH}{\mathbf{\Phi}}
\newcommand{\PHi}{\mathbf{\Phi}^{-1}}
\newcommand{\Com}{\mathbf{\Lambda\Phi\Lambda}^{\top}}
\newcommand{\Gram}{\mathbf{\Lambda}^{\top}\mathbf{\Psi}^{-1}\mathbf{\Lambda}}
\newcommand{\xih}{\hat{\xi}}
\newcommand{\eph}{\hat{\epsilon}}
\newcommand{\bull}{\scaleobj{.4}{\bullet}}
\newcommand{\lambdaDot}{\stackinset{c}{-.1ex}{b}{0.2ex}{$\bull$}{$\LA$}}
\newcommand{\HDot}{\stackinset{c}{0ex}{b}{0.2ex}{$\bull$}{$\mathbf{H}$}}

\newlength{\dhatheight}
\newcommand{\doublehat}[1]{%
	\settoheight{\dhatheight}{\ensuremath{\hat{#1}}}%
	\addtolength{\dhatheight}{-0.35ex}%
	\hat{\vphantom{\rule{1pt}{\dhatheight}}%
		\smash{\hat{#1}}}}

\AtBeginEnvironment{bmatrix}{\setlength{\arraycolsep}{5pt}}
\AtBeginEnvironment{array}{\setlength{\arraycolsep}{5pt}}
\colorlet{BLUE}{blue}
\newcommand{\boo}{\color{blue}{SUPPLEMENT}}
\newcommand{\booz}{\color{blue}{SUPPLEMENTARY MATERIAL}}
\runauthor{C.F.W.\ Peeters}
\begin{document}

%--------------- Front Matter ---------------------------------------
%--------------------------------------------------------------------
\begin{frontmatter}
	\title{Perspectives on Latent Factor Indeterminacy\\and its Implications for Data Representation}
	%\title{A sample article title with some additional note\thanksref{t1}}
	\runtitle{Latent Factor Indeterminacy and its Implications}
	%\thankstext{T1}{A sample additional note to the title.}
	
\begin{aug}
	%%%%%%%%%%%%%%%%%%%%%%%%%%%%%%%%%%%%%%%%%%%%%%%
	%% Only one address is permitted per author. %%
	%% Only division, organization and e-mail is %%
	%% included in the address.                  %%
	%% Additional information can be included in %%
	%% the Acknowledgments section if necessary. %%
	%% ORCID can be inserted by command:         %%
	%% \orcid{0000-0000-0000-0000}               %%
	%%%%%%%%%%%%%%%%%%%%%%%%%%%%%%%%%%%%%%%%%%%%%%%
	\author[A]{\fnms{Carel F.W.}~\snm{Peeters} \ead[label=e1]{carel.peeters@wur.nl}}
	%%%%%%%%%%%%%%%%%%%%%%%%%%%%%%%%%%%%%%%%%%%%%%
	%% Addresses                                %%
	%%%%%%%%%%%%%%%%%%%%%%%%%%%%%%%%%%%%%%%%%%%%%%
	\address[A]{Mathematical \& Statistical Methods group -- Biometris,\\
	Wageningen University \& Research \printead[presep={ ,\ }]{e1}}
\end{aug}

\begin{abstract}
	The common factor analytic model is related to Helmholtz and Boltzmann machines, can be conceived as a linear autoencoder, or can be thought of as a single-hidden-layer generative neural network. 
	We thus consider it a basal generative representation learner that can be used as a minimal model for studying the foundational characteristics of (deep) generative model architectures. 
	We focus on the fundamental problem of indeterminacy in latent factor projections.
	This indeterminacy implies that, even when the intrinsic dimension of the latent vector is known, regularity conditions are met, and rotational indeterminacy is resolved, an inherent indefiniteness in the retrieval of causative latent sources remains: they will be uncertain, distributionally deviant, and non-unique. 
	This can have major implications for data representation but remains an elusive issue, even to practitioners and theorists well-versed in the factor model.
	Moreover, this classic psychometric problem is intricately related to the modern issue of latent variable collapse in the variational autoencoder framework for deep generative modeling.
	Here, we assess this indeterminacy from various perspectives and show how these are mathematically and conceptually related and we discuss subsequent implications for the Psychometrics, Statistics, and Artificial Intelligence communities.
	We show that one has latent factor determinacy across all its facets when the feature-dimension grows to infinity.
	This feeds into an essentially distribution-free estimation approach in the sample case when the number of features grows very large.
	We conclude, as these are emergent properties at scale, that the factor model is suited for representation learning of very-high-dimensional data.
\end{abstract}



\end{frontmatter}
%%%%%%%%%%%%%%%%%%%%%%%%%%%%%%%%%%%%%%%%%%%%%%
%% Please use \tableofcontents for articles %%
%% with 50 pages and more                   %%
%%%%%%%%%%%%%%%%%%%%%%%%%%%%%%%%%%%%%%%%%%%%%%
%\tableofcontents
%--------------------------------------------------------------------
%--------------- Front Matter ---------------------------------------

\begin{bibunit}
%--------------------------------------------------------------------
%%%%%%%%%%%%%%%%%%%%%%%%%%%%%%%%%%%%%%%%%%%%%%%%%%%%%%%%%%%%%%%%%%%%%
%%%%%%%%%%%%%%%% MAIN TEXT %%%%%%%%%%%%%%%%%%%%%%%%%%%%%%%%%%%%%%%%%%
%%%%%%%%%%%%%%%%%%%%%%%%%%%%%%%%%%%%%%%%%%%%%%%%%%%%%%%%%%%%%%%%%%%%%

%--------------- Introduction ---------------------------------------
\section{Introduction}\label{SEC:Intro}
\subsection{The common factor model}\label{SSEC:FAmodel}
Consider the common factor analytic model:
\begin{equation}\label{EQ:FAmodel}
    x := \mathbf{\Lambda}\xi + \epsilon,
\end{equation}
where the random $p$-dimensional vector of observable variables $x$ is driven by the random vector of common latent factors $\xi$ whose dimension $m < p$.
In (\ref{EQ:FAmodel}), $\epsilon \in \mathbb{R}^{p}$ denotes the vector of unique factors while $\mathbf{\Lambda} \in \mathbb{R}^{p \times m}$ denotes the matrix of factor loadings in which each element $\lambda_{jk}$ is the loading of the $j$th variable on the $k$th factor, $j = 1, \ldots, p$, $k = 1, \ldots, m$.
The model can thus be viewed as a multivariate regression model with random covariates that are not directly observable.
The model can also be conceived as a linear version of the Helmholtz machine \citep{ref_Helm96}, a directed generative variant of the Boltzmann machine \citep{ref_RBM,ref_BoltzFA}, or a linear autoencoder \citep{ref_DeepLVMs}.
This model and its variants are staples in the realm of representation learning \citep{ref_FArepresent}, especially in the forms of dimension reduction, unsupervised preprocessing, and pre-training.

We will assume, without loss of generality, that $x$ is standardized.
Furthermore, we postulate (i) $\xi \sim \mathcal{N}_m(\boldsymbol{0}, \mathbf{\Phi})$, with $\mathbf{\Phi} \odot \mathbf{I}_m = \mathbf{I}_m$; (ii) $\epsilon \sim \mathcal{N}_p(\boldsymbol{0}, \mathbf{\Psi})$; and (iii) $\xi \ci \epsilon$.
It is well-known that under these postulates the correlation matrix on $x$, say $\mathbf{\Sigma}_{xx}$, is implied to be:
\begin{equation}\label{EQ:FAcov}
     \mathbf{\Lambda}\mathbf{\Phi}\mathbf{\Lambda}^{\top} + \mathbf{\Psi} \equiv \mathbf{\Sigma}(\mathbf{\Theta})_{xx}.
\end{equation}
The model in (\ref{EQ:FAcov}) is quite general in that it covers the orthogonal exploratory factor model (a variant of the restricted Boltzmann machine), its rotational extension to an oblique model, as well as the confirmatory factor model (the latter two being variants of the unrestricted Boltzmann machine), depending on how one wishes to treat restrictions in the quadruple $\mathbf{\Theta} = \{m, \mathbf{\Psi}, \mathbf{\Lambda}, \mathbf{\Phi}\}$.
Let $r_{\Psi}$ denote the number of (linear) restrictions on $\mathbf{\Psi}$ and let $r_{\Lambda\Phi}$ denote the number of (linear) restrictions on $\{\mathbf{\Lambda,\Phi}\}$.
In the remainder we will then also assume that (iv) $m$ is known and minimal as well as the usual regularity assumptions:
(v) $r_{\Psi} \geq pm + m(m - 1)/2 - r_{\Lambda\Phi}$, with $r_{\Lambda\Phi} \geq m^{2} - m$;
(vi) $\mbox{rank}(\mathbf{\Lambda}) = m$;
(vii) $m < p/2$;
(viii) The restrictions numbered by $r_{\Lambda\Phi}$ provide rotational uniqueness;
(ix) $\mathbf{\Psi} \succ 0$; and
(x) $\mathbf{\Phi} \succ 0$.

Assumption (iv) rules out Reiers{\o}l constructions \citep{ref_Reiersol,ref_EBG78}.
Assumption (v) is a simple general order condition, stating that the number of unknowns in $\mathbf{\Sigma}(\mathbf{\Theta})_{xx}$ cannot exceed the number of nonredundant elements in $\mathbf{\Sigma}_{xx}$.
It is a necessity for the identification of $\mathbf{\Psi}$, which is further contingent upon assumption (vi) \citep{ref_GZ1980,PeetersThesis}.
The full rank of $\mathbf{\Lambda}$ is usually guaranteed under assumption (vii) \citep{ref_AndersonRubinClassic}.
Assumptions (v)--(vii) are thus likely to identify $\mathbf{\Psi}$ given $m$.
Assumption (viii) deals with rotational indeterminacy.
There are many ways to achieve this through the imposition of minimally $m^{2} - m$ additional restrictions.
In the exploratory setting one might consider $\mathbf{\Phi} = \mathbf{I}_m$ and that both $\mathbf{\Psi}$ and $\mathbf{\Lambda}^{\top}\mathbf{\Psi}^{-1}\mathbf{\Lambda}$ be diagonal with an order condition on the latter.
Unrestricted confirmatory settings might assume $\mathbf{\Psi}$ diagonal and impose $m^{2} - m$ fixed-value restrictions on $\mathbf{\Lambda}$ \citep{ref_Peeters12}.
Restricted confirmatory settings might further restrict $\mathbf{\Lambda}$ when allowing for a limited set of non-diagonal free elements in $\mathbf{\Psi}$ \citep{ref_BMW94}.
Assumption (viii) thus identifies (up to possibly polarity reversals in the columns of $\mathbf{\Lambda}$) $\{\mathbf{\Lambda,\Phi}\}$ given $\mathbf{\Psi}$ and $m$.
Note that (i) has already given $\xi$ a scale.
Assumption (x) then states that $\mathbf{\Phi}$ is also invertible, as does (ix) for $\mathbf{\Psi}$.
Hence, assumptions (iv)--(x) imply the identification of $\mathbf{\Theta}$ and the existence of $\mathbf{\Sigma}(\mathbf{\Theta})_{xx}^{-1}$.
An inherent indeterminacy, however, remains.

\subsection{Latent factor indeterminacy}\label{SSEC:Indet}
The retrieval of $\xi$ from $x$ and $\mathbf{\Theta}$ is of key interest.
Such projections provide the basis for data representation.
That is, the learning (through the factor model) of low-dimensional representations that enhance interpretability, reveal hidden features, and support downstream modeling or transfer learning. 
A schematic of our basic model and key interest is given in Figure \ref{FIG:Model}.

One could, given the model, find a proxy for $\xi$ through the decoder $p(\xi|x)$.
But proxies can also be obtained through other means \citep[see, e.g., the least-squares methods in][]{ref_LSscores}.
Hence, we provide a more general definition on what counts as a proxy for $\xi$ (and $\epsilon$):

\begin{definition}[Proxies of $(\xi,\epsilon)$]\label{DEF:Proxy}
	Let our proxy for $\xi$ be a random variate, as a function of $\mathbf{\Theta}$, based on the probability space on which $x$ is defined:
	\begin{equation*}
		\tilde{\xi}_{\mathbf{\Theta},x} \equiv f(\mathbf{\Theta},x).
	\end{equation*}
	A proxy for $\epsilon$ can, under our model and given 
	$\tilde{\xi}_{\mathbf{\Theta},x}$, then be found as:
	\begin{equation*}
		x - \LA\tilde{\xi}_{\mathbf{\Theta},x} \equiv \tilde{\epsilon}_{\mathbf{\Theta},x,\tilde{\xi}}.
	\end{equation*}
	We will use $\tilde{\xi}$ and $\tilde{\epsilon}$ as shorthands for our stated proxies.
\end{definition}

Even when $\mathbf{\Theta}$ is fully identified, the retrieval of $\xi$ by $\tilde{\xi}$ will likely suffer from \emph{factor indeterminacy}, in the sense of:
\begin{enumerate}[i]
	\item \emph{Uncertainty}: there might be (considerable) uncertainty regarding $\tilde{\xi}$ as in $\|\xi - \tilde{\xi}\|_2^2 > 0$, for example;
	\item \emph{Distributional deviance}: $\tilde{\xi} \not\sim \xi$, that is, $\tilde{\xi}$ does not display the postulated distributional behavior of $\xi$; and
	\item \emph{Non-uniqueness}: the possibility to construct, on the basis of $\tilde{\xi}$, an infinity of proxies dependent on arbitrary components, that do adhere to the model postulates but imply differing positionings in latent space.
\end{enumerate}

\begin{figure}[h!]
	\centering
	% Styles
	\tikzset{
		latent/.style={circle, thick, minimum size=1.2cm, draw=black!100, fill=white!100},
		observed/.style={circle, thick, minimum size=1.2cm, draw=black!100, fill=gray!30}}
	\tikzstyle{plate} = [draw, rectangle, rounded corners, minimum size = 2.5cm]
	\def\bottom#1#2{\hbox{\vbox to #1{\vfill\hbox{#2}}}}
	\scalebox{1}{
		\begin{tikzpicture}
			% Nodes
			\node (L) at(4,4) [latent]{$\xi$};
			\node (E) at(0,0) [latent]{$\epsilon$};
			\node (X) at(4,0) [observed]{$x$};
			
			% Edges
			\draw [->, line width=1.1pt] (L) -- (X) node[midway,above right]{$\mathbf{\Lambda}$};
			\draw [dashed,->, line width=1.1pt] (X) to [bend right = 45] node[midway, right]{$\tilde{\xi}_{\mathbf{\Theta},x}$} (L);
			\draw [-, line width=1.1pt]  (E) -- (X) node[midway,above]{$\mathbf{I}_p$};
			
			% Plates
			\draw (0,0) node[plate, text height = 2.2 cm] (PE){$\mathcal{N}_p(\boldsymbol{0},\mathbf{\Psi})$};
			\draw (4,4) node[plate, text depth = 1.9 cm] (PL){$\mathcal{N}_m(\boldsymbol{0},\mathbf{\Phi})$};
			\draw (4,0) node[plate, text height = 2.2 cm] (PX){$\mathcal{N}_p[\boldsymbol{0},\mathbf{\Sigma}(\mathbf{\Theta})_{xx}]$};
		\end{tikzpicture}
	}
	\caption{
		Schematic of the common factor model.
		The nodes indicate our random variables, with latents represented by white nodes and observables in grey.
		The plates contain postulated distributions for $\xi$ and $\epsilon$, and the likelihood (under these postulates) of the observable $x$.
		Here, $\xi$ is the latent causal source for the observable $x$, with weight matrix $\LA$.
		Hence, we have a top-down generative model $p(x|\xi)p(\xi)$, where $p(x|\xi) \sim \mathcal{N}_p(\LA\xi,\PS)$ represents the mapping of $\xi$ to $x$.
		The dashed line represents the core problem of retrieving the causal source through $\mathbf{\Theta}$ and the marginal behavior of $x$.
		This retrieval is represented by the proxy $\tilde{\xi}_{\mathbf{\Theta},x}$.}
		%We will treat the Gaussian case initially but will subsequently treat the distribution-free setting.}
	\label{FIG:Model}
\end{figure}
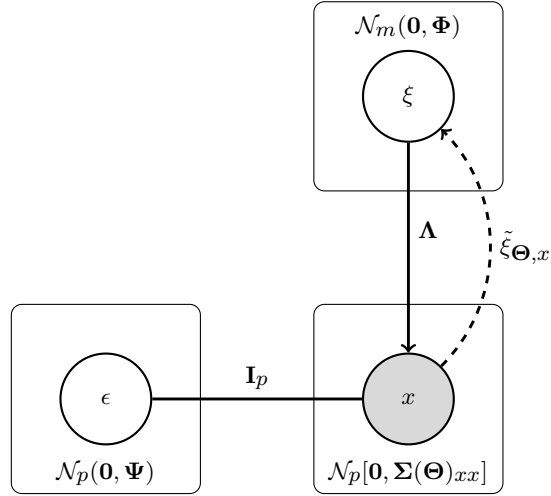

Here, we will show that these reflections of indeterminacy, and the various forms they may take, are related.
In addition, we will explore under what conditions the retrieval is determinate.
Before stating our aims more completely in Section \ref{SSEC:Goal} we will sketch the history of the idea of factor indeterminacy.

\subsection{A short history}\label{SSEC:History}
Spearman \citep{ref_Spearman1904} marked the beginning of the quantitative treatment of latent common variables.
His work to determine a general intelligence factor termed $g$ from collinear scholastic ability measurements culminated in his 1927 book `The Abilities of Man' \citep{ref_SpearmanAbilities}.
In reviewing this book, Wilson \citep{ref_Wilson1928Science} indicated that multiple constructed factors could serve as $g$, thus giving the first  treatment of factor indeterminacy.
This review was followed by a string of papers on indeterminacy, such as \citep{ref_Wilson1928PNAS}, further exploring the nature and geometry of indeterminacy in the $g$-factor model.
In the 1930s Piaggio showed that $g$ could be decomposed into a determinate and an indeterminate part \citep{ref_Piaggo1931}, how one can construct different factors that can serve as $g$ by varying the indeterminate part \citep{ref_Piaggo1933}, and that the indeterminate part vanishes as the test battery probing $g$ grows to infinity \citep{ref_Piaggo1933}.
Kestelman \citep{ref_Kestelman1952} generalized Piaggo's results to the more general orthogonal common factor model with multiple common factors.
Guttman \citep{ref_Guttman1955}, in turn, generalized Kestelman's results to the oblique multiple common factor situation and provided a numerical index of indeterminacy as well as a philosophical embedding.
The 1960s were, as the 1940s, largely silent on the issue.
In the 1970s Sch\"{o}nemann \citep{ref_Schonemann1971} and Sch\"{o}nemann and Wang \citep{ref_SW72} renewed interest in factor indeterminacy, inciting a string of papers.
The most focal work of the 1980s was produced by Bartholomew \citep{ref_Bartholomew1981}.
He argued that factor indeterminacy is simply a reflection of the posterior distribution of $\xi$ given $x$ not being concentrated on a single point and, hence, can be resolved by appropriate uncertainty quantification.
This did not settle the debate as can be inferred from a 1996 special issue in \emph{Multivariate Behavioral Research}.
Its target paper by Maraun \citep{ref_Maraun1996} juxtaposes the posterior distribution and construction stances on factor indeterminacy, siding with the latter.
Subsequent commentaries, marking the last inclusive treatment of factor indeterminacy in the psychometric literature, featured the often sharply conflicting views of proponents of both camps.
For more comprehensive historical overviews of the factor indeterminacy debate in Psychometrics we confine by referring to \citep{ref_SS78,ref_Steiger79,ref_Mulaik05}, and \citep[][Chapter 13]{ref_Mulaik2010}.
A more recent incarnation of the problem can be found in the Artificial Intelligence (AI) community. 
Here, the issue is known as \emph{posterior} or \emph{latent variable collapse} and occurs in the variational autoencoder framework for deep generative modeling when the posterior $p(\xi|x)$ converges to the prior $p(\xi)$ \cite[see, e.g.,][p.\ 576--577]{ref_Bishop&Bishop}.
This work implies that this is merely a reflection of the classic problem of factor indeterminacy.

\subsection{Objectives}\label{SSEC:Goal}
The common factor analytic model may be seen as a canonical model, with many special cases and generalizations, and widespread use in many fields of enquiry \citep{ref_FArepresent}.
The issue of factor indeterminacy, however, remains elusive.
In Psychometrics the issue is known but often not widely understood.
One misunderstanding is that it only affects exploratory factor analysis \citep[compare][]{ref_Vittadini1989,ref_Stevens2002}.
Another that the issue can be ignored if the analysis does not require  factor scores (i.e., realizations of the proxy), e.g., when the (estimated) structure $\mathbf{\Lambda\Phi}$ is the end-point of analysis \citep[compare][]{ref_Guttman1955,ref_MW2020}.
It however also affects confirmatory factor analysis as well as any model (indirectly) build upon common factors, such as structural equation models with latent variables \citep{ref_Bollen1989}.
%Discussion of the issue tends to be convoluted, often fueled by either an ideological taste or distaste of the model, with little reference to technical arguments.
In Statistics the issue tends to be ignored, with most energy spend towards finding estimates for $\mathbf{\Theta}$.
In other fields, however, data representation is the goal, i.e., the projection of $x$ onto $\xi$.
In these fields, such as Machine Learning, the factor model is viewed as a single-hidden-layer (or shallow) generative neural network that may serve as a building block for deep generative architectures \citep{DL16}. 
Moreover, it is often used (in the form of its special case, the probabilistic PCA model) as a minimal model for studying the fundamental characteristics of deep generative models \citep[see, e.g.,][p.\ 827--830]{ref_AImodern}.
Here the issue of factor indeterminacy might seem especially pertinent, but its intricacies and relation to posterior collapse tend to go unnoticed.
Our focus will be a comprehensive technical understanding of factor indeterminacy and its implications.
The sub-objectives are: (i) to assess indeterminacy from various perspectives and show how these are mathematically and conceptually related in order to learn under what conditions we may conclude determinacy; and (ii) to assess the implications of (i) for the Psychometrics, Statistics, and AI communities.
The overall objective is to raise broader awareness.

\subsection{Overview}\label{SSEC:Overview}
Section \ref{SEC:Perspectives} explores the issue of factor indeterminacy from multiple perspectives.
Its starts by showing that the span of the model is inflated in comparison to the span of available information in $x$ (Section \ref{SSEC:Span}).
This lack of information reverberates in incongruities in the postulated model and the model under proxies of $\xi$ and $\epsilon$ (Section \ref{SSEC:ViolatePostulates}).
Our proxy of choice is the best linear estimate $\mathbb{E}(\xi|x)$ which has (of any linear function of $x$) maximum correlation with $\xi$.
Even then we have imperfect correlations between our proxy and $\xi$ (Section \ref{SSEC:CorNotPerfect}).
This observation can be used as a basis for the construction of infinitely many common and unique factors that do adhere to the postulated model (Section \ref{SSEC:Constructions}).
The constructed common factors are then useful in understanding factor indeterminacy from the perspective of correlational geometry (Section \ref{SSEC:Geometry}), which, in turn, is related to the mean-squared error between proxies and idealized common factors (Section \ref{SSEC:MSE}).
Lastly, it is shown how factor indeterminacy may also be understood from the Markov properties of the common factor model (Section \ref{SSEC:Markov}).
Section \ref{SEC:Approaches} then explores approaches towards the indeterminacy problem.
It is argued that the posterior distribution stance does not resolve the issue (Section \ref{SSEC:PosteriorStance}) but essentially provides another view on indeterminacy.
Rather, we reargue that the problem is tackled by considering an infinity of features that adhere to the model (Section \ref{SSEC:InfiniteFeature}) and show how this feeds into a practical, distribution-free estimation and factor scoring approach for realized, high-dimensional data (Section \ref{SSEC:Image}).
Section \ref{SEC:Discussion} provides a discussion on the implications of the foregoing for Psychometrics, Statistics, and AI (including machine learning and deep learning).
Section \ref{SSEC:Reflect} concludes with a reflection.
Mathematical details for Sections \ref{SEC:Perspectives}--\ref{SEC:Discussion} can be found in the Supplementary Material (SM).

\subsection{Main messages}\label{SSEC:Message}
This work has three main messages.
When the feature-dimension grows to infinity, we have latent factor determinacy across all its facets. 
Thus, in very high dimension, \emph{latent projection adheres to concentration of measure} (Sections 2.1--3.2, covering sub-objective (i)).
This feeds into an essentially distribution-free approach in the sample case. 
While we begin our development under Gaussianity (Figure \ref{FIG:Model}) we will see that, in high-dimension, we may drop this assumption.
Even when $p \gg n - 1 > m$ the linear projection of $x$ onto $\xi$ retrieves the true distribution of the latter when the number of features goes to infinity, irrespective of the distribution of $x$. 
Hence, \emph{latent projection is a form of (implicit) regularization}  (Section 3.3, covering the computational implications under sub-objective (ii)).
These findings have many implications for (the usage of) factor projections in Psychometrics, Statistics, and AI.
For example, the model might not be suited for the traditional behavioral survey setting (Psychometrics).
Moreover, a full Bayesian estimation approach will often be untenable in situations of approximate determinacy (Statistics).
Also, posterior collapse is a pervasive characteristic of deep generative models rather than an attribute of training approaches (AI).
Nonetheless, these implications also provide many opportunities for resource-efficient analysis of high-dimensional data. 
As from the perspectives of representation and manifold learning \emph{dimensionality is a blessing rather than a curse for the factor model}, emphasizing its suitability for the analysis of modern high-dimensional data (Sections \ref{SEC:Discussion}--\ref{SSEC:Reflect}, covering the methodological implications under sub-objective (ii)).

\subsection{Notation}\label{SSEC:Notation}
Complete explication of all notation can be found in Section S.1 of the SM.
Here we list some notation that may be considered non-standard or lesser-known.
We let $\mathbf{\Sigma}(\mathbf{\Theta})$ denote a model-implied correlation matrix as a function of (any subcollection of) $\mathbf{\Theta}$.
A conditional model-implied correlation matrix, e.g., the model-implied correlation matrix of $x$ given $\xi$, is denoted as $\mathbf{\Sigma}(\mathbf{\Theta})_{x|\xi}$.
We write $p(\xi|x)$ as a shorthand for the conditional (posterior) distribution $p(\xi|x, \mathbf{\Theta})$.
We use $\delta(z - \iota)$ for the shifted Dirac delta function, in $m-$dimensional Euclidean space, concentrated on $\iota$.
Let $\vartheta$ and $\omega$ be elements of a Hilbert space of $m$-dimensional random vectors. 
Then the $\ell_2$-norm $\|\vartheta\|_2 = [\tr\mathbb{E}(\vartheta\vartheta^{\top})]^{1/2}$ and the inner product $\langle\vartheta,\omega\rangle = \tr\mathbb{E}(\vartheta\omega^{\top})$.
In terms of operations, $\odot$ will denote the Hadamard product and $\circ$ will denote composition. 
For real matrix $\mathbf{A}$ we let $\mathbf{A}^{\odot\frac12} = \big(a_{rc}^{1/2}\big)$ and $\mathbf{A}^{\odot(-1)} = \big(a_{rc}^{-1}\big)$ denote the Hadamard square root and Hadamard inverse, respectively.
These are similarly defined for vectors.
We will consider the elementwise $\ell_{\infty,\infty}$ norm 
$\|\mathbf{A}\|_{\infty,\infty} = \max_{r,c}|a_{rc}|$.
In case $\mathbf{A}$ is symmetric and positive definite, we will also consider the spectral norm $\|\mathbf{A}\|_{2} = e_{\mathrm{max}}(\mathbf{A})$, where $e_{\mathrm{max}}(\mathbf{A})$ is the maximum eigenvalue of $\mathbf{A}$.
We use $\overset{\Pr}{\longrightarrow}$ to denote convergence in probability.
We permit ourselves the convenient abuse of notation $\lim_{p\,\uparrow\,\infty}\xih_p \overset{\Pr}{\longrightarrow} \mathcal{N}_m(\boldsymbol{0}, \PH)$ to represent $\xih_p \overset{\Pr}{\longrightarrow} \xi$ as $p\uparrow\infty$, where $\xi\in\mathcal{N}_m(\boldsymbol{0}, \PH)$.
We will treat the population case, showing that the problem of factor indeterminacy is persistent, irrespective of sample size.
But note that the sample setting is completely analogous.
%--------------- Introduction ---------------------------------------

%--------------- Perspectives ---------------------------------------
%\newpage
\section{Perspectives}\label{SEC:Perspectives}
%%%%%%%%%%%%%%%%%%%%%%%%%%%%
%%%--------------- Span ----
%%%%%%%%%%%%%%%%%%%%%%%%%%%%
\subsection{Model span}\label{SSEC:Span}
The simplest perspective on factor indeterminacy, first voiced by Wilson \citep{ref_Wilson1928PNAS}, comes from assessing the model span.
We may view $\xi$ and $\epsilon$ as basis vectors from which $x$ is linearly determined.
Conversely, we may also view the observable vector $x$ as a basis vector from which $\xi$ and $\epsilon$ may be linearly determined.
The dimension of $x$ is, however, not sufficient to span the joint dimension of $\xi$ and $\epsilon$.
This may also be understood from a simple decomposition on the model-implied correlation matrix:
\begin{equation}\label{EQ:CorMatSpan}\nonumber
    \mathbf{\Sigma}_{xx} = \left[
                        \begin{array}{cc}
                          \mathbf{\Lambda} & \mathbf{I}_p \\
                        \end{array}
                      \right]
                      \left[
                        \begin{array}{cc}
                          \mathbf{\Phi} & \boldsymbol{0} \\
                          \boldsymbol{0} & \mathbf{\Psi} \\
                        \end{array}
                      \right]
                      \left[
                        \begin{array}{c}
                          \mathbf{\Lambda}^{\top} \\
                          \mathbf{I}_p \\
                        \end{array}
                      \right].
\end{equation}
This decomposition conveys that the span of the model-implied correlation matrix, $m + p$, exceeds the span of the correlation matrix on the observables, $p$.
Hence, as will be shown in Section \ref{SSEC:ViolatePostulates}, with finite $p$ observables, there is not enough information to uniquely determine proxies for $m$ common factors in $\xi$ and $p$ unique factors in $\epsilon$ while respecting all postulates of the model (also see Section S.2.1 of the SM).
The current considerations are summarized in the first statement on factor indeterminacy:

\begin{IState}[Inflated span]\label{IDS:span}
Under the postulates of the model, the dimension of the observable random vector $x$ is lower than the joint dimension of the unobservable random vectors $\xi$ and $\epsilon$: $p < m + p$.
Hence, the span of the model-implied correlation matrix exceeds the span of the correlation matrix on the observables:
\begin{equation}
    \mathrm{span}[\mathbf{\Sigma}(\mathbf{\Theta})_{xx}] > \mathrm{span}[\mathbf{\Sigma}_{xx}].
\end{equation}
\end{IState}

%%%%%%%%%%%%%%%%%%%%%%%%%%%%
%%%---Postulate violation --
%%%%%%%%%%%%%%%%%%%%%%%%%%%%
\subsection{Estimators of $(\xi, \epsilon)$}\label{SSEC:ViolatePostulates}
\begin{sloppypar}
Even though the exogeneity of $\xi$ is deemed causal, there is an interest for estimating this unobservable from the observable information contained in $x$ as such projections provide the basis for data representation \citep{ref_FArepresent}.
The implication of the preceding is that we cannot determine $\xi$ and $\epsilon$ from $x$ in a way that befits the postulated model.
To show this, we will consider the joint distribution on $x$, $\xi$, and $\epsilon$ which, through some standard covariance algebra, can be obtained as $[x^{\top}, \xi^{\top}, \epsilon^{\top}]^{\top} \sim \mathcal{N}_{2p + m}\big[\boldsymbol{0}, \mathbf{\Sigma}(\mathbf{\Theta})\big]$, with
\begin{equation}\label{EQ:JoinedCor}
    \mathbf{\Sigma}(\mathbf{\Theta})
    =
    \begin{bmatrix}
    \mathbf{\Sigma}(\mathbf{\Theta})_{xx}    & \mathbf{\Sigma}(\mathbf{\Theta})_{x\xi}   & \mathbf{\Sigma}(\mathbf{\Theta})_{x\epsilon} \\
    \mathbf{\Sigma}(\mathbf{\Theta})_{\xi x} & \mathbf{\Sigma}(\mathbf{\Theta})_{\xi\xi} & \mathbf{\Sigma}(\mathbf{\Theta})_{\xi\epsilon} \\
    \mathbf{\Sigma}(\mathbf{\Theta})_{\epsilon x} & \mathbf{\Sigma}(\mathbf{\Theta})_{\epsilon\xi} & \mathbf{\Sigma}(\mathbf{\Theta})_{\epsilon\epsilon}
    \end{bmatrix}
    =
  \left[\begin{array}{ccc}
    \mathbf{\Lambda}\mathbf{\Phi}\mathbf{\Lambda}^{\top} + \mathbf{\Psi} & \mathbf{\Lambda}\mathbf{\Phi} & \mathbf{\Psi} \\
    \mathbf{\Phi}\mathbf{\Lambda}^{\top}                                 & \cellcolor{lightgray} \mathbf{\Phi}                 & \cellcolor{lightgray} \boldsymbol{0} \\
    \mathbf{\Psi}                                                        & \cellcolor{lightgray} \boldsymbol{0}                & \cellcolor{lightgray}\mathbf{\Psi}
  \end{array}\right].
\end{equation}
In the shaded area of (\ref{EQ:JoinedCor}) we recognize the central model postulates (i)--(iii).
The non-shaded area contains the model implications as derived from these postulates.
This perspective gives a handle on what it means to abide the factor model:

\begin{definition}[Factor solution]\label{DEF:FacSol}
Let $(\tilde{\xi}, \tilde{\epsilon})$ be a proxy of $(\xi, \epsilon)$ in accordance with Definition \ref{DEF:Proxy} and let $\tilde{\mathbf{\Theta}} \equiv \{m, \mathbb{E}(\tilde{\epsilon}\tilde{\epsilon}^{\top}), \mathbf{\Lambda}, \mathbb{E}(\tilde{\xi}\tilde{\xi}^{\top})\}$.
The set $\{\tilde{\mathbf{\Theta}}, (\tilde{\xi}, \tilde{\epsilon})\}$ is then a factor solution for $x$ if and only if $(\tilde{\xi}, \tilde{\epsilon})$ abides the model postulates and, together with $x = \LA\tilde{\xi} + \tilde{\epsilon}$, abides all model implications.
That is, if and only if $\tilde{\mathbf{\Theta}} = \mathbf{\Theta}$ and $[x^{\top}, \tilde{\xi}^{\top}, \tilde{\epsilon}^{\top}]^{\top}$ follows (\ref{EQ:JoinedCor}).
\end{definition}

As the joint distribution is Gaussian, a standard result states that the conditional distribution of $\xi$ given $x$ is Gaussian with expectation $\mathbf{\Sigma}(\mathbf{\Theta})_{\xi x}\mathbf{\Sigma}(\mathbf{\Theta})_{xx}^{-1}x$ and covariance $\mathbf{\Sigma}(\mathbf{\Theta})_{\xi\xi} - \mathbf{\Sigma}(\mathbf{\Theta})_{\xi x}\mathbf{\Sigma}(\mathbf{\Theta})_{xx}^{-1}\mathbf{\Sigma}(\mathbf{\Theta})_{x\xi} \equiv \mathbf{\Sigma}(\mathbf{\Theta})_{\xi|x}$ \citep[see, e.g., Theorem 2.5.1 in][]{ref_AndersonBible}:
\begin{equation}\label{EQ:CondDistXi}
    \xi|x \sim \mathcal{N}_{m}\Big[\mathbf{\Phi}\mathbf{\Lambda}^{\top}\big(\mathbf{\Lambda}\mathbf{\Phi}\mathbf{\Lambda}^{\top} + \mathbf{\Psi}\big)^{-1}x, \mathbf{\Phi} - \mathbf{\Phi}\mathbf{\Lambda}^{\top}\big(\mathbf{\Lambda}\mathbf{\Phi}\mathbf{\Lambda}^{\top} + \mathbf{\Psi}\big)^{-1}\mathbf{\Lambda}\mathbf{\Phi}\Big].
\end{equation}
As the expectation provides the most plausible estimator for $\xi$ through $x$ it is usual to let $\tilde{\xi} = \mathbb{E}(\xi|x)  \equiv \hat{\xi}$.
The estimator $\hat{\xi}$ can be conceived as the regression function of $\xi$ on $x$.
It is the best linear predictor of $\xi$ and is colloquially known as the Thomson score \citep{ref_Thomson1939}.
It may be expressed in the computationally more convenient form $\big(\mathbf{\Phi}^{-1} + \mathbf{\Lambda}^{\top}\mathbf{\Psi}^{-1}\mathbf{\Lambda}\big)^{-1}\mathbf{\Lambda}^{\top}\mathbf{\Psi}^{-1}x$ by Lemma \ref{SMLEM:HUAmat} from Appendix \ref{APP:Identities} (Section S.2.2.1 of the SM).
With $\hat{\xi}$ at hand we may find the corresponding proxy for $\epsilon$:
\begin{equation*}
    \tilde{\epsilon} \equiv \hat{\epsilon} = x - \mathbf{\Lambda}\hat{\xi} = x - \mathbf{\Lambda}\mathbf{\Phi}\mathbf{\Lambda}^{\top}\big(\mathbf{\Lambda}\mathbf{\Phi}\mathbf{\Lambda}^{\top} + \mathbf{\Psi}\big)^{-1}x = \mathbf{\Psi}\big(\mathbf{\Lambda}\mathbf{\Phi}\mathbf{\Lambda}^{\top} + \mathbf{\Psi}\big)^{-1}x,
\end{equation*}
where the last expression implies that the estimator $\hat{\epsilon}$ concurs with $\mathbb{E}(\epsilon|x)$ (Section S.2.2.2 of the SM).
A natural question then arises: Does $x = \mathbf{\Lambda}\hat{\xi} + \hat{\epsilon}$ abide (\ref{EQ:JoinedCor})?
It can be easily verified (Section S.2.2.3 of the SM) that it abides the model implications at the expense of violating the key postulates:
\end{sloppypar}

\begin{IState}[$\hat{\xi}$ and $\hat{\epsilon}$ violate the model postulates]\label{IDS:PostulateViolate}
Let $\hat{\xi}$ and $\hat{\epsilon}$ be the respective regression estimators of $\xi$ and $\epsilon$ on $x$.
The joint distribution of the random variables in $x = \mathbf{\Lambda}\hat{\xi} + \hat{\epsilon}$ adheres to the model-implications in (\ref{EQ:JoinedCor}) by violation of model postulates (i)--(iii):
\begin{align}\label{EQ:PostViolate}
    \mathbb{E}\big(\hat{\xi}\hat{\xi}^{\top}\big) & = \mathbf{\Phi}\mathbf{\Lambda}^{\top}\big(\mathbf{\Lambda}\mathbf{\Phi}\mathbf{\Lambda}^{\top} + \mathbf{\Psi}\big)^{-1}\mathbf{\Lambda}\mathbf{\Phi} \neq \mathbf{\Phi} \\\label{EQ:PostViolate2}
    \mathbb{E}\big(\hat{\epsilon}\hat{\epsilon}^{\top}\big) &= \mathbf{\Psi}\big(\mathbf{\Lambda}\mathbf{\Phi}\mathbf{\Lambda}^{\top} + \mathbf{\Psi}\big)^{-1}\mathbf{\Psi} \neq \mathbf{\Psi} \\\label{EQ:PostViolate3}
    \mathbb{E}\big(\hat{\xi}\hat{\epsilon}^{\top}\big) & = \mathbf{\Sigma}(\mathbf{\Theta})_{\xi|x}\LAt \neq \boldsymbol{0}.
\end{align}
Hence, $\{m, \mathbb{E}(\eph\eph^{\top}), \mathbf{\Lambda}, \mathbb{E}(\xih\xih^{\top}), (\xih,\eph)\}$ is not a factor solution for $x$.
\end{IState}

%%%%%%%%%%%%%%%%%%%%%%%%%%%%
%%%--- SMC -----------------
%%%%%%%%%%%%%%%%%%%%%%%%%%%%
\subsection{Squared multiple correlation}\label{SSEC:CorNotPerfect}
The predictor $\xih$ has the maximum correlation between $\xi$ and the regression function on $x$ \citep[see, e.g., Section 2.5.2 in][]{ref_AndersonBible}.
The diagonal elements of $\mathbb{E}(\xih\xih^{\top})$ then carry the squared multiple correlations between $\xih$ and $\xi$.
This also follows from the fact that $\mathbb{E}(\xih\xih^{\top}) = \mathbb{E}(\xih\xi^{\top})$ (Section S.2.3.1.\ of the SM).
Note that, by the Woodbury matrix identity \citep{ref_Duncan1944,ref_Woodbury}, the following equality holds (Section S.2.3.2 of the SM):
\begin{equation}\label{EQ:SMCwood}
    \mathbf{\Phi}\mathbf{\Lambda}^{\top}\big(\mathbf{\Lambda}\mathbf{\Phi}\mathbf{\Lambda}^{\top} + \mathbf{\Psi}\big)^{-1}\mathbf{\Lambda}\mathbf{\Phi} =
    \mathbf{\Phi} - \big(\mathbf{\Phi}^{-1} + \mathbf{\Lambda}^{\top}\mathbf{\Psi}^{-1}\mathbf{\Lambda}\big)^{-1}.
\end{equation}
Naturally, the diagonal elements of (\ref{EQ:SMCwood}) are nonnegative.
The diagonal elements of the bracketed term on the right-hand side tend to $\mathbf{\Phi}_{kk}$ under a weak factor structure and to $0$ under a strong factor structure.
Hence, in general, $\xi$ will not be perfectly predicted by its regression on $x$:

\begin{IState}[Squared multiple correlations less than unity]\label{IDS:SMC}
The squared multiple correlation coefficient between the best linear estimate $\hat{\xi}_k$ and the true factor $\xi_k$ is less than unity for all $k$:
\begin{equation}\label{EQ:SMCs}
    \Big[\mathbf{\Phi}\mathbf{\Lambda}^{\top}\big(\mathbf{\Lambda}\mathbf{\Phi}\mathbf{\Lambda}^{\top} + \mathbf{\Psi}\big)^{-1}\mathbf{\Lambda}\mathbf{\Phi}\Big]_{kk} <1 \,\,\,\forall k.
\end{equation}
\end{IState}

\begin{remark}
Note that we could make analogous remarks on the squared multiple correlation between $\epsilon$ and its regression on $x$.
It is contained in the diagonal elements of $\PS^{-1/2}\mathbb{E}(\hat{\epsilon}\hat{\epsilon}^{\top})\PS^{-1/2}$ which, in general, will be less than unity (Section S.2.3.3 of the SM).
\end{remark}

%%%%%%%%%%%%%%%%%%%%%%%%%%%%
%%%--- Constructions -------
%%%%%%%%%%%%%%%%%%%%%%%%%%%%
\subsection{Constructions of $(\xi, \epsilon)$}\label{SSEC:Constructions}
The estimators in Section \ref{SSEC:ViolatePostulates} are determinate through their reliance on $x$, but in general do not satisfy the postulates of our model.
It is also possible to \emph{construct} proxies that do satisfy these postulates.
Guttman \citep{ref_Guttman1955} realized that any matrix square root of the difference between $\mathbf{\Phi}$ and $\mathbf{\Phi\Lambda}^{\top}(\mathbf{\Lambda\Phi\Lambda}^{\top} + \mathbf{\Psi})^{-1}\mathbf{\Lambda\Phi}$ is essential to these constructions.
Let $s$ denote an arbitrary random vector such that $s \sim \mathcal{N}_m(\boldsymbol{0}, \mathbf{I}_m) \ci x$.
Moreover, let $\mathbf{H} \in \mathbb{R}^{m \times m}$ be any member of the special orthogonal group $\mathrm{SO}(m)$ \citep{ref_GEOGROU}, i.e., the group (under matrix multiplication) of rotation matrices.
Then one can construct a proxy for $\xi$ according to:
\begin{equation*}
    \xi' \equiv \mathbf{\Phi}\mathbf{\Lambda}^{\top}\big(\mathbf{\Lambda}\mathbf{\Phi}\mathbf{\Lambda}^{\top} + \mathbf{\mathbf{\Psi}}\big)^{-1}x +
    \Big[\mathbf{\Phi} - \mathbf{\Phi}\mathbf{\Lambda}^{\top}\big(\mathbf{\Lambda}\mathbf{\Phi}\mathbf{\Lambda}^{\top} + \mathbf{\mathbf{\Psi}}\big)^{-1}\mathbf{\Lambda}\mathbf{\Phi}\Big]^{1/2}\mathbf{H}s,
\end{equation*}
where the first component on the right-hand side is our best linear predictor $\hat{\xi}$ (the projection of $\xi'$ onto the space spanned by $x$) and where the second component represents the indeterminate orthogonal complement or residual operator.
%This residual operator is indeterminate due to reliance upon an arbitrary rotation of $s$.
%where the first component on the right-hand side is our determinate best linear predictor $\hat{\xi}$ and where the second component is indeterminate due to reliance upon an arbitrary rotation of $s$.
We may recognize $\mathbf{\Phi} - \mathbf{\Phi\Lambda}^{\top}(\mathbf{\Lambda\Phi\Lambda}^{\top} + \mathbf{\Psi})^{-1}\mathbf{\Lambda\Phi}$ as $\mathbf{\Sigma}(\mathbf{\Theta})_{\xi|x}$ from (\ref{EQ:CondDistXi}).
If we then define $\epsilon' \equiv x - \mathbf{\Lambda}\xi'$ we can construct $(\xi', \epsilon')$ according to the system (Sections S.2.4.1 and S.2.4.2 of the SM):
\begin{equation*}
    \begin{bmatrix}
    \xi' \\
    \epsilon'
    \end{bmatrix}
    \equiv
    \begin{bmatrix}
    \hat{\xi}      & \mathbf{\Sigma}(\mathbf{\Theta})_{\xi|x}^{1/2} \\
    \hat{\epsilon} & -\mathbf{\Lambda}\mathbf{\Sigma}(\mathbf{\Theta})_{\xi|x}^{1/2}
    \end{bmatrix}
    \begin{bmatrix}
    1 \\
    \mathbf{H}s
    \end{bmatrix}.
\end{equation*}
These constructions are on the probability space on which $x$ is defined and adhere to all model postulates and implications (Section S.2.4.3 of the SM).
However, there are infinitely many rotations $\mathbf{H}s$ and, hence, infinitely many constructions that abide the model:

\begin{IState}[Infinitely many constructions $(\xi', \epsilon')$ abide the model]\label{IDS:Constructions}
When $\mathbb{E}(\hat{\xi}\hat{\xi}^{\top}) \neq \PH$ we have that, by (\ref{EQ:SMCwood})
\begin{align}\label{EQ:CVARconstruct}\nonumber
    \mathbf{\Sigma}(\mathbf{\Theta})_{\xi|x} &= \mathbf{\Phi} - \Big[\mathbf{\Phi} - \big(\mathbf{\Phi}^{-1} + \mathbf{\Lambda}^{\top}\mathbf{\Psi}^{-1}\mathbf{\Lambda}\big)^{-1}\Big]\\
    &= \big(\mathbf{\Phi}^{-1} + \mathbf{\Lambda}^{\top}\mathbf{\Psi}^{-1}\mathbf{\Lambda}\big)^{-1} \neq \boldsymbol{0},
\end{align}
implying that the set of all random variables $(\xi',\epsilon')$ on the probability space on which $x$ is defined that, together with a given $\mathbf{\Theta}$, satisfies the factor model,
\begin{equation*}
    \Xi \equiv \Big\{\xi'|\exists\epsilon' : \{\mathbf{\Theta},(\xi',\epsilon')\} \,\, \mbox{forms a factor solution for} \,\,x\Big\},
\end{equation*}
is a set of infinite cardinality.
\end{IState}

%%%%%%%%%%%%%%%%%%%%%%%%%%%%
%%%--- Cor Geometry --------
%%%%%%%%%%%%%%%%%%%%%%%%%%%%
\subsection{Correlational geometry}\label{SSEC:Geometry}
We can relate the statements from Sections \ref{SSEC:CorNotPerfect} and \ref{SSEC:Constructions} to correlational geometry.
We have already seen that (\ref{EQ:SMCs}) represents the squared multiple correlation between $\hat{\xi}_k$ and $\xi_k$.
The square root
\begin{equation*}
  \Big[\mathbf{\Phi}\mathbf{\Lambda}^{\top}\big(\mathbf{\Lambda}\mathbf{\Phi}\mathbf{\Lambda}^{\top} + \mathbf{\Psi}\big)^{-1}\mathbf{\Lambda}\mathbf{\Phi}\Big]_{kk}^{1/2} \equiv \rho\big(\xih_{k},\xi_{k}'\big),
\end{equation*}
then represents the multiple correlation between $\hat{\xi}_k$ and any $\xi_k'$ stemming from the construction system (Section S.2.5.1 of the SM).
One can also take interest in the minimal correlation between any two constructions $\xi'$ and $\xi''$ in $\Xi$ \citep[see, e.g.,][]{ref_EBG78}, known as the Guttman criterion \citep{ref_Guttman1955}:
\begin{equation*}
  \gamma_k \equiv \inf\Big\{\rho\big(\xi_{k}',\xi_{k}''\big) | \xi_{k}',\xi_{k}'' \in \Xi\Big\}.
\end{equation*}
Naturally, if $\xi' \equiv \xih + \mathbf{\Sigma}(\mathbf{\Theta})_{\xi|x}^{1/2}\mathbf{H}s$ is in $\Xi$, then the maximally different $\xi'' \equiv \xih - \mathbf{\Sigma}(\mathbf{\Theta})_{\xi|x}^{1/2}\mathbf{H}s$ is too, and the respective factors $\xi_{k}'$ and $\xi_{k}''$ should have minimal correlation.
Independently of $\xi'$ in $\Xi$ we then have that (Section S.2.5.2 of the SM):
\begin{equation*}
  \mathbb{E}\big(\xi'\xi''^{\top}\big) = 2\PH\LAt\big(\Com + \PS\big)^{-1}\LA\PH - \PH = \PH - 2\mathbf{\Sigma}(\mathbf{\Theta})_{\xi|x} \equiv \mathbf{\Gamma},
\end{equation*}
such that $\gamma_k = (\mathbf{\Gamma})_{kk} = 2\rho(\xih_{k},\xi_{k}')^{2} - 1$.

These considerations can be given a geometrical interpretation (see Section S.2.5.3 of the SM for details), which may be viewed as an extension of the considerations in Mulaik \citep{ref_Mulaik1976}.
Define $\mathbf{\Sigma}(\mathbf{\Theta})_{\xi|x}^{1/2}\mathbf{H}s \equiv \tau$ and let $\tau'$ be the rotational representation of $\tau$ corresponding to $\xi'$.
Our vectors of interest, $\xih$, $\tau'$, and $\xi'$ are all elements of a Hilbert space of $m$-dimensional random vectors.
Their elements, $\xih_k$, $\tau'_k$, and $\xi'_k$ can be regarded as elements of a smaller Hilbert space of random variables.
This implies that we can define the $\ell_2$-norm and inner product (for both the $m$-dimensional random vectors and their elements) in terms of covariances.
Let us consider the $k$th element in each of our vectors.
For $\tau'_k$ we, as $\|\xi'_k\|_2 = \|\xih_k + \tau'_k\|_2 = 1$, and $\langle\tau'_k,\xi_k'\rangle = \|\tau'_k\|_2^2$, then have that:
\begin{equation*}
	\frac{\langle\tau'_k,\xi_k'\rangle}{\big\|\tau'_k\big\|_2\big\|\xi'_k\big\|_2} = \big\|\tau'_k\big\|_2 = 
	\Big[\mathbf{\Sigma}(\mathbf{\Theta})_{\xi|x}\Big]_{kk}^{1/2} 
	\equiv \rho(\tau'_k, \xi'_k) = \cos\big[\measuredangle \, \tau'_k,\xi'_k\big].
\end{equation*}
That is, the length of vector $\tau'_k$ also represents the correlation between $\tau'_k$ and  $\xi'_k$, which represents the cosine of their angle.
For $\xih_k$ we similarly have that
\begin{equation*}
	\big\|\xih_k\big\|_2 = 
	\rho(\xih_k, \xi'_k) = 
	\cos\big[\measuredangle \, \xih_k,\xi'_k\big].
\end{equation*}
There are infinitely many rotations of $s$ and thus infinitely many $\tau_k$, all perpendicular to $\xih_k$.
This implies that we have infinitely many alternative constructions at the same angle with $\xih_k$.
Hence, we have a cone with the apex at the origin resting on the equator formed by all rotational representations of $\tau_k$.
All constructions $\xi_k'$ then occupy the locus of the cone around $\xih_k$.
This cone finds itself in the unit sphere as $\|\xi'_k\|_2 = 1$.
The Guttman criterion $\gamma_k$ then represents the cosine of the double angle between $\xih_k$ and $\xi'_k$ and thus the cosine of the angle between any $\xi_{k}'$ and its maximally different counterpart on the opposite side of the cone.
Figure \ref{FIG:Cones2} provides a visualization for a situation of global rotational uniqueness (left-hand side) as well as a situation of local rotational uniqueness (right-hand side). 
In case of local rotational uniqueness polarity reflections are allowed implying a reflection of the cone such that we have a double cone placed apex to apex.
For multiple factors this translates to a (double) hypercone in the hypersphere for the orthogonal case and the (double) oblique hypercone in the hyperspheroid for the oblique case.

\begin{figure}[h!]
	\centering
	\includegraphics[width=\textwidth]{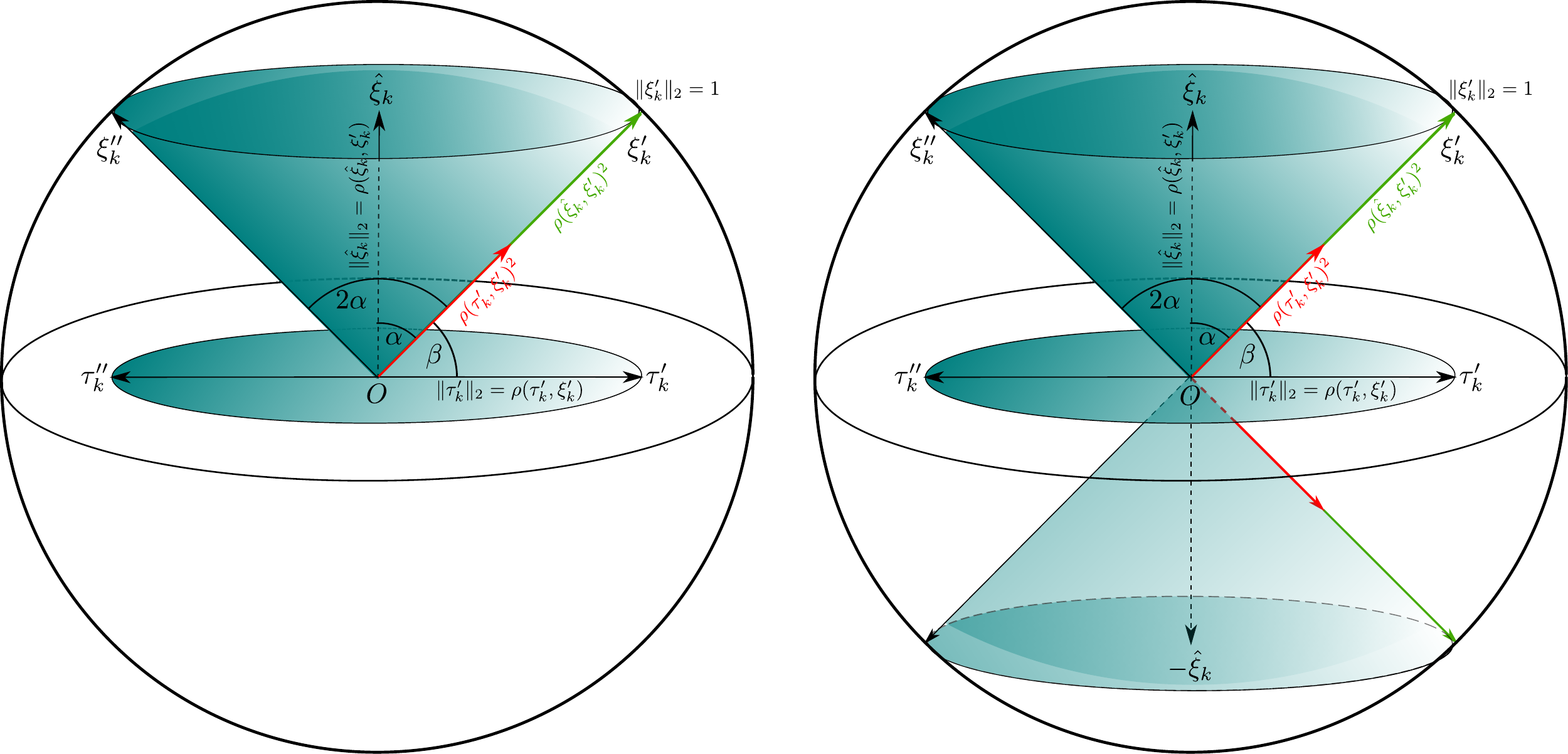}
	\caption{
		A visualization of the correlational geometry of indeterminacy. 
		The angles $\alpha$ and $\beta$ are (in degrees)  $\alpha \equiv \measuredangle \, \xih_k,\xi'_k = \arccos[\rho(\xih_k, \xi'_k)] \cdot 180/\pi$ and $\beta \equiv \measuredangle \, \tau'_k,\xi'_k = \arccos[\rho(\tau'_k, \xi'_k)] \cdot 180/\pi$.
		The double $\alpha$ angle follows from standard trigonometry: 
		$2\alpha = \arccos[2\rho(\xih_k, \xi'_k)^2 - 1] \cdot 180/\pi = \arccos(\gamma_k) \cdot 180/\pi$.
		The unit-length vector $\xi'_k$ consists of the sum of vector projections of $\tau'_k$ (in red) and $\xih_k$ (in green) onto $\xi'_k$.
		These vector projections are of respective lengths $\rho(\tau'_k, \xi'_k)^2$ and $\rho(\xih_k, \xi'_k)^2$, representing squared multiple correlations.
		Hence, the length of $\xi'_k$ is divided into a part predictable by $\xih_k$ and a part dependent on the arbitrary indeterminate component $\tau'_k$.
		For this case we have that $\xi'_k$ is predicted by $\xih_k$ and dependent on $\tau'_k$ in equal parts: $\rho(\xih_k, \xi'_k)^2 = \rho(\tau'_k, \xi'_k)^2 = .5$. 
		Hence, $\rho(\xih_k, \xi'_k) = \sqrt{.5} = .7071068$.
		The angle between $\xih_k$ and any construction $\xi'_k$ is then $\alpha = \arccos(.7071068) \cdot 180/\pi =  45^\circ$ and the angle associated with the Guttman criterion is $2\alpha = 90^\circ$.
		This shows that, under relatively strong correlation between $\xih_k$ and $\xi'_k$, it is already possible to construct alternatives that are uncorrelated.
		%The left-hand figure then represents a situation of global rotational %uniqueness while the right-hand side represents a situation of local %rotational uniqueness.
		The squared multiple correlation $\rho(\tau'_k, \xi'_k)^2$ and the double $\beta$ angle are related to the mean squared error between $\xi'$ and $\xih$ as explained in Section \ref{SSEC:MSE}.
	}
	\label{FIG:Cones2}
\end{figure}

In practice it is usual to inspect the sample counterparts to $\LA$ and $\PH$ to interpret the factors in $\xi$.
However, unless the correlation between factor score and all factor constructions is perfect, these matrices are compatible with differing positions in latent space:

\begin{IState}[$\mathbf{\Theta}$ does not determine the meaning of $\xi$]\label{IDS:CorGeom}
In general we have that
\begin{equation}\label{EQ:GuttGap}
  \gamma_k <1 \,\,\,\forall k.
\end{equation}
Lower values of $\gamma_k$ imply a wider gap between two maximally different $\xi_k'$ and $\xi_k''$.
This, in turn, implies that a given $\mathbf{\Theta}$ can correspond to (radically) different positionings in latent space.
Hence, the (semantic) meaning of $\xi$ cannot be strictly inferred from $\mathbf{\Theta}$.
\end{IState}

\begin{remark}
Analogous statements on the correlational geometry of indeterminacy in the unique factors are possible (Section S.2.5.4 of the SM).
In addition, we can also express the correlations between constructions of  $\xi$ and $\epsilon$ that are maximally different from $0$ (Section S.2.5.5 of the SM).
\end{remark}

%%%%%%%%%%%%%%%%%%%%%%%%%%%%
%%%--- MSE -----------------
%%%%%%%%%%%%%%%%%%%%%%%%%%%%
\subsection{Mean squared error}\label{SSEC:MSE}
We may, related to the foregoing geometrical considerations, assess the mean squared error (MSE) between $\xi'$ and $\xih$.
The MSE is expressible in terms of covariances.
We have, noting $\xi' - \xih = \tau'$, that, for all $\xi' \in \Xi$ (see Section S.2.6.1 of the SM):
\begin{align*}
    \big\|\xi' - \xih\big\|_{2}^{2} = \tr\mathbb{E}\Big[\big(\xi' - \xih\big)\big(\xi' - \xih\big)^{\top}\Big] =
    \tr\mathbb{E}\big(\tau'\tau'^{\top}\big) =
    \big\|\tau'\big\|_{2}^{2} = \tr\Big[\mathbf{\Sigma}(\mathbf{\Theta})_{\xi|x}\Big].
\end{align*}
That is, the MSE equals the trace of the conditional covariance matrix of $\xi$ given $x$.
We may also interpret the MSE as the $k$-sum of squared multiple correlations $\rho(\tau'_k, \xi'_k)^2$.
That is, the sum of those parts of each $\xi'_k$ that are dependent on the arbitrary indeterminate component $\tau'_k$ or, conversely, the sum of the linearly unpredictable parts of each $\xi'_k$.
The MSE between $\xi'$ and $\xih$ can thus be seen as the scalar projection of $\tau'$ onto $\xi'$ (Figure \ref{FIG:Cones2}).

It is insightful to give an equivalent MSE expression in terms of variance and squared bias (Section S.2.6.2 of the SM):
\begin{equation*}
	\big\|\xi' - \xih\big\|_{2}^{2} =
    \tr\Big[\PH\LAt\big(\Com + \PS\big)^{-1}\LA\PH\Big] - \tr\big(\mathbf{\Gamma}\big),
\end{equation*}
implying that the Guttman bound can be seen as the negation of the squared bias between $\xi'$ and $\xih$.
This itself relates to the double $\beta$ angle (Figure \ref{FIG:Cones2}).
We have that (Section S.2.6.3 of the SM) $2\beta = \arccos(-\gamma_k) \cdot 180/\pi$, expressing the degree to which the cone has collapsed onto the equator.
These considerations relate the MSE to conical contraction:

\begin{IState}[Nonzero MSE between $\xi'$ and $\xih$]\label{IDS:MSE}
	In general,
	\begin{equation}\label{EQ:MSEGap}
	\Big[\mathbf{\Sigma}(\mathbf{\Theta})_{\xi|x}\Big]_{kk} > 0\,\,\,\forall k, 
	\end{equation}
	implying the imperfect retraction of the equator formed by the rotational representations of $\tau$ towards the origin and thus a dependency on this arbitrary indeterminate component to abide model postulates. This, in turn, implies an imperfect conical contraction of the $\xi'$ in $\Xi$ around $\xih$.
\end{IState}

\begin{remark}
We can make analogous remarks on the MSE between $\epsilon$ and $\eph$. It amounts to $\tr\big[\mathbf{\Sigma}(\mathbf{\Theta})_{\epsilon|x}\big] = \tr\Big[\PS - \PS\big(\Com + \PS\big)^{-1}\PS\Big]$ (Section S.2.6.4 of the SM).
\end{remark}

%%%%%%%%%%%%%%%%%%%%%%%%%%%%
%%%--- Markov --------------
%%%%%%%%%%%%%%%%%%%%%%%%%%%%
\subsection{Markov properties}\label{SSEC:Markov}
We may also view indeterminacy from the perspective of Markov properties.
Let us consider graphs $\mathcal{G} = (\mathcal{V}, \mathcal{E})$ consisting of a set $\mathcal{V} = \mathcal{V}_{\xi} \cup \mathcal{V}_{x}$ of vertices and set of edges $\mathcal{E} = \mathcal{E}_{\xi \xi} \cup \mathcal{E}_{\xi x} \cup \mathcal{E}_{x x}$.
The vertices correspond to our collection of random variables $\{x,\xi\} \sim \mathcal{N}_{(p + m)}\big[\boldsymbol{0},\mathbf{\Sigma}(\mathbf{\Theta})_{[1:2,1:2]}\big]$, where $\mathbf{\Sigma}(\mathbf{\Theta})_{[1:2,1:2]}$ denotes the top-left $2\times 2$ block of (\ref{EQ:JoinedCor}).
Edges in $\mathcal{E}_{\xi x}$ consist of ordered pairs such that $\xi_{k} \rightarrow x_{j} \in \mathcal{E}_{\xi x}$.
The edges in $\mathcal{E}_{\xi \xi}$ consist of pairs of distinct vertices in $\mathcal{V}_{\xi}$ such that $\xi_{k} \leftrightarrow \xi_{k'} \in \mathcal{E}_{\xi \xi}$.
Edges in $\mathcal{E}_{xx}$ consist of pairs of distinct vertices in $\mathcal{V}_{x}$ such that $x_{j} \leftrightarrow x_{j'} \in \mathcal{E}_{xx}$.
We may connect any particular factor structure to the graph by having edges correspond to nonzero parameters.
The strict uni-directedness of edges in $\mathcal{E}_{\xi x}$ follows from the assumption of exogeneity of $\xi$ (to $x$) and expresses causal influence in the sense that $\xi_{k} \rightarrow x_{j} \in \mathcal{E}_{\xi x}$ if $\lambda_{jk} \neq 0$.
The bidirectional edges $\leftrightarrow$ represent marginal dependence. 
Then $\xi_{k} \leftrightarrow \xi_{k'} \in \mathcal{E}_{\xi \xi}$ if $\phi_{kk'} = \phi_{k'k} \neq 0$ and $x_{j} \leftrightarrow x_{j'} \in \mathcal{E}_{xx}$ if $\psi_{jj'} = \psi_{j'j} \neq 0$.
A graph so obtained is a directed bipartite mixed graph (Section S.2.7.1 of the SM) and can be thought of as a marginalization of the traditional path diagram of the factor model \citep[see, e.g.,][]{ref_Bollen1989} over the error variables (Section S.2.7.2 of the SM).

Inferring the (probabilistic) conditional (in)dependencies (i.e., the Markov properties) associated with such graphs can be achieved generalizing Pearl's notion of $d$-separation \citep{ref_Pearl} to $\mathfrak{m}$-separation \citep[][see Section S.2.7.3 of the SM]{ref_Koster1999,ref_Rich2003}.
We use the calligraphy $\mathfrak{m}$ to avoid confusion with the factor dimension $m$.
Now, consider the model in (\ref{EQ:FAmodel}) satisfying the stated assumptions in Section \ref{SSEC:FAmodel}.
Denote the set of (Gaussian) probability distributions of $\{x,\xi\}$ that satisfy this model by $\mathscr{P}$.
Let $\mathcal{G} = (\mathcal{V}, \mathcal{E})$ be the associated directed bipartite mixed graph.
Let $A$, $B$, and $C$ be (arbitrary) pairwise disjoint subsets of $\mathcal{V}$ with $A$ and $B$ nonempty and $C$ possibly empty. 
Also, let $A\ci_{\mathcal{G}} B \,|\, C$ denote $A$ being $\mathfrak{m}$-separated from $B$ given $C$ in $\mathcal{G}$.
Let us also assume that any off-diagonal zero entry in the model-implied precision matrix $\mathbf{\Sigma}(\mathbf{\Theta})_{[1:2,1:2]}^{-1}$ is never induced by cancellation of constituent parameters (faithfulness).
Then, for any distribution $\mathcal{P}$ compatible with $\mathcal{G}$ \citep{ref_Koster1999}:
\begin{equation*}
	A\ci_{\mathcal{G}} B \,|\, C \Leftrightarrow 
	A\ci B \,|\, C\,[\mathcal{P}], ~~~\forall ~ \mathcal{P} \in \mathscr{P}.
\end{equation*}
That is, graphical $\mathfrak{m}$-separation implies probabilistic separation in the form of conditional independence in all $\mathcal{P} \in \mathscr{P}$ (global $\mathcal{G}$-Markovness) and vice versa (Markov perfectness). 
This allows us to characterize the Markov properties of the model by means of (the support of) the concentration or precision matrix $\mathbf{\Sigma}(\mathbf{\Theta})_{[1:2,1:2]}^{-1}$.

Let $\mathbf{\Pi} \equiv \mathbf{\Sigma}(\mathbf{\Theta})_{x\xi}\mathbf{\Sigma}(\mathbf{\Theta})_{\xi\xi}^{-1}$.
The model-implied precision matrix on the joint observables $x$ and latents $\xi$ can then be found as (Section S.2.7.4 of the SM):
\begin{align*}
%	\begin{bmatrix}
%		\mathbf{\Sigma}(\mathbf{\Theta})_{xx}  &  \mathbf{\Sigma}(\mathbf{\Theta})_{x\xi} \\
%		\mathbf{\Sigma}(\mathbf{\Theta})_{\xi x} & \mathbf{\Sigma}(\mathbf{\Theta})_{\xi\xi}
%	\end{bmatrix}^{-1}
%	=
	\begin{bmatrix}
		\mathbf{\Sigma}(\mathbf{\Theta})_{x|\xi}^{-1}  &  -\mathbf{\Sigma}(\mathbf{\Theta})_{x|\xi}^{-1}\mathbf{\Pi} \\
		-\mathbf{\Pi}^{\top}\mathbf{\Sigma}(\mathbf{\Theta})_{x|\xi}^{-1}  & 
		\mathbf{\Sigma}(\mathbf{\Theta})_{\xi|x}^{-1}
	\end{bmatrix} 
	= 
	\begin{bmatrix}
		\PSi  &  -\PSi\LA \\
		-\LAt\PSi & \PHi + \Gram
	\end{bmatrix}.
\end{align*}
Hence, the observables $x$ are independent given the latents $\xi$, except for those observables that are chained by correlated measurement errors.
From an identification perspective this means that the $x$, predominantly, are conditionally independent.
An intuitive understanding of this condition forms the rationale for metrics such as the Kaiser-Meyer-Olkin index \citep{ref_KMO}.
Of course we cannot condition on an unobserved source.
A necessity for full identification of the model is then that the conditional precision $\mathbf{\Sigma}(\mathbf{\Theta})_{x|\xi}^{-1}$ can be obtained from the precision on the observables alone or, in other words, if the conditioning does not have to rely on the latent factors, that is: $\mathbf{\Sigma}(\mathbf{\Theta})_{xx}^{-1} = \mathbf{\Sigma}(\mathbf{\Theta})_{x|\xi}^{-1}$.

\begin{IState}[Non-assessable conditional independence properties]\label{IDS:Markov}
	In general
	\begin{equation}\label{EQ:CIGap}
		\big(\LA\PH\LAt + \PS\big)^{-1} \neq \mathbf{\Psi}^{-1}.
	\end{equation}
	Hence, the conditional independence properties on the observables as implied by the model cannot be assessed with (the sufficient statistic for) the observable data.
\end{IState}

%The condition under which (\ref{EQ:CIGap}) would be an equality is intricately related to the determinacy of the latents (Section \ref{SSEC:InfiniteFeature}).
%This connects to the following remark:

\begin{remark}
	We can give an alternative view on Indeterminacy Statement \ref{IDS:Markov} in which we would condition on our best linear predictor $\xih$.
	We would then have determinacy when $\mathbf{\Sigma}(\mathbf{\Theta})_{x|\xih} = \mathbf{\Sigma}(\mathbf{\Theta})_{x|\xi}$.
	This concurs with demanding that (\ref{EQ:PostViolate}) from Indeterminacy Statement \ref{IDS:PostulateViolate} expresses an equivalency (Section S.2.7.5 of the SM).
\end{remark}
%--------------- Perspectives ---------------------------------------

%--------------- Approaches -----------------------------------------
\section{Approaches}\label{SEC:Approaches}
%%%%%%%%%%%%%%%%%%%%%%%%%%%%
%%%--- Posterior approach --
%%%%%%%%%%%%%%%%%%%%%%%%%%%%
\subsection{Factor estimators as posterior constructs}\label{SSEC:PosteriorStance}
Bartholomew \citep{ref_Bartholomew1981} argued that many psychometricians had treated $\xi$ as an unknown parameter rather than a random variable.
To him, factor indeterminacy reflected the failure to quantify the uncertainty regarding estimators of $\xi$.
A Bayesian quantification of this uncertainty would then resolve the issue.
He viewed the distributional assumption on $\xi$ as a prior distribution and treated (\ref{EQ:CondDistXi}) as the posterior of $\xi$ (see Section S.3.1.1 of the SM for some considerations).
Using sample estimate $\hat{\mathbf{\Theta}}$ of $\mathbf{\Theta}$ would then result in an empirical Bayes evaluation of the estimator $\xih$ and its associated uncertainty.
Aitkin and Aitkin \citep{ref_AA2005} would reiterate this stance for the full Bayesian data-augmentation approach using counterfactual realizations of $\xi$.

We concur that solely using $\xih$ ignores its uncertainty and that assessment is important.
This assessment can be made by full or empirical Bayesian approaches (or by evaluation of any quantity indicative of indeterminacy as given in Indeterminacy Statements \ref{IDS:span}--\ref{IDS:Markov}).
It does not, however, resolve the core problem: $(\xih,\eph)$, as posterior locations, will not be part of a factor solution for $x$ in the sense of Definition \ref{DEF:FacSol}, implying that the postulates from which the model implications are derived, cannot be retrieved from those same implications.
Moreover, $\xih$ is a random variable from which scorings (realizations in latent space) are obtained. 
The uncertainty expressed by $\mathbf{\Sigma}(\mathbf{\Theta})_{\xi|x}$ implies that there are infinitely many alternative scoring distributions (see Section S.3.1.2 of the SM for additional considerations on the ramifications).
Hence, we are interested in characterizing the posterior plausibility of alternative scoring distributions.
A simple approach would be to define a credibility interval on the (marginal) posterior standard deviations:
\begin{equation*}
	\xih \pm c\,\mathrm{diag}\Big[\mathbf{\Sigma}(\mathbf{\Theta})_{\xi|x}^{\odot\frac12}\Big],
\end{equation*}
where $c$ is a scalar constant.
This represents an $m$-orthotope based on simultaneous shifts upward or downward from $\xih$.
The value of $c$ beyond which the posterior plausibility becomes negligible is $4$ (Section S.3.1.3 of the SM).
We may also, moving beyond marginal variances, consider the Mahalanobis distance between $\xi|x$ and $\xih$ to obtain a credibility region:
\begin{equation*}
	\mathrm{CR}({\xih}) \equiv 
	\bigg\{\xi|x \in \mathbb{R}^{m} \,\Big| \Big(\xi|x - \xih\Big)^{\top}\mathbf{\Sigma}(\mathbf{\Theta})_{\xi|x}^{-1}\Big(\xi|x - \xih\Big) \leq \chi^2_{m,\alpha}\bigg\},
\end{equation*}
representing the interior of an iso-distance $m$-dimensional ellipsoid centered at $\xih$ and distance constant $\chi^2_{m,\alpha}$, where the latter represents the $\chi^2$-value with $m$-degrees of freedom located at $\alpha \in (0,1)$.
Under our distributional assumptions $\mathrm{Pr}(\xi|x \in \mathrm{CR}({\xih})) = 1 - \alpha$ \citep[][p.\ 47, p.\ 79]{ref_AndersonBible}.
The volume of this ellipsoid is a measure of dispersion around $\xih$ and is proportional to $\big|\mathbf{\Sigma}(\mathbf{\Theta})_{\xi|x}^{1/2}\big|$ \citep[][p.\ 301]{ref_Cramer46}.
Clearly, given the duality between $\mathbb{E}(\xih\xih^{\top})$ and $\mathbf{\Sigma}(\mathbf{\Theta})_{\xi|x}$ in the sense of $\mathbb{E}(\xih\xih^{\top}) = \PH - \mathbf{\Sigma}(\mathbf{\Theta})_{\xi|x}$, $\mathbf{\Sigma}(\mathbf{\Theta})_{\xi|x}$ needs to approximate $\boldsymbol{0}$ for concentration on $\xih$ and for $\xih$ to behave as $\xi$.
If not, we have an infinity of alternative scoring distributions and the most plausible one, $\xih$, does not behave as $\xi$.
As $\|\xi - \xih\|_{2}^{2}$ becomes larger, it is increasingly possible to have alternative scoring distributions whose locations have high posterior plausibility, but whom would simultaneously imply radically differing positioning in latent space.
We exemplify these considerations with a single latent dimension in Figure \ref{FIG:BIndet2} before stating our last indeterminacy statement.

\begin{figure}[h!]
	\centering
	\includegraphics[width=\textwidth]{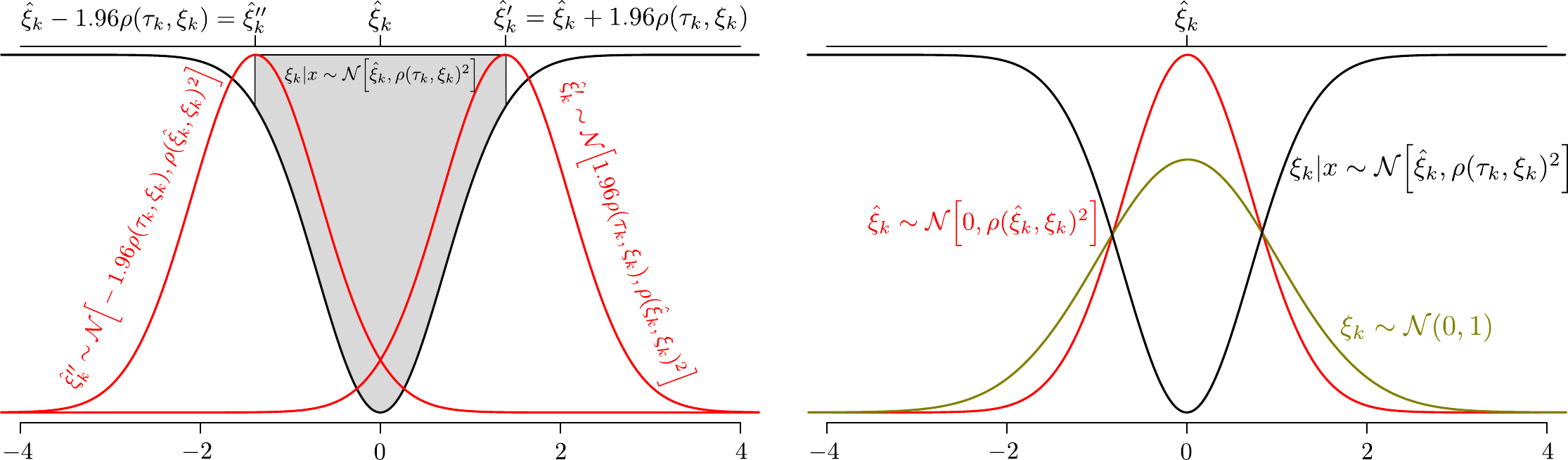}
	\caption{
		Visualization of the Bayesian view on indeterminacy in one dimension involving $\xi_k$ (see Section S.3.1.4 of the SM for an expanded figure).
		The inverted distribution (in black) is the posterior of $\xi_k|x \sim \mathcal{N}\big[\xih_k,\rho(\tau_k,\xi_k)^2\big]$.
		The posterior spread is an expression of the distance between the postulated $\xi_k$ and the best linear predictor $\xih_k$ and represents the uncertainty around the latter.
		We depict a situation in which $\rho(\tau_k,\xi_k)^2 = \rho(\xih_k,\xi_k)^2 = .5$ (as in Figure \ref{FIG:Cones2}).
		On the left-hand side we see two scoring distributions (in red) whose locations are respectively located at the lower and upper bounds of the $95\%$ credible interval for $\xi_k|x$.
		These scoring distributions are distributed as $\xi''_k \sim \mathcal{N}\big[-1.96\rho(\tau_k,\xi_k),\rho(\xih_k,\xi_k)^2\big]$ and $\xi'_k \sim \mathcal{N}\big[1.96\rho(\tau_k,\xi_k),\rho(\xih_k,\xi_k)^2\big]$
		(note that we now use primes to indicate locations alternative to $\xih$ rather than the constructions from Section \ref{SSEC:Constructions}).
		These scoring distributions will produce realizations (factor scores) that imply very different positionings in latent space (mostly negative versus mostly positive).
		On the right-hand side we see the scoring distribution with the highest posterior plausibility, $\xih_k \sim \mathcal{N}\big[0,\rho(\xih_k,\xi_k)^2\big]$.
		It will score on a narrower domain than the idealized postulate $\xi_k \sim \mathcal{N}(0,1)$ (in green).
		As the distribution of $\xi|x$ widens the scoring distributions narrow and the production of very different positions in latent space can already happen for alternative scoring distributions whose locations have high posterior plausibility.
		For determinacy one would like $\xi|x$ to concentrate on $\xih$ (Section S.3.1.5 of the SM).
		The closeness of alternative scoring distributions of interest may be quantified with an overlap coefficient or an appropriate distance such as the Hellinger distance (Section S.3.1.6 of the SM).
		Such metrics can also be used to draw a Bayesian analogy to the Guttman criterion (Section S.3.1.7 of the SM).
	}
	\label{FIG:BIndet2}
\end{figure}

\begin{IState}[Infinite superposition of scoring distributions]\label{IDS:PosteriorPoint2}
	In general $\mathbf{\Sigma}(\mathbf{\Theta})_{\xi|x} \neq \boldsymbol{0}$, such that
	\begin{equation*}
		\xi|x \nsim \lim\limits_{\mathbf{\Sigma}(\mathbf{\Theta})_{\xi|x} \rightarrow \boldsymbol{0}} \mathcal{N}_{m}\Big[\xih, \mathbf{\Sigma}(\mathbf{\Theta})_{\xi|x}\Big] = \delta\Big(\xi|x - \xih\Big).
	\end{equation*}
	That is, the conditional distribution of $\xi|x$ is not a Dirac measure concentrated on the point of highest posterior plausibility $\xih$.
	As a result, $\xih$ does not have the distributional behavior of the postulated $\xi$:
	\begin{equation*}
		\xih \nsim \mathcal{N}_{m}(\boldsymbol{0}, \PH).
	\end{equation*}
	Moreover, we then have an infinite superposition of scoring distributions whose locations have posterior plausibility determined by  $\mathbf{\Sigma}(\mathbf{\Theta})_{\xi|x}$.
\end{IState}

\begin{remark}
The considerations for $\epsilon$ lead to an analogous indeterminacy statement for $\epsilon|x$ and $\eph$ (Section S.3.1.8 of the SM).
\end{remark}

%%%%%%%%%%%%%%%%%%%%%%%%%%%%
%%%--- Infinite features ---
%%%%%%%%%%%%%%%%%%%%%%%%%%%%
\subsection{Infinite feature space}\label{SSEC:InfiniteFeature}
We show that we attain factor determinacy when the feature-space adhering to the model grows to infinity.
We do so by first generalizing a result by Guttman \citep{ref_Guttman1955} to any identified factor structure (all proofs are deferred to Appendix \ref{APP:Proofs}).

\begin{theorem}[High-dimensional limit of the inverse factor analytic kernel]
	\label{PROP:GramLimit}
	Consider model (\ref{EQ:FAmodel}) and the associated postulates and assumptions from Section \ref{SSEC:FAmodel}.
	Let these assumptions hold in a dimensional sequence $p$.
	Let $m$ be finite and fixed.
	Let a feature being an indicator for a factor be expressed by $\lambda_{jk} \neq 0$.
	Assume that, for all $j$ and $k$, $(\PS^{-1/2}\LA)_{jk} = \sum_{j'}(\PS^{-1/2})_{jj'}\lambda_{j'k} = 0\,$ only if $\,(\PS^{-1/2})_{jj'}\lambda_{j'k} = 0\,$ for all $j'$.
	Additionally assume that each factor has a number of indicators that is expressible as a fraction of $p$, $p/d_k \equiv c_k$, where $1 \leq d_k < p$ is a positive, finite divisor for factor $k$ such that $c_k \in \mathbb{N}_{\infty}$.
	Moreover, let there, in the sequence $p$, be a weakest precision-weighted indicator such that $\inf\big\{|(\PS^{-1/2}\LA)_{jk}| ~|~ |(\PS^{-1/2}\LA)_{jk}| \neq 0\big\} = \varphi$, with $\varphi$ strictly positive and arbitrarily small.
	Then
	\begin{equation*}
		\lim\limits_{p\,\uparrow\,\infty} \Big\|\big(\PHi + \Gram\big)^{-1}\Big\|_2  = 0.
	\end{equation*}
	The condition that $p$ grows to infinity is both necessary and sufficient.
\end{theorem}

\begin{remark}
	Note that Theorem \ref{PROP:GramLimit} can also hold for growing $m$.
	This factor dimension should then remain finite as $p$ grows to infinity and each additional dimension should be subject to the indicator conditions stated in the theorem.
\end{remark}

\begin{remark}
	The assumption on $\PS^{-1/2}\LA$, stating that its elements are never zero due to cancellation of constituent parameters, can be considered a parametrized faithfulness assumption. 
	The set of parameters associated with distributions in which this assumption does not hold is of Lebesgue measure $0$ \citep{ref_Faithfulness}. 
	Hence, it can be considered a mild assumption.
	The additional assumption on $\PS^{-1/2}\LA$ (there being a smallest absolute nonzero value over the precision-weighted indicators) expresses that the (ordered) terms of the sequence over the nonzero (squared or absolute) elements in any column of $\PS^{-1/2}\LA$ are not subject to a convergent series.
	We also consider this assumption to be mild.
	Hence, the factor-regime may be weak in the absolute magnitude of the precision-weighted loadings but must be strong in the number of indicators per factor.
	Our regime can thus be described as spectrally strong (eigenvalues grow with $p$) delocalization (precision-weighted loadings spread across many features).
\end{remark}

Theorem \ref{PROP:GramLimit} implies that $\mathbf{\Sigma}(\mathbf{\Theta})_{\xi|x}$ converges in the spectral norm, and thus also entrywise, to $0$ as $p\uparrow\infty$.
Using Theorem \ref{PROP:GramLimit} we can assess the high-dimensional convergence of core factor analytic quantities.
These are $\mathbb{E}\big(\hat{\xi}\hat{\xi}^{\top}\big)$ and $\mathbb{E}\big(\hat{\epsilon}\hat{\epsilon}^{\top}\big)$ from the factor solution standpoint, and the inverse of the model-based correlation matrix.

\begin{corollary}[High-dimensional factor analytic matrix convergence]
	\label{CORROL:FacLimits}
	Let Theorem \ref{PROP:GramLimit} hold. 
	Then the following convergences hold:
	\begin{enumerate}
		\item[~~i.] $\lim\limits_{p\,\uparrow\,\infty} \Big\|\PH - \PH\LAt\big(\Com + \PS\big)^{-1}\LA\PH\Big\|_2 = 0$; and
		\item[~ii.] $\lim\limits_{p\,\uparrow\,\infty} \Big\|\PS - \PS\big(\Com + \PS\big)^{-1}\PS\Big\|_{\infty,\infty} = 0$.
	\end{enumerate}
	If we additionally let the restrictions on $\PS$ counted by $r_{\Psi}$ be exclusion restrictions such that $\PS$ is permutable to a block-diagonal matrix in which each block is of finite dimension, then the following convergence also holds:
	\begin{enumerate}
		\item[iii.] $\lim\limits_{p\,\uparrow\,\infty} \Big\|\PSi - \big(\Com + \PS\big)^{-1}\Big\|_{\infty,\infty} = 0$.
	\end{enumerate}
	The condition that $p$ grows to infinity is both necessary and sufficient.
\end{corollary}

Corollary \ref{CORROL:FacLimits} implies entrywise convergence of $\mathbb{E}(\hat{\xi}\hat{\xi}^{\top})$ to $\PH$,  $\mathbb{E}(\hat{\epsilon}\hat{\epsilon}^{\top})$ to $\PS$, and $\mathbf{\Sigma}(\mathbf{\Theta})_{xx}^{-1}$ to $\PSi$ by spectrally strong delocalization as $p\uparrow\infty$.
With the results in Theorem \ref{PROP:GramLimit} and Corollary \ref{CORROL:FacLimits} we can find that, when $p\uparrow\infty$, Indeterminacy Statements \ref{IDS:span}--\ref{IDS:Markov} adhere to limiting equalities.

\begin{corollary}[Limiting equalities]
	\label{CORROL:Equalities}
	Let Theorem \ref{PROP:GramLimit} hold.
	Then the inequalities in Indeterminacy Statements \ref{IDS:span} to \ref{IDS:MSE} become (entrywise) limiting equalities if and only if $p\uparrow\infty$.
	When the additional assumption from  Corollary \ref{CORROL:FacLimits} similarly holds, then the inequality from Indeterminacy Statement \ref{IDS:Markov} also becomes an entrywise limiting equality if and only if $p\uparrow\infty$.
\end{corollary}

Corollary \ref{CORROL:Equalities} implies the following.
When $p\uparrow\infty$ the spans of $\mathbf{\Sigma}(\mathbf{\Theta})_{xx}$ and $\mathbf{\Sigma}_{xx}$ concur and $\{m, \mathbb{E}(\eph\eph^{\top}), \mathbf{\Lambda}, \mathbb{E}(\xih\xih^{\top}), (\xih,\eph)\}$ is a factor solution for $x$.
In addition, the squared multiple correlation coefficient between the best linear estimate $\xih_k$ and $\xi_k$ is then unity for all $k$.
Also, all $\xi' \in \Xi$ become indistinguishable due to conical contraction around $\xih$.
Furthermore, the vanishing partials in $\mathbf{\Sigma}(\mathbf{\Theta})_{x|\xi}^{-1}$ can then be represented by the precision matrix on the observables.
Moreover, we then also observe posterior concentration of measure, resolving Indeterminacy Statement \ref{IDS:PosteriorPoint2}, as the following corollary indicates.

\begin{corollary}[Posterior concentration of measure]
	\label{CORROL:Concentration}
	Let Theorem \ref{PROP:GramLimit} hold.
	Let $\xi|x_p$ and $\xih_p$ denote sequences of $\xi|x$ and $\xih$ dependent on feature-dimension $p$.
	Then $p\uparrow\infty$ is necessary and sufficient for the following convergences:
	\begin{enumerate}
		\item[~~i.] $\lim\limits_{p\,\uparrow\,\infty} \xi|x_p \overset{\Pr}{\longrightarrow} \delta\Big(\xi|x_p - \xih_p\Big)$; and
		\item[~ii.] $\lim\limits_{p\,\uparrow\,\infty} \xih_p \overset{\Pr}{\longrightarrow} \mathcal{N}_{m}(\boldsymbol{0}, \PH)$.
	\end{enumerate}
\end{corollary}

Corollary \ref{CORROL:Concentration} can be seen as a multivariate instance of concentration of measure \cite{Milman, Tala}. 
The distribution of $\xi|x$ becomes a Dirac measure concentrated on $\xih$ when $p\uparrow\infty$. 
In that instance, all there is to know about $\xi$ through $x$, can be obtained from $\xih$. 
The behavior of this proxy is then in accordance with Postulate (i) from Section \ref{SSEC:FAmodel}.
Corollaries \ref{CORROL:Equalities} and \ref{CORROL:Concentration} imply that we can achieve factor determinacy across all its facets in the feature-limit. 

\begin{remark}
	Similar determinacy results under $p\uparrow\infty$ can be obtained for proxy $\eph$ of the unique factors $\epsilon$ (Section S.3.2.1 of the SM). 
	An interesting point about factor model collapse is also to be made if we assume $p = \infty$ (Section S.3.2.2 of the SM).
\end{remark}

%%%%%%%%%%%%%%%%%%%%%%%%%%%%
%%%--- HD estimation -------
%%%%%%%%%%%%%%%%%%%%%%%%%%%%
\subsection{Distribution-free canonical estimation when $p$ is large}\label{SSEC:Image}
In the preceding sections we have worked under Gaussianity and tacit identification.
We explicitly note, however, that (the metrics in) Indeterminacy Statements \ref{IDS:span}--\ref{IDS:Markov} do not hinge upon distributional assumptions.
Nor do the results in Theorem \ref{PROP:GramLimit} and Corollaries \ref{CORROL:FacLimits} and \ref{CORROL:Equalities}.
Also, irrespective of the distribution of $x$, $\xih$ remains the best linear predictor for $\xi$ and $\lim_{p\,\uparrow\,\infty}\|\xi - \xih_p\|_2^{2} = 0$ by the results of Section \ref{SSEC:InfiniteFeature}.
Let us now assume the model (\ref{EQ:FAmodel}) but with $\epsilon \sim\, _{p}(\boldsymbol{0},\PS)$ and $\xi \sim\, _{m}(\boldsymbol{0},\PH)$, implying $x \sim\, _{p}[\boldsymbol{0},\mathbf{\Sigma}(\mathbf{\Theta})_{xx}]$.
We thus make no distributional assumptions on our random variables other than being describable by a specific location and scale.
It is then of interest to understand, for this situation, the distributional behavior of $\xih$ when $p$ grows very large. 
We will do so in a setting of interest for data analytic practice, clarifying in passing that also Corollary \ref{CORROL:Concentration} holds in general.

To do so we first note that each oblique representation has equivalent orthogonal representations. 
We may always find a Gram decomposition of $\PH$ such that $\PH = \mathbf{K}^{-1}\mathbf{K}^{-\top}$ for any $\mathbf{K} \in \mathcal{K} \equiv \{\mathbf{K} \in \mathbb{R}^{m \times m} ~|~ |\mathbf{K}| \neq 0 \wedge \mathbf{K}^{-1}\mathbf{K}^{-\top} = \PH\}$.
From the model perspective then
\begin{equation*}
x:= \LA\xi + \epsilon = \LA\mathbf{K}^{-1}\mathbf{K}\xi + \epsilon = \lambdaDot\dot{\xi} + \epsilon = \lambdaDot\HDot\HDot^{\top}\dot{\xi} + \epsilon,
\end{equation*} 
where $\mathbf{K}\xi \equiv \dot{\xi} \sim\, _{m}(\boldsymbol{0}, \mathbf{I}_m)$ by the affine property, $\lambdaDot \equiv \LA\mathbf{K}^{-1}$, and $\HDot \in \mathrm{O}(m) \equiv \{\HDot \in \mathbb{R}^{m \times m} ~|~ \HDot^{\top}\HDot = \HDot\HDot^{\top} = \mathbf{I}_m\}$.
From the moment perspective we then have the equivalence $\mathbf{\Sigma}(\mathbf{\Theta})_{xx} = \Com + \PS = \LA\mathbf{K}^{-1}\mathbf{K}^{-\top}\LAt + \PS = \lambdaDot\lambdaDot^{\top} + \PS =  \lambdaDot\HDot\HDot^{\top}\lambdaDot^{\top} + \PS$ (Section S.3.3.1 of the SM).
Hence, any oblique representation has equivalent (rotated) orthogonal representations.
We want to use $\lambdaDot\lambdaDot^{\top} + \PS$ as the basis for a canonical representation. 
Let us make the additional assumption that $\PS$ is diagonal.
Notice that the above implies $\PS^{-1/2}\mathbf{\Sigma}_{xx}\PS^{-1/2} - \mathbf{I}_p = \PS^{-1/2}\lambdaDot\lambdaDot^{\top}\PS^{-1/2}$ and that the of rank of these expressions is $m$.
The canonical representation can then be found as (Section S.3.3.2 of the SM):
\begin{equation}\label{EQ:EstEq}
	\begin{aligned}
		\lambdaDot &= \PS^{1/2}\mathbf{V}_m\big(\mathbf{E}_m - \mathbf{I}_m\big)^{1/2}\\
		\PS &= \big(\mathbf{\Sigma}_{xx} - \lambdaDot\lambdaDot^{\top}\big) \odot \mathbf{I}_p
	\end{aligned}.
\end{equation}
In (\ref{EQ:EstEq}) $\mathbf{E}_m = \mathrm{diag}[\boldsymbol{e}_m(\mathbf{\Sigma}^{\PS}_{xx})]$, with $\boldsymbol{e}_m(\mathbf{\Sigma}^{\PS}_{xx})$ the vector of eigenvalues $e_1(\mathbf{\Sigma}^{\PS}_{xx}) \geq \dots \geq e_m(\mathbf{\Sigma}^{\PS}_{xx})$ with $\mathbf{\Sigma}^{\PS}_{xx} \equiv \PS^{-1/2}\mathbf{\Sigma}_{xx}\PS^{-1/2}$.
The corresponding matrix of eigenvectors is $\mathbf{V}_m$.
In the remainder we will use $(e,\bm{\mathrm{v}})$ as shorthand for an eigenpair of $\mathbf{\Sigma}^{\PS}_{xx}$.
Clearly, $\lambdaDot^{\top}\PSi\lambdaDot$ is diagonal and ordered, connecting to the canonical identification constraint mentioned in Section \ref{SSEC:FAmodel} (Section S.3.3.3 of the SM).
We now have the tools for the following result to clarify the properties of $\xih$ at scale under weak distributional assumptions and practical identification restrictions.

\begin{corollary}[Factor analytic projection pursuit]
	\label{CORROL:BLPAdist}
	Let the conditions and assumptions of Section \ref{SSEC:FAmodel}
	hold, except for those on distribution.
	Let $x := \lambdaDot\dot{\xi} + \epsilon$ be the orthogonal representation of model (\ref{EQ:FAmodel}) with $\epsilon \sim\, _{p}(\boldsymbol{0},\PS)$, $\dot{\xi} \sim\, _{m}(\boldsymbol{0},\mathbf{I}_m)$, and
	$x \sim\, _{p}[\boldsymbol{0},\mathbf{\Sigma}(\mathbf{\Theta})_{xx}]$
	with $\mathbf{\Sigma}(\mathbf{\Theta})_{xx} = \Com + \PS = \lambdaDot\lambdaDot^{\top} + \PS$.
	Let $\lambdaDot$ and $\PS$ be in the canonical reflection in the sense of (\ref{EQ:EstEq}) such that $\PS = \mathrm{diag}[\psi_{11}, \ldots, \psi_{pp}]$.
	Let $\xih^0 = (\mathbf{I}_m + \lambdaDot^{\top}\PSi\lambdaDot)^{-1}\lambdaDot^{\top}\PSi x$ be the best linear predictor $\xih$ under the canonical matrices and let $\xih^0_p$ denote a sequence of $\xih^0$ dependent on feature-dimension $p$.
	Moreover, let $\HDot \in \mathrm{O}(m)$ and $\mathbf{K} \in \mathcal{K}$.
	Now, let Theorem \ref{PROP:GramLimit} hold. 
	Then $p\uparrow\infty$ is necessary and sufficient for the following convergences:
	\begin{enumerate}
		\item[~~i.] $\lim\limits_{p\,\uparrow\,\infty} \xih_p^0 \overset{\Pr}{\longrightarrow} \dot{\xi} \sim\, _{m}(\boldsymbol{0},\mathbf{I}_m)$;
		\item[~ii.] $\lim\limits_{p\,\uparrow\,\infty} \HDot^{\top}\xih_p^0 \overset{\Pr}{\longrightarrow} \HDot^{\top}\dot{\xi} \sim\, _{m}(\boldsymbol{0},\mathbf{I}_m)$; and
		\item[iii.] $\lim\limits_{p\,\uparrow\,\infty} \mathbf{K}^{-1}\xih_p^0 \overset{\Pr}{\longrightarrow} \mathbf{K}^{-1}\dot{\xi} \sim\, _{m}(\boldsymbol{0}, \PH)$.
	\end{enumerate}
\end{corollary}

Corollary \ref{CORROL:BLPAdist} implies that the canonical best linear predictor $\xih^0$ asymptotically retrieves the true generative latent factor $\dot{\xi}$.
Let $\bm{\mathrm{a}} = \PS^{-1/2}\bm{\mathrm{v}}$.
The canonical matrices then stem from the first $m$ non-trivial eigensolutions $(e,\bm{\mathrm{a}})$ to the generalized eigenvalue problem $\mathbf{\Sigma}_{xx}\bm{\mathrm{a}} = e\PS\bm{\mathrm{a}}$ (Section S.3.3.4 of the SM).
These solutions are defined by maximization of the generalized Rayleigh-Ritz ratio  $(\bm{\mathrm{a}}^{\top}\mathbf{\Sigma}_{xx}\bm{\mathrm{a}})/
(\bm{\mathrm{a}}^{\top}\PS\bm{\mathrm{a}})$ \citep[][p.\ 176--180]{ref_HJ85}.
The $m$ canonical directions are thus not typical directions.
They represent (in the regime of Theorem \ref{PROP:GramLimit}) a-typical directions in the projection space where the common variance (signal) is largest relative to the error variance (noise).
These directions, in tandem with the best linear predictor, point towards the low-dimensional signal and override the usual tendency of projection geometry in high-dimension to linearize \citep[Hall-Li effect][]{ref_HLeffect} and Gaussianize \citep[Diaconis-Freedman effect,][]{ref_DFeffect}.
Rotations that retain a-typicality of direction (i.e., that preserve the regime of Theorem \ref{PROP:GramLimit}) then provide rotated representations of the true generative latent factor.
The canonical proxy can then be seen as a representational anchor.
These effects are irrespective of the distribution of $x$.
We thus view latent projection as a form of implicit regularization.
We also note that our projection is scale-invariant.

\begin{proposition}[Scale-invariance of the best linear predictor]
	\label{PROP:Invariant}
	Let $\dot{x}$ denote $x$ on its original measurement scale and let $x$ denote a standardized variable in the sense of 
	\begin{equation*}
		\big(\mathbf{\Sigma}_{\dot{x}\dot{x}} \odot \mathbf{I}_p\big)^{-1/2}(\dot{x} - \mu) \equiv x := \LA\xi + \epsilon.
	\end{equation*}
	Now, let $\mathbf{C} \in \mathbb{R}^{p\times p}$ be any diagonal matrix such that $\mathbf{C} \succ 0$. 
	For any such scaling matrix $\mathbf{C}x := \mathbf{C}\LA\xi + \mathbf{C}\epsilon$ and $\mathbf{C}\mathbf{\Sigma}(\mathbf{\Theta})_{xx}\mathbf{C} = \mathbf{C}\Com\mathbf{C} + \mathbf{C}\PS\mathbf{C}$.
	Let $\xih_{\mathbf{C}}$ (generically) denote our best linear predictor under $\mathbf{C}x$. The best linear predictor is then scale-invariant in the sense of $\xih_{\mathbf{C}} = \xih$.
\end{proposition}

The preceding results directly point to a practical estimation and scoring approach when $p$ is large for realized data matrix $\mathbf{X} \in \mathbb{R}^{n\times p}$ of $x$. 
The fixed-point iteration procedure presented in Algorithm \ref{ALG:Estimation} summarizes this approach.
We will fully explicate the algorithm elsewhere, but will make various interpretative remarks below.

\begin{algorithm}
	\caption{(High-dimensional factor analytic representation learning).
		The procedure pivotally hinges, from a computational perspective, on the leading $m$ eigenpairs of $\hat{\mathbf{\Sigma}}^{\PS}_{xx} = \hat{\mathbf{\Psi}}^{-1/2}\hat{\mathbf{\Sigma}}_{xx}\hat{\mathbf{\Psi}}^{-1/2}$.
		Below, we use the Implicitly Restarted Lanczos Bidiagonalization Method (IRLBM) \cite{ref_IRLBM} to efficiently and stably compute the leading $m$ right-singular pairs $(\hat{d}_j, \hat{\bm{\mathrm{v}}}_{j})$ for $\mathbf{X}\hat{\PS}^{-1/2}(n-1)^{-1/2} \equiv \tilde{\mathbf{X}}$.
		A pair $(\hat{d}_j, \hat{\bm{\mathrm{v}}}_{j})$ consists of estimates of the $j$th singular value $d_j$ and the right-singular vector $\bm{\mathrm{v}}_{j}$.
		It is well-known that the right-singular vectors are the eigenvectors of $\tilde{\mathbf{X}}^{\top}\tilde{\mathbf{X}}$ and that the singular values are the square roots of its eigenvalues.
		We can then retrieve the eigenpair $(\hat{e}_j, \hat{\bm{\mathrm{v}}}_{j})$ of $\hat{\mathbf{\Sigma}}^{\PS}_{xx}$ as $(\hat{d}_{j}^{2}, \hat{\bm{\mathrm{v}}}_{j})$.
		Operating on $\tilde{\mathbf{X}}$ instead of directly on $\hat{\mathbf{\Sigma}}^{\PS}_{xx}$ lowers the computational complexity from quadratic in $p$ to linear in $p$ and avoids computation and storage of $\hat{\mathbf{\Sigma}}^{\PS}_{xx}$.
		For further understanding of the algorithm let $(\hat{\bm{\mathrm{d}}}_m, \hat{\mathbf{V}}_m)$ denote the vector of the leading $m$ singular values and the corresponding matrix $\hat{\mathbf{V}}_m \in \mathbb{R}^{p\times m}$ of right-singular vectors.
		Let $\mathtt{var}(\mathbf{X})$ represent the computational operation retrieving the vector of column variances of $\mathbf{X}$.
		Also, let $\mathcal{M} \equiv \{\mathbf{M} \in \mathbb{R}^{m \times m} ~|~ \mathbf{M} \in \mathrm{O}(m) \vee \mathbf{M} \in \mathcal{K}\}$.
		Then, from an estimation perspective, $\mathbf{M}^{-1}\mathbf{M}^{-\top} = \hat{\PH}$ in the oblique situation, $\mathbf{M}^{-1}\mathbf{M}^{-\top} = \mathbf{I}_m$ in the orthogonal situation, and in the case of no rotation $\mathbf{M} = \mathbf{I}_m$ (also see Section S.3.3.6 of the SM).
		Aditionally, let $\doublehat{\mathbf{\Xi}} = [\doublehat{\boldsymbol{\xi}}_1, \ldots, \doublehat{\boldsymbol{\xi}}_n]$ define the projected data in terms of $m$-variate realized factor scores on $n$ samples.
	}
	\label{ALG:Estimation}
	\begin{algorithmic}[1]
		\Require
		$\mathbf{X}$ 
		\Comment{Centered data matrix} 
		\\
		$m, \hat{\PS}$ 
		\Comment{Choice $m$, initial estimate $\hat{\PS}$, robust choice: $q\cdot\mathrm{diag}(\mathtt{var}(\mathbf{X}))$, $q \in (0,1)$} 
		\\
		$c = 0, \varpi > 0, \dot{\varpi} > \varpi$  
		\Comment{Index, tolerance, and starting criterion}
		\While {$\dot{\varpi} > \varpi$}
		\State ~~~~~~~~\,\,\,$\hat{\mathbf{\Psi}}(c) \gets \hat{\mathbf{\Psi}}$
		\Comment{Set convergence trail}
		\State ~~~~~~~~~~~~\,\,\,\,\,$\tilde{\mathbf{X}} \gets \mathbf{X}\hat{\PS}^{-1/2}(n-1)^{-1/2}$
		\Comment{Gram element of uniqueness-whitened second moment matrix}
		\State $\big(\hat{\bm{\mathrm{d}}}_m, \hat{\mathbf{V}}_m\big) \gets \Call{IRLBM}{\tilde{\mathbf{X}},m}$
		\Comment{IRLBM extracting top $m$ right-singular pairs of $\tilde{\mathbf{X}}$}
		\State ~~~~~~~~~~~~~\,\,\,\,$\hat{\lambdaDot} \gets \hat{\mathbf{\Psi}}^{1/2}\hat{\mathbf{V}}_m\Big[\mathrm{diag}\big(\hat{\bm{\mathrm{d}}}_m^{\odot2}\big) - \mathbf{I}_m\Big]^{1/2}$  
		\Comment{Update $\hat{\lambdaDot}$}
		\State ~~~~~~~~~~~~~~\,\,$\hat{\mathbf{\Psi}} \gets \Big[\mathrm{diag}(\mathtt{var}(\mathbf{X})) - \hat{\lambdaDot}\hat{\lambdaDot}^{\top}\Big] \odot \mathbf{I}_p$ 
		\Comment{Update $\hat{\mathbf{\Psi}}$} 
		\State ~~~~~~~~~~~~~~\,\,$\dot{\varpi} \gets \max_{j}\big|\big(\hat{\mathbf{\Psi}}(c) - \hat{\mathbf{\Psi}}\big)_{jj}\big|$
		\Comment{Criterion update}
		\State ~~~~~~~~~~~~~~~\,\,\,$c \gets c + 1$
		\Comment{Update index for convergence trail}
		\EndWhile
		\State If desired, find rotation matrix $\mathbf{M} \in \mathcal{M}$
		\Comment{Implies rotation $\hat{\lambdaDot}\mathbf{M}$ and $\mathbf{M}^{-1}\mathbf{M}^{-\top} = \hat{\PH}$}
		\State $\doublehat{\boldsymbol{\xi}}_i \gets \mathbf{M}^{-1}\big(\mathbf{I}_m + \hat{\lambdaDot}^{\top}\hat{\mathbf{\Psi}}^{-1}\hat{\lambdaDot}\big)^{-1}\hat{\lambdaDot}^{\top}\hat{\mathbf{\Psi}}^{-1}\boldsymbol{\mathrm{x}}_i, ~~i = 1,\ldots,n$  
		\Comment{Empirical Bayes realization factor scores}
		\Ensure $\{\hat{\mathbf{\Psi}}, \hat{\lambdaDot}\mathbf{M}, (\mathbf{M}^{\top}\mathbf{M})^{-1}, \doublehat{\mathbf{\Xi}}\}$
		\Comment{Output collection}
	\end{algorithmic}
\end{algorithm}

The estimating equations in (\ref{EQ:EstEq}) are identical to those of Gaussian Maximum Likelihood (ML) approaches.
This can be taken as an attest to the robustness of ML or as a supporting argument for our distribution-free approach.
Fixed-point iteration procedures based on (\ref{EQ:EstEq}) go back to Lawley \cite{ref_Lawley40} and Rao \cite{ref_Rao55} but were abandoned quickly due to poor convergence \cite{ref_Howe55}.
They where replaced in favor of direct numerical likelihood optimization as a function of both $\lambdaDot$ and $\PS$, based on either Fletcher-Powell \cite{ref_Joreskogg67} or Newton-Rhapson \cite{ref_Jenn69} methods.
These approaches, however, assume Gaussianity, require $\hat{\mathbf{\Sigma}}_{xx} \succ 0$, and are of a computational complexity cubic in $p$, thus excluding datasets with $p > n$. 
Algorithm \ref{ALG:Estimation} needs neither Gaussian features nor a p.d.\ second central moment matrix.
Algorithm \ref{ALG:Estimation} only needs $p > n - 1 > m$ when (\ref{EQ:FAcov}) approximately holds in the data.
Clearly, the $m$ leading eigenvalues of $\hat{\mathbf{\Sigma}}^{\PS}_{xx}$ grow, under the regime of Theorem \ref{PROP:GramLimit}, without bound as $p\uparrow\infty$ while the remaining eigenvalues remain bounded.
This accumulation of eigenmass can be retrieved when $n - 1 > m$.
The algorithm then gains consistency and stability as $p$ grows, acts as a contraction-mapping near a solution \cite{ref_Sundberg16}, and preserves the p.d.-ness of $\hat{\PS}$ \cite{ref_Robertson07,ref_Sundberg16}.
It thus avoids the convergence issues of early implementations in low-dimension.
It is of particular use when $p \gg n - 1 > m$.
Note that the pseudo-algorithm is geared towards insightfulness and that many operations can be handled efficiently.
For example, while we represent the matrix $\hat{\mathbf{\Psi}}$ we only need to operate with its diagonal elements $\hat{\bm{\psi}} = \mathrm{diag}(\hat{\mathbf{\Psi}})$.
This allows us, for example, to obtain $\hat{\mathbf{\Psi}}^{-1/2}$ swiftly as $\mathrm{diag}\big[(\hat{\bm{\psi}}^{\odot \frac{1}{2}})^{\odot(-1)}\big]$ and to obtain updates of $\hat{\mathbf{\Psi}}$ in the vector-space only.
The computational complexity of the Algorithm is then dominated by the IRLBM step.
This step is of complexity $\mathcal{O}(mnp)$ \cite{ref_DaiDutta}.
We note that the IRLBM can also be used to swiftly determine a choice for $m$ based on eigenconcentration metrics.
The procedure is scale-equivariant in the sense that initialization with $\mathbf{XC}$ will produce estimates $\mathbf{C}\hat{\lambdaDot}$ and $\mathbf{C}\hat{\PS}\mathbf{C}$ (Section S.3.3.5 of the SM).
The realized factor scores are scale-invariant by Proposition \ref{PROP:Invariant}.
Interpretation of the loadings-structure and the latent factors can be enhanced by employing rotations.
Rotations in $\mathcal{M}$ that tend to preserve a-typicality of direction are orthogonal or oblique simple structure rotations (Section S.3.3.6 of the SM). 

To summarize, the algorithm is distribution-free, scale-equivariant, of low computational complexity, and gains stability as $p$ grows.
One may view the algorithm as a distribution-free extension of \cite{ref_DaiDutta}.
Or as an expansion of the procedure implicit in \cite{ref_Sundberg16}, geared towards computational efficiency in high $p$ and retrieval of the source data $\doublehat{\mathbf{\Xi}}$.
An empirical determinacy evaluation of the factors therein can then be given by the metrics in or derived from Indeterminacy Statements \ref{IDS:span}--\ref{IDS:Markov}.
The assertions made above are supported in Section S.3.3.7 of the SM.
%--------------- Approaches -----------------------------------------

%--------------- Discussion -----------------------------------------
\section{Discussion}\label{SEC:Discussion}
%%%%%%%%%%%%%%%%%%%%%%%%%%%%
%%%--- Summary -------------
%%%%%%%%%%%%%%%%%%%%%%%%%%%%
\subsection{Summary}\label{SSEC:Summary}
We have treated the inverse problem of retrieving the latent generative factors in common factor analysis.
The randomness of these factors induces an inherent indeterminacy in their retrieval at which we have taken a comprehensive conceptual and mathematical look.
We have seen that the retrieval becomes determinate across all its aspects when the feature-space grows to infinity.
This result can be exploited in the form of an essentially distribution-free estimation approach for the sample situation in which $p \gg n - 1 > m$. 
The empirical Bayes factor projection will then behave as the true generative latent, irrespective of the distribution of $x$.

Note that we have not treated the situation in which some or all $\psi_{jj} = 0$ (indicated as Heywood cases in the Psychometric literature). 
The case in which some $\psi_{jj} = 0$ is treated by Krijnen \citep{ref_KrijnenHeywood} and may be seen as a special variant of the context treated in Section S.2.1 of the SM.
Note that we also could have entertained alternative proxies.
We choose a linear function of $x$ as linear combinations tend to asymptotically satisfy linearity conditions. 
Our linear proxy of choice was then the conditional mean $\mathbb{E}(\xi|x)$.
While there are other regression-type proxies, these are subject to the same determinacy criteria and will have higher MSE (see Section S.4.1 of the SM).

In the remainder we are going to contextualize our first two main messages by reviewing the methodological and philosophical implications for the Psychometrics, Statistics, and Artificial Intelligence communities. 
Many of the following notes (and the concurrent notes in Sections S.4.2.1, S.4.3.1, and S.4.4.1 of the SM) may serve as the basis for further research.

%%%%%%%%%%%%%%%%%%%%%%%%%%%%
%%%--- Psychometrics -------
%%%%%%%%%%%%%%%%%%%%%%%%%%%%
\subsection{On implications for Psychometrics}\label{SSEC:Psycho}
\subsubsection{The divide between exploratory and confirmatory modes}
\label{SSSEC:Divide}
This work implies that the divide between exploratory and confirmatory schools should be discarded in situations that are appreciable from a determinacy standpoint (i.e., in high dimension).
With $p$ high (hundreds, thousands, or even more) it will no longer be practical to pre-specify a factor structure by exclusion restrictions in the parameter matrices as done in confirmatory factor analysis \citep[cf.][]{ref_Peeters12}.
On the other hand this work also suggests that exploratory factor analysis is no \emph{post facto} detection method, nor a purely abductive endeavor \citep[cf.][]{ref_HaigAbductive} if it is to be of use in terms of determinate latent projections. 
That is, blind factor analysis in which the method is applied to any arbitrary collection of items will not discover usable or determinate latent sources \citep{ref_MM_Indet, ref_Mulaik2010}. 
In high dimension, the approach will have to be a natural blend of confirmatory and exploratory modes: sampling from a feature-space that has an underlying conceptual idea of the latent features that are being probed (confirmatory) coupled with flexible, automated determination of parameter estimates (an exploratory, distributed representation).

\subsubsection{The ontological status of latent factors}
\label{SSSEC:Status}
The classic Psychometric literature assumed that one of the $\xi' \in \Xi$ (in our cone) represented the true latent vector \citep{ref_Cattell78, ref_EBG78, ref_Mulaik2010}.
However, the results on conical contraction indicate that, in the limit, all $\xi' \in \Xi$ are indistinguishable from $\xih$. 
This brings us to the ontological status of the latent common factors \citep{ref_BorsOnto}.
Many works in Psychology suffer from the nominal fallacy \citep{ref_KuaNominal}.
However, giving a latent factor a natural language description does not amount to explaining it nor to identifying its existence.
This work implies that latent factors in the common factor analytic model are perhaps best viewed through the lens of the manifold hypothesis.
This hypothesis observes that high-dimensional data often concentrates on low-dimensional manifolds embedded in the high-dimensional space \citep{ref_Manifold}.
Indeed, if $\hat{\mathbf{\Sigma}}^{+}_{xx}$ is (near-)diagonal then the factor model forms a possible data generating mechanism and the data possibly lie near an identifiable latent linear manifold.
Whether this manifold can be assigned a natural language description depends somewhat on the field of study. 
For the Behavioral and Social Sciences this might be a nicety, but it is not a necessity for (successful) application in, for example, imaging and remote sensing. 
We then make a case for a pragmatist-realist point of view: when the factor model is an adequate description of the data-generating mechanism then a successful latent projection justifies commitment to the reality of the latent entity, irrespective of the possibility for a natural language description. 
To put it bluntly: if it works, it works.

\subsubsection{Hierarchical factor analysis}
\label{SSSEC:HFA}
Hierarchical factor analysis has been popular in behavioral research.
It consists of a hierarchy of linear latent factor models that may be viewed as a generative deep learner under Gaussian latents and errors with identity activation.
Indeed, while Amari \citep{ref_AmariDeep} and Ivakhnenko and Lapa \citep{ref_IvaDeep} are generally considered to be the first publications in feedforward deep learners, Schmid and Leiman \citep{ref_SchmidtDeep} unwittingly published the first instance of a generative deep learner. 
However, with linear activation they produce a composition that is simply an extended linear map (see Section S.4.2.2 of the SM).
In addition, higher-order latent factors are likely to suffer from (high) indeterminacy as they are represented by relatively few lower-order latent factors.
We will expand upon this topic in Sections \ref{SSSEC:Collapse} and \ref{SSSEC:Layers}.

%%%%%%%%%%%%%%%%%%%%%%%%%%%%
%%%--- Statistics ----------
%%%%%%%%%%%%%%%%%%%%%%%%%%%%
\subsection{On implications for Statistics}\label{SSEC:Stats}
\subsubsection{Bayesian estimation}
\label{SSSEC:Estimate}
The developments indicate that the observable dimensions needed to achieve practical determinacy discard the possibility of a full Bayesian approach towards the factor model.
This model is notorious for slow-mixing.
In addition, with $p$ ranging in the hundreds, thousands, or even higher, posterior simulation (by any form of Markov chain Monte Carlo) is rendered impractical or computationally impossible due to the curse of dimensionality. 
Variational inference approximations to the posterior are more efficient from a computational view.
However, for very large $p$, runtimes for the factor model will still be too considerable to be deemed practicable \citep{ref_Munch_FR}.
We see, as this work implies that the uncertainty in estimates of $\{\LA,\PH,\PS\}$ is inconsequential for the determinacy of factor projection, two alternative possibilities for Bayesian modeling.
First, in situations of (near) determinacy one could obtain, through Algorithm \ref{ALG:Estimation}, factor projections to be used in further full Bayesian modeling. 
Second, if interested in uncertainties regarding $\{\LA,\PH,\PS\}$, one could use concentration in terms of a (near) determinate projection to perform MCMC conditional on a fixation of this projection. 

\subsubsection{Suitability assessment}
\label{SSSEC:Suitable}
We may also use the findings on concentration to assess if the factor model could be adequate for the data at hand.
Let $\hat{\mathbf{\Sigma}}_{xx}$ be a suitable sample-estimate of the second central moment of the observables.
If $\hat{\mathbf{\Sigma}}^{+}_{xx}$ is (near) diagonal we have indication that (a) the model in (\ref{EQ:FAmodel}) is a possible description of the data generating mechanism and, if so, (b) one can obtain a factor projection that is (near) determinate. 
Hence, assessing the generalized inverse of a sample estimate of the covariance matrix provides a probe regarding the factorability of the data.
Note that the (near) diagonality of $\hat{\mathbf{\Sigma}}^{+}_{xx}$ is a necessary but not sufficient condition for factorability (see Section S.4.3.2 of the SM).
Nonetheless, the assessment of $\hat{\mathbf{\Sigma}}^{+}_{xx}$ can provide an alternative to methods that require positive definite estimates of $\mathbf{\Sigma}$, such as the Kaiser-Meyer-Olkin index \citep{ref_KMO}.

\subsubsection{Sparse versus distributed representation}
\label{SSSEC:Distribute}
Recently, $\ell_1$-norm and $\ell_2$-norm penalties have been incorporated in factor analytic estimation procedures to deal with the burden of high-dimensionality \citep[see, e.g.,][]{ref_Yama2015, ref_SparseEFA}. 
Usage of $\ell_1$-norm penalties will result in sparse loadings (weights).
While a sparse weight structure may enhance the interpretation of loadings, the regularization implies some suppression of signal and some loss of fit. 
In addition, a sparse structure will diminish determinacy.
The developments above imply that a componential or distributed representation \citep[involving all states of the latent factors,][]{ref_HG97} leads to better determined latent factor projections.
From this perspective one could prefer the estimation procedure as advocated here followed by some form of simple structure rotation \citep{ref_VARIMAX,ref_PROMAX}.
Such rotations provide an \emph{approximately} sparse representation.
Hence, they support interpretation, but do not diminish determinacy.
Also, distributed representations are known to be more robust against noise and information loss \citep[][p.\ 78]{ref_AImodern}.
Section \ref{SSEC:Image} implies we can achieve this without computationally expensive explicit regularization.

%%%%%%%%%%%%%%%%%%%%%%%%%%%%
%%%--- AI ------------------
%%%%%%%%%%%%%%%%%%%%%%%%%%%%
\subsection{On implications for Artificial Intelligence}\label{SSEC:AI}
\subsubsection{Posterior collapse}
\label{SSSEC:Collapse}
Much recent work in unsupervised representation learning focuses on deep generative latent variable models.
These models are based on the idea of autoencoders, especially the variational autoencoder \citep{ref_KWel} in which intractable posteriors are approximated by variational posteriors drawn from computationally tractable families of distributions.
These rely on approximate methods such as the evidence lower bound (ELBO).
This type of approximate Bayesian learning has deep connections to probabilistic PCA \citep{ref_Nakajma, ref_Lucas}, which is a special case of the model introduced in Section \ref{SSEC:FAmodel}.
Variational autoencoders often encounter posterior collapse \citep{ref_Bowman, ref_KingmaNeu}, which implies $p(\xi|x) = p(\xi)$ or, in words: the posterior of the latent variables equals the postulated prior.
This phenomenon is interpreted in the sense of the generative model being informationally replete and is often attributed to the use of the ELBO \citep{ref_Ravazi}, spurious local maxima \citep{ref_Lucas}, optimization \citep{ref_Alemi}, or the flexible capacity of neural networks \citep{ref_Dai20}.

We have treated the factor model as a minimal model, i.e., as a linear autoencoder with a tractable posterior.
This implies that the phenomenon of posterior collapse is persistent, in the sense that posterior collapse is a reflection of factor indeterminacy.
The developments in Section \ref{SEC:Perspectives} indicate that $p(\xi|x)$ does not directly reflect information on positionings in latent space. 
Rather, it reflects the uncertainty around its location, itself a random variable with which placings (of objects, persons, samples) in latent space are possible. 
The covariance matrix $\mathbf{\Sigma}(\mathbf{\Theta})_{\xi|x}$ of $p(\xi|x)$ carries the MSE between $\xi$ and the most plausible location for latent positioning $\xih$.
It implies a duality between $\mathbf{\Sigma}(\mathbf{\Theta})_{\xi|x}$ and $\mathbb{E}\big(\xih\xih^{\top}\big)$ in the sense of $\mathbf{\Sigma}(\mathbf{\Theta})_{\xi|x} = \PH - \mathbb{E}\big(\xih\xih^{\top}\big)$.
In case $\mathbb{E}\big(\xih\xih^{\top}\big) = \PH\LAt(\Com + \PS)^{-1}\LA\PH \approx \boldsymbol{0}$, then $\xi|x$ is approximately distributed as $\mathcal{N}_m\big(\xih, \PH\big)$.
From a realization perspective its behavior is $\mathcal{N}_m\big[\mathbb{E}(\xih), \PH\big] \overset{d}{=} \mathcal{N}_m(\boldsymbol{0}, \PH) = p(\xi)$.
The actual scoring in latent space could then be any near-constant on the support of $p(\xi|x)$ as any $\xih_k \pm c\rho(\tau_k,\xi_k)$ is approximately a Dirac delta function centered on $\pm c\rho(\tau_k,\xi_k)$ with posterior plausibility determined by $p(\xi_k|x)$.
This implies complete uncertainty regarding placings in latent space.
For this case there is no underlying factor model where $\xi$ exerts a retrievable signal to $x$. 
This can be due to $\epsilon \approx x$ or the dimension of $p$ being too low, in the sense of inadequate probing of the feature-space indicative of $\xi$.
Thus, in the determinate situation $p(\xi|x)$ is a Dirac measure on $\xih$ and $\xih$ behaves exactly as postulated in $p(\xi)$.
Latent scoring collapses when $p(\xi|x)$ approximates $p(\xi)$.
From this perspective it would thus be more appropriate to speak of latent variable collapse rather than posterior collapse.
This has immediate repercussions for deep generative learning models. 

\subsubsection{Layeredness in generative models}
\label{SSSEC:Layers}
Dellaleau and Bengio \citep{ref_DelBeng} have, as do universal approximation theorems in general \citep{ref_Universal}, shown that, for feedforward architectures, a deep distributed learner can give a more efficient representation than a shallow one.
A shallow architecture can learn the same functions, but will need (much) more neurons. 
The same observation is made by Hinton and Salakhutdinov \citep{ref_Hinton2006}.
For generative architectures the situation is different.
If we focus on our minimal generative architecture, the latent factor model as a single-hidden-layer generative model, we de facto have that $x$ (the observable space) provides the neurons.
Our results show that, to achieve determinacy in this architecture, the number of neurons must grow large.
From this perspective, when $p$ is very large, a single generative layer is sufficient to provide a latent representation without uncertainty that depends only linearly on $x$.
From a modern (high-throughput) data collection perspective, obtaining additional observational neurons is relatively cheap (see also Section \ref{SSEC:Reflect}).

Indeed, deep generative architectures will encounter the propagation of indeterminacy to higher-order connective layers, heightening the probability of encountering latent variable collapse with each consecutive layer.
To see this we will do a thought experiment.
Say we have an unsupervised deep generative model expressed as a hierarchy of generalized factor models \citep{ref_Chen} that can recursively be decomposed as:
\begin{equation*}
	\xi^{(l - 1)} := f^{(l)}\Big(\LA^{(l)}\xi^{(l)} + \epsilon^{(l)}\Big),
\end{equation*}
for $l = 1, \ldots, L$. 
In this expression the $\xi^{(l)}$ are the generative hidden features for the $l$th layer, with $\LA^{(l)}$ and $\epsilon^{(l)}$ the corresponding weight matrix and error vector.
It is usual to let the distributional assumptions on $\xi^{(l)}$ and  $\epsilon^{(l)}$ be simple such as $\xi^{(l)} \sim \mathcal{N}_{m^{(l)}}\big(\boldsymbol{0}, \PH^{(l)}\big)$ and  $\epsilon^{(l)} \sim \mathcal{N}_{m^{(l-1)}}\big(\boldsymbol{0}, \PS^{(l)}\big)$.
Proxy $\hat{\xi}^{(l)}$ for $\xi^{(l)}$ can be determined as $\xih^{(l)} \equiv \dot{\mathbb{E}}\big[\xi^{(l)}|\xih^{(l-1)}\big]$, where $\dot{\mathbb{E}}[\cdot|\cdot]$ is a possibly nonlinear expectation.
Naturally, $\xih^{(0)} = x$.
Now, let $f^{(l)}(\cdot)$ denote a (semi-affine) activation function of the $l$th layer.
If we then define
\begin{equation*}
	f^{(l)}\Big(\LA^{(l)}\xi^{(l)} + \epsilon^{(l)}\Big) \equiv \tilde{f}^{(l)},
\end{equation*}
we have the following composition:
\begin{equation*}
	\Big(\tilde{f}^{(L)} \circ \tilde{f}^{(L - 1)} \circ \cdots \circ \tilde{f}^{(1)}\Big).
\end{equation*}
This would define a generative (top-down or feedback) architecture, in sense of (now using the arrows to express computational rather than graphical or causal structure in the sense of Section \ref{SSEC:Markov}):
\begin{equation*}
	\xi^{(L)} \longrightarrow \xi^{(L - 1)} \longrightarrow \cdots \longrightarrow \xi^{(0)} = x.
\end{equation*}
Now, unless $\PS^{(l)} = \boldsymbol{0}$ (weights having perfect communality), the identification results from Section \ref{SSEC:FAmodel} suggest that the dimension $m^{(l)}$ of $\xi^{(l)}$ should be small relative to the dimension $p^{(l-1)} = m^{(l-1)}$ of $\xi^{(l-1)}$.
This implies, for each consecutive layer, a lesser number of projective indicators for the connective generative latent vector.
Except for degenerate cases, one is then bound to encounter a layer for which $p\big(\xi^{(l)} | \xi^{(l-1)}\big) = p\big(\xi^{(l)}\big)$, i.e., latent variable collapse.
This is most easily seen for the case in which the $f^{(l)}(\cdot)$ represent identity functions such that we have linear activation and where $\dot{\mathbb{E}}[\cdot|\cdot]$ signifies the regular expectation operator (bringing us to the hierarchical factor analysis discussed in Section \ref{SSSEC:HFA}), but carries over to the more general model.

This suggests, to avoid upward propagation of indeterminacy and encountering latent variable collapse, to combine a determinate generative layer with an architecture where the higher-order latents are formative (i.e., bottom-up, feedforward, or deterministic) in nature:
\begin{equation*}
	\bar{\xi}^{(L)} \longleftarrow \cdots \longleftarrow \bar{\xi}^{(2)} \longleftarrow \xi^{(1)} \longrightarrow \xi^{(0)} = x,
\end{equation*}
where the $\bar{\xi}^{(l)}$ are formative latents.
That is, a shallow generative part followed by a deeper formative structure (see also Section S.4.4.2 of the SM).
This can be considered a form of (unsupervised) pre-training \citep{ref_PreTrain}.
It is also commensurate with the idea of predictive coding networks \citep{ref_RaoPC99,ref_HanPC18}. 
%--------------- Discussion -----------------------------------------

%--------------- Conclusion -----------------------------------------
\section{Conclusion}\label{SSEC:Reflect}
%\subsection{A checkered beginning}
%The factor model originated in the Psychometric literature in the early 1900s.
%Since its inception, the model was surrounded by a certain mystique, partly fuelled by a misunderstanding of latent factor indeterminacy and the associated problems of identifying and interpreting latent projections.
%Especially in the 1950s and 1960s factor analysis was applied blindly to any collection of variables in the hope that the underlying latent sources would be `discovered' \citep[][Chapter 1]{ref_Mulaik2010}.
%While popular in the Social and Behavioral Sciences, this mystique had led to a certain disdain for the procedure amongst (applied) statisticians, see, e.g., \citep[the preface to][]{ref_Basi} and \citep[][p.\ 263]{ref_Feinstein}.
%One could say that the model was borne before its time \cite{ref_Batholomew_Origins}.

%\subsection{The blessing of dimensionality}
%At current, with renewed interest in (statistical) learning theory in this second wave of Artificial Intelligence, the model has broken free of its Psychometric mold.
%The model can be thought of as a generative statistical representation learner \citep{ref_FArepresent} with many variants (Helmholtz and Boltzmann machines, linear autoencoder) and special cases (probabilistic PCA), or as a base learner for the understanding and building of deep, generative architectures \citep{ref_AImodern}.
%As such, it has found renewed interest \citep[e.g.,][]{ref_Rohe23}.

Latent variable models are often divided into the classic single-hidden-layer and neural-network based deep architectures, with the former thought of as interpretable but offering limited model capacity and vice versa for the latter \citep{ref_CM}.
This work has shown that, from the generative factor perspective, the observables $x$ are the neurons and, as their dimension grows large, one can retrieve the generative latent sources without uncertainty and this retrieval is only linearly dependent on $x$.
This result depends on the dimension $p$ of $x$, not on its sample dimension.
It points again to a concentration of measure phenomenon that, due to its simplified (informational) geometry, drastically simplifies the retrieval of latent sources in high dimension \citep{ref_Kainen1997, ref_Garbon18}.
As such, the factor model is very suited for the analysis of  $p \gg n - 1 > m$ data.
With the adoption of modern data collection technology, the marginal cost of collecting measurements on additional features has shrunk dramatically.
As a result,  $p \gg n$ data are now ubiquitous, such as those stemming from high-throughput technologies and high-resolution imaging.
A universal feature of such high-dimensional data seems to be that they are approximately low-rank \citep{ref_Udell}.
For the factor model concentration of measure, the resulting determinacy of the retrieved factors, the linear dependency of this retrieval on $x$, and the exact localization of scorings in latent space, are all emergent properties at scale.
Hence, for the factor model dimensionality is a blessing rather than a curse \citep{ref_Donoho2000}.
%Factor analysis might have been borne before its time, but the era of modern high-dimensional data collection might see it come of age.
%--------------- Conclusion -----------------------------------------

%--------------------------------------------------------------------
%%%%%%%%%%%%%%%%%%%%%%%%%%%%%%%%%%%%%%%%%%%%%%%%%%%%%%%%%%%%%%%%%%%%%
%%%%%%%%%%%%%%%% Appendix and backmatter %%%%%%%%%%%%%%%%%%%%%%%%%%%%
%%%%%%%%%%%%%%%%%%%%%%%%%%%%%%%%%%%%%%%%%%%%%%%%%%%%%%%%%%%%%%%%%%%%%

%%--------------------------------------------------------------------
%%--------------- Appendix -------------------------------------------
\begin{appendix}
%%%%%%%%%%%%%%%%%%%%%%%%%%%%
%%%--- Identities ----------
%%%%%%%%%%%%%%%%%%%%%%%%%%%%
\section{Useful identities}\label{APP:Identities}
We establish some useful matrix identities that support the developments in the (supplementary) text as well as the proofs in Appendix \ref{APP:Proofs}.
The first identity may be understood as a matrix-variant of the Hua-indentity \citep{ref_HUAident}.
It may also be understood as representing the left- and right-most extremes of the Woodbury matrix identity \citep{ref_Duncan1944,ref_Woodbury}  under consecutive application of the push-through identity \citep{ref_Push}:
\begin{lemma}[Hua-type matrix identity]\label{SMLEM:HUAmat}
	Let $\mathbf{P} \in \mathbb{R}^{a \times a}$, $\mathbf{L} \in \mathbb{R}^{b \times a}$, and $\mathbf{R} \in \mathbb{R}^{b \times b}$.
	Moreover, let $\mathbf{R}$ and $\mathbf{LPL}^{\top} + \mathbf{R}$ be nonsingular.
	The following identity then holds:
	\begin{equation}\label{SMEQ:HUAmat}
		\mathbf{R}^{-1} - \mathbf{R}^{-1}\mathbf{LPL}^{\top}\big(\mathbf{LPL}^{\top} + \mathbf{R}\big)^{-1} =
		\mathbf{R}^{-1} - \big(\mathbf{LPL}^{\top} + \mathbf{R}\big)^{-1}\mathbf{LPL}^{\top}\mathbf{R}^{-1} =
		\big(\mathbf{LPL}^{\top} + \mathbf{R}\big)^{-1}.
	\end{equation}
\end{lemma}

\begin{sloppypar}
\begin{proof}
	We will proceed by direct proof.
	The last equality implies that pre- and post-multiplication by $(\mathbf{LPL}^{\top} + \mathbf{R})$ should give an identity matrix.
	First consider post-multiplying the left-hand side expression in (\ref{SMEQ:HUAmat}) with $(\mathbf{LPL}^{\top} + \mathbf{R})$:
	\begin{align*}
		\Big[\mathbf{R}^{-1} - \mathbf{R}^{-1}\mathbf{LPL}^{\top}\big(\mathbf{LPL}^{\top} + \mathbf{R}\big)^{-1}\Big]\big(\mathbf{LPL}^{\top} + \mathbf{R}\big) =
		\mathbf{R}^{-1}\mathbf{LPL}^{\top} + \mathbf{I}_b - \mathbf{R}^{-1}\mathbf{LPL}^{\top} = \mathbf{I}_b.
	\end{align*}
	Let us now establish the first equality in (\ref{SMEQ:HUAmat}).
	Note that
	\begin{equation}\label{EQ:CyclePre}
		\mathbf{R}^{-1} - \mathbf{R}^{-1}\mathbf{LPL}^{\top}\big(\mathbf{LPL}^{\top} + \mathbf{R}\big)^{-1} =
		\mathbf{R}^{-1} - \mathbf{R}^{-1}\mathbf{LPL}^{\top}\mathbf{R}^{-1}\big(\mathbf{LPL}^{\top}\mathbf{R}^{-1} + \mathbf{I}_b\big)^{-1}.
	\end{equation}
	Sequential application of Sylvester's determinant identity (sometimes called the Weinstein-Aronzajn identity) \citep[][Chapter 4, Section 6]{ref_Kato76} gives that
	\begin{equation}\label{EQ:Cycle}
		\big|\mathbf{LPL}^{\top}\mathbf{R}^{-1} + \mathbf{I}_b\big| = 
		\big|\mathbf{PL}^{\top}\mathbf{R}^{-1}\mathbf{L} + \mathbf{I}_a\big| = 
		\big|\mathbf{L}^{\top}\mathbf{R}^{-1}\mathbf{LP} + \mathbf{I}_a\big| = 
		\big|\mathbf{R}^{-1}\mathbf{LPL}^{\top} + \mathbf{I}_b\big|.
	\end{equation}
	Taking the first determinant expression above we see that
	\begin{equation*}
		\big|\mathbf{LPL}^{\top}\mathbf{R}^{-1} + \mathbf{I}_b\big| = 
		\big|\big(\mathbf{LPL}^{\top} + \mathbf{R}\big)\mathbf{R}^{-1}\big| = \big|\mathbf{LPL}^{\top} + \mathbf{R}\big|\big|\mathbf{R}^{-1}\big| > 0,
	\end{equation*}
	by the nonsingularity assumption on $\mathbf{R}$ and $\mathbf{LPL}^{\top} + \mathbf{R}$.
	As a consequence, all inverses of the (cyclic) matrix sums expressed in (\ref{EQ:Cycle}) exist.
	Hence, the right-hand side of (\ref{EQ:CyclePre}) may then, by sequential application of the push-through identity \citep{ref_Push}, be shown to be equivalent to
	\begin{equation*}
		\mathbf{R}^{-1} - \big(\mathbf{R}^{-1}\mathbf{LPL}^{\top} + \mathbf{I}_b\big)^{-1}\mathbf{R}^{-1}\mathbf{LPL}^{\top}\mathbf{R}^{-1} =
		\mathbf{R}^{-1} - \big(\mathbf{LPL}^{\top} + \mathbf{R}\big)^{-1}\mathbf{LPL}^{\top}\mathbf{R}^{-1},
	\end{equation*}
	establishing the first equality in (\ref{SMEQ:HUAmat}).
		Now, pre-multiplying the middle expression with $(\mathbf{LPL}^{\top} + \mathbf{R})$ gives:
	\begin{align*}
		\big(\mathbf{LPL}^{\top} + \mathbf{R}\big)\Big[\mathbf{R}^{-1} - \big(\mathbf{LPL}^{\top} + \mathbf{R}\big)^{-1}\mathbf{LPL}^{\top}\mathbf{R}^{-1}\Big] =
		\mathbf{LPL}^{\top}\mathbf{R}^{-1} + \mathbf{I}_b - \mathbf{LPL}^{\top}\mathbf{R}^{-1} = \mathbf{I}_b.
	\end{align*}
	This establishes all equalities stated in (\ref{SMEQ:HUAmat}).
\end{proof}
\end{sloppypar}

The second identity can be thought of as an expanded Woodbury matrix identity.
It can be obtained by direct application of stated identity followed by some ready algebra (or, alternatively, by repeated usage of Lemma \ref{SMLEM:HUAmat}).
\begin{lemma}[Expanded Woodbury matrix identity]\label{SMLEM:ExpWood}
	Let $\mathbf{P} \in \mathbb{R}^{a \times a}$, $\mathbf{L} \in \mathbb{R}^{b \times a}$, and $\mathbf{R} \in \mathbb{R}^{b \times b}$.
	Moreover, let $\mathbf{R}$, $\mathbf{LPL}^{\top} + \mathbf{R}$, and $\mathbf{P}$ be nonsingular.
The following identity then holds:
	\begin{equation}\label{SMEQ:ExpWood}
		\mathbf{R} - \mathbf{R}\big(\mathbf{LPL}^{\top} + \mathbf{R}\big)^{-1}\mathbf{R} =
			\mathbf{L}\big(\mathbf{L}^{\top}\mathbf{R}^{-1}\mathbf{L} + \mathbf{P}^{-1}\big)^{-1}\mathbf{L}^{\top}.
	\end{equation}
\end{lemma}

%\begin{proof}
%	We again proceed by direct proof. 
%	By Lemma \ref{SMLEM:HUAmat} we have
%	\begin{align*}
%		\mathbf{R}\big(\mathbf{LPL}^{\top} + \mathbf{R}\big)^{-1}\mathbf{R} &= 
%		\mathbf{R}\Big[\mathbf{R}^{-1} - \mathbf{R}^{-1}\mathbf{LPL}^{\top}\big(\mathbf{LPL}^{\top} + \mathbf{R}  \big)^{-1}\Big]\mathbf{R} \\
%		&= \mathbf{R} - \mathbf{LPL}^{\top}\big(\mathbf{LPL}^{\top} + \mathbf{R} \big)^{-1}\mathbf{R} \\
%		&= \mathbf{R} - \mathbf{LPL}^{\top}\Big[\mathbf{R}^{-1} - \big(\mathbf{LPL}^{\top} + \mathbf{R}  \big)^{-1} \mathbf{LPL}^{\top}\mathbf{R}^{-1}\Big]\mathbf{R}.
%	\end{align*}
%	The last expression may be expanded as
%	\begin{equation*}
%		\mathbf{R} - \mathbf{LPL}^{\top}\Big[\mathbf{I}_b - \big(\mathbf{LPL}^{\top} + \mathbf{R}\big)^{-1} \mathbf{LPL}^{\top}\Big] = \mathbf{R} - \mathbf{LPL}^{\top} - \mathbf{LPL}^{\top}\big(\mathbf{LPL}^{\top} + \mathbf{R}\big)^{-1}\mathbf{LPL}^{\top},
%	\end{equation*}
%	which can be factorized as
%	\begin{equation*}
%		\mathbf{R} - \mathbf{L}\Big[\mathbf{P} - \mathbf{PL}^{\top}\big(\mathbf{LPL}^{\top} + \mathbf{R}\big)^{-1}\mathbf{LP}\Big]\mathbf{L}^{\top}.
%	\end{equation*}
%	The bracketed term in the preceding expression can then be factorized as $\big(\mathbf{L}^{\top}\mathbf{R}^{-1}\mathbf{L} + \mathbf{P}^{-1}\big)^{-1}$ by the Woodbury matrix identity.
%	The stated result (\ref{SMEQ:ExpWood}) follows.
%\end{proof}

\begin{remark}
	Note that the lemma's also hold in the more general case where we replace $\mathbf{L}^{\top}$ by $\mathbf{Q} \in \mathbb{R}^{a \times b}$.
	Usage of $\mathbf{L}^{\top}$, however, makes it easier to see the connection to the factor model.
	In addition, in case of symmetry of $\mathbf{R}$ and $\mathbf{P}$ (as in the factor model where $\mathbf{R} = \PS$ and $\mathbf{P} = \PH$), the proof of Lemma \ref{SMLEM:HUAmat} is even more direct.
	In that case $(\mathbf{LPL}^{\top} + \mathbf{R})^{-1}$ is symmetric and the term $\mathbf{R}^{-1}\mathbf{LPL}^{\top}(\mathbf{LPL}^{\top} + \mathbf{R})^{-1}$ must also be symmetric.
	Taking its transpose then directly establishes the  first equality in (\ref{SMEQ:HUAmat}).
\end{remark}

%%%%%%%%%%%%%%%%%%%%%%%%%%%%
%%%--- Proofs --------------
%%%%%%%%%%%%%%%%%%%%%%%%%%%%
\section{Proofs}\label{APP:Proofs}
We will first establish the proof of Theorem \ref{PROP:GramLimit}.
Our general strategy for the remaining proofs will be to convert our statements of interest in a manner that pivotally hinges upon the result of Theorem \ref{PROP:GramLimit}.

\begin{proof}[Proof of Theorem \ref{PROP:GramLimit}]
Our approach will be to gradually simplify the problem until a manageable series remains.
We first observe that, as $\PH \succ 0$, 
\begin{equation*}
	\big(\Gram\big)_{kk} < \big(\PHi + \Gram\big)_{kk},\,\,\,\,\forall k.
\end{equation*}
Post- and pre-multiplication of the matrix-terms on both sides with $(\PHi + \Gram)^{-1}$ and $(\Gram)^{-1}$ respectively (or vice versa), immediately gives:
\begin{equation*}
	\Big[\big(\PHi + \Gram\big)^{-1}\Big]_{kk} < \Big[\big(\Gram\big)^{-1}\Big]_{kk},\,\,\,\,\forall k.
\end{equation*}
We then note that the matrix terms are residual correlation matrices (of our best linear and best linear unbiased proxies for $\xi$, also see Sections S.2.4.1 and S.4.1 of the SM).
As a consequence, $(\Gram)^{-1} \longrightarrow \boldsymbol{0}$ implies $(\PHi + \Gram)^{-1} \longrightarrow \boldsymbol{0}$.
Hence, our first simplification is to work with $(\Gram)^{-1}$.

Our second simplification stems from the Gramian status of our matrices of interest. 
From our assumptions it is immediate that $\PHi + \Gram \succ 0$.
Hence, $\PHi + \Gram$ has a Gram decomposition.
Focusing on (as justified by our first simplification) the dominant factor $\Gram$ we note, as $\PS \succ 0$, that there exists a unique inverse square root $\PS^{-1/2}$ that is p.d.\ and symmetric as a consequence of \citep[Theorem 1, p.\ 115 in][]{ref_Serre}.
This implies that the Gram decomposition of $\Gram$ can be written as
\begin{equation*}
	\LAt\PS^{-1/2}\PS^{-1/2}\LA = \big(\PS^{-1/2}\LA\big)^{\top}\big(\PS^{-1/2}\LA\big) \equiv \LA^{*\top}\LA^{*}.
\end{equation*}

Our third simplification stems from the behavior of the elements of $(\LA^{*\top}\LA^{*})^{-1}$ when $p$ grows large.
Let $\boldsymbol{\lambda}_k^{*}$ represent the $k$th column of $\LA^{*}$.
Then $\LA^{*\top}\LA^{*}$ is a Gram matrix with basis $\{\boldsymbol{\lambda}_1^{*}, \ldots, \boldsymbol{\lambda}_m^{*}\}$.
Let the orthogonal projection of $\boldsymbol{\lambda}_k^{*}$ onto the subspace spanned by $\{\boldsymbol{\lambda}_{k'}^{*}\}_{k'\neq k}^{m}$ be denoted by $\hat{\boldsymbol{\lambda}}_k^{*}$, and define $\tilde{\boldsymbol{\lambda}}_k^{*} \equiv \boldsymbol{\lambda}_k^{*} - \hat{\boldsymbol{\lambda}}_k^{*}$.
As $\LA^{*}$ is of full rank, the coefficients $b_{kk'}$ in $\hat{\boldsymbol{\lambda}}_k^{*} = \sum_{k'\neq k}b_{kk'}\boldsymbol{\lambda}_{k'}^{*}$ are uniquely determined.
By Theorem 2.1 of \citep{ref_GramPino}, we can then express the elements of $(\LA^{*\top}\LA^{*})^{-1}$  as:
\begin{align*}
	\Big[\big(\LA^{*\top}\LA^{*}\big)^{-1}\Big]_{kk} &= \big\|\tilde{\boldsymbol{\lambda}}_k^{*}\big\|_2^{-2}, \,\,\,k = 1,\dots,m, ~\mbox{and} \\
	\Big[\big(\LA^{*\top}\LA^{*}\big)^{-1}\Big]_{kk'} &= -b_{kk'}\big\|\tilde{\boldsymbol{\lambda}}_k^{*}\big\|_2^{-2}, \,\,\,k\neq k', ~~k,k' = 1,\dots,m.
\end{align*} 
Observing that $\|\hat{\boldsymbol{\lambda}}_k^{*}\|_2 \leq \|\boldsymbol{\lambda}_{k'}^{*}\|_2$ and using the reverse triangle inequality, we have
\begin{equation}\label{EQ:ReverseT}
	\Big(\big\|\boldsymbol{\lambda}_k^{*}\big\|_2 - \big\|\hat{\boldsymbol{\lambda}}_k^{*}\big\|_2\Big)^2
	\leq \big\|\tilde{\boldsymbol{\lambda}}_k^{*}\big\|_2^{2} \leq \big\|\boldsymbol{\lambda}_k^{*}\big\|_2^{2}.
\end{equation}
By the results in \citep{ref_KainenKurk}, all vectors $\boldsymbol{\lambda}_k^{*}, \boldsymbol{\lambda}_{k'}^{*}$ are pairwise $\eta$-orthogonal for $k\neq k'$ with probability tending to $1$ as $p\uparrow\infty$.
That is, for every $\eta > 0$, and for all $k\neq k'$,
\begin{equation*}
	\Pr \Bigg(\bigg|\frac{\langle\boldsymbol{\lambda}^*_k,\boldsymbol{\lambda}^*_{k'}\rangle}{\big\|\boldsymbol{\lambda}^*_k\big\|_2\big\|\boldsymbol{\lambda}^*_{k'}\big\|_2}\bigg| \geq \eta \Bigg) \longrightarrow 0
	\,\,\,\,\,\mbox{as} \,\, p\uparrow\infty.
\end{equation*}
As a direct consequence $b_{kk'} \longrightarrow 0$ for all $k\neq k'$ as $p\uparrow\infty$.
Another direct consequence, by (\ref{EQ:ReverseT}), is that $\|\tilde{\boldsymbol{\lambda}}_k^{*}\|_2^{2}$ will tend to get upper- and lower-bounded by $\|\boldsymbol{\lambda}_k^{*}\|_2^{2}$ as $p\uparrow\infty$.
These consequences imply
\begin{equation*}
	e_{k}\Big[\big(\LA^{*\top}\LA^{*}\big)^{-1}\Big] \longrightarrow \|\boldsymbol{\lambda}_k^{*}\|_2^{-2}
	\,\,\,\,\,\mbox{as} \,\, p\uparrow\infty,
\end{equation*}
for all $k$.
%This means that we only have to assess the limiting behavior of  $\|\boldsymbol{\lambda}_k^{*}\|_2^{-2}$ for all $k$ as $p\uparrow\infty$.
With this third simplification it is now left to show that $\lim_{p\,\uparrow\,z} \big\|\boldsymbol{\lambda}^*_k\big\|^2_2 = \infty$ if and only if $z = \infty$.

Assume $z = \infty$. 
Note that if $c_k$ is the number of indicators for factor $k$, by our assumption, the number of nonzero elements in $\boldsymbol{\lambda}_k^{*}$ is $c'_k \geq c_k$.
Also, by assumption there is a weakest precision-weighted indicator such that $\inf\big\{|(\PS^{-1/2}\LA)_{jk}| ~|~ |(\PS^{-1/2}\LA)_{jk}| \neq 0\big\} = \varphi > 0$.
Let $h_k$ be an index summing over all non-null elements in $\boldsymbol{\lambda}_k^{*}$, $h_k = 1, \ldots, c'_k$.
Now, if $p\uparrow\infty$, $c'_k\uparrow\infty$.
Then $\|\boldsymbol{\lambda}^*_k\|^2_2$ is lower-bounded in the limit by
\begin{equation*}
	\lim\limits_{p\,\uparrow\,\infty} \sum_{h_k = 1}^{c'_k} \varphi^2 = \lim\limits_{p\,\uparrow\,\infty} \frac{p\varphi^2}{d_k'} = \infty, \,\,\,\,\forall k.
\end{equation*}
Now assume that the series diverges but that $z$ is finite, hence, that $p$ grows to a finite positive integer.
Then we would have, for any $k$, a finite sum of finite elements, resulting in a finite evaluation of the series.
This provides a contradiction by which we conclude necessity.
The logical implication is that $e_{\mathrm{max}}\big[(\LA^{*\top}\LA^{*})^{-1}\big] \longrightarrow 0$ if and only if $p\uparrow\infty$.
The theorem follows.
\end{proof}

\begin{remark}
	Note that, in applications, often (orthogonal or oblique) rotations of $\LA$ are employed that are (approximately) sparse. 
	Such rotations emulate factorially pure structures (in which a feature loads on one and only one factor), thus enhancing quasi-orthogonality.
	In a perfectly factorially pure structure the columns of $\LA$ are perfectly pairwise orthogonal at any $p$, but still requires $p\uparrow\infty$ for $(\mathbf{\Phi}^{-1} + \mathbf{\Lambda}^{\top}\mathbf{\Psi}^{-1}\mathbf{\Lambda})^{-1}$ to vanish.
	From a practical standpoint this vanishing is enhanced by high absolute loadings.
\end{remark}

\begin{proof}[Proof of Corollary \ref{CORROL:FacLimits}]
We will visit each of the statements in turn.
For Corollary \ref{CORROL:FacLimits}.i we have, by (\ref{EQ:SMCwood}) and Theorem \ref{PROP:GramLimit}, that
\begin{equation*}
	\lim\limits_{p\,\uparrow\,\infty} \Big\|\PH - \PH\LAt\big(\Com + \PS\big)^{-1}\LA\PH\Big\|_2 = \lim\limits_{p\,\uparrow\,\infty} \Big\|\big(\mathbf{\Phi}^{-1} + \mathbf{\Lambda}^{\top}\mathbf{\Psi}^{-1}\mathbf{\Lambda}\big)^{-1} \Big\|_2 = 0.
\end{equation*}
For Corollary \ref{CORROL:FacLimits}.ii we note that
\begin{equation*}
	\Big\|\PS - \PS\big(\Com + \PS\big)^{-1}\PS\Big\|_{\infty,\infty} =
	\Big\|\LA\big(\PHi + \Gram\big)^{-1}\LAt\Big\|_{\infty,\infty},
\end{equation*}
by Lemma \ref{SMLEM:ExpWood}.
Let $\boldsymbol{\lambda}_j = [\lambda_{j1}, \ldots, \lambda_{jm}] \in \mathbb{R}^{m}$ denote the $j$th row of $\LA$ and define $\Delta_{\PS} \equiv \LA\big(\PHi + \Gram\big)^{-1}\LAt$. We can then find the individual elements of $\Delta_{\PS}$ as
\begin{equation*}
	(\Delta_{\PS})_{jj'} = \boldsymbol{\lambda}_j\big(\PHi + \Gram\big)^{-1} \boldsymbol{\lambda}_{j'}^{\top},
\end{equation*}
in which we recognize a bilinear form. 
By the Cauchy-Schwarz inequality and the induced matrix norm it then follows that \cite{ref_Zhang_bilinear}:
\begin{equation*}
	\big|(\Delta_{\PS})_{jj'}\big| \leq \big\|\boldsymbol{\lambda}_j\big\|_2\, \Big\|\big(\PHi + \Gram\big)^{-1}\Big\|_2\, \big\|\boldsymbol{\lambda}_{j'}\big\|_2.
\end{equation*}
The maximum over $j$ and $j'$ follows naturally as:
\begin{equation*}
	\max_{j,j'}\big|(\Delta_{\PS})_{jj'}\big| \leq \Big(\max_{j} \big\|\boldsymbol{\lambda}_j\big\|_2\Big)^2 \,\Big\|\big(\PHi + \Gram\big)^{-1}\Big\|_2.
\end{equation*}
Clearly, $\|\boldsymbol{\lambda}_j\|_2$ is dependent on a bounded $m$-sum for all $j$.
Hence, 
\begin{align*}
	\lim\limits_{p\,\uparrow\,\infty} \max_{j,j'}\big|(\Delta_{\PS})_{jj'}\big| \leq \Big(\max_{j} \big\|\boldsymbol{\lambda}_j\big\|_2\Big)^2 \lim\limits_{p\,\uparrow\,\infty} \Big\|\big(\PHi + \Gram\big)^{-1}\Big\|_2 = \,0,
\end{align*}
by Theorem \ref{PROP:GramLimit}.
For Corollary \ref{CORROL:FacLimits}.iii we note that, by the Woodbury identity, we have
\begin{equation*}
	\big(\Com + \PS\big)^{-1} = \PSi - \PSi\LA\big(\PHi + \Gram\big)^{-1}\LAt\PSi,
\end{equation*}
implying that
\begin{equation*}
	\Big\|\PSi - \big(\Com + \PS\big)^{-1}\Big\|_{\infty,\infty} =
	\Big\|\PSi\LA\big(\PHi + \Gram\big)^{-1}\LAt\PSi\Big\|_{\infty,\infty}.
\end{equation*}
Define $\PSi\LA\big(\PHi + \Gram\big)^{-1}\LAt\PSi \equiv \Delta_{\PSi}$ and $\PSi\LA \equiv \LA_{\PSi}$ and let $(\LA_{\PSi})_{j,\bullet}$ denote the $j$th row of $\LA_{\PSi}$.
Then, extending the previous argument from \ref{CORROL:FacLimits}.ii,
\begin{equation*}
	\big|(\Delta_{\PSi})_{jj'}\big| \leq \big\|(\LA_{\PSi})_{j,\bullet}\big\|_2\, \Big\|\big(\PHi + \Gram\big)^{-1}\Big\|_2\, \big\|(\LA_{\PSi})_{j',\bullet}\big\|_2.
\end{equation*}
By the similarity invariance of analytic matrix functions the (block) zero structure of $\PS$ is preserved in $\PSi$.
Thus $(\LA_{\PSi})_{jk}$ is bounded for all $j,k$ and as a result $\|(\LA_{\PSi})_{j',\bullet}\|_2$ is bounded for all $j$.
It then follows that 
\begin{equation*}
	\lim\limits_{p\,\uparrow\,\infty} \max_{j,j'}\big|(\Delta_{\PSi})_{jj'}\big| \leq \Big(\max_{j} \big\|(\LA_{\PSi})_{j,\bullet}\big\|_2\Big)^2 \lim\limits_{p\,\uparrow\,\infty} \Big\|\big(\PHi + \Gram\big)^{-1}\Big\|_2 = \,0,
\end{equation*}
as a direct consequence of Theorem \ref{PROP:GramLimit}.
In i., ii., and iii., the necessity and sufficiency of $p$ growing to infinity stems directly from Theorem \ref{PROP:GramLimit}.
The result follows.
\end{proof}

In the remainder we will let expressions such as $\lim_{p\,\uparrow\,\infty} \mathbf{B}_p = \boldsymbol{0}$ denote general entrywise convergence to the null matrix. 
This will allow us to avoid some notational clutter.

\begin{proof}[Proof of Corollary \ref{CORROL:Equalities}]
We will be using direct proofs to show our limiting equalities.
A necessary condition for determinacy is concurrent spans for $    \mathbf{\Sigma}(\mathbf{\Theta})_{xx}$ and $\mathbf{\Sigma}_{xx}$.
For (\ref{EQ:CorMatSpan}) from \emph{Indeterminacy Statement} \ref{IDS:span} it is sufficient to show that the feature-limit of $m + p$ divided by the dominant factor is unity.
For finite $m$ it is then immediate that
\begin{equation*}
\lim\limits_{p\,\uparrow\,\infty} \bigg(\frac{p}{p} + \frac{m}{p}\bigg) = 1 + \lim\limits_{p\,\uparrow\,\infty} \frac{m}{p} = 1.
\end{equation*}
Limiting equalities for (\ref{EQ:PostViolate}) and (\ref{EQ:PostViolate2}) from \emph{Indeterminacy Statement} \ref{IDS:PostulateViolate} follow directly from Corollary \ref{CORROL:FacLimits}.i and \ref{CORROL:FacLimits}.ii, respectively.
For (\ref{EQ:PostViolate3}) we first note the equivalence $\mathbf{\Sigma}(\mathbf{\Theta})_{\xi|x}\LAt = \mathbf{I}_m\mathbf{\Sigma}(\mathbf{\Theta})_{\xi|x}\LAt$, allowing us to express the elements of $\mathbf{\Sigma}(\mathbf{\Theta})_{\xi|x}\LAt$ in bilinear form:
\begin{equation*}
	\Big[\mathbf{\Sigma}(\mathbf{\Theta})_{\xi|x}\LAt\Big]_{kj} = (\mathbf{I}_m)_{k,\bullet}\mathbf{\Sigma}(\mathbf{\Theta})_{\xi|x}\boldsymbol{\lambda}_{j}^{\top},
\end{equation*}
where $(\mathbf{I}_m)_{k,\bullet}$ denotes the $k$th row of $\mathbf{I}_m$.
Now, $\|(\mathbf{I}_m)_{k,\bullet}\|_2 = 1$ for all $k$ and $\|\boldsymbol{\lambda}_j\|_2$ implies a bounded $m$-sum for all $j$.
Hence,
\begin{align*}
	\lim\limits_{p\,\uparrow\,\infty} \Big\|\mathbf{\Sigma}(\mathbf{\Theta})_{\xi|x_p}\LAt\Big\|_{\infty,\infty}
	&= \,\lim\limits_{p\,\uparrow\,\infty} \max_{k,j}\bigg|\Big[\mathbf{\Sigma}(\mathbf{\Theta})_{\xi|x_p}\LAt\Big]_{kj}\bigg| \\
	&\leq \,\lim\limits_{p\,\uparrow\,\infty} \max_{k,j}
	\Big(\big\|(\mathbf{I}_m)_{k,\bullet}\big\|_2\,\big\|\mathbf{\Sigma}(\mathbf{\Theta})_{\xi|x_p}\big\|_2\,\big\|\boldsymbol{\lambda}_{j}\big\|_2\Big)\\
	&= \, \bigg(\lim\limits_{p\,\uparrow\,\infty} \Big\|\big(\PHi + \Gram\big)^{-1}\Big\|_2\bigg) \max_{j}\big\|\boldsymbol{\lambda}_{j}\big\|_2\\
	&= \,0,
\end{align*}
by the developments in the proof of Corollary \ref{CORROL:FacLimits}.ii and Theorem \ref{PROP:GramLimit}.
Regarding (\ref{EQ:SMCs}) from \emph{Indeterminacy Statement} \ref{IDS:SMC} we, by the result in Theorem \ref{PROP:GramLimit}, directly write
\begin{equation*}
	  \bigg[\mathbf{\Phi} - \lim\limits_{p\,\uparrow\,\infty} \big(\mathbf{\Phi}^{-1} + \mathbf{\Lambda}^{\top}\mathbf{\Psi}^{-1}\mathbf{\Lambda}\big)^{-1} \bigg]_{kk} =
	  (\PH)_{kk} = 1, \,\,\,\forall k.
\end{equation*} 
The limiting equality for (\ref{EQ:CVARconstruct}) from \emph{Indeterminacy Statement} \ref{IDS:Constructions} follows directly from Theorem \ref{PROP:GramLimit}.
Regarding (\ref{EQ:GuttGap}) from \emph{Indeterminacy Statement} \ref{IDS:CorGeom} we have, by Theorem \ref{PROP:GramLimit}, that
\begin{equation*}
	\lim\limits_{p\,\uparrow\,\infty} \gamma_k = \bigg[\mathbf{\Phi} - 2\lim\limits_{p\,\uparrow\,\infty} \big(\mathbf{\Phi}^{-1} + \mathbf{\Lambda}^{\top}\mathbf{\Psi}^{-1}\mathbf{\Lambda}\big)^{-1} \bigg]_{kk} =
	(\PH)_{kk} = 1, \,\,\,\forall k.
\end{equation*}
The limiting equality for (\ref{EQ:MSEGap}) from \emph{Indeterminacy Statement} \ref{IDS:MSE} follows directly from Theorem \ref{PROP:GramLimit}.
The limiting entrywise equality for (\ref{EQ:CIGap}) from \emph{Indeterminacy Statement} \ref{IDS:Markov} follows directly from Corollary \ref{CORROL:FacLimits}.iii.
Sufficiency and necessity of $p\uparrow\infty$ for these limiting equalities is directly inherited from Theorem \ref{PROP:GramLimit}.
\end{proof}

\begin{proof}[Proof of Corollary \ref{CORROL:Concentration}]
Our general strategy will be to utilize mean-square convergence.
For Corollary \ref{CORROL:Concentration}.i we have
\begin{align*}
	\lim\limits_{p\,\uparrow\,\infty} \mathbb{E} \Big[\big\|\xi|x_p - \xih_p\big\|_2^2\Big] &= \lim\limits_{p\,\uparrow\,\infty} \mathbb{E} \Big[\big\|\xi|x_p - \mathbb{E}(\xi|x_p)\big\|_2^2\Big]\\
	&= \lim\limits_{p\,\uparrow\,\infty} \mathbb{E} \Big\{\tr\mathbb{E}\Big[\big(\xi|x_p - \mathbb{E}(\xi|x_p)\big)\big(\xi|x_p - \mathbb{E}(\xi|x_p)\big)^{\top}\Big]\Big\}\\
	&= \lim\limits_{p\,\uparrow\,\infty} \mathbb{E} \Big[\tr\big(\PHi + \Gram\big)^{-1}\Big]\\
	&= \tr\bigg[\lim\limits_{p\,\uparrow\,\infty}\big(\PHi + \Gram\big)^{-1}\bigg] = 0,
\end{align*}
by Theorem \ref{PROP:GramLimit}.
Hence, $\xi|x_p$ converges in mean-square to the constant random variable $\xih_p$, i.e., $\lim_{p\,\uparrow\,\infty} \xi|x_p \overset{\ell_2}{\longrightarrow} \delta(\xi|x_p - \xih_p)$.
For Corollary \ref{CORROL:Concentration}.ii we have
\begin{align*}
	\lim\limits_{p\,\uparrow\,\infty} \mathbb{E} \Big[\big\|\xih_p - \xi\big\|_2^2\Big] &=
	\lim\limits_{p\,\uparrow\,\infty} \mathbb{E} \Big\{\tr\mathbb{E}\Big[\big(\xih_p - \xi\big)\big(\xih_p - \xi\big)^{\top}\Big]\Big\}\\
	&= \lim\limits_{p\,\uparrow\,\infty} \mathbb{E} \Big[\tr\big(\PHi + \Gram\big)^{-1}\Big]\\
	&= \tr\bigg[\lim\limits_{p\,\uparrow\,\infty}\big(\PHi + \Gram\big)^{-1}\bigg] = 0,
\end{align*}
by Theorem \ref{PROP:GramLimit}.
Hence, $\lim_{p\,\uparrow\,\infty} \xih_p \overset{\ell_2}{\longrightarrow} \xi$.
As convergence in mean-square implies convergence in probability \cite[Theorem 3.1.,][]{ref_Gut2005Main}, we may conclude $\lim_{p\,\uparrow\,\infty} \xi|x_p \overset{\Pr}{\longrightarrow} \delta(\xi|x_p - \xih_p)$ and $\lim_{p\,\uparrow\,\infty} \xih_p \overset{\Pr}{\longrightarrow} \xi$.
Sufficiency and necessity of $p\uparrow\infty$ is directly inherited from Theorem \ref{PROP:GramLimit}.
The result follows.
\end{proof}

\begin{proof}[Proof of Corollary \ref{CORROL:BLPAdist}]
We will again use mean-square convergence. 
First, let $\mathbf{M} \in \mathcal{M}$, with $\mathcal{M} \equiv 
\{\mathbf{M} \in \mathbb{R}^{m \times m} ~|~ \mathbf{M} \in \mathrm{O}(m) \vee \mathbf{M} \in \mathcal{K}\}$.
This set also includes the identity rotation.
We then establish that
\begin{align}\label{EQ:CCC}
	\big\|\mathbf{M}^{-1}\big(\xih^0 - \dot{\xi}\big)\big\|_2^2 
	%&= 
	%\tr\bigg\{\mathbf{M}^{-1}\mathbb{E}\Big[\big(\xih^0 - \dot{\xi}\big) \big(\xih^0 - \dot{\xi}\big)^{\top}\Big]\mathbf{M}^{-\top}\bigg\} \\\nonumber
	%&= 
	%\tr\Big[ \mathbf{M}^{\top}\big(\mathbf{I}_m + \lambdaDot^{\top}\PSi\lambdaDot\big)^{-1}\mathbf{M}\Big]^{-1} \\
	&= \tr\big(\mathbf{M}^{\top}\mathbf{M} + \mathbf{M}^{\top}\lambdaDot^{\top}\PSi\lambdaDot\mathbf{M}\big)^{-1},
\end{align}
by use of standard covariance algebra and the reverse order law.
When $\mathbf{M}$ is the identity rotation $\mathbf{I}_m$ we have that (\ref{EQ:CCC}) equals $\tr(\mathbf{I}_m + \lambdaDot^{\top}\PSi\lambdaDot)^{-1}$, the canonical MSE between $\xih^0$ and $\dot{\xi}$.
When $\mathbf{M} = \HDot \in \mathrm{O}(m)$ (\ref{EQ:CCC}) equals the orthogonally rotated canonical MSE $\tr(\HDot^{\top}\HDot + \HDot^{\top}\lambdaDot^{\top}\PSi\lambdaDot\HDot)^{-1} = 
\tr(\mathbf{I}_m + \lambdaDot'^{\top}\PSi\lambdaDot')^{-1}$.
And when $\mathbf{M} = \mathbf{K} \in \mathcal{K}$ (\ref{EQ:CCC}) equals  $\tr(\mathbf{K}^{\top}\mathbf{K} + \mathbf{K}^{\top}\lambdaDot^{\top}\PSi\lambdaDot\mathbf{K})^{-1} = 
\tr(\PHi + \Gram)^{-1}$, as $(\mathbf{K}^{-1}\mathbf{K}^{-\top})^{-1} = \PHi = \mathbf{K}^{\top}\mathbf{K}$ and $\lambdaDot\mathbf{K} = \LA\mathbf{K}^{-1}\mathbf{K} = \LA$.
Using developments in the proof of Corollary \ref{CORROL:Concentration} we then have
\begin{align*}
	\lim\limits_{p\,\uparrow\,\infty} \mathbb{E} \Big[\big\|\mathbf{M}^{-1}\big(\xih^0_p - \dot{\xi}\big)\big\|_2^2\Big] &= \tr\bigg[\lim\limits_{p\,\uparrow\,\infty}\big(\mathbf{M}^{\top}\mathbf{M} + \mathbf{M}^{\top}\lambdaDot^{\top}\PSi\lambdaDot\mathbf{M}\big)^{-1}\bigg] = 0,
	\end{align*}
by (special cases of) Theorem \ref{PROP:GramLimit}.
Thus, $\lim_{p\,\uparrow\,\infty} \mathbf{M}^{-1}\xih^0_p \overset{\ell_2}{\longrightarrow} \mathbf{M}^{-1}\dot{\xi}$.
By the affine property $\mathbf{M}^{-1}\dot{\xi} \sim\, _{m}(\boldsymbol{0},\mathbf{M}^{-1}\mathbf{M}^{-\top})$.
Also, we again note that convergence in mean-square implies convergence in probability.
Sufficiency and necessity of $p\uparrow\infty$ is directly inherited from Theorem \ref{PROP:GramLimit}.
Corollaries \ref{CORROL:BLPAdist}.i--\ref{CORROL:BLPAdist}.iii then readily follow.
\end{proof}

\begin{proof}[Proof of Proposition \ref{PROP:Invariant}]
We can directly establish that
\begin{align*}
	\xih_{\mathbf{C}} &= (\PHi + \LAt\mathbf{C}\mathbf{C}^{-1}\PSi\mathbf{C}^{-1}\mathbf{C}\LA)^{-1}
	\LAt\mathbf{C}\mathbf{C}^{-1}\PSi\mathbf{C}^{-1}\mathbf{C}
	\big(\mathbf{\Sigma}_{\dot{x}\dot{x}} \odot \mathbf{I}_p\big)^{-1/2}(\dot{x} - \mu)\\
	&= (\PHi + \Gram)^{-1}\LAt\PSi x \\
	&= \xih,
\end{align*}
from which the statement directly follows.
\end{proof}
\end{appendix}
%%--------------- Appendix -------------------------------------------
%%--------------------------------------------------------------------

%--------------- AFC ------------------------------------------------
%--------------------------------------------------------------------
%\medskip
\begin{acks}[Acknowledgments]
Dedicated to the memories of the two pops, Carel Frederik Wilhelm Peeters sr.\ (21/2/1933 -- 22/8/2019) and Martinus van Embden (14/10/1939 -- 19/2/2022).
De reiziger reist voort.
\end{acks}

\begin{sloppypar}
\begin{acks}[Software]
An implementation of Algorithm \ref{ALG:Estimation} (as well as supporting functions) can be found in the software package \textsf{HAMMER}, available from the Comprehensive \textsf{R} Archive Network (CRAN) at \href{https://CRAN.R-project.org/package=HAMMER}
{DOI:10.32614/CRAN.package.HAMMER}.
\end{acks}
\end{sloppypar}

\begin{supplement}
The supplementary file \textbf{Perspectives on Latent Factor Indeterminacy and its Implications for Data Representation: \textcolor{blue}{Supplementary Material}} contains additional mathematical details for Sections \ref{SEC:Perspectives}--\ref{SEC:Discussion} as well as exemplifications of Algorithm \ref{ALG:Estimation}.
\end{supplement}
%\bigskip
%--------------------------------------------------------------------
%--------------- AFC ------------------------------------------------

%%--------------------------------------------------------------------
%%--------------- References -----------------------------------------
%\bibliographystyle{imsart-number}
%\bibliography{FAIn_refs}
\putbib[FAIn]
\end{bibunit}
%%--------------- References -----------------------------------------
%%--------------------------------------------------------------------

\clearpage

%--------------------------------------------------------------------
%--------------- Supplementary Numbering & Countering ---------------
% Prefix S to figures, tables, etc
\renewcommand{\theequation}{S.\arabic{equation}}
\renewcommand{\thelemma}{S.\arabic{lemma}}
\renewcommand{\thecorollary}{S.\arabic{corollary}}
\renewcommand{\theremark}{S.\arabic{remark}}
\renewcommand{\thedefinition}{S.\arabic{definition}}
\renewcommand{\theexample}{S.\arabic{example}}
\renewcommand{\thefigure}{S.\arabic{figure}}
\renewcommand{\thetable}{S.\arabic{table}}
\renewcommand{\thealgorithm}{S.\arabic{algorithm}}
\renewcommand{\bibnumfmt}[1]{[S.#1]}
\renewcommand{\citenumfont}[1]{S.#1}
\renewcommand{\thesection}{S.\arabic{section}}
\renewcommand{\thepage}{S.\arabic{page}}

% reset counters
\setcounter{section}{0}
\setcounter{subsection}{0}
\setcounter{equation}{0}
\setcounter{remark}{0}
\setcounter{figure}{0}
\setcounter{table}{0}
\setcounter{algorithm}{0}
\setcounter{page}{1}
%--------------- Supplementary Numbering & Countering ---------------
%--------------------------------------------------------------------

%--------------- Front Matter ---------------------------------------
%--------------------------------------------------------------------
\begin{frontmatter}
	\title{Perspectives on Latent Factor Indeterminacy\\and its Implications for Data Representation\\ \booz}
	%\title{A sample article title with some additional note\thanksref{t1}}
	\runtitle{Latent Factor Indeterminacy and its Implications: \boo}
	%\thankstext{T1}{A sample additional note to the title.}
	
	\begin{aug}
		%%%%%%%%%%%%%%%%%%%%%%%%%%%%%%%%%%%%%%%%%%%%%%%
		%% Only one address is permitted per author. %%
		%% Only division, organization and e-mail is %%
		%% included in the address.                  %%
		%% Additional information can be included in %%
		%% the Acknowledgments section if necessary. %%
		%% ORCID can be inserted by command:         %%
		%% \orcid{0000-0000-0000-0000}               %%
		%%%%%%%%%%%%%%%%%%%%%%%%%%%%%%%%%%%%%%%%%%%%%%%
		\author[A]{\fnms{Carel F.W.}~\snm{Peeters} \ead[label=e1]{carel.peeters@wur.nl}}
		%%%%%%%%%%%%%%%%%%%%%%%%%%%%%%%%%%%%%%%%%%%%%%
		%% Addresses                                %%
		%%%%%%%%%%%%%%%%%%%%%%%%%%%%%%%%%%%%%%%%%%%%%%
		\address[A]{Mathematical \& Statistical Methods group (Biometris),\\
			Wageningen University \& Research \printead[presep={ ,\ }]{e1}}
	\end{aug}
	
	\begin{keyword}[class=MSC]
		\kwd[Primary ]{62H25}
		\kwd{62A01}
		\kwd[; secondary ]{62H12}
		\kwd{68T07}
	\end{keyword}
	
	\begin{keyword}
		\kwd{Empirical Bayes}
		\kwd{Factor analysis}
		\kwd{Generative representation learning}
		\kwd{High-dimensional data}
		\kwd{Inverse problem}
		\kwd{Latent representation}
		\kwd{Latent variable collapse}
		\kwd{Linear manifold}
	\end{keyword}
	
\end{frontmatter}
%%%%%%%%%%%%%%%%%%%%%%%%%%%%%%%%%%%%%%%%%%%%%%
%% Please use \tableofcontents for articles %%
%% with 50 pages and more                   %%
%%%%%%%%%%%%%%%%%%%%%%%%%%%%%%%%%%%%%%%%%%%%%%
%\tableofcontents
%--------------------------------------------------------------------
%--------------- Front Matter ---------------------------------------

\begin{bibunit}
%--------------------------------------------------------------------
%%%%%%%%%%%%%%%%%%%%%%%%%%%%%%%%%%%%%%%%%%%%%%%%%%%%%%%%%%%%%%%%%%%%%
%%%%%%%%%%%%%%%% MAIN TEXT %%%%%%%%%%%%%%%%%%%%%%%%%%%%%%%%%%%%%%%%%%
%%%%%%%%%%%%%%%%%%%%%%%%%%%%%%%%%%%%%%%%%%%%%%%%%%%%%%%%%%%%%%%%%%%%%

%------------------------------------------------------------------------
%--------------- INTRO --------------------------------------------------
\section{Introduction}
\subsection{Setup supplement}
This supplement contains details and derivations in support of the Main Text (Peeters, C.F.W., \emph{Perspectives on Latent Factor Indeterminacy and its Implications for Data Representation}).
Sectioning in this Supplementary Material (SM) is concurrent with the Main Text (MT).
Sections \ref{SMSEC:Supp2}--\ref{SMSEC:Supp4} respectively contain supporting materials for Sections 2--4 of the MT.
Subsections are then also concurrent.
For example, details for Section 2.1 of the MT can be found in Section \ref{SMSSEC:Supp2.1} of this SM. 
Table \ref{SMTABLE:ReadGuide2} contains a reading guide, indicating how the main messages (Section 1.6 of the MT) are related to the objectives (Section 1.4 of the MT) and sectioning in both the MT and this SM.

\begin{table}[h!]
	\caption{Reading guide}
	\label{SMTABLE:ReadGuide2}
	\resizebox{\textwidth}{!}{%
		\begin{tabular}{llcc}
			\toprule
			Main message        & Objective                      & Section MT & Section SM     \\ \midrule
			Latent projection adheres to concentration of measure & Objective (i)                  & 2.1 -- 3.2 & S.2.1 -- S.3.2 \\
			Latent projection is a form of implicit regularization & Objective (ii): computational  & 3.3        & S.3.3          \\
			Dimensionality is a blessing for the factor model & Objective (ii): methodological & 4 -- 5     & S.4            \\ \bottomrule
		\end{tabular}%
	}
\end{table}

\subsection{Notation}
Our notation (used in the MT and this SM) is fully explicated here.
We use the assigment operator $(:=)$ in (1) of the MT as the assumption of factor exogeneity is essentially causal (see \citep{SMref_Pearl2009} for a discussion on this issue).
The symbol $\ci$ is used to denote independence or orthogonality.
We use $\Pr(A)$ to denote the probability of event $A$.
We let $\measuredangle(\boldsymbol{a},\boldsymbol{b})$ represent the angle between vectors $\boldsymbol{a}$ and $\boldsymbol{b}$.
At times we write $\proj_{\boldsymbol{a}}\boldsymbol{b}$ for the projection of $\boldsymbol{b}$ onto $\boldsymbol{a}$.
We let $\mathbb{N}_{\infty}$ denote the set of extended natural numbers.
In general we denote sets with set-builder notation $\mathcal{A} \equiv \{a|P(a)\}$, that is, the set of all $a$ for which the logical formula $P(a)$ is true.
The context distinguishes the set from our preferred ordering of brackets $\{[()]\}$ as well as parameter collections denoted $\{\cdot\}$.
Specific notation on random variables, matrices, norms, graphical modeling, and convergence is found below.

\subsubsection{Random variables and distributions}
Random variables are denoted in lower-case italics.
The expected value of, say, $z$ is denoted $\mathbb{E}(z)$.
When $z$ is a centered random variable its covariance matrix can be represented as $\mathbb{E}(zz^{\top})$.
The population correlation matrix is denoted $\mathbf{\Sigma}$.
We let $\mathbf{\Sigma}(\mathbf{\Theta})$ denote a model-implied correlation matrix as a function of (any subcollection of) $\mathbf{\Theta}$.
Correlation matrices may, in the remainder, bear double subscripts referring to block-partitioning, e.g., $\mathbf{\Sigma}_{xx}$ denotes the correlation matrix between the random variables in collection $x$, while $\mathbf{\Sigma}(\mathbf{\Theta})_{x\xi}$ refers to the model-implied correlation matrix between the random variables in collections $x$ and $\xi$.
A conditional model-implied correlation matrix, e.g., the model-implied correlation matrix of $x$ given $\xi$, is denoted as $\mathbf{\Sigma}(\mathbf{\Theta})_{x|\xi}$.
In case $x$ is a Gaussian $p-$dimensional random variable with location $\boldsymbol{m}$ and scale matrix $\mathbf{S}$ we write $x \sim \mathcal{N}_p(\boldsymbol{m}, \mathbf{S})$.
In case $x$ is a $p-$dimensional random variable with \emph{any} distribution describable by location $\boldsymbol{m}$ and scale matrix $\mathbf{S}$ we write $x \sim\, _p(\boldsymbol{m}, \mathbf{S})$.
We will write $p(x|\xi)$ and $p(\xi|x)$ as shorthands for the conditional (posterior) distributions  $p(x|\xi, \mathbf{\Theta})$ and $p(\xi|x, \mathbf{\Theta})$.
These distributions would respectively be referred to as the \emph{decoder} and \emph{encoder} in the learning community.
We use $\delta(z - \iota)$ for the shifted Dirac delta function, in $m-$dimensional Euclidean space, concentrated on $\iota$.
Let $\vartheta$ and $\omega$ be elements of a Hilbert space of $m$-dimensional random vectors. 
Then the $\ell_2$-norm $\|\vartheta\|_2 = [\tr\mathbb{E}(\vartheta\vartheta^{\top})]^{1/2}$ and the inner product $\langle\vartheta,\omega\rangle = \tr\mathbb{E}(\vartheta\omega^{\top})$.

\subsubsection{Matrices, vectors, and their operations}
Matrices are thus denoted by boldface capitals.
Matrix $\mathbf{A}$ bears elements $a_{rc}$ unless clarity of notation commands the use of $(\mathbf{A})_{rc}$.
In terms of operations, $\odot$ will denote the Hadamard product, $\circ$ will denote composition, and $\otimes$ will denote the Kronecker product. 
For a partitioned matrix $\mathbf{A}$ we let $\mathbf{A}_{rc}$ denote the $rc$th block. 
For the $c$th column of the block-partitioned $\mathbf{A}$ we write $\mathbf{A}_{[c]}$.
We let $\mathbf{A} \boxdot \mathbf{B} = (\mathbf{A}_{rc}\mathbf{B}_{rc})_{rc}$ denote the block Hadamard product in case of two partitioned matrices $\mathbf{A}$ and $\mathbf{B}$ that are block commuting in the sense of every block in $\mathbf{A}$ being conformable with respect to matrix multiplication with each corresponding block in $\mathbf{B}$.
Transposition of $\mathbf{A}$ is denoted $\mathbf{A}^{\top}$.
For real matrix $\mathbf{A}$ we let $\mathbf{A}^{\odot2} = \big(a_{rc}^{2}\big)$, $\mathbf{A}^{\odot\frac12} = \big(a_{rc}^{1/2}\big)$, and $\mathbf{A}^{\odot(-1)} = \big(a_{rc}^{-1}\big)$ respectively denote the Hadamard square, Hadamard square root, and Hadamard inverse.
These are similarly defined for vectors.
We use $\tr(\mathbf{A})$ and $|\mathbf{A}|$ to respectively denote the trace and determinant of $\mathbf{A}$ when it is square.
If a square matrix $\mathbf{A}$ is positive definite we will write $\mathbf{A} \succ 0$ and $\mathbf{A}^{-1}$ will denote it's inverse.
We will use $\mathbf{A}^{-\top}$ as a shorthand for $(\mathbf{A}^{-1})^{\top}$.
For a square symmetric matrix $\mathbf{A} \succ 0$ we will let $\mathbf{A}^{1/2}$ denote it's (principal) matrix square root.
The generalized inverse of (a possibly singular or rectangular matrix) $\mathbf{B}$ is denoted $\mathbf{B}^{+}$.
The $(p \times p)$-dimensional identity matrix is denoted $\mathbf{I}_{p}$ and $\boldsymbol{0}$ may indicate either a vector or a matrix as should be clear from its context.
We use $\boldsymbol{1}_p$ to denote the $p$-dimensional vectors of ones.
For two $p$-dimensional vectors $\boldsymbol{\mathrm{a}}$ and $\boldsymbol{\mathrm{b}}$ we define the inner product space $\langle\boldsymbol{\mathrm{a}},\boldsymbol{\mathrm{b}}\rangle = \boldsymbol{\mathrm{a}}^{\top}\boldsymbol{\mathrm{b}} = \sum_{j = 1}^{p}a_jb_j$.
Then, $\mathrm{diag}(\boldsymbol{\mathrm{a}})$ represents $(\boldsymbol{\mathrm{a}}\boldsymbol{1}_p^{\top})\odot\mathbf{I}_{p}$.
Hence, this operator returns a square diagonal matrix with the elements of $\boldsymbol{\mathrm{a}}$ on the main diagonal.
For a square $\mathbf{A}$ we let $\mathrm{diag}(\mathbf{A})$ represent the operator that returns a column vector of the main diagonal elements of $\mathbf{A}$.

\subsubsection{Matrix norms}
We will treat matrix convergence from the perspective of various norms. 
For real $\mathbf{A}$ we will consider the elementwise $\ell_{\infty,\infty}$ norm 
$\|\mathbf{A}\|_{\infty,\infty} = \max_{r,c}|a_{rc}|$ and the Frobenius norm $\|\mathbf{A}\|_F = \sqrt{\sum_r\sum_c |a_{rc}|}$.
In case $\mathbf{A}$ is symmetric and positive definite, we will also consider the spectral norm $\|\mathbf{A}\|_{2} = e_{\mathrm{max}}(\mathbf{A})$, where $e_{\mathrm{max}}(\mathbf{A})$ is the maximum eigenvalue of $\mathbf{A}$. 

\subsubsection{Graphical modeling and algorithms}
We will also be using the language of (mixed) graphical modeling.
We let unidirectional edges $\rightarrow$ on ordered pairs (of random variables) represent functional (or causal) influence, while we let bidirectional edges $\leftrightarrow$ between unordered pairs represent marginal dependence.
In case of usage in algorithmic representation, $\leftarrow$ will denote computational assignment.
We use $\mathcal{O}()$ to express (worst-case) asymptotic time complexity of an operation or algorithm.
In addition, we use long unidirectional edges $\longrightarrow$ to express computational structure or convergence, as should be clear from the context.

\subsubsection{Asymptotics}
We will consider asymptotics in the high-dimensional, fixed sample size regime.
We will use  $\lim_{p\,\uparrow\,\infty}$ to denote the left-hand dimensional limit (with puncture) at infinity.
In addition, we will use $\rightsquigarrow$ to imply convergence in distribution and $\overset{\Pr}{\longrightarrow}$ to imply convergence in probability.
Moreover, we will use $\overset{\ell_2}{\longrightarrow}$ to imply convergence in mean-square.
We permit ourselves the convenient abuse of notation $\lim_{p\,\uparrow\,\infty}\xih_p \overset{\Pr}{\longrightarrow} \mathcal{N}_m(\boldsymbol{0}, \PH)$ to represent $\xih_p \overset{\Pr}{\longrightarrow} \xi$ as $p\uparrow\infty$, where $\xi\in\mathcal{N}_m(\boldsymbol{0}, \PH)$.

%\subsubsection{Proxies}
%--------------- INTRO --------------------------------------------------
%------------------------------------------------------------------------

%------------------------------------------------------------------------
%--------------- Details 2 ----------------------------------------------
\section{Details for Section 2 of the Main Text}\label{SMSEC:Supp2}
%%%%%%%%%%%%%%%%%%%%%%%%%%%%
%%%-Details Span -----------
%%%%%%%%%%%%%%%%%%%%%%%%%%%%
\subsection{Details for Section 2.1 of the Main Text}\label{SMSSEC:Supp2.1}
Note that we explicitly, as opposed to the classic psychometric literature, conceive of both $\xi$ and $\epsilon$ as random variables instead of unknown parameters.
Hence, indeterminacy should not be understood in terms of counting equations and unknowns, but rather as the degree to which the information in $x$ can be used to determine locations for $\xi$ and $\epsilon$ uniquely.
As the joint dimension of $\xi$ and $\epsilon$ exceeds the dimension of $x$ it is clear that this degree is not perfect.
We may understand this by first considering a simpler model\footnote{
	There is some controversy in the literature on using assignment systems as equation systems.
	However, note that our assumption of exogeneity is explicitly causal in the sense that the exogenous latent factors are considered temporally prior in relation to the observable features.
	This causal connotation is then the reason why the assignment symbol ($:=$) is retained in our representation of the factor model: it is a descriptor of responses in $x$ as a result of changes in $\xi$.
}:
\begin{equation*}
	x := \mathbf{\Lambda}\xi,
\end{equation*}
adhering to all postulates and stipulations from Section 1.1 of the MT, including $\xi \sim \mathcal{N}_m(\boldsymbol{0}, \mathbf{\Phi})$, with the exception that we now have no errors of measurement.
The joint covariance matrix on $x$ and $\xi$ can then be found as:
\begin{equation*}
	\begin{bmatrix}
		\mathbf{\Lambda}\mathbf{\Phi}\mathbf{\Lambda}^{\top}   & \mathbf{\Lambda}\mathbf{\Phi} \\
		\mathbf{\Phi}\mathbf{\Lambda}^{\top}  & \PH
	\end{bmatrix}.
\end{equation*}
As the joint distribution is Gaussian the conditional distribution of $\xi$ given $x$ is also Gaussian (see Section 2.2 of the MT):
\begin{equation*}
	\xi|x \sim \mathcal{N}_{m}\Big[\PH\LAt\big(\Com\big)^{+}x, \PH - \PH\LAt\big(\Com\big)^{+}\LA\PH\Big],
\end{equation*}
where $(\cdot)^{+}$ denotes the Moore-Penrose inverse \cite{Moore,Penrose}.
Let
\begin{equation}\label{SMEQ:ConDist}
	\mathbf{\Sigma}(\mathbf{\Theta})_{\xi|x} \equiv \PH - \PH\LAt\big(\Com\big)^{+}\LA\PH.
\end{equation}
Note that, by the results in Hartwig \citep{Hartwig_SM},
\begin{equation*}
	(\Com\big)^{+} = \LA\big(\LAt\LA\big)^{-1}\PHi\big(\LAt\LA\big)^{-1}\LAt,
\end{equation*}
such that
\begin{equation}\label{SMEQ:SmallModVar}
	\PH\LAt\big(\Com\big)^{+}\LA\PH = \PH\LAt\LA\big(\LAt\LA\big)^{-1}\PHi(\LAt\LA\big)^{-1}\LAt\LA\PH = \PH.
\end{equation}
It is then clear that (\ref{SMEQ:ConDist}) evaluates to:
\begin{equation*}
	\mathbf{\Sigma}(\mathbf{\Theta})_{\xi|x} = \PH - \PH\LAt\LA\big(\LAt\LA\big)^{-1}\PHi(\LAt\LA\big)^{-1}\LAt\LA\PH = \PH - \PH = \boldsymbol{0}.
\end{equation*}
Hence, the conditional distribution of $\xi$ given $x$ is a point distribution (or, rather, a Dirac delta function) located at $\PH\LAt\big(\Com\big)^{+}x$.
This implies that all that we can know about the latent $\xi$ from the observable $x$ follows from this location.
We may then examine the distributional behavior of this location.
Let $\xih \equiv \PH\LAt\big(\Com\big)^{+}x$.
Then $\mathbb{E}\big(\xih\big) = \boldsymbol{0}$ and, using (\ref{SMEQ:SmallModVar}):
\begin{align*}
	\mathbb{E}\big(\xih\xih^{\top}\big) &= \PH\LAt\big(\Com\big)^{+}\mathbb{E}\big(xx^{\top}\big)\big(\Com\big)^{+}\LA\PH \\
	&= \PH\LAt\big(\Com\big)^{+}\big(\Com\big)\big(\Com\big)^{+}\LA\PH\\
	&= \PH\LAt\big(\Com\big)^{+}\LA\PH \\
	&= \PH.
\end{align*}
We can conclude that $\xih \sim \mathcal{N}_m(\boldsymbol{0}, \mathbf{\Phi})$ and thus behaves as the postulated $\xi$.
The simpler model above is determined in the sense that we can retrieve (the behavior of) $\xi$ from $x$.
(Similar reasoning can be employed for the retrieval of $\epsilon$ from $x$ in the simplified model in which we do have $\epsilon$ but where $\xi$ is considered non-random).

When the model is expanded with the unique factors in $\epsilon$ we have that its variance matrix $\PS$ enters the inverse such that $\big(\Com\big)^{+}$ is replaced by $\big(\Com + \PS\big)^{-1}$.
Clearly, $\mathbf{\Sigma}(\mathbf{\Theta})_{\xi|x}$ will then no longer evaluate to $\boldsymbol{0}$ in general.
This holds even when $\PS$ is parametrized in terms of low numbers of parameters such as in probabilistic PCA, which postulates $\PS = \sigma\mathbf{I}_p$ \citep{TBprobPCA}.
The conditional distribution of $\xi$ given $x$ is then no longer a point distribution (shifted Dirac delta function) and the distributional behavior of the location does not concur with the initial postulate.
These and other implications are further elaborated in Section 2.2 ff.\ of the MT.

%%%%%%%%%%%%%%%%%%%%%%%%%%%%
%%%-Details Post. violation-
%%%%%%%%%%%%%%%%%%%%%%%%%%%%
\subsection{Details for Section 2.2 of the Main Text}\label{SMSSEC:Supp2.2}
Section \ref{SMSSSEC:Supp2.2.1} explains a useful alternative expression for  $\hat{\xi}$.
Section \ref{SMSSEC:Regresse} contains an insightful identity for $\hat{\epsilon}$.
Lastly, Section \ref{SMSSEC:HatAdhere} details the adherence of $\hat{\xi}$ and $\hat{\epsilon}$ to model postulates and implications.

\subsubsection{Expressing \texorpdfstring{$\hat{\xi}$}{}}\label{SMSSSEC:Supp2.2.1}
Note that, in practice, it is useful to express the matrix of regression coefficients $\mathbf{\Sigma}(\mathbf{\Theta})_{\xi x}\mathbf{\Sigma}(\mathbf{\Theta})_{xx}^{-1}$ in an alternative way:
\begin{align*}
	\mathbf{\Sigma}(\mathbf{\Theta})_{\xi x}\mathbf{\Sigma}(\mathbf{\Theta})_{xx}^{-1}
	&= \mathbf{\Phi}\mathbf{\Lambda}^{\top}\big(\mathbf{\Lambda}\mathbf{\Phi}\mathbf{\Lambda}^{\top} + \mathbf{\Psi}\big)^{-1}\\
	&= \Big[\mathbf{\Phi} - \mathbf{\Phi}\mathbf{\Lambda}^{\top}\mathbf{\Psi}^{-1}\mathbf{\Lambda}\big(\mathbf{\Phi}^{-1} + \mathbf{\Lambda}^{\top}\mathbf{\Psi}^{-1}\mathbf{\Lambda}\big)^{-1}\Big]\mathbf{\Lambda}^{\top}\mathbf{\Psi}^{-1}\\
	&= \big(\mathbf{\Phi}^{-1} + \mathbf{\Lambda}^{\top}\mathbf{\Psi}^{-1}\mathbf{\Lambda}\big)^{-1}\mathbf{\Lambda}^{\top}\mathbf{\Psi}^{-1},
\end{align*}
by consecutive usage of the Woodbury matrix identity \citep{ref_Duncan1944_SM,ref_Woodbury_SM} and Lemma 1 of the MT.
Note that $\mathbf{\Psi}$ in general will be sparse if not diagonal.
Its inverse is thus dependent upon inversion of blocks whose dimensions are much lower than $(p \times p)$ and in general lower than $(m \times m)$.
The last expression is then convenient in practice as its worst-case asymptotic time complexity is dependent upon inversion of an $(m \times m)$-dimensional matrix rather than a $(p \times p)$-dimensional matrix.

\subsubsection{Expressing \texorpdfstring{$\hat{\epsilon}$}{}}\label{SMSSEC:Regresse}
The regression function of $\epsilon$ on $x$ is
\begin{equation*}
	\mathbf{\Sigma}(\mathbf{\Theta})_{\epsilon x}\mathbf{\Sigma}(\mathbf{\Theta})_{xx}^{-1}x = \mathbf{\Psi}\big(\mathbf{\Lambda}\mathbf{\Phi}\mathbf{\Lambda}^{\top} + \mathbf{\Psi}\big)^{-1}x = \mathbb{E}(\epsilon|x).
\end{equation*}
It is insightful to explain the following identity: $\mathbb{E}(\epsilon|x) = x - \mathbf{\Lambda}\hat{\xi}$.
By the Woodbury matrix identity we have that:
\begin{align}\label{SMEQ:We}\nonumber
	\mathbf{\Psi}\big(\mathbf{\Lambda}\mathbf{\Phi}\mathbf{\Lambda}^{\top} + \mathbf{\Psi}\big)^{-1}x &=
	\mathbf{\Psi}\Big[\mathbf{\Psi}^{-1} - \mathbf{\Psi}^{-1}\mathbf{\Lambda}\big(\mathbf{\Phi}^{-1} + \mathbf{\Lambda}^{\top}\mathbf{\Psi}^{-1}\mathbf{\Lambda}\big)^{-1}\mathbf{\Lambda}^{\top}\mathbf{\Psi}^{-1}\Big]x \\\nonumber
	&= \Big[\mathbf{I}_p - \LA\big(\PHi + \Gram\big)^{-1}\LAt\PSi\Big]x \\
	&= x - \LA\big(\PHi + \Gram\big)^{-1}\LAt\PSi x.
\end{align}
The term $\big(\PHi + \Gram\big)^{-1}$ may be expanded by (again) the Woodbury matrix identity so that (\ref{SMEQ:We}) equals:
\begin{equation}\label{SMEQ:We2}
	x - \LA\Big[\PH - \PH\LAt\big(\Com + \PS\big)^{-1}\LA\PH\Big]\LAt\PSi x.
\end{equation}
Note that (\ref{SMEQ:We2}), by bringing the matrix terms inside the brackets and subsequently factorizing out $\Com$, can be written as:
\begin{equation}\label{SMEQ:We3}
	x - \Com\Big[\PSi - \big(\Com + \PS\big)^{-1}\Com\PSi\Big]x.
\end{equation}
By Lemma 1 of the MT, we have that the square-bracketed term in (\ref{SMEQ:We3}) abides the following equality:
\begin{equation*}
	\PSi - \big(\Com + \PS\big)^{-1}\Com\PSi = \big(\Com + \PS\big)^{-1},
\end{equation*}
implying that (\ref{SMEQ:We3}) may be written as
\begin{equation*}
	x - \Com\big(\Com + \PS\big)^{-1}x,
\end{equation*}
showing that $\mathbb{E}(\epsilon|x) = x - \mathbf{\Lambda}\hat{\xi}$.

\subsubsection{Adherence of \texorpdfstring{$\hat{\xi}$}{} and \texorpdfstring{$\hat{\epsilon}$}{} to model postulates and implications}\label{SMSSEC:HatAdhere}
As $\mathbb{E}(x) = 0$, we have that the expectations of $\xih$, $\eph$, and $\LA\xih + \eph$ are also zero.
By normality assumption on $x$ we have that the joint distribution of $x$, $\xih$, and $\eph$ is also Gaussian (although, as we shall later see, Gaussianity is by no means a necessity).
Moreover, 
\begin{equation*}
	\LA\xih + \eph = \big(\Com + \PS\big)\big(\Com + \PS\big)^{-1}x
\end{equation*}
and thus retrieves $x$.
Hence, to evaluate if $(\xih,\eph)$ is part of a factor solution for $x$ we need to assess if the joint correlation matrix on $x = \LA\xih + \eph$, $\xih$, and $\eph$ concurs with (4) of the MT.

\begin{sloppypar}
	Let us first find $\mathbb{E}(\xih\xih^{\top})$, $\mathbb{E}(\eph\eph^{\top})$, and $\mathbb{E}(\xih\eph^{\top})$ using standard covariance algebra.
	For $\mathbb{E}(\xih\xih^{\top})$ we find:
	\begin{align*}
		\mathbb{E}\big(\xih\xih^{\top}\big) &= \mathbb{E}\bigg\{\PH\LAt\big(\Com + \PS\big)^{-1}x\Big[\PH\LAt\big(\Com + \PS\big)^{-1}x\Big]^{\top}\bigg\} \\
		&= \PH\LAt\big(\Com + \PS\big)^{-1}\mathbb{E}\big(xx^{\top}\big)\big(\Com + \PS\big)^{-1}\LA\PH \\
		&= \PH\LAt\big(\Com + \PS\big)^{-1}\LA\PH.
	\end{align*}
	For $\mathbb{E}(\eph\eph^{\top})$ we find:
	\begin{align*}
		\mathbb{E}\big(\eph\eph^{\top}\big) &= \mathbb{E}\bigg\{\PS\big(\Com + \PS\big)^{-1}x\Big[\PS\big(\Com + \PS\big)^{-1}x\Big]^{\top}\bigg\} \\
		&= \PS\big(\Com + \PS\big)^{-1}\mathbb{E}\big(xx^{\top}\big)\big(\Com + \PS\big)^{-1}\PS \\
		&= \PS\big(\Com + \PS\big)^{-1}\PS.
	\end{align*}
	For finding $\mathbb{E}(\xih\eph^{\top})$ it is convenient to use $\eph = x - \Com\big(\Com + \PS\big)^{-1}x$.
	Hence,
	\begin{align*}
		\mathbb{E}\big(\xih\eph^{\top}\big) &= \mathbb{E}\bigg\{\PH\LAt\big(\Com + \PS\big)^{-1}x\Big[x^{\top} - x^{\top}\big(\Com + \PS\big)^{-1}\Com\Big]\bigg\} \\
		&= \PH\LAt\big(\Com + \PS\big)^{-1}\mathbb{E}\big(xx^{\top}\big) \\
		& \qquad - \PH\LAt\big(\Com + \PS\big)^{-1}\mathbb{E}\big(xx^{\top}\big)\big(\Com + \PS\big)^{-1}\Com \\
		&= \PH\LAt - \PH\LAt\big(\Com + \PS\big)^{-1}\Com \\
		&= \Big[\PH - \PH\LAt\big(\Com + \PS\big)^{-1}\LA\PH \Big]\LAt = \mathbf{\Sigma}(\mathbf{\Theta})_{\xi|x}\LAt.
	\end{align*}
	Clearly, these correlation matrices violate model postulates (i)--(iii) from Section 1.1 of the MT.
	With these correlations at hand, we can evaluate the correlation matrix of $x = \LA\xih + \eph$:
	\begin{align*}
		\mathbb{E}\Big[\big(\mathbf{\Lambda}\hat{\xi} + \hat{\epsilon}\big)&\big(\mathbf{\Lambda}\hat{\xi} + \hat{\epsilon}\big)^{\top}\Big] \\
		&=
		\mathbf{\Lambda}\mathbb{E}\big(\hat{\xi}\hat{\xi}^{\top}\big)\mathbf{\Lambda}^{\top} + \mathbf{\Lambda}\mathbb{E}\big(\hat{\xi}\hat{\epsilon}^{\top}\big) + \mathbb{E}\big(\hat{\epsilon}\hat{\xi}^{\top}\big)\mathbf{\Lambda}^{\top} + \mathbb{E}\big(\hat{\epsilon}\hat{\epsilon}^{\top}\big) \\
		&= -\Com\big(\Com + \PS\big)^{-1}\Com + \PS\big(\Com + \PS\big)^{-1}\PS + 2\Com \\
		&= \big(-\mathbf{\Lambda\Phi\Lambda}^{\top} + \mathbf{\Psi}\big)\big(\mathbf{\Lambda\Phi\Lambda}^{\top} + \mathbf{\Psi}\big)^{-1}\big(\mathbf{\Lambda\Phi\Lambda}^{\top} + \mathbf{\Psi}\big) + 2\mathbf{\Lambda\Phi\Lambda}^{\top} \\
		&= -\mathbf{\Lambda\Phi\Lambda}^{\top} + \mathbf{\Psi} + 2\mathbf{\Lambda\Phi\Lambda}^{\top} \\
		&= \mathbf{\Lambda\Phi\Lambda}^{\top} + \mathbf{\Psi},
	\end{align*}
	indicating that $x = \LA\xih + \eph$ abides the fundamental equation of factor analysis.
	Moreover,
	\begin{align*}
		\mathbb{E}\big(x,\xih^{\top}\big) &= \mathbb{E}\Big[xx^{\top}\big(\Com + \PS\big)^{-1}\LA\PH\Big]
		= \mathbb{E}\big(xx^{\top}\big)\big(\Com + \PS\big)^{-1}\LA\PH
		= \LA\PH,
	\end{align*}
	and
	\begin{align*}
		\mathbb{E}\big(x,\eph^{\top}\big) &= \mathbb{E}\Big[xx^{\top}\big(\Com + \PS\big)^{-1}\PS\Big]
		= \mathbb{E}\big(xx^{\top}\big)\big(\Com + \PS\big)^{-1}\PS
		= \PS.
	\end{align*}
	Thus,
	\begin{equation*}
		\begin{bmatrix}
			x\\
			\xih\\
			\eph
		\end{bmatrix}
		\sim \mathcal{N}_{2p + m}\big[\boldsymbol{0}, \dot{\mathbf{\Sigma}}(\mathbf{\Theta})\big],
	\end{equation*}
	with
	\begin{align*}
		%\begin{small}
		\dot{\mathbf{\Sigma}}(\mathbf{\Theta}) &=
		\begin{bmatrix}
			\mathbf{\Lambda\Phi\Lambda}^{\top} + \mathbf{\Psi}  & \LA\PH   & \PS \\
			\PH\LAt                                             & \PH\LAt\big(\Com + \PS\big)^{-1}\LA\PH       & \mathbf{\Sigma}(\mathbf{\Theta})_{\xi|x}\LAt \\
			\PS                                                 & \LA\mathbf{\Sigma}(\mathbf{\Theta})_{\xi|x} & \PS\big(\Com + \PS\big)^{-1}\PS
		\end{bmatrix} \\
		&=
		\begin{bmatrix}
			\mathbf{\Lambda\Phi\Lambda}^{\top} + \mathbf{\Psi}  & \LA\PH   & \PS \\
			\PH\LAt                                             & \PH - \mathbf{\Sigma}(\mathbf{\Theta})_{\xi|x}        & \mathbf{\Sigma}(\mathbf{\Theta})_{\xi|x}\LAt \\
			\PS                                                 & \LA\mathbf{\Sigma}(\mathbf{\Theta})_{\xi|x} & \PS - \mathbf{\Sigma}(\mathbf{\Theta})_{\epsilon|x}
		\end{bmatrix},
		%\end{small}
	\end{align*}
	%\begin{equation}\nonumber
	%\begin{scriptsize}
	%    \begin{bmatrix}
		%    x\\
		%    \xih\\
		%    \eph
		%    \end{bmatrix}
	%    \sim \mathcal{N}_{2p + m}\left\{
	%    \begin{bmatrix}
		%    \boldsymbol{0} \\
		%    \boldsymbol{0} \\
		%    \boldsymbol{0}
		%    \end{bmatrix},
	%    \begin{bmatrix}
		%    \mathbf{\Lambda\Phi\Lambda}^{\top} + \mathbf{\Psi}  & \LA\PH   & \PS \\
		%    \PH\LAt                                             & \PH\LAt\big(\Com + \PS\big)^{-1}\LA\PH       & \PH\LAt - \PH\LAt\big(\Com + \PS\big)^{-1}\Com \\
		%    \PS                                                 & \LA\PH - \Com\big(\Com + \PS\big)^{-1}\LA\PH & \PS\big(\Com + \PS\big)^{-1}\PS
		%    \end{bmatrix}
	%    \right\},
	%\end{scriptsize}
	%\end{equation}
	abides the model implications but violates the model postulates (the second identity uses results from Sections \ref{SMSSEC:EpressXiXiH} and \ref{SMSSSEC:Geom_unique}).
	This implies that $\{m, \mathbb{E}(\eph\eph^{\top}), \mathbf{\Lambda}, \mathbb{E}(\xih\xih^{\top}), (\xih,\eph)\}$ is not a factor solution for $x$.
\end{sloppypar}

\begin{remark}
	At this point one may wonder if, when the model implications are met, it matters that the postulates are violated.
	The problem is that the model implications are no longer based on the postulates.
	This means that the (realizations of) $x$ may be produced by a different mechanism from which the model (and all its implications) are initially derived and generally understood.
	That is, the model then takes on a different meaning.
	Moreover, the violations are an expression of uncertainty that is propagated to data representation (in the sense of latent projection).
	That is, they prohibit a \emph{unique} placement of observations in latent space.
	Or, in other words, they imply uncertainty in latent projection under the factor model as further elaborated in Section 2.3 ff.\ of the MT.
\end{remark}

%%%%%%%%%%%%%%%%%%%%%%%%%%%%
%%%-Details SMC ------------
%%%%%%%%%%%%%%%%%%%%%%%%%%%%
\subsection{Details for Section 2.3 of the Main Text}\label{SMSSEC:Supp2.3}
Section \ref{SMSSEC:EpressXiXiHE} shows the equivalence of $\mathbb{E}(\xih\xih^{\top})$ and $\mathbb{E}(\xih\xi^{\top})$.
Section \ref{SMSSEC:EpressXiXiH} contains a useful alternative expression for $\mathbb{E}(\xih\xih^{\top})$.
Lastly, Section \ref{SMSSEC:SMCeps} details the squared multiple correlation between $\epsilon$ and its regression on $x$.

\subsubsection{The equivalence of \texorpdfstring{$\mathbb{E}(\xih\xih^{\top})$}{} and \texorpdfstring{$\mathbb{E}(\xih\xi^{\top})$}{}}\label{SMSSEC:EpressXiXiHE}
Consider
\begin{align*}
	\mathbb{E}\big(\xih\xi^{\top}\big) = \PH\LAt\big(\Com + \PS\big)^{-1}\mathbb{E}\big(x\xi^{\top}\big) = \PH\LAt\big(\Com + \PS\big)^{-1}\LA\PH,
\end{align*}
showing that $\mathbb{E}(\xih\xih^{\top}) = \mathbb{E}(\xih\xi^{\top})$.

\subsubsection{Expressing \texorpdfstring{$\mathbb{E}(\xih\xih^{\top})$}{}}\label{SMSSEC:EpressXiXiH}
Consider the following identity, following from \citep{ref_Duncan1944_SM,ref_Woodbury_SM}:
\begin{equation}\label{SMEQ:Wood}
	\big(\PHi + \Gram \big)^{-1} = \PH - \PH\LAt\big(\Com + \PS\big)^{-1}\LA\PH.
\end{equation}
It is then immediate that:
\begin{align*}
	\PH - \big(\PHi + \Gram \big)^{-1} &= \PH - \Big[\PH - \PH\LAt\big(\Com + \PS\big)^{-1}\LA\PH\Big] \\
	&= \PH\LAt\big(\Com + \PS\big)^{-1}\LA\PH.
\end{align*}
Hence, $\mathbb{E}(\xih\xih^{\top}) = \PH - \mathbf{\Sigma}(\mathbf{\Theta})_{\xi|x}$.
Similarly, one can show that  $\mathbb{E}(\eph\eph^{\top}) = \PS - \mathbf{\Sigma}(\mathbf{\Theta})_{\epsilon|x}$ (see Section \ref{SMSSSEC:Geom_unique}).

\subsubsection{Squared multiple correlation between $\epsilon$ and its regression on $x$}\label{SMSSEC:SMCeps}
We could also have written the model in (1) of the MT as $x := \LA\xi + \PS^{1/2}\varepsilon$, where $\varepsilon \sim \mathcal{N}_{p}(\boldsymbol{0}, \mathbf{I}_p)$, and where $\PS^{1/2}$ is the coefficients matrix for the unique factors $\varepsilon$.
The covariance matrix of the regression of $\varepsilon$ on $x$ then contains squared multiple correlations in its diagonal.
Now, consider $\hat{\varepsilon} = \PS^{-1/2}(x - \LA\xih)$ which, with a similar development as in Section \ref{SMSSEC:Regresse}, can be shown to be the regression function of $\varepsilon$ on $x$. Its correlation matrix is found as:
\begin{align*}
	\mathbb{E}\big(\hat{\varepsilon}\hat{\varepsilon}^{\top}\big) = \PS^{-1/2}\mathbb{E}\Big[\big(x - \LA\xih\big)\big(x - \LA\xih\big)^{\top}\Big]\PS^{-1/2},
\end{align*}
indicating that the squared multiple correlation of $\epsilon$ and its regression on $x$ is contained in the diagonal elements of:
\begin{align}\label{SMEQ:SMCe}
	\PS^{-1/2}\mathbb{E}\big(\eph\eph^{\top}\big)\PS^{-1/2}.
\end{align}
By the Woodbury identity we may express (\ref{SMEQ:SMCe}) as:
\begin{align*}
	\PS^{-1/2}\mathbb{E}\big(\eph\eph^{\top}\big)\PS^{-1/2} &= \PS^{-1/2}\Big[\PS - \LA\big(\PHi + \Gram\big)^{-1}\LAt\Big]\PS^{-1/2} \\
	&= \mathbf{I}_p - \PS^{-1/2}\LA\big(\PHi + \Gram\big)^{-1}\LAt\PS^{-1/2}.
\end{align*}
This last expression gives a handle on its behavior.
Under a weak structure the diagonal elements of the second term will tend to unity.
Under a strong structure the diagonal elements of the second term will tend to $0$.
Hence, the squared multiple correlation of $\epsilon$ and its regression on $x$ is nonnegative, and generally smaller than unity.

%%%%%%%%%%%%%%%%%%%%%%%%%%%%
%%%-Details Construction ---
%%%%%%%%%%%%%%%%%%%%%%%%%%%%
\subsection{Details for Section 2.4 of the Main Text}\label{SMSEC:Supp2.4}
Section \ref{SMSSEC:resid} gives expression to the residual correlation matrix of $\xi - \xih$.
This matrix is of importance for a construction system, given in Section \ref{SMSSEC:Construct}, that constructs common and unique factors that possess all characteristics of the postulated common and unique factors.
Section \ref{SMSSEC:Abide} then indeed shows that constructions stemming from this system adhere to all postulates and implications of the factor model.

\subsubsection{The residual correlation matrix for $\xi$}\label{SMSSEC:resid}
The retrieval of $\xi$ from $x$ will in general not be perfect.
Hence, we can take interest in the residual correlation matrix of $\xi - \xih$:
\begin{equation}\label{SMEQ:RCM}
	\mathbb{E}\Big[\big(\xi - \xih\big)\big(\xi - \xih\big)^{\top}\Big] = \mathbb{E}\big(\xi\xi^{\top}\big) - \mathbb{E}\big(\xi\xih^{\top}\big) - \mathbb{E}\big(\xih\xi^{\top}\big) + \mathbb{E}\big(\xih\xih^{\top}\big).
\end{equation}
Now, $\mathbb{E}(\xi\xi^{\top}) = \PH$ as postulated and from Sections \ref{SMSSEC:EpressXiXiHE} and \ref{SMSSEC:EpressXiXiH} we have that $\mathbb{E}(\xi\xih^{\top}) = \mathbb{E}(\xih\xi^{\top}) = \mathbb{E}(\xih\xih^{\top}) = \PH\LAt\big(\Com + \PS\big)^{-1}\LA\PH$.
Hence, (\ref{SMEQ:RCM}) can be expressed as:
\begin{equation*}
	\mathbb{E}\big(\xi\xi^{\top}\big) - \mathbb{E}\big(\xih\xih^{\top}\big) = \PH - \PH\LAt\big(\Com + \PS\big)^{-1}\LA\PH.
\end{equation*}
This last expression now points to the following equalities:
\begin{equation}\label{SMEQ:resCorIdent}
	\mathbb{E}\Big[\big(\xi - \xih\big)\big(\xi - \xih\big)^{\top}\Big] = \mathbb{E}\big(\xi\xi^{\top}\big) - \mathbb{E}\big(\xih\xih^{\top}\big) = \mathbf{\Sigma}(\mathbf{\Theta})_{\xi|x}.
\end{equation}
That is, the residual correlation matrix of $\xi$ with the linearly-predictable parts partialled out equals the matrix with the dispersion of prediction errors and equals the conditional covariance matrix of $\xi$ given $x$.
This matrix is of importance for the construction system that constructs common and unique factors that adhere to all model postulates.

\subsubsection{A construction system}\label{SMSSEC:Construct}
We are interested in constructing common and unique factors that possess all characteristics of the postulated common and unique factors.
Begin by defining
\begin{equation*}
	\xi'\equiv \xih + \tau,
\end{equation*}
with $\tau = \xi - \xih$ such that the construction $\xi'$ takes on the properties of the true factors in $\xi$.
The term $\tau$ contains that part of $\xi$ that is not linearly predictable from $x$.
The behavior of $\tau$ then indicates what this non-predictable part needs to adhere to in order to construct a $\xi'$ that possesses all desired qualities.
Clearly, $\mathbb{E}(\tau) = \boldsymbol{0}$ and, by (\ref{SMEQ:resCorIdent}), $\mathbb{E}(\tau\tau^{\top}) = \mathbf{\Sigma}(\mathbf{\Theta})_{\xi|x}$.
Moreover,
\begin{align*}
	\mathbb{E}\big(x\tau^{\top}\big) &= \mathbb{E}\bigg\{x\Big[\xi - \PH\LAt\big(\Com + \PS\big)^{-1}x\Big]^{\top}\bigg\} \\
	&= \mathbb{E}\big(x\xi^{\top}\big) - \mathbb{E}\big(xx^{\top}\big)\big(\Com + \PS\big)^{-1}\LA\PH \\
	&= \LA\PH - \LA\PH = \boldsymbol{0}.
\end{align*}
Hence,
\begin{equation*}
	\tau \sim \mathcal{N}_m\big[\boldsymbol{0}, \mathbf{\Sigma}(\mathbf{\Theta})_{\xi|x}\big] \ci x.
\end{equation*}
As $\xi$ cannot be obtained, we need to devise a random variable that behaves exactly as $\tau$.
To do this, first note that, by our model assumptions, $\PHi + \Gram$ is p.d.
Then
\begin{equation*}
	\PH - \PH\LAt\big(\Com + \PS\big)^{-1}\LA\PH = \big(\PHi + \Gram\big)^{-1}
\end{equation*}
is also p.d.
This implies that the matrix square root $\mathbf{\Sigma}(\mathbf{\Theta})_{\xi|x}^{1/2}$ exists.
Now, let $s$ be an arbitrary $m$-dimensional random vector distributed as $\mathcal{N}_m(\boldsymbol{0}, \mathbf{I}_m)$ with the further property that $s \ci x$.
Moreover, let $\mathbf{H} \in \mathbb{R}^{m \times m}$ be any member of the special orthogonal group $\mathrm{SO}(m)$ \citep{ref_GEOGROU_SM}, i.e., the group (under matrix multiplication) of rotation matrices.
By the assumptions on $s$, the characteristics of rotation matrices, and the affine property of Gaussians we then indeed have that:
\begin{equation*}
	\mathbf{\Sigma}(\mathbf{\Theta})_{\xi|x}^{1/2}\mathbf{H}s \sim \mathcal{N}_m\big[\boldsymbol{0}, \mathbf{\Sigma}(\mathbf{\Theta})_{\xi|x}\big] \ci x.
\end{equation*}

The above implies that we can construct proxies for $\xi$ according to:
\begin{align*}
	\xi' &\equiv \PH\LAt\big(\Com + \PS\big)^{-1}x + \Big[\PH - \PH\LAt\big(\Com + \PS\big)^{-1}\LA\PH\Big]^{1/2}\mathbf{H}s \\
	&= \xih + \mathbf{\Sigma}(\mathbf{\Theta})_{\xi|x}^{1/2}\mathbf{H}s.
\end{align*}
We can then find the corresponding $\epsilon'$ as:
\begin{align*}
	\epsilon' &\equiv x - \LA\xi'\\
	&= x - \Com\big(\Com + \PS\big)^{-1}x - \LA\mathbf{\Sigma}(\mathbf{\Theta})_{\xi|x}^{1/2}\mathbf{H}s \\
	&= \eph - \LA\mathbf{\Sigma}(\mathbf{\Theta})_{\xi|x}^{1/2}\mathbf{H}s.
\end{align*}
The system for constructing $(\xi', \epsilon')$ naturally follows:
\begin{equation}\label{SMEQ:ConstSyst}
	\begin{bmatrix}
		\xi' \\
		\epsilon'
	\end{bmatrix}
	\equiv
	\begin{bmatrix}
		\hat{\xi}      & \mathbf{\Sigma}(\mathbf{\Theta})_{\xi|x}^{1/2} \\
		\hat{\epsilon} & -\mathbf{\Lambda}\mathbf{\Sigma}(\mathbf{\Theta})_{\xi|x}^{1/2}
	\end{bmatrix}
	\begin{bmatrix}
		1 \\
		\mathbf{H}s
	\end{bmatrix}.
\end{equation}
The group $\mathrm{SO}(m)$ has infinitely many elements.
There are thus infinitely many rotations $\mathbf{H}s$ of $s$ and, hence, infinitely many constructions $(\xi', \epsilon')$.\footnote{
	It is possible to formulate an alternative construction system with the same implications.
	Consider writing the model of interest as
	\begin{equation*}
		x := \LA\PH^{1/2}\dot{\xi} + \PS^{1/2}\varepsilon,
	\end{equation*}
	where $\dot{\xi} \sim \mathcal{N}_m(\boldsymbol{0}, \mathbf{I}_m)$, $\varepsilon \sim \mathcal{N}_p(\boldsymbol{0}, \mathbf{I}_p)$, and all other assumptions of Section 1.1 of the MT in place.
	Then we may write
	\begin{equation*}
		x := 
		\begin{bmatrix}
			\LA\PH^{1/2}  & \PS^{1/2}
		\end{bmatrix}
		\begin{bmatrix}
			\dot{\xi}  \\
			\varepsilon
		\end{bmatrix}
		\equiv \mathbf{C}\varrho,
	\end{equation*}
	with $\mathbf{C} \in \mathbb{R}^{p \times (m + p)}$ and $\varrho \in \mathbb{R}^{m + p}$.
	One solution (or proxy) for $\varrho$ in this system is through the generalized inverse: $\hat{\varrho} = \mathbf{C}^{+}x$.
	As $\mathbf{C}\mathbf{C}^{\top} = \Com + \PS$ is invertible, our generalized inverse is surjective, such that
	\begin{equation*}
		\mathbf{C}^{+} = \mathbf{C}^{\top}\big(\mathbf{C}\mathbf{C}^{\top}\big)^{-1}.
	\end{equation*}
	This leads to 
	\begin{equation*}
		\hat{\varrho} = \mathbf{C}^{+}x = 
		\begin{bmatrix}
			\PH^{1/2}\LAt\big(\Com + \PS\big)^{-1}x  \\
			\PS^{1/2}(\Com + \PS\big)^{-1}x
		\end{bmatrix} = 
		\begin{bmatrix}
			\PH^{-1/2}\xih  \\
			\PS^{-1/2}\eph
		\end{bmatrix},
	\end{equation*}
	in which we indeed recognize scaled versions of our best linear estimates $\xih$ and $\eph$ found previously (affine scaling by $\PH^{-1/2}$ and $\PS^{-1/2}$ leads to identity covariance matrices for $\xi$ and $\epsilon$, such that $\dot{\xi} = \PH^{-1/2}\xi$ and $\varepsilon = \PS^{-1/2}\epsilon$).
	There is, however, an infinite number of constructions
	\begin{equation}\label{SMEQ:ALTconstr}
		\varrho' = \mathbf{C}^{+}x + \Big(\mathbf{I}_{(m + p)} - \mathbf{C}^{+}\mathbf{C}\Big)\dot{\mathbf{H}}z,
	\end{equation}
	where $z \sim \mathcal{N}_{(m + p)}\big(\boldsymbol{0}, \mathbf{I}_{(m + p)}\big) \ci x$ and $\dot{\mathbf{H}} \in \mathbb{R}^{(m + p) \times (m + p)}$ represents any member of the special orthogonal group $\mathrm{SO}(m + p)$,
	that also provide a solution as
	\begin{align*}
		x &= \mathbf{C}\varrho' = \mathbf{C}\bigg[\hat{\varrho} + \Big(\mathbf{I}_{(m + p)} - \mathbf{C}^{+}\mathbf{C}\Big)\dot{\mathbf{H}}z\bigg] = 
		\mathbf{C}\hat{\varrho} + \big(\mathbf{C} - \mathbf{C}\mathbf{C}^{+}\mathbf{C}\big)\dot{\mathbf{H}}z \\
		&= \mathbf{C}\hat{\varrho},
	\end{align*}
	due to the fundamental property $\mathbf{C}\mathbf{C}^{+}\mathbf{C} = \mathbf{C}$.
	The alternative construction system (\ref{SMEQ:ALTconstr}) can be seen as (a random variable take on) the general (least-squares-type) solution to $x := \mathbf{C}\varrho$ with $\mathbf{C}^{+}\mathbf{C}$ the orthogonal projector that projects $\varrho$ onto $x$ and where $\mathbf{I}_{(m + p)} - \mathbf{C}^{+}\mathbf{C}$ represents its orthogonal complement (or projection onto the null space) \citep[see, e.g.][Chapter 2]{SMref_MTricks}.
	All implications of the construction system of Section \ref{SMSSEC:Construct} also follow from the construction system (\ref{SMEQ:ALTconstr}).
	We prefer the construction system of Section \ref{SMSSEC:Construct} as it is more explicit in expressing the dependency of $\eph$ ($\epsilon'$) on $\xih$ ($\xi'$) and as it is deemed easier to comprehend.
}
From a realization perspective one could say that $s$ represents random numbers from the standard normal distribution that may receive an arbitrary rotation.

\subsubsection{Adherence of $\xi'$ and $\epsilon'$ to model postulates and implications}\label{SMSSEC:Abide}
\begin{sloppypar}
	Constructions $(\xi', \epsilon')$ abide all model postulates and model implications.
	As $\mathbb{E}(x) = 0$ and $\mathbb{E}(s) = 0$, we have that the expectations of $\xi'$, $\epsilon'$, and $\LA\xi' + \epsilon'$ are also zero.
	By normality assumptions on $x$ and $s$ we have that the joint distribution of $x$, $\xi'$ and $\epsilon'$ is also Gaussian (although, as we shall later see, Gaussianity is by no means a necessity).
	Moreover,
	\begin{equation*}
		\LA\xi' + \epsilon' = \LA\Big[\xih + \mathbf{\Sigma}(\mathbf{\Theta})_{\xi|x}^{1/2}\mathbf{H}s\Big] + \eph - \LA\mathbf{\Sigma}(\mathbf{\Theta})_{\xi|x}^{1/2}\mathbf{H}s = \LA\xih + \eph,
	\end{equation*}
	and thus retrieves $x$ (Section \ref{SMSSEC:HatAdhere}).
	Hence, to evaluate if $(\xi',\epsilon')$ is part of a factor solution for $x$ we need to assess if the joint correlation matrix for the random variables in $x = \LA\xi' + \epsilon'$ concurs with (4) of the MT.
	While this should be immediate from the setup of the constructions in Section \ref{SMSSEC:Construct}, it is insightful to show it algebraically.
	To avoid notational clutter, we give the algebraic development for the identity rotation first.
	For this assessment it is useful to keep in mind the results from Section \ref{SMSSEC:HatAdhere} as well as the characteristics of $s$.
	We begin by assessing $\mathbb{E}(\xi'\xi'^{\top})$:
	\begin{align*}
		\mathbb{E}\big(\xi'\xi'^{\top}\big) &= \mathbb{E}\bigg\{\Big[\xih + \mathbf{\Sigma}(\mathbf{\Theta})_{\xi|x}^{1/2}s\Big]\Big[\xih + \mathbf{\Sigma}(\mathbf{\Theta})_{\xi|x}^{1/2}s\Big]^{\top}\bigg\} \\
		&= \mathbb{E}\big(\xih\xih^{\top}\big) + \PH\LAt\big(\Com + \PS\big)^{-1}\mathbb{E}\big(xs^{\top}\big)\mathbf{\Sigma}(\mathbf{\Theta})_{\xi|x}^{1/2} \\
		&\qquad\qquad ~~+ \mathbf{\Sigma}(\mathbf{\Theta})_{\xi|x}^{1/2}\mathbb{E}\big(sx^{\top}\big)\big(\Com + \PS\big)^{-1}\LA\PH +
		\mathbf{\Sigma}(\mathbf{\Theta})_{\xi|x}^{1/2}\mathbb{E}\big(ss^{\top}\big)\mathbf{\Sigma}(\mathbf{\Theta})_{\xi|x}^{1/2} \\
		&= \PH\LAt\big(\Com + \PS\big)^{-1}\LA\PH + \mathbf{\Sigma}(\mathbf{\Theta})_{\xi|x} \\
		&= \PH\LAt\big(\Com + \PS\big)^{-1}\LA\PH + \Big[\PH - \PH\LAt\big(\Com + \PS\big)^{-1}\LA\PH\Big] \\
		&= \PH.
	\end{align*}
	We proceed with $\mathbb{E}(\epsilon'\epsilon'^{\top})$:
	\begin{align}\label{SMEQ:Eee}\nonumber
		\mathbb{E}\big(\epsilon'\epsilon'^{\top}\big) &= \mathbb{E}\bigg\{\Big[\eph - \LA\mathbf{\Sigma}(\mathbf{\Theta})_{\xi|x}^{1/2}s\Big]\Big[\eph - \LA\mathbf{\Sigma}(\mathbf{\Theta})_{\xi|x}^{1/2}s\Big]^{\top}\bigg\}\\\nonumber
		&= \mathbb{E}\big(\eph\eph^{\top}\big) - \PS\big(\Com + \PS\big)^{-1}\mathbb{E}\big(xs^{\top}\big)\mathbf{\Sigma}(\mathbf{\Theta})_{\xi|x}^{1/2}\LAt \\\nonumber
		&\qquad\qquad ~~- \LA\mathbf{\Sigma}(\mathbf{\Theta})_{\xi|x}^{1/2}\mathbb{E}\big(sx^{\top}\big)\big(\Com + \PS\big)^{-1}\PS \\\nonumber
		&\qquad\qquad ~~+
		\LA\mathbf{\Sigma}(\mathbf{\Theta})_{\xi|x}^{1/2}\mathbb{E}\big(ss^{\top}\big)\mathbf{\Sigma}(\mathbf{\Theta})_{\xi|x}^{1/2}\LAt \\
		&= \PS\big(\Com + \PS\big)^{-1}\PS + \LA\mathbf{\Sigma}(\mathbf{\Theta})_{\xi|x}\LAt.
	\end{align}
	Let us look at the second term in (\ref{SMEQ:Eee}):
	\begin{align*}
		\LA\mathbf{\Sigma}(\mathbf{\Theta})_{\xi|x}\LAt &= \Com - \Com\big(\Com + \PS\big)^{-1}\Com \\
		& = \Com\Big[\mathbf{I}_p - \big(\Com + \PS\big)^{-1}\Com\Big] \\
		& = \Com\Big[\mathbf{I}_p - \big(\Com + \PS\big)^{-1}\Com\Big]\PSi\PS \\
		& = \Com\Big[\PSi - \big(\Com + \PS\big)^{-1}\Com\PSi\Big]\PS \\
		& = \Com\big(\Com + \PS\big)^{-1}\PS,
	\end{align*}
	where the last step is justified by Lemma 1 of the MT.
	Hence, (\ref{SMEQ:Eee}) can be expressed as:
	\begin{align*}
		\PS\big(\Com + \PS\big)^{-1}\PS + \Com\big(\Com + \PS\big)^{-1}\PS &= \big(\Com + \PS\big)\big(\Com + \PS\big)^{-1}\PS \\
		&= \PS.
	\end{align*}
	\begin{remark}
		By Lemma 2 from the MT we can give another equivalence for the second term in (\ref{SMEQ:Eee}),
		\begin{equation*}
			\LA\mathbf{\Sigma}(\mathbf{\Theta})_{\xi|x}\LAt = 
			\PS - \PS\big(\Com + \PS\big)^{-1}\PS,
		\end{equation*}
		immediately implying the following identity:
		\begin{equation*}
			\PS - \PS\big(\Com + \PS\big)^{-1}\PS = 
			\Com\big(\Com + \PS\big)^{-1}\PS.
		\end{equation*}
		This identity can be confirmed by using $\PS = \mathbf{\Sigma}(\mathbf{\Theta})_{xx} - \Com$ and $\Com = \mathbf{\Sigma}(\mathbf{\Theta})_{xx} - \PS$:
		\begin{align*}
			\PS - \PS\big(\Com + \PS\big)^{-1}\PS &= \mathbf{\Sigma}(\mathbf{\Theta})_{xx} - \Com - \Big[\mathbf{\Sigma}(\mathbf{\Theta})_{xx} - \Com\Big]\mathbf{\Sigma}(\mathbf{\Theta})_{xx}^{-1}\PS \\
			&= \mathbf{\Sigma}(\mathbf{\Theta})_{xx} - \PS - \Com + \Com\mathbf{\Sigma}(\mathbf{\Theta})_{xx}^{-1}\PS \\
			&= \Com\big(\Com + \PS\big)^{-1}\PS.
		\end{align*}
	\end{remark}
	We continue with $\mathbb{E}(\xi'\epsilon'^{\top})$:
	\begin{align*}
		\mathbb{E}\big(\xi'\epsilon'^{\top}\big) &= \mathbb{E}\bigg\{\Big[\xih + \mathbf{\Sigma}(\mathbf{\Theta})_{\xi|x}^{1/2}s\Big]\Big[\eph - \LA\mathbf{\Sigma}(\mathbf{\Theta})_{\xi|x}^{1/2}s\Big]^{\top}\bigg\} \\
		&= \mathbb{E}\big(\xih\eph^{\top}\big) - \PH\LAt\big(\Com + \PS\big)^{-1}\mathbb{E}\big(xs^{\top}\big)\mathbf{\Sigma}(\mathbf{\Theta})_{\xi|x}^{1/2}\LAt \\
		&\qquad\qquad ~~+ \mathbf{\Sigma}(\mathbf{\Theta})_{\xi|x}^{1/2}\mathbb{E}\big(sx^{\top}\big)\big(\Com + \PS\big)^{-1}\PS - \mathbf{\Sigma}(\mathbf{\Theta})_{\xi|x}^{1/2}\mathbb{E}\big(ss^{\top}\big)\mathbf{\Sigma}(\mathbf{\Theta})_{\xi|x}^{1/2}\LAt \\
		&= \PH\LAt - \PH\LAt\big(\Com + \PS\big)^{-1}\LA\PH\LAt - \mathbf{\Sigma}(\mathbf{\Theta})_{\xi|x}\LAt \\
		&= \Big[\PH - \PH\LAt\big(\Com + \PS\big)^{-1}\LA\PH\Big]\LAt - \Big[\PH - \PH\LAt\big(\Com + \PS\big)^{-1}\LA\PH\Big]\LAt \\
		&= \boldsymbol{0}.
	\end{align*}
	Hence, these correlation matrices abide model postulates (i)--(iii) from Section 1.1 of the MT.
	Using these results we also indeed have that:
	\begin{align*}
		\mathbb{E}\Big[\big(\mathbf{\Lambda}\xi' + \epsilon'\big)\big(\mathbf{\Lambda}\xi' + \epsilon'\big)^{\top}\Big] &= \LA\mathbb{E}\big(\xi'\xi'^{\top}\big)\LAt + \LA\mathbb{E}\big(\xi'\epsilon'^{\top}\big) + \mathbb{E}\big(\epsilon'\xi'^{\top}\big)\LAt + \mathbb{E}\big(\epsilon'\epsilon'^{\top}\big) \\
		&= \Com + \PS,
	\end{align*}
	showing that $\mathbf{\Lambda}\xi' + \epsilon'$ abides the fundamental equation of factor analysis.
	Moreover,
	\begin{align*}
		\mathbb{E}\big(\xi'x^{\top}\big) = \mathbb{E}\bigg\{\Big[\xih + \mathbf{\Sigma}(\mathbf{\Theta})_{\xi|x}^{1/2}s\Big]x^{\top}\bigg\}
		&= \PH\LAt\big(\Com + \PS\big)^{-1}\mathbb{E}\big(xx^{\top}\big) + \mathbf{\Sigma}(\mathbf{\Theta})_{\xi|x}^{1/2}\mathbb{E}\big(sx^{\top}\big) \\
		&= \PH\LAt,
	\end{align*}
	and
	\begin{align*}
		\mathbb{E}\big(\epsilon'x^{\top}\big) = \mathbb{E}\bigg\{\Big[\eph - \LA\mathbf{\Sigma}(\mathbf{\Theta})_{\xi|x}^{1/2}s\Big]x^{\top}\bigg\}
		&= \PS\big(\Com + \PS\big)^{-1}\mathbb{E}\big(xx^{\top}\big) - \LA\mathbf{\Sigma}(\mathbf{\Theta})_{\xi|x}^{1/2}\mathbb{E}\big(sx^{\top}\big) \\
		&= \PS,
	\end{align*}
	showing that $\mathbf{\Lambda}\xi' + \epsilon'$ abides all implications derived from the initial postulates.
	Now, as $\mathbf{H}\mathbb{E}(ss^{\top})\mathbf{H}^{\top} = \mathbf{H}\mathbf{H}^{\top} = \mathbf{I}_m$ and $\mathbf{H}\mathbb{E}(sx^{\top}) = \boldsymbol{0}$ it is immediate that the results above hold for any rotation $\mathbf{H}s$ of $s$.
	Hence, $\{m, \mathbb{E}(\epsilon'\epsilon'^{\top}), \mathbf{\Lambda}, \mathbb{E}(\xi'\xi'^{\top}), (\xi',\epsilon')\} = \{m, \PS, \mathbf{\Lambda}, \PH, (\xi',\epsilon')\}$ is a factor solution for $x$ for all $(\xi',\epsilon')$ constructed according to (\ref{SMEQ:ConstSyst}).
\end{sloppypar}

%%%%%%%%%%%%%%%%%%%%%%%%%%%%
%%%-Details Cor. Geometry --
%%%%%%%%%%%%%%%%%%%%%%%%%%%%
\subsection{Details for Section 2.5 of the Main Text}\label{SMSEC:Supp2.5}
Section \ref{SMSSS:Corxixi} expresses the multiple correlation between $\hat{\xi}_k$ and any $\xi_k'$ stemming from the construction system (Section \ref{SMSSEC:Construct}).
Section \ref{SMSSEC:Gut} shows the derivation behind $\gamma_k$, the minimal correlation between the $k$th factors from any two constructions $\xi'$ and $\xi''$ in $\Xi$. 
Section \ref{SMSSS:FigCone} provides all details behind Figure 2 from the MT.
Section \ref{SMSSSEC:Geom_unique} specifies the correlational geometry of indeterminacy in the unique factors.
Lastly, Section \ref{SMSSSEC:MinCorXiEp} expresses the minimal correlation between constructions of $\xi$ and $\epsilon$.

\subsubsection{The correlation between $\xi'$ and \texorpdfstring{$\xih$}{}}
\label{SMSSS:Corxixi}
As any constructed $\xi'$ behaves as $\xi$, it is immediate that
\begin{equation}\label{SMEQ:SMCsr}
	\Big[\PH\LAt\big(\Com + \PS\big)^{-1}\LA\PH\Big]_{kk}
\end{equation}
also represents the squared multiple correlation between $\xih_k$ and any construction $\xi_k'$.
We may also assess this algebraically.
For any $\xi'$ we have:
\begin{align*}
	\mathbb{E}\big(\xi'\xih^{\top}\big) &= \mathbb{E}\bigg\{\Big[\xih + \mathbf{\Sigma}(\mathbf{\Theta})_{\xi|x}^{1/2}\mathbf{H}s\Big]\xih^{\top}\bigg\} \\
	&= \mathbb{E}\big(\xih\xih^{\top}\big) + \mathbf{\Sigma}(\mathbf{\Theta})_{\xi|x}^{1/2}\mathbf{H}\mathbb{E}\big(sx^{\top}\big)\big(\Com + \PS\big)^{-1}\LA\PH \\
	&= \mathbb{E}\big(\xih\xih^{\top}\big) = \PH\LAt\big(\Com + \PS\big)^{-1}\LA\PH.
\end{align*}
The square root of (\ref{SMEQ:SMCsr}) then represents the multiple correlation between $\hat{\xi}_k$ and any $\xi_k'$ stemming from the construction system.

\subsubsection{Expressing the Guttman criterion}\label{SMSSEC:Gut}
We are interested in expressing \citep{ref_Guttman1955_SM}:
\begin{equation*}
	\gamma_k \equiv \inf\Big\{\rho\big(\xi_{k}',\xi_{k}''\big) | \xi_{k}',\xi_{k}'' \in \Xi\Big\}.
\end{equation*}
Any two maximally different $\xi'$ and $\xi''$ are expressible as:
\begin{align*}
	\xi' &\equiv \xih + \mathbf{\Sigma}(\mathbf{\Theta})_{\xi|x}^{1/2}\mathbf{H}s \\
	\xi'' &\equiv \xih - \mathbf{\Sigma}(\mathbf{\Theta})_{\xi|x}^{1/2}\mathbf{H}s.
\end{align*}
For any $\xi'$ in $\Xi$ we then have:
\begin{align}\label{SMEQ:GB}\nonumber
	\mathbb{E}\big(\xi'\xi''^{\top}\big) &= \mathbb{E}\bigg\{\Big[\xih + \mathbf{\Sigma}(\mathbf{\Theta})_{\xi|x}^{1/2}\mathbf{H}s\Big]\Big[\xih - \mathbf{\Sigma}(\mathbf{\Theta})_{\xi|x}^{1/2}\mathbf{H}s\Big]^{\top}\bigg\} \\\nonumber
	&= \mathbb{E}\big(\xih\xih^{\top}\big) - \PH\LAt\big(\Com + \PS\big)^{-1}\mathbb{E}\big(xs^{\top}\big)\mathbf{H}^{\top}\mathbf{\Sigma}(\mathbf{\Theta})_{\xi|x}^{1/2} \\\nonumber
	&\qquad\qquad ~~+ \mathbf{\Sigma}(\mathbf{\Theta})_{\xi|x}^{1/2}\mathbf{H}\mathbb{E}\big(sx^{\top}\big)\big(\Com + \PS\big)^{-1}\LA\PH \\\nonumber
	&\qquad\qquad ~~-
	\mathbf{\Sigma}(\mathbf{\Theta})_{\xi|x}^{1/2}\mathbf{H}\mathbb{E}\big(ss^{\top}\big)\mathbf{H}^{\top}\mathbf{\Sigma}(\mathbf{\Theta})_{\xi|x}^{1/2} \\\nonumber
	&= \PH\LAt\big(\Com + \PS\big)^{-1}\LA\PH - \Big[\PH - \PH\LAt\big(\Com + \PS\big)^{-1}\LA\PH\Big] \\
	&= 2\PH\LAt\big(\Com + \PS\big)^{-1}\LA\PH - \PH.
\end{align}
The minimal correlations we are looking for are thus contained in the diagonal elements of (\ref{SMEQ:GB}).
An alternative expression to (\ref{SMEQ:GB}) can also be given.
From (\ref{SMEQ:Wood}) it is immediate that:
\begin{equation*}
	2\PH\LAt\big(\Com + \PS\big)^{-1}\LA\PH - \PH = \PH - 2\big(\PHi + \Gram \big)^{-1} \equiv \mathbf{\Gamma}.
\end{equation*}
The latter expression may be of interest from a computational perspective.
Again, $\PS$ will in general be sparse if not diagonal such that its inverse is dependent upon blocks whose dimensions are much lower than $(p \times p)$ and also lower than $(m \times m)$.
The worst-case asymptotic time complexity of the latter expression is then $\mathcal{O}(m^{3})$ rather than $\mathcal{O}(p^{3})$.

Relatedly, we can then also express the variance of the difference between two maximally different constructions \cite{Gut_SM},
\begin{align}\label{SMEQ:MaxVar}\nonumber
	\mathbb{E}\Big[\big(\xi' - \xi''\big)\big(\xi' - \xi''\big)^{\top}\Big] &= \mathbb{E}\bigg\{2\mathbf{\Sigma}(\mathbf{\Theta})_{\xi|x}^{1/2}\mathbf{H}s\Big[2\mathbf{\Sigma}(\mathbf{\Theta})_{\xi|x}^{1/2}\mathbf{H}s\Big]^{\top}\bigg\} \\\nonumber
	&= 4\mathbf{\Sigma}(\mathbf{\Theta})_{\xi|x}^{1/2}\mathbf{H}\mathbb{E}\big(ss^{\top}\big)\mathbf{H}^{\top}\mathbf{\Sigma}(\mathbf{\Theta})_{\xi|x}^{1/2} \\
	&= 4\big(\PHi + \Gram \big)^{-1},
\end{align}
which can be taken as (yet) another measure for the degree of (in)determinacy.
For determinacy one would require this maximal variance to be $\boldsymbol{0}$.
From Section 2.6 of the MT it will be clear that the trace of (\ref{SMEQ:MaxVar}) equals four times the mean squared error between $\xi'$ and $\xih$.

\subsubsection{Understanding Figure 2}
\label{SMSSS:FigCone}
We may understand random variables and their associated probabilities from the perspective of measure theory, especially the notion of the Hilbert space \citep{ref_Hilbert_SM}.
Our vectors of interest, $\xih$ and $\xi'$, are expressible as linear combinations of random vectors and thus are elements of a Hilbert space (of $m$-dimensional random vectors).
The $\ell_2$-norm of an element $\xih$ in the space of $m$-dimensional random vectors is then:
\begin{equation}\label{SMEQ:norm}
	\|\xih\|_2 = \Big[\tr\mathbb{E}\big(\xih\xih^{\top}\big)\Big]^{1/2}.
\end{equation}
As both $\xih$ and $\xi'$ are elements of our Hilbert space we can define their inner product as \citep[][p. 81]{ref_Luenberg_SM}:
\begin{equation*}
	\langle\xih,\xi'\rangle = \tr\mathbb{E}\big(\xih\xi'^{\top}\big).
\end{equation*}
The quantity in (\ref{SMEQ:norm}) fulfills, as a consequence of the Minkowsky inequality \citep{Minko_SM}, the requirements of a norm and can be used to define distance between random variables \citep[][p. 131]{Gut_SM}.
This allows us to express the multidimensional ($m$-ary) correlation between $\xih$ and $\xi'$ as:
\begin{equation*}
	\frac{\langle\xih,\xi'\rangle}{\big\|\xih\big\|_2\big\|\xi'\big\|_2} = \frac{\tr\mathbb{E}\big(\xih\xi'^{\top}\big)}{\Big[\tr\mathbb{E}\big(\xih\xih^{\top}\big)\Big]^{1/2} m^{1/2}}.
\end{equation*}

The multidimensional correlation is relatively hard to interpret.
However, the elements of $\xih$ and $\xi'$ can be regarded as elements of a smaller Hilbert space of random variables.
This allows us to define the norm and inner-product also for individual elements $\xih_k$ and $\xi'_k$:
\begin{equation*}
	\big\|\xih_k\big\|_2 = \Big[\mathbb{E}\big(\xih\xih^{\top}\big)\Big]_{kk}^{1/2},
\end{equation*}
and
\begin{equation*}
	\langle\xih_k,\xi_k'\rangle = \Big[\mathbb{E}\big(\xih\xi'^{\top}\big)\Big]_{kk}.
\end{equation*}
The correlation between $\xih_k$ and $\xi'_k$ can then be expressed as:
\begin{align*}
	\frac{\langle\xih_k,\xi_k'\rangle}{\big\|\xih_k\big\|_2\big\|\xi'_k\big\|_2} &= \frac{\langle\xih_k,\xi_k'\rangle}{\big\|\xih_k\big\|_2} \\
	&= \frac{\Big[\PH\LAt\big(\Com + \PS\big)^{-1}\LA\PH\Big]_{kk}}
	{\Big[\PH\LAt\big(\Com + \PS\big)^{-1}\LA\PH\Big]_{kk}^{1/2}} \\
	&= \Big[\PH\LAt\big(\Com + \PS\big)^{-1}\LA\PH\Big]_{kk}^{1/2} \\
	&\equiv \rho(\xih_k, \xi'_k) = \cos\big[\measuredangle \, \xih_k,\xi'_k\big],
\end{align*}
where we have used that $\|\xi_k'\|_2 = 1$, the covariance results obtained earlier, and basic trigonometry. 
From the above we note the implication that $\|\xih_k\|_2 = \rho(\xih_k, \xi'_k)$.
This means that, next to the cosine of their angle, the length of the vector $\xih_k$ also represents the correlation between $\xih_k$ and $\xi'_k$.

We can similarly express the correlation between $\mathbf{\Sigma}(\mathbf{\Theta})_{\xi|x}^{1/2}\mathbf{H}s \equiv \tau$ (see Section \ref{SMSSEC:Construct}) and $\xi'$. 
Let $\tau'$ be the rotational representation of $\tau$ corresponding to $\xi'$.
Now, $\tau'$ is an element of a Hilbert space, as are its individual elements $\tau'_k$.
Let us find the $\ell_2$-norm of $\tau'_k$ and the inner product of $\tau'_k$ and $\xi'_k$.
For all $\tau'_k$ as elements of the set defined by the rotations $\mathbf{H}s$ of $s$ we have (see Section \ref{SMSSEC:Construct}):
\begin{equation*}
	\big\|\tau'_k\big\|_2 = \Big[\mathbb{E}\big(\tau'\tau'^{\top}\big)\Big]_{kk}^{1/2} = 
	\Big[\mathbf{\Sigma}(\mathbf{\Theta})_{\xi|x}\Big]_{kk}^{1/2}.
\end{equation*}
For the inner product of interest we find:
\begin{align*}
	\langle\tau'_k,\xi_k'\rangle &= \bigg[\mathbb{E}\bigg\{\mathbf{\Sigma}(\mathbf{\Theta})_{\xi|x}^{1/2}\mathbf{H}s
	\Big[\xih + \mathbf{\Sigma}(\mathbf{\Theta})_{\xi|x}^{1/2}\mathbf{H}s\Big]^{\top}\bigg\}
	\bigg]_{kk} \\
	&= \bigg[\mathbf{\Sigma}(\mathbf{\Theta})_{\xi|x}^{1/2}\mathbf{H}\mathbb{E}(s
	s^{\top})\mathbf{H}^{\top}\mathbf{\Sigma}(\mathbf{\Theta})_{\xi|x}^{1/2}\bigg]_{kk} \\
	&= \Big[\mathbf{\Sigma}(\mathbf{\Theta})_{\xi|x}\Big]_{kk},
\end{align*}
where we have used that $\mathbb{E}(x\tau') = 0$, implying $\xih \ci \tau'$.
This allows us to express the correlation between $\tau'_k$ and $\xi'_k$:
\begin{align*}
	\frac{\langle\tau'_k,\xi_k'\rangle}{\big\|\tau'_k\big\|_2\big\|\xi'_k\big\|_2} = \frac{\langle\tau'_k,\xi_k'\rangle}{\big\|\tau'_k\big\|_2} 
	= \frac{\Big[\mathbf{\Sigma}(\mathbf{\Theta})_{\xi|x}\Big]_{kk}}
	{\Big[\mathbf{\Sigma}(\mathbf{\Theta})_{\xi|x}\Big]_{kk}^{1/2}} 
	= \Big[\mathbf{\Sigma}(\mathbf{\Theta})_{\xi|x}\Big]_{kk}^{1/2} 
	\equiv \rho(\tau'_k, \xi'_k) = \cos\big[\measuredangle \, \tau'_k,\xi'_k\big].
\end{align*}
Here we note the implication that $\|\tau'_k\|_2 = \rho(\tau'_k, \xi'_k)$.
Hence, the length of the vector $\tau'_k$ represents the correlation between $\tau'_k$ and $\xi'_k$.

We note that there are infinitely many rotations of $s$ and thus infinitely many $\tau'_k$ with corresponding $\xi'_k$.
These rotational representations form an equator.
On it rests the cone, around $\xih_k$, with all alternative constructions $\xi'_k$.
This cone finds itself in the unit sphere as, using $\xih \ci \tau'$ and the results from Section \ref{SMSSEC:Abide}:
\begin{align}\label{SMEQ:Uni}\nonumber
	\big\|\xi'_k\big\|_2 &= \big\|\xih_k + \tau'_k\big\|_2 \\\nonumber
	&= \bigg\{\mathbb{E}\Big[\big(\xih + \tau'\big)\big(\xih + \tau'\big)^{\top}\Big]\bigg\}_{kk}^{1/2} \\\nonumber
	&= \Big[\mathbb{E}\big(\xih\xih^{\top}\big) + \mathbb{E}\big(\tau'\tau'^{\top}\big)\Big]_{kk}^{1/2} \\\nonumber
	&= \Big[\PH\LAt\big(\Com + \PS\big)^{-1}\LA\PH + \PH - \PH\LAt\big(\Com + \PS\big)^{-1}\LA\PH\Big]_{kk}^{1/2} \\
	&= \big[\PH\big]_{kk}^{1/2} = \sqrt{1} = 1.
\end{align}

We can relate the foregoing also to the geometry of squared multiple correlations.
Either from (\ref{SMEQ:Uni}) or from the orthogonality of $\xih_k$ and $\tau'_k$ in conjunction with the Pythagorean theorem on a unit-length hypothenuse, it immediately follows that 
\begin{equation}\label{SMEQ:LeXi}
	\big\|\xi'_k\big\|_2 = \Big[\PH\LAt\big(\Com + \PS\big)^{-1}\LA\PH\Big]_{kk} + \Big[\mathbf{\Sigma}(\mathbf{\Theta})_{\xi|x}\Big]_{kk} = \rho(\xih_k, \xi'_k)^2 + \rho(\tau'_k, \xi'_k)^2 = 1.
\end{equation}
Hence, the length of $\xi'_k$ is divided into the part predictable by $\xih_k$ and the part dependent on the arbitrary indeterminate component $\tau'_k$.
This may also be obtained from the sum of vector projections of $\tau_k'$ and $\xih_k$ onto $\xi_k'$:
\begin{align*}
	\proj_{\xi_k'}\tau_k' + \proj_{\xi_k'}\xih_k &= 
	\Bigg[\frac{\langle\tau'_k,\xi_k'\rangle}{\big\|\xi'_k\big\|_2} + \frac{\langle\xih_k,\xi_k'\rangle}{\big\|\xi'_k\big\|_2}\Bigg]
	\frac{\xi'_k}{\big\|\xi'_k\big\|_2} \\
	&= \Big[\langle\tau'_k,\xi_k'\rangle + \langle\xih_k,\xi_k'\rangle\Big]\xi'_k \\
	&= \Big[\rho(\tau'_k, \xi'_k)^2 + \rho(\xih_k, \xi'_k)^2\Big]\xi'_k = \xi'_k.
\end{align*}
The sum of vector projections naturally retrieves the vector $\xi'_k$.
The bracketed term contains the respective scalar resolutes of $\tau_k'$ and $\xih_k$ in the direction of $\xi_k'$. 
Hence, they represent the lengths of the respective vector projections. 

To complete our geometric understanding, we will translate the correlations of interest to angles.
For the angle (in degrees) between $\xih_k$ and $\xi'_k$ we have:
\begin{equation*}
	\measuredangle \, \xih_k,\xi'_k \equiv \alpha = \arccos\big[\rho(\xih_k, \xi_k')\big] \cdot \frac{180}{\pi}.
\end{equation*}
Analogously, the angle between $\tau_k'$ and $\xi'_k$ can be retrieved as:
\begin{equation*}
	\measuredangle \, \tau_k', \xi'_k \equiv \beta = \arccos\big[\rho(\tau_k', \xi_k')\big] \cdot \frac{180}{\pi}.
\end{equation*}
For the cosine of the double $\alpha$ angle we remember the standard trigonometric identity $\cos 2\alpha = 2\cos^2 \alpha - 1$. 
From this identity it immediately follows that
\begin{equation*}
	2\alpha = \arccos\big[2\rho(\xih_k, \xi_k')^2 - 1\big] \cdot \frac{180}{\pi} = \arccos(\gamma_k) \cdot \frac{180}{\pi},
\end{equation*}
with the last identity stemming from the Guttman criterion in Section \ref{SMSSEC:Gut} and the correlations defined above (and in Section 2.5 of the MT), as
\begin{equation*}
	\Big[2\PH\LAt\big(\Com + \PS\big)^{-1}\LA\PH - \PH\Big]_{kk} = (\mathbf{\Gamma})_{kk} = \gamma_k =  2\rho(\xih_k, \xi_k')^2 - 1.
\end{equation*}
The Guttman criterion then represents the cosine of the double angle between $\xih_k$ and $\xi_k'$ and thus the cosine of the angle between any $\xi_k'$ and its maximally different counterpart on the opposite side of our cone.

In case of local rotational uniqueness $2^m$ combinations of polarity reversals in the columns of $\LA$ (and thus in the sign of $\xi$) are allowed \citep{Peeters2012_SM}. In this case, $x := \LA\mathbf{P}\mathbf{P}\xi + \epsilon$, where $\mathbf{P} \equiv \mathrm{diag}[\pm 1,\ldots,\pm 1]$ denotes the signature matrix for which $\mathbf{P}^{-1} = \mathbf{P}^{\top} = \mathbf{P}$.
Geometrically, this implies that each $\xih_k$ and all $\xi'_k$ may have a reflection corresponding to the negated $\xi_k$. 
Hence, we then have a double cone, one around $\xih_k$ (whose locus is occupied by all constructions $\xi'_k$) and one around $-\xih_k$ (whose locus is occupied by all negated constructions $\xi'_k$).
This double cone is placed apex to apex, with the apexes resting on the equator formed by all rotational representations of $\tau_k$.
For multiple factors this translates to a (double) hypercone in the hypersphere for the orthogonal case in which the factors are uncorrelated, and the (double) oblique hypercone in the hyperspheroid for the oblique case in which the factors are allowed to correlate.

All developments in this subsection can be used to support (understanding of) Section 2.5 and Figure 2 of the MT.
Let us now fuel our geometrical intuition with some examples.

\begin{example}\label{SMEX:Gman}
	Say that $\xi'_k$ is predicted by $\xih_k$ and $\tau'_k$ in equal parts. 
	That is, $\rho(\xih_k, \xi'_k)^2 = \rho(\tau'_k, \xi'_k)^2 = .5$.
	In this case the correlation between $\xi'_k$ and $\xih_k$ amounts to
	$\rho(\xih_k, \xi'_k) = \sqrt{.5} = .7071068$.
	The associated angles are $\alpha = \arccos(.7071068) \cdot 180/\pi =  45^\circ$ and $2\alpha = 90^\circ$.
	Clearly, when $\rho(\xih_k, \xi'_k)^2 < \rho(\tau'_k, \xi'_k)^2$ the angle $2\alpha$ widens beyond orthogonality.
	Thus when the correlation between $\xih_k$ and any construction $\xi'_k$ drops below $\sqrt{.5}$ the correlation between maximally different constructions $\xi'_k$ and $\xi''_k$ will become negative.
	Note that a correlation of $\sqrt{.5}$ is regarded as (moderately) high in standard accounts of correlation \citep[][pp.\ 77--81]{Cohen_SM}.
	The implication is that a relatively high correlation between the latent factor and its best linear predictor can still imply the possibility of factor constructions that are uncorrelated or negatively correlated.
	This is a situation emphasized by Guttman \citep{ref_Guttman1955_SM} as well as the situation visualized in Figure 2 of the MT.
	For reference, this case is again visualized (in an enlarged fashion) in Figure \ref{SMFIG:Cones} below. 
	\QEDE
\end{example}

\begin{example}
	Say we desire the correlation between two maximally different factor constructions to equal or exceed $.8$.
	We then have the following condition:
	\begin{equation*}
		2\rho(\xih_k, \xi_k')^2 - 1 \geq .8,
	\end{equation*}
	implying that 
	\begin{equation*}
		\rho(\xih_k, \xi_k') \geq \sqrt\frac{.8 + 1}{2} \approx .95.
	\end{equation*}
	Hence, the correlation between the postulated latent factor and the best linear proxy needs to be close to unity in order to achieve a strong correlation between maximally different factor constructions.
	\QEDE
\end{example}

\begin{figure}[h!]
	\centering
	\includegraphics[width=.62\textwidth]{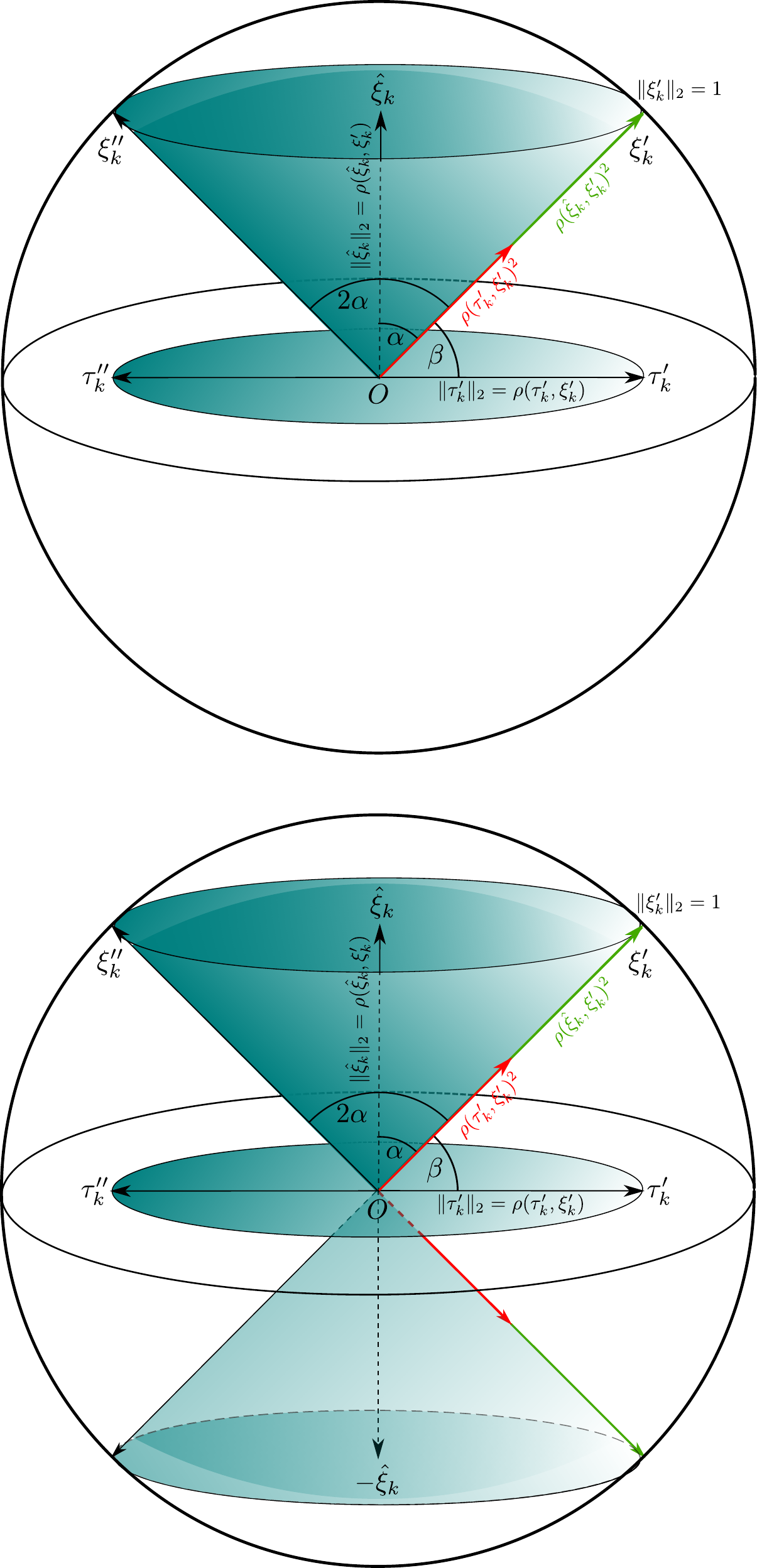}
	\caption{
		A visualization of the correlational geometry of indeterminacy when $\rho(\xih_k, \xi'_k)^2 = \rho(\tau'_k, \xi'_k)^2 = .5$ in case of global (upper figure) and local (lower figure) rotational uniqueness of $\xi$. 
	}
	\label{SMFIG:Cones}
\end{figure}

\begin{example}
	We note that many trigonometric identities can be obtained from the angles $\alpha$ and $\beta$.
	For example, $\alpha + \beta$ must amount to $90^\circ$.
	In order to assess, we need to evaluate the following identity \citep[][p.\ 80]{Handbook_SM}:
	\begin{equation*}
		\Big[\arccos(\rho_1) + \arccos(\rho_2)\Big] \cdot \frac{180}{\pi} = 
		\arccos\bigg[\rho_1\rho_2 - \sqrt{(1 - \rho_1^2)(1 - \rho_2^2)}\bigg] \cdot \frac{180}{\pi},
	\end{equation*}
	with $\rho_1 = \rho(\xih_k, \xi_k')$ and $\rho_2 = \rho(\tau'_k, \xi'_k)$.
	From (\ref{SMEQ:LeXi}) we then note that $1 - \rho_1^2 = \rho_2^2$ and $1 - \rho_2^2 = \rho_1^2$. 
	The evaluation $\arccos(0) \cdot 180/\pi = 90^\circ$ then readily follows.
	We leave the evaluation of other trigonometric identities to the reader.
	\QEDE
\end{example}

\subsubsection{The correlational geometry of indeterminacy in the unique factors}\label{SMSSSEC:Geom_unique}
Let us first look at the correlation between $\epsilon'$ and $\eph$.
As any constructed $\epsilon'$ behaves as $\epsilon$ (Section \ref{SMSSEC:Abide}) it is immediate (Section \ref{SMSSEC:SMCeps}) that 
\begin{equation}\label{SMEQ:SMCep}
	\Big[\PS^{-1/2}\mathbb{E}\big(\eph\eph^{\top}\big)\PS^{-1/2}\Big]_{jj},
\end{equation}
represents the squared multiple correlation between $\eph_j$ and any construction $\epsilon'_{j}$.
Its square root then represents the multiple correlation between $\eph_j$ and any $\epsilon'_{j}$ stemming from the construction system, for which we will write $\rho(\eph_j, \epsilon'_{j})$.
We may also assess this algebraically.
It is immediate that
\begin{equation*}
	\mathbb{E}\big(\epsilon'\eph^{\top}\big) = \mathbb{E}\bigg\{\Big[\eph - \LA\mathbf{\Sigma}(\mathbf{\Theta})_{\xi|x}^{1/2}\mathbf{H}s\Big]\eph^{\top}\bigg\}
	= \mathbb{E}\big(\eph\eph^{\top}\big),
\end{equation*}
when we realize that $\eph$ can be written as $x - \Com\big(\Com + \PS\big)^{-1}x$ (Section \ref{SMSSEC:Regresse}) and that $x \ci s$ (Section \ref{SMSSEC:Construct}).
This implies, from the development in Section \ref{SMSSEC:SMCeps}, that (\ref{SMEQ:SMCep}) carries the squared multiple correlation between between $\eph_j$ and any construction $\epsilon'_{j}$ and that its square root represents $\rho(\eph_j, \epsilon'_{j})$.

Now we turn to the Guttman bound for alternative constructions $\epsilon'_j$ and $\epsilon''_{j}$ for which we will first establish the conditional distribution $\epsilon | x$.
The expectation $\mathbb{E}(\epsilon | x)$ was established in Section \ref{SMSSEC:Regresse}.
The conditional covariance can be found as:
\begin{equation*}
	\mathbf{\Sigma}(\mathbf{\Theta})_{\epsilon|x} = \mathbf{\Sigma}(\mathbf{\Theta})_{\epsilon\epsilon} - \mathbf{\Sigma}(\mathbf{\Theta})_{\epsilon x}\mathbf{\Sigma}(\mathbf{\Theta})_{xx}^{-1}\mathbf{\Sigma}(\mathbf{\Theta})_{x\epsilon} = \PS - \PS\big(\Com + \PS\big)^{-1}\PS.
\end{equation*}
Hence, 
\begin{equation*}
	\epsilon | x \sim \mathcal{N}_p\Big[\PS\big(\Com + \PS\big)^{-1}x, \PS - \PS\big(\Com + \PS\big)^{-1}\PS\Big].
\end{equation*}
We then take interest in
\begin{equation*}
	\upsilon_j \equiv \inf\Big\{\rho\big(\epsilon_{j}',\epsilon_{j}''\big) | \epsilon_{j}',\epsilon_{j}'' \in \Xi\Big\},
\end{equation*}
the minimal correlation between the $j$th elements in any two constructions $\epsilon'$ and $\epsilon''$ in $\Xi$.
When we recall $\xi' \equiv \xih + \mathbf{\Sigma}(\mathbf{\Theta})_{\xi|x}^{1/2}\mathbf{H}s$ and the maximally different $\xi'' \equiv \xih - \mathbf{\Sigma}(\mathbf{\Theta})_{\xi|x}^{1/2}\mathbf{H}s$  we can define (Section \ref{SMSSEC:Construct}):
\begin{equation*}
	\epsilon' \equiv x - \LA\xi' = \eph - \LA\mathbf{\Sigma}(\mathbf{\Theta})_{\xi|x}^{1/2}\mathbf{H}s,
\end{equation*}
and the maximally different
\begin{align*}
	\epsilon'' &\equiv x - \LA\xi'' \\
	&= x - \LA\Big[\xih - \mathbf{\Sigma}(\mathbf{\Theta})_{\xi|x}^{1/2}\mathbf{H}s\Big] \\
	&= x - \Com\big(\Com + \PS\big)^{-1}x + \LA\mathbf{\Sigma}(\mathbf{\Theta})_{\xi|x}^{1/2}\mathbf{H}s \\
	&= \eph + \LA\mathbf{\Sigma}(\mathbf{\Theta})_{\xi|x}^{1/2}\mathbf{H}s.
\end{align*}
We then have that, independently of $\epsilon' \in \Xi$,
\begin{align}\label{SMEQ:infep}\nonumber
	\mathbb{E}\big(\epsilon'\epsilon''^{\top}\big) &= 
	\mathbb{E}\bigg\{\Big[\eph - \LA\mathbf{\Sigma}(\mathbf{\Theta})_{\xi|x}^{1/2}\mathbf{H}s\Big]
	\Big[\eph + \LA\mathbf{\Sigma}(\mathbf{\Theta})_{\xi|x}^{1/2}\mathbf{H}s\Big]^{\top}\bigg\} \\
	&= \PS\big(\Com + \PS\big)^{-1}\PS - \Com\big(\Com + \PS\big)^{-1}\PS,
\end{align}
as a direct consequence of the developments in Section \ref{SMSSEC:HatAdhere}.
We will write (\ref{SMEQ:infep}) into a more convenient form.
Let us focus on the second term.
Note that
\begin{align*}
	- \Com\big(\Com + \PS\big)^{-1}\PS &= \PS\Big[\PSi - \PSi\Com\big(\Com + \PS\big)^{-1}\Big]\PS - \PS \\
	&= \PS(\Com + \PS\big)^{-1}\PS - \PS,
\end{align*}
by Lemma 1 of the MT.
We can then rewrite (\ref{SMEQ:infep}) as
\begin{equation*}
	\PS\big(\Com + \PS\big)^{-1}\PS + \PS(\Com + \PS\big)^{-1}\PS - \PS = 
	2\PS(\Com + \PS\big)^{-1}\PS - \PS.
\end{equation*}
Analogous to the Guttman bound on constructions of $\xi$ we have the following equivalence:
\begin{equation*}
	2\PS(\Com + \PS\big)^{-1}\PS - \PS = \PS - 2\mathbf{\Sigma}(\mathbf{\Theta})_{\epsilon|x} \equiv \mathbf{\Upsilon}.
\end{equation*}
From the developments in Section \ref{SMSSEC:SMCeps} we then understand that the minimal correlations we seek are contained in 
\begin{equation*}
	\upsilon_j = \big(\PS^{-1/2}\mathbf{\Upsilon}\PS^{-1/2}\big)_{jj} = 
	\Big[2\PS^{1/2}(\Com + \PS\big)^{-1}\PS^{1/2} - \mathbf{I}_p\Big]_{jj} = 2\rho(\eph_j, \epsilon'_{j})^2 - 1.
\end{equation*}
We can also express the related variance of the difference between two maximally different constructions $\epsilon'$ and $\epsilon''$:
\begin{align*}
	\mathbb{E}\Big[\big(\epsilon' - \epsilon''\big)\big(\epsilon' - \epsilon''\big)^{\top}\Big] &= \mathbb{E}\bigg\{-2\LA\mathbf{\Sigma}(\mathbf{\Theta})_{\xi|x}^{1/2}\mathbf{H}s\Big[-2\LA\mathbf{\Sigma}(\mathbf{\Theta})_{\xi|x}^{1/2}\mathbf{H}s\Big]^{\top}\bigg\} \\
	&= 4\LA\mathbf{\Sigma}(\mathbf{\Theta})_{\xi|x}^{1/2}\mathbf{H}\mathbb{E}\big(ss^{\top}\big)\mathbf{H}^{\top}\mathbf{\Sigma}(\mathbf{\Theta})_{\xi|x}^{1/2}\LAt \\
	&= 4\LA\mathbf{\Sigma}(\mathbf{\Theta})_{\xi|x}\LAt \\
	&= 4\mathbf{\Sigma}(\mathbf{\Theta})_{\epsilon|x},
\end{align*}
where the last step follows directly from Lemma 2 of the MT.
Note that these considerations can be given a geometrical interpretation analogous to the correlational geometry of indeterminacy in the common factors as described in the preceding subsection.

\subsubsection{Correlations between constructions of $\xi$ and $\epsilon$ maximally different from $0$}\label{SMSSSEC:MinCorXiEp}
Above we have worked out Guttman bounds for constructions of $\xi$ as well as constructions of $\epsilon$. 
We may also ask ourselves to find the correlation between the $k$th element in a construction of $\xi$ and the $j$th element in a construction of $\epsilon$, say $\rho(\xi_{k}',\epsilon_{j}'')$, maximally different from the postulated ideal of $0$.
We are then, as these correlations may be both upwards or downwards from $0$, after establishing:
\begin{equation*}
	\kappa_{kj} \equiv \sup\Big\{|\rho\big(\xi_{k}',\epsilon_{j}''\big)| \,\big|\, \xi_{k}',\epsilon_{j}'' \in \Xi\Big\},
\end{equation*}
where $|\rho\big(\xi_{k}',\epsilon_{j}''\big)|$ denotes the absolute value of $\rho\big(\xi_{k}',\epsilon_{j}''\big)$.
Two maximally different $\xi'$ and $\epsilon''$ would be where $\epsilon''$ belongs to $\xi''$ maximally different from $\xi'$.
Hence, two maximally different $\xi'$ and $\epsilon''$ (in standardized form) can be expressed as:
\begin{align*}
	\xi' &\equiv \xih + \mathbf{\Sigma}(\mathbf{\Theta})_{\xi|x}^{1/2}\mathbf{H}s \\
	\PS^{-1/2}\epsilon'' &\equiv \PS^{-1/2}\eph + \PS^{-1/2}\LA \mathbf{\Sigma}(\mathbf{\Theta})_{\xi|x}^{1/2}\mathbf{H}s.
\end{align*}
We then have that, for any $\xi'$ and its maximally different $\epsilon''$ in $\Xi$,
\begin{align*}
	\mathbb{E}\Big[\xi'\big(\PS^{-1/2}\epsilon''\big)^{\top}\Big] &= 
	\mathbb{E}\bigg\{\Big[\xih + \mathbf{\Sigma}(\mathbf{\Theta})_{\xi|x}^{1/2}\mathbf{H}s\Big]
	\Big[\eph + \LA\mathbf{\Sigma}(\mathbf{\Theta})_{\xi|x}^{1/2}\mathbf{H}s\Big]^{\top}\bigg\}\PS^{-1/2} \\
	&= \Big[\mathbb{E}\big(\xih\eph^{\top}\big) + \mathbf{\Sigma}(\mathbf{\Theta})_{\xi|x}^{1/2}\mathbf{H}\mathbb{E}\big(ss^{\top}\big)\mathbf{H}^{\top}\mathbf{\Sigma}(\mathbf{\Theta})_{\xi|x}^{1/2}\LAt\Big]\PS^{-1/2} \\ 
	&= \Big[\mathbf{\Sigma}(\mathbf{\Theta})_{\xi|x}\LAt + \mathbf{\Sigma}(\mathbf{\Theta})_{\xi|x}\LAt\Big]\PS^{-1/2} \\ 
	&= 2\mathbf{\Sigma}(\mathbf{\Theta})_{\xi|x}\LAt\PS^{-1/2},
\end{align*}
by consecutively using that that $x \ci s$ (Section \ref{SMSSEC:Construct}), $\mathbf{H}\mathbb{E}(ss^{\top})\mathbf{H}^{\top} = \mathbf{H}\mathbf{H}^{\top} = \mathbf{I}_m$ (Section \ref{SMSSEC:Abide}), and earlier results on the expression of $\mathbb{E}\big(\xih\eph^{\top}\big)$ (Section \ref{SMSSEC:HatAdhere}).
Hence, 
\begin{equation*}
	\kappa_{kj} = \Big[2\mathbf{\Sigma}(\mathbf{\Theta})_{\xi|x}\LAt\PS^{-1/2}\Big]_{kj}.
\end{equation*}

%%%%%%%%%%%%%%%%%%%%%%%%%%%%
%%%-Details MSE ------------
%%%%%%%%%%%%%%%%%%%%%%%%%%%%
\subsection{Details for Section 2.6 of the Main Text}\label{SMSEC:Supp2.6}
Section \ref{SMSSSEC:MSExi} expresses the mean squared error (MSE) between $\xi'$ and $\xih$ in terms of covariances.
Section \ref{SMSSEC:compose} then explicates this MSE as a sum of variance and squared bias. 
Section \ref{SMSSEC:DoubleBeta} subsequently relates the squared bias term to the negation of the Guttman bound and the double $\beta$ angle from Figure 2 of the MT.
Lastly, Section \ref{SMSSEC:MSEep} explicates the mean squared error between $\epsilon'$ and $\eph$.

\subsubsection{The mean squared error between $\xi'$ and \texorpdfstring{$\xih$}{}}
\label{SMSSSEC:MSExi}
We have previously seen that we may use the $\ell_2$-norm to define distance between random variables.
This implies that we can also express the MSE in terms of covariances.
Irrespective of $\xi' \in \Xi$ we have that:
\begin{align*}
	\big\|\xi' - \xih\big\|_{2}^{2} &= \tr\mathbb{E}\Big[\big(\xi' - \xih\big)\big(\xi' - \xih\big)^{\top}\Big] \\
	&= \tr\Big[\mathbb{E}\big(\xi'\xi'^{\top}\big) - \mathbb{E}\big(\xih\xih^{\top}\big)\Big] \\
	&= \tr\Big[\PH - \PH\LAt\big(\Com + \PS\big)^{-1}\LA\PH\Big] \\
	&= \tr\Big[\mathbf{\Sigma}(\mathbf{\Theta})_{\xi|x}\Big],
\end{align*}
where the second equivalence follows from Section \ref{SMSSEC:resid}.
As $\xi' - \xih \equiv \tau'$, we have that the MSE between $\xi'$ and $\xih$ can also be expressed as:
\begin{equation*}
	\big\|\xi' - \xih\big\|_{2}^{2} = \big\|\tau'\big\|_{2}^{2} = \tr\mathbb{E}\big(\tau'\tau'^{\top}\big).
\end{equation*}
Hence, it expresses the $k$-sum of squared multiple correlations $\rho(\tau_k', \xi_k')^2$.

\subsubsection{The mean squared error as a composition of variance and bias}
\label{SMSSEC:compose}
It can be insightful to write our MSE of interest as a sum of variance and squared bias. 
Using $\mathbb{E}(\xi\xih^{\top}) = \mathbb{E}(\xih\xi^{\top}) = \mathbb{E}(\xih\xih^{\top})$ (Sections \ref{SMSSEC:EpressXiXiHE} and \ref{SMSSEC:EpressXiXiH}) we have:
\begin{equation}\label{SMEQ:varbias}
	\mathbb{E}\Big[\big(\xi' - \xih\big)\big(\xi' - \xih\big)^{\top}\Big] = 
	\underbrace{\mathbb{E}\big(\xih\xih^{\top}\big)}_{\mathrm{variance}} + 
	\underbrace{\mathbb{E}\big(\xi'\xi'^{\top}\big) - 2\mathbb{E}\big(\xih\xih^{\top}\big)}_{\mathrm{squared \,bias}}.
\end{equation}
From our previous developments expression (\ref{SMEQ:varbias}) amounts to:
\begin{equation*}\label{SMEQ:varbiasexplicate}
	\PH\LAt\big(\Com + \PS\big)^{-1}\LA\PH + \PH -2\PH\LAt\big(\Com + \PS\big)^{-1}\LA\PH.
\end{equation*}
This last expression, using information from Section \ref{SMSSEC:Gut}, can be rewritten as:
\begin{align*}
	\PH\LAt\big(\Com + \PS\big)^{-1}\LA\PH + 2\mathbf{\Sigma}(\mathbf{\Theta})_{\xi|x} - \PH = 
	\PH\LAt\big(\Com + \PS\big)^{-1}\LA\PH - \mathbf{\Gamma},
\end{align*}
directly implying that we can express our MSE of interest also as:
\begin{equation*}
	\tr\Big[\PH\LAt\big(\Com + \PS\big)^{-1}\LA\PH\Big] - \tr\big(\mathbf{\Gamma}\big).
\end{equation*}
The Guttman bound may thus be seen as the negation of the squared bias between $\xi'$ and $\xih$.

\subsubsection{The double $\beta$ angle}
\label{SMSSEC:DoubleBeta}
Consider the $\beta$ angle from Figure 2 of the MT.
From Section \ref{SMSSS:FigCone} we have that the cosine of this angle represents the correlation $\rho(\tau_k', \xi_k')$ between the arbitrary (rotated) component $\tau_k'$ and construction $\xi_k'$.
By the well-known double angle identities we then have that
\begin{equation}\label{SMEQ:doubbeta}
	2\beta = \arccos\big[2\rho(\tau_k', \xi_k')^2 - 1\big] \cdot \frac{180}{\pi}
\end{equation}
expresses the double angle (in degrees).
From the developments in Section \ref{SMSSS:FigCone} and Figure 2 we have the following identity:
\begin{equation*}
	\sqrt{1 - \rho(\xih_k, \xi_k')^2} = \rho(\tau_k', \xi_k').
\end{equation*}
Using this identity, we can establish further identities:
\begin{align*}
	2\rho(\tau_k', \xi_k')^2 - 1 &= 2\bigg[\sqrt{1 - \rho(\xih_k, \xi_k')^2}\bigg]^2 - 1 \\
	&= 2\Big[1 - \rho(\xih_k, \xi_k')^2\Big] -1 \\
	%&= 2 - 2\rho(\xih_k, \xi_k')^2 - 1 \\
	&= 1 - 2\rho(\xih_k, \xi_k')^2 \\
	&= -\gamma_k,
\end{align*}
implying that (\ref{SMEQ:doubbeta}) can also be expressed as:
\begin{equation}\label{SMEQ:doubbeta2}
	\arccos(-\gamma_k) \cdot \frac{180}{\pi}.
\end{equation}

\begin{remark}
	We could have obtained expression (\ref{SMEQ:doubbeta2}) also from our previously found matrix expressions by noticing that:
	\begin{equation*}
		2\rho(\tau_k', \xi_k')^2 - 1 = \Big[2\mathbf{\Sigma}(\mathbf{\Theta})_{\xi|x} - \PH\Big]_{kk} = -\Big[\PH - 2\mathbf{\Sigma}(\mathbf{\Theta})_{\xi|x} \Big]_{kk} = -(\mathbf{\Gamma})_{kk} = -\gamma_k.
	\end{equation*}
\end{remark}

\begin{remark}
	The correlation $\rho(\xih_k, \xi_k')$ can be evaluated in empirical settings. The quantities $[1 - \rho(\xih_k, \xi_k')^2]^{1/2}$ and $-\gamma_k$ may then serve as (alternative) metrics for indeterminacy in the sense that they can be used to quantify, through $\rho(\xih_k, \xi_k')$, the dependency on arbitrary indeterminate components $\tau'_k$. 
\end{remark}

Let us look at some intuitive examples of $2\beta$ in the context of Figure 2 of the MT.

\begin{example}
	At one extreme we find the full protrusion of the equator formed by the rotational representations of $\tau$ and thus a full collapse of the cone onto our equator. 
	In this case $\rho(\xih_k, \xi_k')^2 = 0$ such that $\gamma_k = -1$. This implies that $2\beta = \arccos(1) \cdot 180/\pi = 0^\circ$. This expresses full dependency on the arbitrary indeterminate component $\tau'_k$ or complete bias.
	\QEDE
\end{example}

\begin{example}
	At the other extreme we find a full contraction of the cone around $\xih_k$.
	In that case $\rho(\xih_k, \xi_k')^2 = 1$ such that $\gamma_k = 1$.
	This implies that $2\beta = \arccos(-1) \cdot 180/\pi = 180^\circ$, expressing absence of dependency on the arbitrary indeterminate component $\tau'_k$ and thus absence of bias.
	\QEDE
\end{example}

\subsubsection{The mean squared error between $\epsilon'$ and \texorpdfstring{$\eph$}{}}
\label{SMSSEC:MSEep}
We can also express the MSE between $\epsilon' \in \Xi$ and $\eph$.
First we note that, as $s \ci x$:
\begin{equation*}
	\mathbb{E}\big(\epsilon'\eph^{\top}\big) = \mathbb{E}\bigg\{\Big[\eph + \mathbf{\Sigma}(\mathbf{\Theta})_{\xi|x}^{1/2}\mathbf{H}s\Big]\eph^{\top}\bigg\} = \mathbb{E}\big(\eph\eph^{\top}\big).
\end{equation*}
Using this implication we have that:
\begin{align*}
	\big\|\epsilon' - \eph\big\|_{2}^{2} &= \tr\mathbb{E}\Big[\big(\epsilon' - \eph\big)\big(\epsilon' - \eph\big)^{\top}\Big] \\
	&= \tr\Big[\mathbb{E}\big(\epsilon'\epsilon'^{\top}\big) - \mathbb{E}\big(\eph\eph^{\top}\big)\Big] \\
	&= \tr\Big[\PS - \PS\big(\Com + \PS\big)^{-1}\PS\Big] \\
	&= \tr\Big[\mathbf{\Sigma}(\mathbf{\Theta})_{\epsilon|x}\Big].
\end{align*}
These findings can be given analogous geometrical considerations as with $\xi'$ and $\xih$.

%%%%%%%%%%%%%%%%%%%%%%%%%%%%
%%%-Details Markov ---------
%%%%%%%%%%%%%%%%%%%%%%%%%%%%
\subsection{Details for Section 2.7 of the Main Text}\label{SMSEC:Supp2.7}
Section \ref{SMSSEC:Bipartite} gives a formal definition of the directed bipartite mixed graph with which the factor model may be viewed as a graphical object.
This graphical object may be seen as the marginalization of the traditional path diagram of the factor model over the error variables.
Section \ref{SMSSEC:Marginalize} explains why we may perform this marginalization.
Section \ref{SMSSEC:Moralize} explains the separation criterion with which we can understand the Markov properties of our graphical object.
Section \ref{SMSSEC:PrecMat} derives the precision matrix whose support has a one-on-one relation with graphical separation in our object of interest.
Lastly, Section \ref{SMSSEC:AddPers} provides an additional view on (in)determinacy from the perspective of Markov properties.

\subsubsection{Directed bipartite mixed graphs}
\label{SMSSEC:Bipartite}
The path diagram of a factor model marginalized over the unique (error) factors is a special instance of a directed acyclic mixed graph.
It can be seen as a directed bipartite mixed graph:
\begin{definition}[Directed bipartite mixed graphs]\label{DEF:DPG}
	Consider the random variables $\xi \in \mathbb{R}^m$ and $x \in \mathbb{R}^p$.
	A graph $\mathcal{G} = (\mathcal{V}, \mathcal{E})$ is a \emph{directed bipartite graph} if $\mathcal{V} = \mathcal{V}_{\xi} \cup \mathcal{V}_{x}$ with $\mathcal{V}_{\xi}$ and $\mathcal{V}_{x}$ disjoint, and $\mathcal{E} = \mathcal{E}_{\xi x}$ such that ordered pairs $\xi_{k} \rightarrow x_{j} \in \mathcal{E}_{\xi x}$.
	That is, each edge in $\mathcal{E}_{\xi x}$ originates on one vertex in $\mathcal{V}_{\xi}$ and is directed towards one vertex in $\mathcal{V}_{x}$.
	We have a \emph{directed bipartite mixed graph} if, furthermore, $\mathcal{E} = \mathcal{E}_{\xi \xi} \cup \mathcal{E}_{\xi x} \cup \mathcal{E}_{x x}$ with $\mathcal{E}_{\xi \xi}$, $\mathcal{E}_{\xi x}$, and $\mathcal{E}_{x x}$ mutually disjoint, with edges in $\mathcal{E}_{\xi \xi}$ consisting of pairs of distinct vertices in $\mathcal{V}_{\xi}$ such that $\xi_{k} \leftrightarrow \xi_{k'} \in \mathcal{E}_{\xi \xi}$ and with edges in $\mathcal{E}_{x}$ consisting of pairs of distinct vertices in $\mathcal{V}_{xx}$ such that $x_{j} \leftrightarrow x_{j'} \in \mathcal{E}_{xx}$.
	The directed bipartite (mixed) graph is termed \emph{directionally complete} if $\xi_{k} \rightarrow x_{j} \in \mathcal{E}_{\xi x} \,\forall (k,j), k = 1, \ldots, m$ and $j = 1, \ldots, p$.
\end{definition}
We have cast this definition in our factor model notation.
Hence, the traditional (orthogonally rotated) exploratory factor model can be viewed as a complete directed bipartite graph.
The traditional confirmatory model -- with exclusion restrictions in $\LA$, $\PH \neq \mathbf{I}_m$, and a possibly non-diagonal $\PS$ -- may be seen as a directed bipartite mixed graph.
An obliquely rotated exploratory model can then be conceived as a directionally complete bipartite mixed graph.

\subsubsection{Marginalizing the factor model path diagram over $\epsilon$}
\label{SMSSEC:Marginalize}
Traditionally, factor model structures have been visualized by way of path diagrams \citep[see, e.g.,][]{SMref_Bollen1989}.
These correspond closely to the type of graph described in Section 2.7 of the MT and defined in Definition \ref{DEF:DPG}.
They differ in the sense that there is a directed edge from each unique factor $\epsilon_j$ to the corresponding observable $x_j$ and there may be bidirected edges $\epsilon_{j} \leftrightarrow \epsilon_{j'}$ if $\psi_{jj'} = \psi_{j'j} \neq 0$.
Hence, the type of graph described in Section 2.7 and implied by Definition \ref{DEF:DPG} may be seen as a marginalization of the traditional path diagram over the unique factors in $\epsilon$.
This of course renders the question of why we may perform this marginalization.
We will explain the various reasons below.

First, usage of $\epsilon_j \rightarrow x_j$ would violate the meaning of $\rightarrow$ in the causal graph.
That is, the $\epsilon_j$ do not causally influence the $x_j$.
They are rather seen as disturbance or residual terms.
(Note that this is also the reason for an undirected rather than directed connection between $\epsilon$ and $x$ in the model diagram in Figure 1 of the MT).

Second, the vector $[x^{\top}, \xi^{\top}, \epsilon^{\top}]^{\top}$ has a degenerate (normal) distribution.
From (4) of the MT it is immediate that
\begin{equation*}
	\mathbf{\Sigma}(\mathbf{\Theta})_{[1]} - \mathbf{\Sigma}(\mathbf{\Theta})_{[2]} \boxdot \big(\boldsymbol{1}_3 \otimes \LAt\big) = \mathbf{\Sigma}(\mathbf{\Theta})_{[3]}.
\end{equation*}

Third, Koster \citep{ref_Koster1996_SM,ref_Koster1999_SM} has shown that deleting the error variables (from the path diagram) and connecting $x_{j}$ and $x_{j'}$ by a double-headed arrow $x_{j} \leftrightarrow x_{j'}$ whenever $\psi_{jj'} = \psi_{j'j} \neq 0$, retains a one-on-one correspondence between the resulting graph (as per Definition \ref{DEF:DPG}) and the central model equation (Equation (1) from the MT).
This may also be (indirectly) seen from the symmetric indefinite factorization of the Choleski decomposition of the joint distribution of $[x^{\top}, \xi^{\top}]^{\top}$,
\begin{align*}
	\begin{bmatrix}
		\LA\phantom{^{\top}}  &  \mathbf{I}_p \\
		\mathbf{I}_m & \boldsymbol{0}
	\end{bmatrix}
	\begin{bmatrix}
		~~\PH\phantom{^{\top}}  &  \boldsymbol{0} \\
		\boldsymbol{0} & ~~\PS\phantom{^{\top}} 
	\end{bmatrix}
	\begin{bmatrix}
		\LAt  &  \mathbf{I}_m \\
		\mathbf{I}_p & \boldsymbol{0}
	\end{bmatrix}
	=
	\begin{bmatrix}
		\Com + \PS  & \LA\PH \\
		\PH\LAt  & \PH
	\end{bmatrix},
\end{align*}
which depends pivotally on the covariance matrix of the joint distribution $[\xi^{\top}, \epsilon^{\top}]^{\top}$.

\subsubsection{$\mathfrak{m}$-separation}
\label{SMSSEC:Moralize}
The problem of inferring the (global) Markov properties (i.e., the
probabilistic conditional (in)dependencies) from probabilistic graphs has been an active programme of research \citep{ref_Koster1996_SM,SMref_SpirtesSMR98,ref_Koster1999_SM}.
Directly reading off the Markov properties from graphs implied by Definition \ref{DEF:DPG} can be achieved by generalizing Pearl's notion of \emph{d}-separation \citep{SMref_Pearl2009} to $\mathfrak{m}$-separation \citep{SMref_SpirtesSMR98,ref_Koster1999_SM,SMref_Rich2003}. 
Before defining this graphical separation criterion some further notation
and terminology needs to be introduced. 
Let $e$ denote a generic edge in $\mathcal{E}$ and let $v$ denote a generic vertex in $\mathcal{V}$. 
Additionally, let $v_a$ and $v_b$ denote two distinct vertices. 
A \emph{path} $\rho_{ab}$ from $v_a$ to $v_b$ is then defined as an ordered $r$-tuple of edges $\rho_{ab} = (e_{1},\ldots, e_{r})$, $r \geq 1$, for which there exists $r + 1$ distinct vertices $\{ v_{0}, \ldots, v_{r} \}$, s.t. $v_{0} = v_{a}$ and $v_{r} = v_{b}$ and $e_{l} \in \{ v_{l-1}\rightarrow
v_{l},v_{l-1}\leftarrow v_{l}, v_{l-1}\leftrightarrow v_{l} \}$ for $l = 1, \ldots, r$. 
If $r > 1$, vertices in $\{ v_{l}| 0 < l < r\}$ are termed \emph{intermediate}. 
A path $\rho_{ab}$ is a \emph{directed path} $\rho_{ab}^{\rightarrow}$ if $e_{l} = v_{l-1} \rightarrow v_{l}$ $\forall 1 \leq l \leq r$. 
Now, define the set of \emph{descendants} for vertex $v_{a}$ as $\mbox{de}(v_{a}) \equiv \{v \in \mathcal{V}| \mbox{there exists in}~ \mathcal{G} ~\mbox{a directed path from}~ v_{a} ~\mbox{to}~ v \}$. 
This definition carries over to sets of vertices, e.g., let $A$ denote a
set of vertices, then $\mbox{de}(A) = \cup_{v \in A} \mbox{de}(v)$.
Fusing the work of Spirtes \emph{et al.}\ \citep{SMref_SpirtesSMR98}, Koster \citep{ref_Koster1999_SM}, and Richardson \citep{SMref_Rich2003}, we define:
\begin{definition}[$\mathfrak{m}$-separation]\label{mSep}
	Let $v_{l}$ be an intermediate vertex on path $\rho_{ab} = (e_{1},
	\ldots, e_{r})$ from $v_a$ to $v_b$ (from $a$ to $b$ for short).
	Then $v_{l}$ is a \emph{collider} on $\rho_{ab}$ if $e_{l} \in \{
	v_{l-1}\rightarrow v_{l},v_{l-1}\leftrightarrow v_{l} \}$ $\wedge$
	$e_{l+1} \in \{ v_{l}\leftarrow v_{l+1},v_{l}\leftrightarrow v_{l+1}
	\}$. If the edges preceding and succeeding $v_{l}$ on $\rho_{ab}$ do
	not have meeting arrowheads at $v_{l}$, then the intermediate vertex
	$v_{l}$ is a \emph{non-collider} on $\rho_{ab}$, i.e., $e_{l} =
	v_{l-1}\leftarrow v_{l} \,\wedge\, e_{l+1} \in \{ v_{l}\leftarrow
	v_{l+1},v_{l}\rightarrow v_{l+1},v_{l}\leftrightarrow v_{l+1} \}$ or
	$e_{l+1} = v_{l}\rightarrow v_{l+1} \wedge e_{l} \in \{
	v_{l-1}\leftarrow v_{l},v_{l-1}\rightarrow
	v_{l},v_{l-1}\leftrightarrow v_{l} \}$. A path $\rho_{ab}$ from $a$
	to $b$ in $\mathcal{G}$ is then \emph{pathwise $\mathfrak{m}$-separated}
	by a set of vertices $C \subseteq \mathcal{V}\setminus\{a,b\}$ if and only if
	\begin{enumerate}
		\item[i.] $\{ v_{l}| v_{l} ~\mbox{is a non-collider on}~ \rho_{ab} \} \cap C \neq \emptyset$; or
		%\item[2.] $\{ \mathcal{V}_{l}| \mathcal{V}_{l} ~\mbox{is a collider on}~ \rho \} \nsubseteq \mbox{\emph{an}}(c)$.
		\item[ii.] $\exists \{ v_{l}| v_{l} ~\mbox{is a collider on}~ \rho_{ab} \} \equiv S$ s.t. $S \cap C = \emptyset ~\wedge~
		\mbox{\emph{de}}(S) \cap C = \emptyset$.
	\end{enumerate}
	If $C$ pathwise $\mathfrak{m}$-separates every path from $a$ to $b$, then
	$a$ to $b$ are said to be \emph{$\mathfrak{m}$-separated} given $C$. If
	$C$ does not $\mathfrak{m}$-separate $a$ from $b$, then $a$ and $b$ are
	said to be \emph{$\mathfrak{m}$-connected} given $C$.
\end{definition}

The $\mathfrak{m}$-separation criterion defines an conditional (in)dependency model over $\mathcal{V}$ that extends to (arbitrary) pairwise disjoint subsets $A, B$, and $C$ (C may be empty): If $A$ is $\mathfrak{m}$-separated from $B$ by $C$ in $\mathcal{G}$, then, in any distribution $\mathcal{P}$ compatible with $\mathcal{G}$, $A \ci B~|~C ~[\mathcal{P}]$. 
If we assume that probabilistic independencies are due to separation in the graph (faithfulness assumption) then $\mathcal{G}$, $A \ci B~|~C ~[\mathcal{P}]$ implies that $A$ is $\mathfrak{m}$-separated from $B$ by $C$ in $\mathcal{G}$.
Koster \citep{ref_Koster1999_SM} gives notable results on $\mathfrak{m}$-separation. 
For (arbitrary) pairwise disjoint subsets $A, B$, and $C$ (C may be empty): whenever $v_a \in A$ and $v_b \in B$, $A$ and $B$ are $\mathfrak{m}$-separated by $C$ \emph{iff} there are no paths from $a$ to $b$ in which every intermediate vertex is a collider. 
These results allow for an extension of the moralization criterion \citep{ref_Moral_SM} on the basis of which a transformed graph can be constructed that immediately conveys the conditional independencies by undirected graphical adjacency.
Moreover, it allows us to characterize the Markov properties of the model by means of (the support of) the concentration or precision matrix (on $x$ and $\xi$).

\subsubsection{Precision matrix derivation}
\label{SMSSEC:PrecMat}
We are interested in finding the inverse
\begin{align*}
	\begin{bmatrix}
		\mathbf{\Sigma}(\mathbf{\Theta})_{xx}  &  \mathbf{\Sigma}(\mathbf{\Theta})_{x\xi} \\
		\mathbf{\Sigma}(\mathbf{\Theta})_{\xi x} & \mathbf{\Sigma}(\mathbf{\Theta})_{\xi\xi}
	\end{bmatrix}^{-1}.
\end{align*}
By our model postulates and assumptions $\mathbf{\Sigma}(\mathbf{\Theta})_{\xi\xi}$ has nonsingular Schur complement $\mathbf{E} = \mathbf{\Sigma}(\mathbf{\Theta})_{xx} - \mathbf{\Sigma}(\mathbf{\Theta})_{x\xi}\mathbf{\Sigma}(\mathbf{\Theta})_{\xi\xi}^{-1}\mathbf{\Sigma}(\mathbf{\Theta})_{\xi x}$.
We may then obtain the desired inverse through block matrix inversion:
\begin{align}\label{SMEQ:PrecMat}
	\begin{bmatrix}
		\mathbf{E}^{-1}  &  -\mathbf{E}^{-1}\mathbf{\Sigma}(\mathbf{\Theta})_{x\xi}\mathbf{\Sigma}(\mathbf{\Theta})_{\xi\xi}^{-1} \\
		-\mathbf{\Sigma}(\mathbf{\Theta})_{\xi\xi}^{-1}\mathbf{\Sigma}(\mathbf{\Theta})_{\xi x}\mathbf{E}^{-1} & \mathbf{\Sigma}(\mathbf{\Theta})_{\xi\xi}^{-1} + \mathbf{\Sigma}(\mathbf{\Theta})_{\xi\xi}^{-1}\mathbf{\Sigma}(\mathbf{\Theta})_{\xi x}\mathbf{E}^{-1}\mathbf{\Sigma}(\mathbf{\Theta})_{x\xi}\mathbf{\Sigma}(\mathbf{\Theta})_{\xi\xi}^{-1}
	\end{bmatrix}.
\end{align}
We then recall Theorem 2.5.1 from Anderson \citep{SMref_AndersonBible}, showing that $\mathbf{\Pi} \equiv \mathbf{\Sigma}(\mathbf{\Theta})_{x\xi}\mathbf{\Sigma}(\mathbf{\Theta})_{\xi\xi}^{-1}$ is the matrix of regression coefficients of $x$ on $\xi$ and $\mathbf{E} = \mathbf{\Sigma}(\mathbf{\Theta})_{x|\xi}$.
Hence, we can express the inverse (\ref{SMEQ:PrecMat}) as:
\begin{align*}
	\begin{bmatrix}
		\mathbf{\Sigma}(\mathbf{\Theta})_{x|\xi}^{-1}  &  	-\mathbf{\Sigma}(\mathbf{\Theta})_{x|\xi}^{-1}\mathbf{\Pi} \\
		-\mathbf{\Pi}^{\top}\mathbf{\Sigma}(\mathbf{\Theta})_{x|\xi}^{-1}  & 
		\mathbf{\Sigma}(\mathbf{\Theta})_{\xi\xi}^{-1} + \mathbf{\Pi}^{\top}\mathbf{\Sigma}(\mathbf{\Theta})_{x|\xi}^{-1}\mathbf{\Pi}
	\end{bmatrix}.
\end{align*}
At this point we note that the desired inverse may also be defined in terms of the Schur complement on $\mathbf{\Sigma}(\mathbf{\Theta})_{xx}$, which is $\mathbf{\Sigma}(\mathbf{\Theta})_{\xi\xi} - \mathbf{\Sigma}(\mathbf{\Theta})_{\xi x}\mathbf{\Sigma}(\mathbf{\Theta})_{xx}^{-1}\mathbf{\Sigma}(\mathbf{\Theta})_{x\xi} = \mathbf{\Sigma}(\mathbf{\Theta})_{\xi|x}$, implying that $\mathbf{\Sigma}(\mathbf{\Theta})_{\xi\xi}^{-1} + \mathbf{\Pi}^{\top}\mathbf{\Sigma}(\mathbf{\Theta})_{x|\xi}^{-1}\mathbf{\Pi} = \mathbf{\Sigma}(\mathbf{\Theta})_{\xi|x}^{-1}$.
Some ready algebra then gives $\mathbf{\Pi} = \LA\PH\PHi = \LA$ and $	\mathbf{\Sigma}(\mathbf{\Theta})_{x|\xi}^{-1} = \big(\Com + \PS - \LA\PH\PHi\PH\LAt\big)^{-1} = \PSi$, resulting in
\begin{align*}
	\begin{bmatrix}
		\mathbf{\Sigma}(\mathbf{\Theta})_{xx}  &  \mathbf{\Sigma}(\mathbf{\Theta})_{x\xi} \\
		\mathbf{\Sigma}(\mathbf{\Theta})_{\xi x} & \mathbf{\Sigma}(\mathbf{\Theta})_{\xi\xi}
	\end{bmatrix}^{-1}
	&=
	\begin{bmatrix}
		\mathbf{\Sigma}(\mathbf{\Theta})_{x|\xi}^{-1}  &  -\mathbf{\Sigma}(\mathbf{\Theta})_{x|\xi}^{-1}\mathbf{\Pi} \\
		-\mathbf{\Pi}^{\top}\mathbf{\Sigma}(\mathbf{\Theta})_{x|\xi}^{-1}  & 
		\mathbf{\Sigma}(\mathbf{\Theta})_{\xi|x}^{-1}
	\end{bmatrix} \\
	&= 
	\begin{bmatrix}
		\PSi  &  -\PSi\LA \\
		-\LAt\PSi & \PHi + \Gram
	\end{bmatrix}.
\end{align*}

\subsubsection{An alternate view on (in)determinacy}
\label{SMSSEC:AddPers}
One way to view determinacy from the Markov perspective is to desire $\mathbf{\Sigma}(\mathbf{\Theta})_{xx}^{-1} = \mathbf{\Sigma}(\mathbf{\Theta})_{x|\xi}^{-1}$.
This view is taken in the MT.
An alternate view on determinacy from the Markov perspective is that the following equivalency should hold: $\mathbf{\Sigma}(\mathbf{\Theta})_{x|\xih} = \mathbf{\Sigma}(\mathbf{\Theta})_{x|\xi}$.
Naturally, this implies that $[x^{\top}, \hat{\xi}^{\top}]^{\top}$ should be distributed as $[x^{\top}, \xi^{\top}]^{\top}$.
Assessing the covariance of $[x^{\top}, \hat{\xi}^{\top}]^{\top}$,
\begin{align*}
	\begin{bmatrix}
		\Com + \PS  & \LA\PH \\
		\PH\LAt  & \PH\LAt\big(\Com + \PS\big)^{-1}\LA\PH
	\end{bmatrix},
\end{align*}
it is immediate that this requires the earlier established condition that $\PH\LAt\big(\Com + \PS\big)^{-1}\LA\PH$ should equal $\PH$.
%--------------- Details 2 ----------------------------------------------
%------------------------------------------------------------------------

%------------------------------------------------------------------------
%--------------- Details 3 ----------------------------------------------
\section{Details for Section 3 of the Main Text}\label{SMSEC:Supp3}
%%%%%%%%%%%%%%%%%%%%%%%%%%%%
%%%-Details Bayesian View --
%%%%%%%%%%%%%%%%%%%%%%%%%%%%
\subsection{Details for Section 3.1 of the Main Text}
\label{SMSSEC:SuppBayes}
Section \ref{SMSSSEC:Consider} contains additional considerations on usage of the distributional assumption on $\xi$ as a prior in Bayesian machinery.
Section \ref{SMSSSEC:Uncertain} discusses some ramifications of having, from the Bayesian viewpoint on indeterminacy, alternative scoring distributions.
Section \ref{SMSSSEC:BayesRVs} characterizes the posterior plausibility (of locations of) alternative scoring distributions in case one defines a credibility interval on the marginal posterior standard deviations. 
Section \ref{SMSSSEC:BayesExpandFig} gives, for additional clarity, an expanded visualization of Figure 3 of the MT.
This visualization pertains to indeterminacy in terms of alternative scoring distributions for a single latent $\xi_k$.
Section \ref{SMSSSEC:Concentrate} then visualizes that determinacy implies  concentration of the posterior $\xi|x$ on its location $\xih$.
Quantification of the closeness of (or distance between) alternative scoring distributions is exemplified in Section \ref{SMSSSEC:Overlap}.
The metrics stemming from this quantification can be used to draw analogies between the Bayesian and construction stances on indeterminacy as exemplified in Section \ref{SMSSSEC:BayesASP}.
Lastly, Section \ref{SMSSSEC:ConcentrateEp} considers the Bayesian indeterminacy statement for $\epsilon|x$ and $\eph$.

\subsubsection{Additional Bayesian considerations}
\label{SMSSSEC:Consider}
At this point we may pose the question: May we treat the distributional assumption on $\xi$ as a prior?
We argue that one cannot view $p(\xi)$ as a prior in the usual Bayesian sense.
As $\xi$ is unobservable it can only be known through projection.
As such, $p(\xi)$ is itself decisionally indeterminate as there are infinitely many choices for $p(\xi)$ with associated $p(\xi|x)$ and there exists no empirical information to discriminate these.
Hence, $p(\xi)$ reflects the distribution on which one wishes to project the observables and it is indeed different distributional choices for $p(\xi)$ that underlie the distributional variants of common factor analysis (we will see in Section 3.3 of the MT that, in high-dimension, a distribution-free approach is possible that retrieves the true distribution of the latent source).
As such we call statement (i) from Section 1.1 a postulate rather than a prior.
Nonetheless, it can indeed be used as a prior in Bayesian machinery.

\subsubsection{The ramifications of having alternative scoring distributions}
\label{SMSSSEC:Uncertain}
There are several reasons why, also from a Bayesian view, it matters to have a handle on indeterminacy. 
From the Bayesian perspective this indeterminacy implies the existence of alternative scoring distributions. 

The first reason is that any further analysis that uses scoring according to a choice of scoring distribution ($\xih$ or otherwise) will deal with (unquantified) propagation of uncertainty. 
If there is uncertainty regarding the location $\xih$ there is propagation of uncertainty in any operation that uses $\xih$.
There are several examples.
One would be (generative) deep learning in which we have propagation of uncertainty (indeterminacy) upward in the latent hierarchy (see Sections 4.2.3, 4.4.1, and 4.4.2 of the MT).
Another example would be using $\xih$ in a classification or prediction problem.
Say we consider the following generalized linear model:
\begin{equation*}
	g\{\mathbb{E}(y)\} = \beta_0 + \beta^{\top}\xih,
\end{equation*}
where $g$ represents an appropriate link function.
Any affine transformation of $\xih$ would not alter the fit of this model.
However, if there is uncertainty regarding $\xih$ there is uncertainty regarding the actual fit in case $\xih$ would have been determinate.
Usage of constructions $\xi'$ (as in Section 2.4 of the MT) does not alleviate this problem as
\begin{equation*}
	g\{\mathbb{E}(y)\} = \beta_0 + \beta^{\top}\xi' \neq \beta_0 + \beta^{\top}\xi'',
\end{equation*}
and would introduce dependency on an arbitrary component.

A second reason stems from considering the factor space as a normed space.
A realization near $0$ would then imply average, negative below average, and positive above average performance on the latent trait, respectively. 
When there is uncertainty regarding $\xih$ there is an infinity of alternative scoring distributions and, hence, uncertainty in terms of ordination (the placement of objects, persons, or samples in latent space).
As $\|\xi - \xih\|_2^2$ becomes larger, it is increasingly possible to have alternative scoring distributions whose locations have high posterior plausibility, but whom would simultaneously imply radically differing positioning in latent space (going from positive to negative and vice versa).
See also Figure 3 of the MT.

A third, related reason is that any procedure hinging upon data representation is in need of a single choice of realization. 
That choice can have implications for representational success.
For example, let $\dot{\mathbf{X}}$ be a matrix representing a gray-scale image and assume that we want to compress this image by the factor model.
Clearly, the reconstruction 
\begin{equation*}
	\hat{\dot{\mathbf{X}}} = \bm{1}_n\bar{\boldsymbol{\mathrm{x}}}^{\top} + \doublehat{\mathbf{\Xi}}\hat{\LA}^{\top} \neq \,\,\bm{1}_n\bar{\boldsymbol{\mathrm{x}}}^{\top} + \doublehat{\mathbf{\Xi}}'\hat{\LA}^{\top},
\end{equation*}
where $\doublehat{\mathbf{\Xi}}'$ is the realized factor score matrix based on the posteriorly plausible $\xih'$ (rather than $\xih$).
Which one to choose?

\subsubsection{Characterizing posterior plausibility of alternative scoring distributions}
\label{SMSSSEC:BayesRVs}
Here we characterize the posterior plausibility (of locations of) alternative scoring distributions in case one defines a credibility interval on the marginal posterior standard deviations.
This implies a hyperrectangular region that can be characterized as a product of marginal intervals.
Hence, for each $\xi_k|x$ we are then interested in
\begin{align*}
	\Pr\Big(\xih_k - c\rho(\tau_k,\xi_{k})\, \leq \,\, &\xi_k|x \, \leq \, \xih_k + c\rho(\tau_k,\xi_{k}) \Big) = \\
	&\int_{\xih - c\rho(\tau_k,\xi_{k})}^{\xih + c\rho(\tau_k,\xi_{k})} \frac{1}{\sqrt{2\pi}\rho(\tau_k,\xi_{k})} \exp\Bigg\{-\frac{1}{2}\bigg[\frac{\xi_k|x - \xih_k}{\rho(\tau_k,\xi_{k})}\bigg]^2\Bigg\} \,d\xi_k|x,
\end{align*}
for some $c$.
When we make the change of variable $u = \big(\xi_k|x - \xih_k\big)/\rho(\tau_k,\xi_{k})$ we can express and evaluate the integral as:
\begin{equation*}
	\frac{1}{\sqrt{2\pi}}\int_{-c}^{c}\exp\bigg\{-\frac{1}{2}u^2\bigg\} \,d u = \frac{1}{2} \erf{\bigg(\frac{u}{\sqrt{2}}\bigg)}\Bigg|_{-c}^{c} = 	\erf{\bigg(\frac{c}{\sqrt{2}}\bigg)},
\end{equation*}
where $\erf()$ denotes the Gauss error function.
The last step follows from the property of this function being an odd function.
It is well-known that this function cannot be expressed as elementary functions \citep{ref_Liouville_SM}.
Hence, we need to evaluate $\erf(c/\sqrt{2})$ by numerical approximation to find our desired probability for any finite $c$.
Several such approximations are available.
We use a rational approximation.
If we let $z(y)$ be the approximation to $\erf(y)$ then we can define the relative error
\begin{equation*}
	|\erf(y) - z(y)| \equiv |\zeta(y)|.
\end{equation*}
Then we can approximate \citep[][p.\ 299]{ref_Gautschi_SM}
\begin{equation*}
	\erf(y) \approx 1 - \Big(a_1 t + a_2 t^2 + a_3 t^3 + a_4 t^4 + a_5 t^5\Big)e^{-y^2},
\end{equation*}
where 
\begin{equation*}
	t = \frac{1}{1 + wy},
\end{equation*}
and $w = .3275911$, $a_1 = .254829592$, $a_2 = -.284496736$, $a_3 = 1.421413741$, $a_4 = -1.453152027$, $a_5 = 1.061405429$,
with relative error 
\begin{equation*}
	|\zeta(y)| \leq 1.5 \times 10^{-7}.
\end{equation*}
When $c = 4$ we let $y = 4/\sqrt{2}$ to find, using this approximation, that
\begin{equation*}
	\Pr\Big(\xih_k - 4\rho(\tau_k,\xi_{k})\, \leq \,\, \xi_k|x \, \leq \, \xih_k + 4\rho(\tau_k,\xi_{k}) \Big) \approx .99994.
\end{equation*}
Hence, the posterior plausibility becomes negligible more than 4 standard deviations upward or downward from $\xi_k$.
Using this same approximation one can evaluate that $c = 1.96$ implies that $95\%$ of the posterior probability mass is contained within approximately $1.96$ standard deviations from $\xih_k$ (as used in, for example, Figure 3 of the MT and Figure \ref{SMFIG:BayesIndet2} below).

\subsubsection{Expanded visualization}\label{SMSSSEC:BayesExpandFig}
Figure \ref{SMFIG:BayesIndet2} expands Figure 3 of the MT.
It depicts the Bayesian view on indeterminacy in a single dimension involving $\xi_k$. 
The top panels are also represented in the MT.
The bottom panels represent the same visualization but has, for additional clarity, a shifted rather than overlaying representation of the posterior distribution $\xi_k|x$.
The inverted distribution (in black) is the posterior of $\xi_k|x \sim \mathcal{N}\big[\xih_k,\rho(\tau_k,\xi_k)^2\big]$.
The posterior spread is an expression of the distance between the postulated $\xi_k$ and the best linear predictor $\xih_k$ and represents the uncertainty around the latter.
We depict a situation in which $\rho(\tau_k,\xi_k)^2 = \rho(\xih_k,\xi_k)^2 = .5$.
In the left-hand panels we see two scoring distributions (in red) whose locations are respectively located at the lower and upper bounds of the $95\%$ credible interval for $\xi|x$.
These scoring distributions are distributed as $\xi''_k \sim \mathcal{N}\big[-1.96\rho(\tau_k,\xi_k),\rho(\xih_k,\xi_k)^2\big]$ and $\xi'_k \sim \mathcal{N}\big[1.96\rho(\tau_k,\xi_k),\rho(\xih_k,\xi_k)^2\big]$
(note that we now use primes to indicate locations alternative to $\xih$ rather than the constructions from Section \ref{SMSSEC:Construct}).
These scoring distributions will produce realizations (factor scores) that imply very different positionings in latent space (mostly negative versus mostly positive).
In the right-hand panels we see the scoring distribution with the highest posterior plausibility, $\xih_k \sim \mathcal{N}\big[0,\rho(\xih_k,\xi_k)^2\big]$.
It will score on a narrower domain than the idealized postulate $\xi_k \sim \mathcal{N}(0,1)$ (in green).
As the distribution of $\xi|x$ widens the scoring distributions narrow and the production of very different positions in latent space can already happen for alternative scoring distributions whose locations have high posterior plausibility.
For determinacy one would like $\xi|x$ to concentrate on $\xih$.

\begin{figure}[h!]
	\centering
	\includegraphics[width=.93\textwidth]{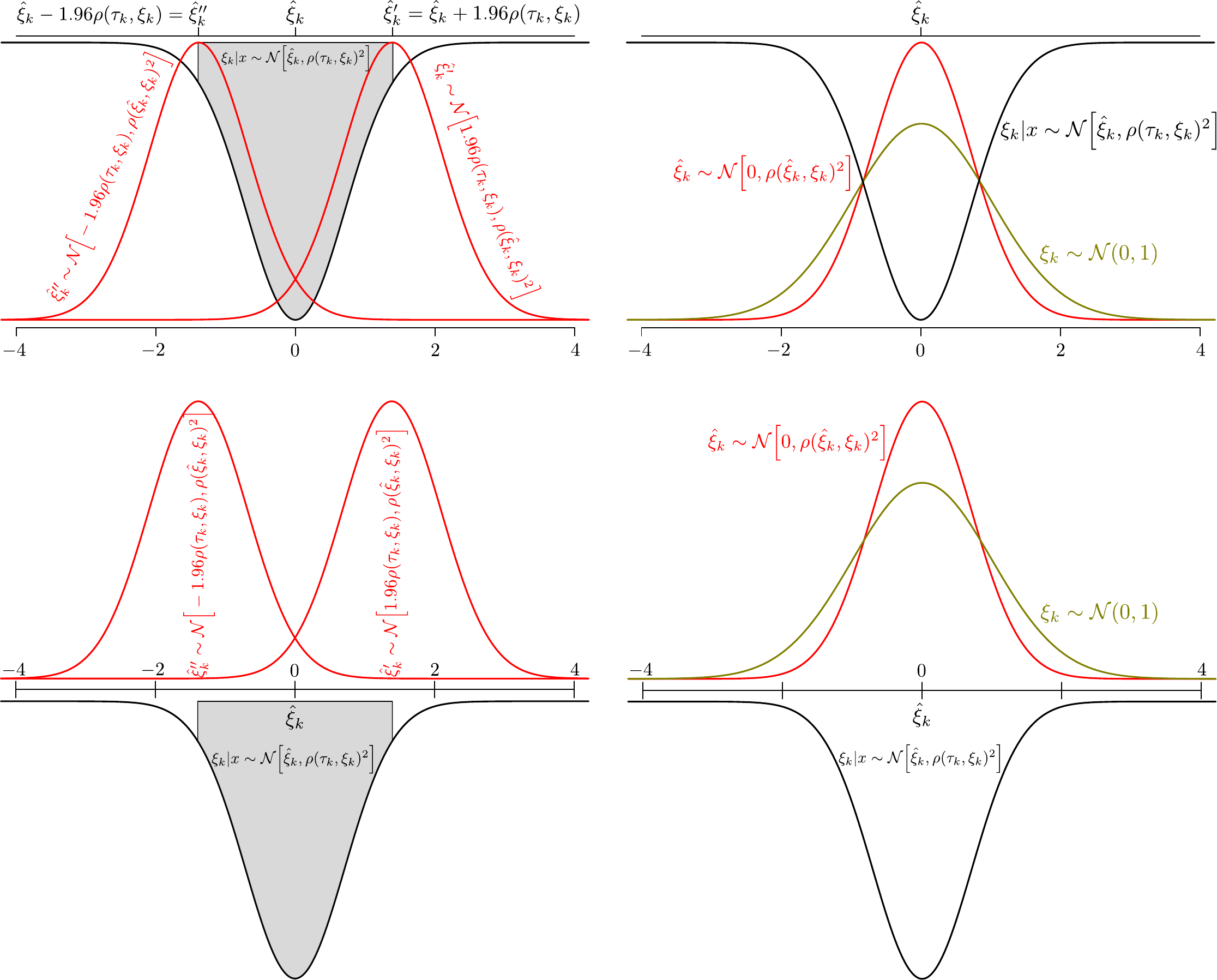}
	\caption{
		Visualization of the Bayesian view on indeterminacy in one dimension involving $\xi_k$.
		This is an expanded version of Figure 3 of the MT.
		For additional clarity, the posterior distribution of $\xi_k|x$ is not only represented as overlaying the scoring distributions (as in the top panels) but also represented in a shifted fashion below the scoring distributions (as in the bottom panels).
	}
	\label{SMFIG:BayesIndet2}
\end{figure}

\subsubsection{Posterior concentration}
\label{SMSSSEC:Concentrate}
There is a duality between $\mathbb{E}\big(\xih\xih^{\top}\big)$ and $\mathbf{\Sigma}(\mathbf{\Theta})_{\xi|x}$ in the sense of $\mathbb{E}\big(\xih\xih^{\top}\big) = \PH - \mathbf{\Sigma}(\mathbf{\Theta})_{\xi|x}$.
Clearly, $\mathbf{\Sigma}(\mathbf{\Theta})_{\xi|x}$ then needs to approximate $\boldsymbol{0}$ for posterior concentration on $\xih$ and for $\xih$ to behave as the postulated $\xi$.
We visualize this posterior concentration in a single dimension involving $\xi_k$ in Figure \ref{SMFIG:BayesIndetOverview}.
In the left-hand panels we see the posterior $p(\xi_k|x)$ (black, inverted) and two scoring distributions (red) with their locations located at the $95\%$ credibility bounds of $p(\xi_k|x)$.
The right-hand panels show again $p(\xi_k|x)$ as well as the scoring distribution at the point of highest posterior plausibility (red) and the behavior of the postulated latent (green).
When $\rho(\tau_k,\xi_{k})^2 = .9$ (upper panels) we have that $p(\xi_k|x) \approx p(\xi_k)$ (we will argue in Sections 4.4.1 and 4.4.2 of the MT that this is directly related to the posterior collapse phenomenon in the variational autoencoder framework for deep generative learning).
We then have narrow distributions with which we can do scoring and we can have alternative scoring distributions at plausible posterior locations that produce wildly different placings in latent space. 
The situation becomes more determinate as $\|\xi_k - \xih_k\|_2^2 = \rho(\tau_k,\xi_{k})^2$ becomes smaller (visualized with, consecutively down from the upper panels, $\rho(\tau_k,\xi_{k})^2 = (.75, .5, .25, .1)$).
We see, as the situation becomes more determinate, that alternative scoring distributions start to overlap and that the scoring distribution at the most plausible posterior location will start to behave as the postulated $\xi_{k}$.
In a determinate situation the posterior would fully concentrate on $\xih_k$.

\begin{figure}[h!]
	\centering
	\includegraphics[width=.975\textwidth]{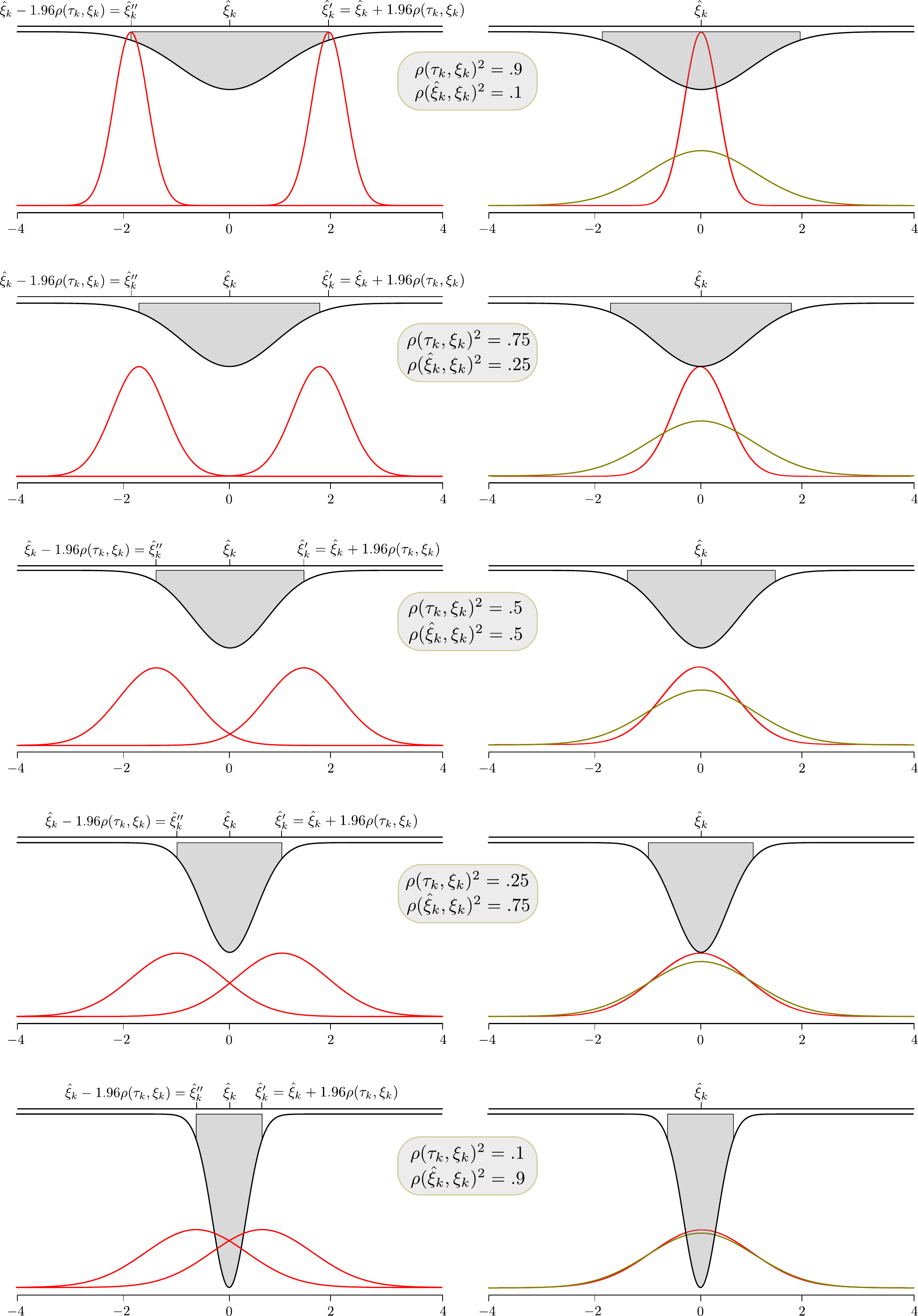}
	\caption{
		Visualization of determinacy through posterior concentration in one dimension involving $\xi_{k}$.
		The setup of each row of panels is analogous to the upper-panels of  Figure \ref{SMFIG:BayesIndet2}. 
		The depth of indeterminacy decreases as $\|\xi_k - \xih_k\|_2^2 = \rho(\tau_k,\xi_{k})^2$ becomes smaller.
		This is visualized with, from to top bottom, $\rho(\tau_k,\xi_{k})^2 = (.9, .75, .5, .25, .1)$.
	}
	\label{SMFIG:BayesIndetOverview}
\end{figure}

\subsubsection{Quantifying closeness of alternative scoring distributions}
\label{SMSSSEC:Overlap}
Here we explore quantifying the closeness of alternative scoring distributions.
One simple closeness measure is the amount of overlap between two probability distributions as probed through an overlapping coefficient \citep{SMref_OLC}.
This coefficient is a function of the Mahalanobis distance\citep{SMref_Maha30, SMref_Maha36}.
We could also use distance measures such as the Hellinger distance \citep{SMref_Hellinger1909} or the Bhattacharyya distance \citep{SMref_Bhatta}.
Both can be used as measures of divergence quantifying the dissimilarity between two probability distributions. 
These distances can, in case we take interest in two alternative scoring distributions that are equidistant from $\xih$, be given simplified expressions as a direct function of the Mahalanobis distance. 

Let us take a look at the univariate situation first.
Say we take interest in the scoring distributions associated with the random variables $\xih_k \pm c\rho(\tau_k,\xi_k)$:
\begin{align*}
	p\Big(\xih_k - c\rho(\tau_k,\xi_k)\Big) &= \mathcal{N}\Big[-c\rho(\tau_k,\xi_k),\rho(\xih_k,\xi_k)^2\Big] \\
	p\Big(\xih_k + c\rho(\tau_k,\xi_k)\Big) &= \mathcal{N}\Big[c\rho(\tau_k,\xi_k),\rho(\xih_k,\xi_k)^2\Big].
\end{align*}
We are then interested in distance quantifications between two alternative scoring distributions whose locations find themselves on equidistant credibility bounds around the point of highest posterior plausibility $\xih_k$. 
Our distances of interest have a closed-form expression in case of two Gaussians \citep{SMref_Kailath}.
For the associated Bhattacharyya distance we then find:
\begin{align*}
	\mathrm{B}_{D}\bigg[p\Big(\xih_k - c\rho(\tau_k,\xi_k)\Big)&,p\Big(\xih_k + c\rho(\tau_k,\xi_k)\Big)\bigg] \\
	&= \frac{1}{4} \frac{\big[-c\rho(\tau_k,\xi_k) - c\rho(\tau_k,\xi_k)\big]^2}{2\rho(\xih_k,\xi_k)^2} + \frac{1}{2}\ln\Bigg[\frac{2\rho(\xih_k,\xi_k)^2}{2\rho(\xih_k,\xi_k)^2}\Bigg] \\
	&= \frac{1}{4}\frac{\big[-2c\rho(\tau_k,\xi_k) \big]^2}{2\rho(\xih_k,\xi_k)^2} + \frac{1}{2}\ln(1) \\
	&= \frac{c^2\rho(\tau_k,\xi_k)^2}{2\rho(\xih_k,\xi_k)^2},
\end{align*}
where we recognize (up to a multiplicative factor) the Mahalanobis distance.
Clearly (using $\mathrm{B}_{D}$ as a shorthand for our Bhattacharyya distance of interest) $0 \leq \mathrm{B}_{D} \leq \infty$.
This distance is not a proper metric as it does not satisfy the triangle inequality \citep{SMref_Kailath}.
The related Hellinger distance
\begin{align*}
	\mathrm{H}_{D}\bigg[p\Big(\xih_k - c\rho(\tau_k,\xi_k)\Big),& \,p\Big(\xih_k + c\rho(\tau_k,\xi_k)\Big)\bigg] = \\
	&\sqrt{1 - \exp\Bigg\{-\mathrm{B}_{D}\bigg[p\Big(\xih_k - c\rho(\tau_k,\xi_k)\Big),p\Big(\xih_k + c\rho(\tau_k,\xi_k)\Big)\bigg]\Bigg\}}
\end{align*}
does obey the triangle inequality.
Clearly, $0 \leq \mathrm{H}_{D} \leq 1$.

We may also find the multivariate analogues for our distances of interest.
Say we are interested in two alternative multivariate scoring distributions simultaneously shifted to be equidistant from $\xih$ according to 
\begin{equation*}
	\xih \pm c\,\mathrm{diag}\Big[\mathbf{\Sigma}(\mathbf{\Theta})_{\xi|x}^{\odot\frac12}\Big].
\end{equation*}
We are then interested in the following two multivariate scoring distributions:
\begin{align*}
	p\Big(\xih -c\,\mathrm{diag}\Big[\mathbf{\Sigma}(\mathbf{\Theta})_{\xi|x}^{\odot\frac12}\Big]\Big) &= \mathcal{N}_m\Bigg[-c\,\mathrm{diag}\Big[\mathbf{\Sigma}(\mathbf{\Theta})_{\xi|x}^{\odot\frac12}\Big],\Big(\PH - \mathbf{\Sigma}(\mathbf{\Theta})_{\xi|x}\Big)^{-1}\Bigg], \\
	p\Big(\xih + c\,\mathrm{diag}\Big[\mathbf{\Sigma}(\mathbf{\Theta})_{\xi|x}^{\odot\frac12}\Big]\Big) &= \mathcal{N}_m\Bigg[c\,\mathrm{diag}\Big[\mathbf{\Sigma}(\mathbf{\Theta})_{\xi|x}^{\odot\frac12}\Big],\Big(\PH - \mathbf{\Sigma}(\mathbf{\Theta})_{\xi|x}\Big)^{-1}\Bigg].
\end{align*}
Using the closed-form expression in case of two multivariate Gaussians \citep{SMref_Kailath} and some ready algebra we find
\begin{align*}
	\mathrm{B}_{D}\bigg[p\Big(\xih -c\,\mathrm{diag}\Big[\mathbf{\Sigma}(\mathbf{\Theta})_{\xi|x}^{\odot\frac12}\Big]\Big),& \,p\Big(\xih + c\,\mathrm{diag}\Big[\mathbf{\Sigma}(\mathbf{\Theta})_{\xi|x}^{\odot\frac12}\Big]\Big)\bigg] = \\
	&
	\frac{c^2}{2}\mathrm{diag}\Big[\mathbf{\Sigma}(\mathbf{\Theta})_{\xi|x}^{\odot\frac12}\Big]^{\top} \Big(\PH - \mathbf{\Sigma}(\mathbf{\Theta})_{\xi|x}\Big)^{-1}\mathrm{diag}\Big[\mathbf{\Sigma}(\mathbf{\Theta})_{\xi|x}^{\odot\frac12}\Big],
\end{align*}
for the multivariate Bhattacharyya distance and
\begin{align*}
	\mathrm{H}_{D}\bigg[p\Big(\xih \,-\, &c\,\mathrm{diag}\Big[\mathbf{\Sigma}(\mathbf{\Theta})_{\xi|x}^{\odot\frac12}\Big]\Big), \,p\Big(\xih + c\,\mathrm{diag}\Big[\mathbf{\Sigma}(\mathbf{\Theta})_{\xi|x}^{\odot\frac12}\Big]\Big)\bigg] = \\
	&\sqrt{1 - \exp\Bigg\{-\mathrm{B}_{D}\bigg[p\Big(\xih -c\,\mathrm{diag}\Big[\mathbf{\Sigma}(\mathbf{\Theta})_{\xi|x}^{\odot\frac12}\Big]\Big), p\Big(\xih + c\,\mathrm{diag}\Big[\mathbf{\Sigma}(\mathbf{\Theta})_{\xi|x}^{\odot\frac12}\Big]\Big)\bigg]\Bigg\}}
\end{align*}
for the multivariate Hellinger distance.

We note that these quantities can be used as alternative metrics for the quantification of indeterminacy.
We now exemplify these metrics with an example.

\begin{example}
	We consider the situation depicted in the left-hand panels of Figure \ref{SMFIG:BayesIndet2} where we have a single latent factor $\xi_{k}$ and $\rho(\tau_k,\xi_k)^2 = \rho(\xih_k,\xi_k)^2 = .5$.
	We take interest in two scoring distributions whose locations are respectively located at the lower and upper bounds of the $95\%$ credible interval for $\xi|x$.
	These scoring distributions are then distributed as $\xi''_k \sim \mathcal{N}\big[-1.96\rho(\tau_k,\xi_k),\rho(\xih_k,\xi_k)^2\big]$ and $\xi'_k \sim \mathcal{N}\big[1.96\rho(\tau_k,\xi_k),\rho(\xih_k,\xi_k)^2\big]$.
	These scoring distributions will produce realizations (factor scores) with little overlap and, hence, imply very different positionings in latent space (mostly negative versus mostly positive).
	This will also be reflected in various overlap of closeness metrics.
	The upper grey area in Figure \ref{SMFIG:BayesOC} represents the overlap.
	The corresponding overlapping coefficient amounts to $.05$ approximately.
	For our distances of interest we find (using shorthand notation): $\mathrm{B}_D \approx 1.92$ and $\mathrm{H}_D \approx .924$.
	As we would look at wider intervals or when $\rho(\tau_k,\xi_k)^2$ increases towards $1$ the overlapping coefficient will go to $0$, the Hellinger distance will go to $1$, while the Bhattacharyya distance will increase (without bound).
	\begin{figure}[h!]
		\centering
		\includegraphics[width=.6\textwidth]{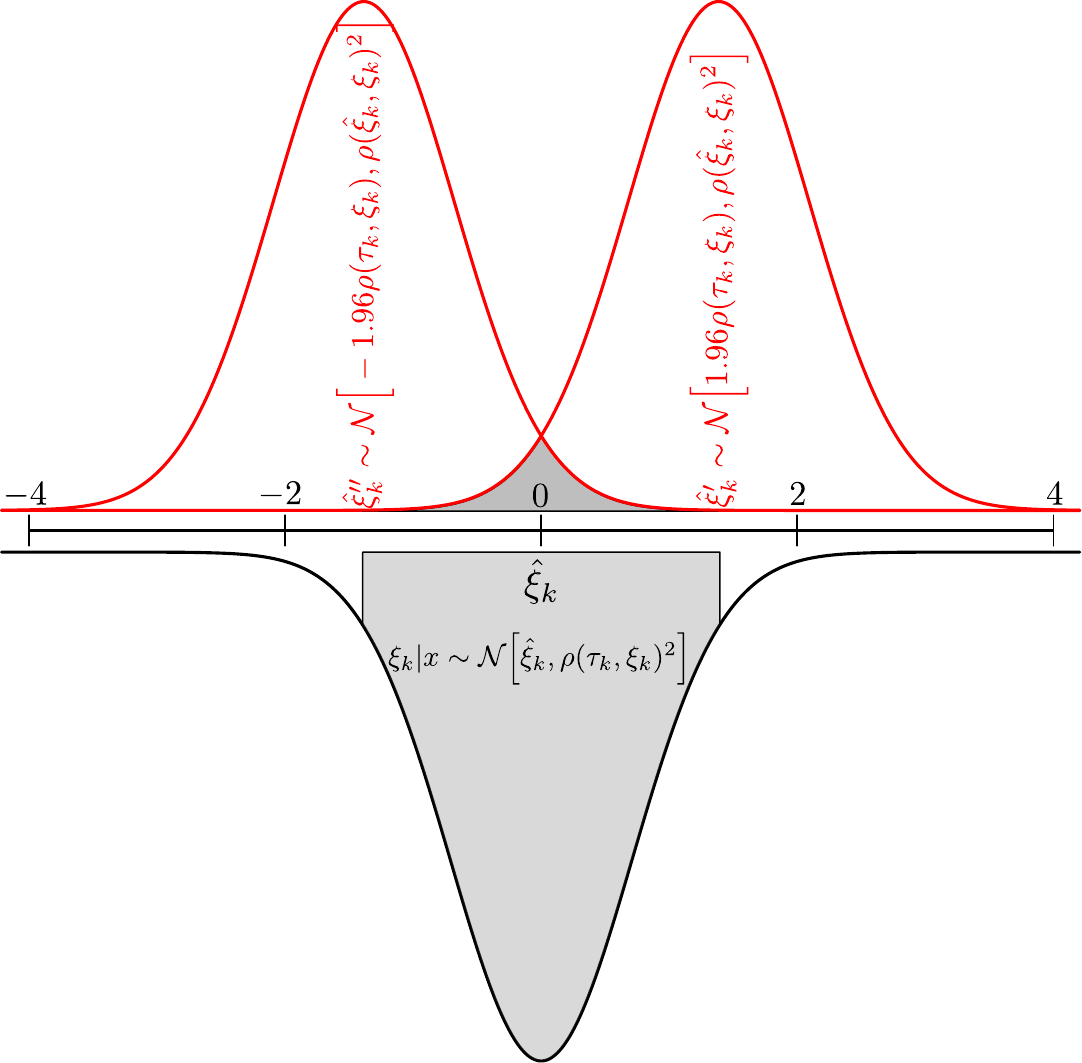}
		\caption{
			Visualizing closeness quantification for two alternative scoring distributions in the Bayesian view on indeterminacy.
			The grey area in the inverted distribution represents the $95\%$ credible interval for the posterior distribution of $\xi_k|x$.
			The two scoring distributions (in red) have their locations respectively located at the lower and upper bounds of this credible interval.
			The grey area in between the scoring distributions represents their overlap.
		}
		\label{SMFIG:BayesOC}
	\end{figure}
	\QEDE
\end{example}

\subsubsection{Analogues between the Bayesian posterior and construction positions}\label{SMSSSEC:BayesASP}
The construction (Section 2.4 of the MT) and Bayesian (Section 3.1 of the MT) stances on indeterminacy are often juxtaposed in the Psychometric literature \citep[see, e.g.,][and the associated discussion papers]{SMref_Maraun1996}.
Note, however, that we give a random variable treatment of the construction stance.
And while the Bayesian view does not involve rotation such that the Guttman criterion is not an appropriate measure of indeterminacy from the Bayesian stance, both the construction and Bayesian stances hinge upon $\mathbf{\Sigma}(\mathbf{\Theta})_{\xi|x}$.
Hence, they should not be seen as fundamentally different.
Indeed, we can draw analogues between the Guttman criterion and Bayesian quantification of alternative scoring distributions as the following example will clarify.

\begin{example}
	Consider again the situation in a single latent dimension for which $\rho(\tau_k,\xi_k)^2 = \rho(\xih_k,\xi_k)^2 = .5$ (as in Example \ref{SMEX:Gman}).
	In this case $\gamma_k = 2\rho(\xih_k,\xi_k)^2 - 1 = 0$, representing an angle of $\arccos(0) \cdot 180/\pi =  90^\circ$.
	This indicates maximally different constructions (in the sense of Section 2.4 of the MT) that are uncorrelated.
	The Bayesian analogy would then be to assess (the distance between or overlap in) two maximally different scoring distributions, who have their locations (equidistantly from $\xih_k$) at the boundaries of the posterior plausibility of $\xi_{k}|x$.
	That is, the scoring distributions associated with $\xih_k \pm 4\rho(\tau_k,\xi_k)$.
	Figure \ref{SMFIG:BayesASP} depicts this situation.
	The overlapping coefficient then indeed goes to $0$ ($6.334 \times 10^{-5}$ with absolute error $< 2.1 \times 10^{-7}$), $\mathrm{H}_D$ goes to $1 \,(.9998)$, and $\mathrm{B}_D = 8$.
	As the situation would become more indeterminate, in the sense of $\rho(\tau_k,\xi_k)^2/\rho(\xih_k,\xi_k)^2 > 1$, the overlapping coefficient and  $\mathrm{H}_D$ would remain (approximately) $0$ and $1$, respectively, while $\mathrm{B}_D$ would increase (without bound).
	From this view one could say that the Guttman criterion has low support from a Bayesian plausibility perspective, crystallizing what McDonald \citep{SMref_McDon} alluded to, but did not articulate fully \citep[cf.][]{SMref_Mulaik1976}.
	\begin{figure}[h!]
		\centering
		\includegraphics[width=.7\textwidth]{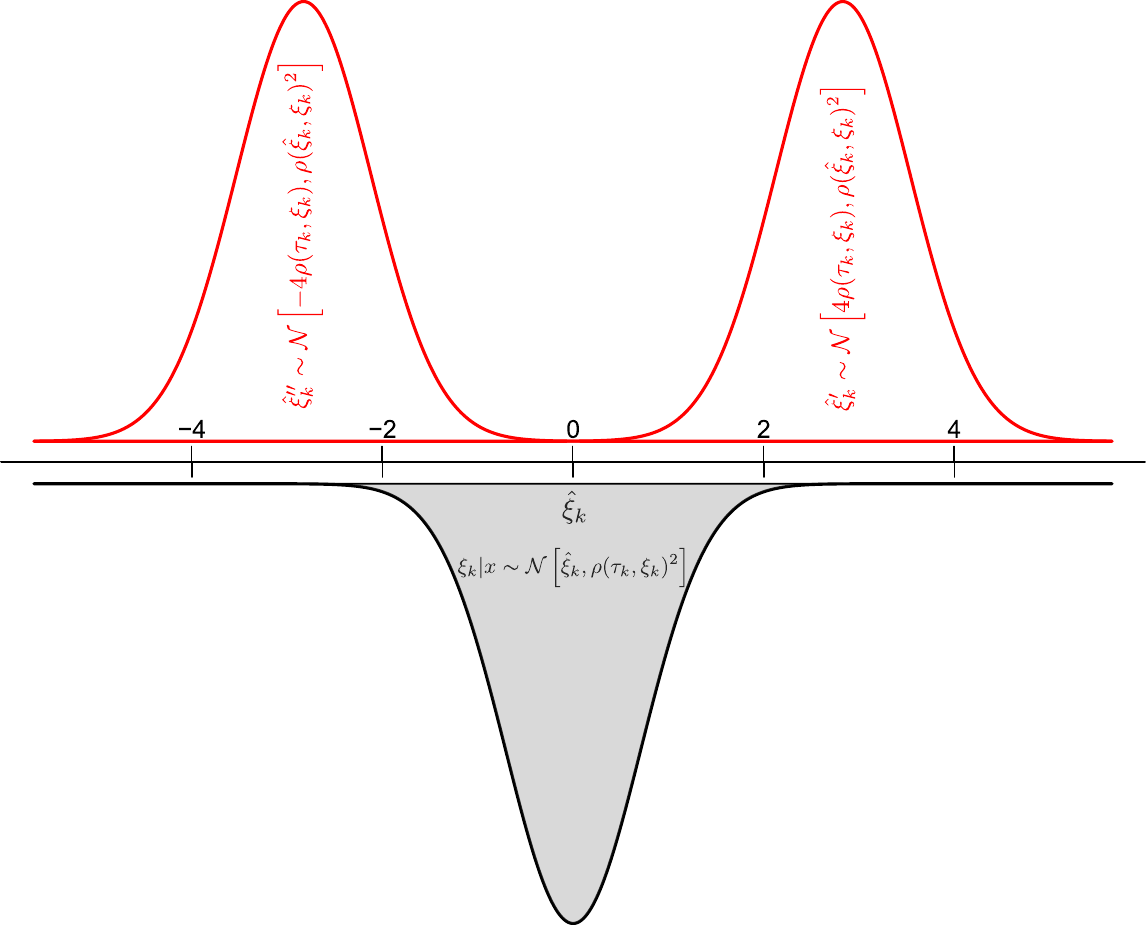}
		\caption{
			Visualization of the Bayesian analogue to the Guttman criterion.
			Scoring distributions are then assessed that are located at the boundaries of the posterior plausibility of $\xi|x$.
		}
		\label{SMFIG:BayesASP}
	\end{figure}
	\QEDE
\end{example}

\subsubsection{Posterior concentration of $\epsilon|x$}
\label{SMSSSEC:ConcentrateEp}
We will also consider the Bayesian indeterminacy statement for $\epsilon|x$ and $\eph$.
This is completely analogous to Indeterminacy Statement 8 from the MT.
By Lemma 2 from the MT we have that $\mathbf{\Sigma}(\mathbf{\Theta})_{\epsilon|x} = \LA\mathbf{\Sigma}(\mathbf{\Theta})_{\xi|x}\LAt$.
Hence, when $\mathbf{\Sigma}(\mathbf{\Theta})_{\xi|x} \neq \boldsymbol{0}$, 
\begin{equation*}
	\epsilon|x \nsim \lim\limits_{\mathbf{\Sigma}(\mathbf{\Theta})_{\epsilon|x} \rightarrow \boldsymbol{0}} \mathcal{N}_{p}\Big[\eph, \mathbf{\Sigma}(\mathbf{\Theta})_{\epsilon|x}\Big] = \delta\Big(\epsilon|x - \eph\Big),
\end{equation*}
and 
\begin{equation*}
	\eph \nsim \mathcal{N}_{p}(\boldsymbol{0}, \PS).
\end{equation*}
That is, the conditional distribution $\epsilon|x$ is not a Dirac measure concentrated on the point of highest posterior plausibility $\eph$.
As a result, $\eph$ does not have the distributional behavior of the postulated $\epsilon$.

%%%%%%%%%%%%%%%%%%%%%%%%%%%%
%%%-Details Equalities Eps -
%%%%%%%%%%%%%%%%%%%%%%%%%%%%
\subsection{Details for Section 3.2 of the Main Text}
\label{SMSSEC:SuppEqualitiesEpsilon}
The determinacy results for $\xih$ have their counterpart in $\eph$.
These are explicated in Section \ref{SMSSSEC:DetEp}.
An interesting point about factor model collapse is to be made if we assume $p = \infty$.
This is explicated in Section \ref{SMSSSEC:collapse}.

\subsubsection{Determinacy results for the unique factors}
\label{SMSSSEC:DetEp}
Here we show the determinacy results under $p\uparrow\infty$ for proxy $\eph$
of the unique factors $\epsilon$.
These concern Remarks 1--3 and 5 of the MT.
They are collected in the following corollary.
Below, we let $\lim_{p\,\uparrow\,\infty} \mathbf{B}_p = \boldsymbol{0}$ be shorthand notation for general entrywise convergence to the null matrix, 
allowing us to avoid some notational clutter.

\begin{corollary}[Determinacy in the limit for the unique factors]
	\label{CORROL:EpsilonLimits}
	Let the setting and assumptions stated in Theorem 1 and Corollary 1 of the MT hold. 
	Let $\epsilon|x_p$, $\eph_p$, and $\mathbf{\Sigma}(\mathbf{\Theta})_{\epsilon|x_p}$ denote the feature-dimension dependent sequence of $\epsilon|x$, $\eph$, and $\mathbf{\Sigma}(\mathbf{\Theta})_{\epsilon|x}$.
	Then the following convergences hold:
	\begin{enumerate}
		\item[~~i.] $\lim\limits_{p\,\uparrow\,\infty} \big\|\mathbf{I}_p - \PS^{-1/2}\mathbb{E}\big(\eph_p\eph_p^{\top})\PS^{-1/2}\big\|_{\infty,\infty} = 0$;
		\item[~ii.] $\lim\limits_{p\,\uparrow\,\infty} \upsilon_j = 1, \,\,\,\,\forall j$;
		\item[iii.] $\lim\limits_{p\,\uparrow\,\infty} \tr\big[\mathbf{\Sigma}(\mathbf{\Theta})_{\epsilon|x_p}\big] = 0$;
		\item[\,iv.] $\lim\limits_{p\,\uparrow\,\infty} \epsilon|x_p \overset{\Pr}{\longrightarrow} \delta\Big(\epsilon|x_p - \eph_p\Big)$; and
		\item[~\,v.] $\lim\limits_{p\,\uparrow\,\infty} \eph_p \overset{\Pr}{\longrightarrow} \epsilon$.
	\end{enumerate}
	The condition that $p\uparrow\infty$ is both necessary and sufficient.
\end{corollary}

\begin{proof}[Proof of Corollary \ref{CORROL:EpsilonLimits}]
	We will visit of each of the statements in turn.
	For Corollary \ref{CORROL:EpsilonLimits}.i we first note that
	\begin{align}\label{SMEQ:Eplim}\nonumber
		\PS^{-1/2}\mathbb{E}\big(\eph_p\eph_p^{\top})\PS^{-1/2} &= \PS^{-1/2}\Big[\PS\big(\Com + \PS\big)^{-1}\PS\Big]\PS^{-1/2} \\\nonumber
		&= \PS^{-1/2}\Big[\PS - \mathbf{\Sigma}(\mathbf{\Theta})_{\epsilon|x_p}\Big]\PS^{-1/2} \\\nonumber
		&= \PS^{-1/2}\Big[\PS - \LA\mathbf{\Sigma}(\mathbf{\Theta})_{\xi|x_p}\LAt\Big]\PS^{-1/2} \\\nonumber
		&= \mathbf{I}_p - \PS^{-1/2}\LA\mathbf{\Sigma}(\mathbf{\Theta})_{\xi|x_p}\LAt\PS^{-1/2},
	\end{align}
	by direct application of Lemma 2 of the MT and results from Section \ref{SMSSSEC:Geom_unique}.
	Hence, we have the following equivalence:
	\begin{equation*}
		\big\|\mathbf{I}_p - \PS^{-1/2}\mathbb{E}\big(\eph_p\eph_p^{\top})\PS^{-1/2}\big\|_{\infty,\infty} \,=\,
		\big\|\PS^{-1/2}\LA\mathbf{\Sigma}(\mathbf{\Theta})_{\xi|x_p}\LAt\PS^{-1/2}\big\|_{\infty,\infty}.
	\end{equation*}
	Define $\PS^{-1/2}\LA\mathbf{\Sigma}(\mathbf{\Theta})_{\xi|x_p}\LAt\PS^{-1/2} \equiv \Delta_{\mathbf{I}}$ and $\PS^{-1/2}\LA \equiv \LA_{\PS^{-1/2}}$ and let $(\LA_{\PS^{-1/2}})_{j,\bullet}$ denote the $j$th row of $\LA_{\PS^{-1/2}}$.
	By the similarity invariance of analytic matrix functions the (block) zero structure of $\PS$ is preserved in $\PS^{-1/2}$.
	Thus $(\LA_{\PS^{-1/2}})_{jk}$ is bounded for all $j,k$ and as a result $\|(\LA_{\PS^{-1/2}})_{j,\bullet}\|_2$ is bounded for all $j$.
	Extending the development in the proof of Corollary 1.iii of the MT then gives that 
	\begin{equation*}
		\lim\limits_{p\,\uparrow\,\infty} \max_{j,j'}\big|(\Delta_{\mathbf{I}})_{jj'}\big| \leq \Big(\max_{j} \big\|(\LA_{\PS^{-1/2}})_{j,\bullet}\big\|_2\Big)^2 \lim\limits_{p\,\uparrow\,\infty} \Big\|\big(\PHi + \Gram\big)^{-1}\Big\|_2 = \,0,
	\end{equation*}
	as a direct consequence of Theorem 1 of the MT.
	
	For Corollary \ref{CORROL:EpsilonLimits}.ii we recall from Section \ref{SMSSSEC:Geom_unique} that, for each $j$,
	\begin{equation*}
		\upsilon_j = \bigg\{\PS^{-1/2}\Big[\PS - 2\mathbf{\Sigma}(\mathbf{\Theta})_{\epsilon|x}\Big]\PS^{-1/2}\bigg\}_{jj} = \Big[\mathbf{I}_p - 2\PS^{-1/2}\mathbf{\Sigma}(\mathbf{\Theta})_{\epsilon|x}\PS^{-1/2}\Big]_{jj}.
	\end{equation*}
	One thus desires entrywise convergence of $\mathbf{I}_p - 2\PS^{-1/2}\mathbf{\Sigma}(\mathbf{\Theta})_{\epsilon|x_p}\PS^{-1/2}$ to $\mathbf{I}_p$.
	Hence, we need to assess
	\begin{equation*}
		\lim\limits_{p\,\uparrow\,\infty} \Big\|\mathbf{I}_p - \Big[\mathbf{I}_p - 2\PS^{-1/2}\mathbf{\Sigma}(\mathbf{\Theta})_{\epsilon|x_p}\PS^{-1/2}\Big]
		\Big\|_{\infty,\infty} \,
		= \,
		\lim\limits_{p\,\uparrow\,\infty} \Big\| 2\PS^{-1/2}\mathbf{\Sigma}(\mathbf{\Theta})_{\epsilon|x_p}\PS^{-1/2}
		\Big\|_{\infty,\infty}.
	\end{equation*}
	Then, by application of Lemma 2 of the MT, the developments in the proof of Corollary \ref{CORROL:EpsilonLimits}.i, and Theorem 1 of the MT,
	\begin{align*}
		\lim\limits_{p\,\uparrow\,\infty} \Big\| 2\PS^{-1/2}\mathbf{\Sigma}(\mathbf{\Theta})_{\epsilon|x_p}\PS^{-1/2}
		\Big\|_{\infty,\infty} \, 
		&= 	\,\lim\limits_{p\,\uparrow\,\infty} \Big\| 	2\PS^{-1/2} \LA\mathbf{\Sigma}(\mathbf{\Theta})_{\xi|x_p}\LAt\PS^{-1/2}
		\Big\|_{\infty,\infty} \\
		&= \,2\lim\limits_{p\,\uparrow\,\infty} \max_{j,j'}\big|(\Delta_{\mathbf{I}})_{jj'}\big| \\
		&\leq \,2\,\Big(\max_{j} \big\|(\LA_{\PS^{-1/2}})_{j,\bullet}\big\|_2\Big)^2 \lim\limits_{p\,\uparrow\,\infty} \big\|\mathbf{\Sigma}(\mathbf{\Theta})_{\xi|x_p}\big\|_2 \\
		&= \,0.
	\end{align*}
	As a direct consequence, $\lim_{p\,\uparrow\,\infty} \upsilon_j = 1$ for all $j$.
	
	For Corollary \ref{CORROL:EpsilonLimits}.iii we can write, by application of Lemma 2 and Corollary 1.ii of the MT,
	\begin{align*}
		\lim\limits_{p\,\uparrow\,\infty} \tr\big[\mathbf{\Sigma}(\mathbf{\Theta})_{\epsilon|x_p}\big] &=  \lim\limits_{p\,\uparrow\,\infty} \tr \Big[ \LA\mathbf{\Sigma}(\mathbf{\Theta})_{\xi|x_p}\LAt\Big] \\
		&= \tr \bigg[\lim\limits_{p\,\uparrow\,\infty} \LA\big(\PHi + \Gram\big)^{-1}\LAt\bigg] \\
		&= 0.
	\end{align*}
	
	For Corollary \ref{CORROL:EpsilonLimits}.iv we have, utilizing mean-square convergence, that
	\begin{align*}
		\lim\limits_{p\,\uparrow\,\infty} \mathbb{E} \Big[\big\|\epsilon|x_p - \eph_p\big\|_2^2\Big] &= \lim\limits_{p\,\uparrow\,\infty} \mathbb{E} \Big[\big\|\epsilon|x_p - \mathbb{E}(\epsilon|x_p)\big\|_2^2\Big]\\
		&= \lim\limits_{p\,\uparrow\,\infty} \mathbb{E} \Big[\tr\mathbf{\Sigma}(\mathbf{\Theta})_{\epsilon|x_p}\Big]\\
		&= \lim\limits_{p\,\uparrow\,\infty} \mathbb{E}\bigg\{\tr \Big[ \LA\mathbf{\Sigma}(\mathbf{\Theta})_{\xi|x_p}\LAt\Big]\bigg\} \\
		&= \tr \bigg[\lim\limits_{p\,\uparrow\,\infty} \LA\big(\PHi + \Gram\big)^{-1}\LAt\bigg] \\
		&= 0.
	\end{align*}
	Hence, $\lim_{p\,\uparrow\,\infty} \epsilon|x_p \overset{\ell_2}{\longrightarrow} \delta(\epsilon|x_p - \eph_p)$ and thus $\lim_{p\,\uparrow\,\infty} \epsilon|x_p \overset{\Pr}{\longrightarrow} \delta(\epsilon|x_p - \eph_p)$.
	
	For Corollary \ref{CORROL:EpsilonLimits}.v we have, using the results above,
	\begin{align*}
		\lim\limits_{p\,\uparrow\,\infty} \mathbb{E} \Big[\big\|\eph_p - \epsilon\big\|_2^2\Big] &=
		\lim\limits_{p\,\uparrow\,\infty} \mathbb{E} \Big\{\tr\mathbb{E}\Big[\big(\eph_p - \epsilon\big)\big(\eph_p - \epsilon\big)^{\top}\Big]\Big\}\\
		&= \lim\limits_{p\,\uparrow\,\infty} \mathbb{E} \Big[\tr\mathbf{\Sigma}(\mathbf{\Theta})_{\epsilon|x_p}\Big] \\
		&= 0.
	\end{align*}
	Hence, $\lim_{p\,\uparrow\,\infty} \eph_p \overset{\ell_2}{\longrightarrow} \epsilon$ and thus $\lim_{p\,\uparrow\,\infty} \eph_p \overset{\Pr}{\longrightarrow} \epsilon$.
	
	Sufficiency and necessity of $p\uparrow\infty$ for Corollary \ref{CORROL:EpsilonLimits}.i--v is directly inherited from Theorem 1 of the MT.
	The result follows.
\end{proof}

The interpretation of these results in analogous to the interpretation of Corollaries 2 and 3 of the MT.
We can carry these results over to the correlations between constructions of $\xi$ and $\epsilon$ maximally different from $0$ (see Section \ref{SMSSSEC:MinCorXiEp}), as the following corollary indicates.

\begin{corollary}[Vanishing correlations between maximally different constructions of $\xi$ and $\epsilon$]
	\label{CORROL:EpsilonLimits2}
	Let the setting and assumptions stated in Theorem 1 and Corollary 1 of the MT hold. 
	Then $p\uparrow\infty$ is necessary and sufficient for the following limit to hold:
	\begin{equation*}
		\lim\limits_{p\,\uparrow\,\infty} \kappa_{kj} = 0, \,\,\,\,\forall k,j.
	\end{equation*}
\end{corollary}

\begin{proof}[Proof of Corollary \ref{CORROL:EpsilonLimits2}]
	By the results in Section \ref{SMSSSEC:MinCorXiEp} we recall
	\begin{equation*}
		\kappa_{kj} = \Big[2\mathbf{\Sigma}(\mathbf{\Theta})_{\xi|x}\LAt\PS^{-1/2}\Big]_{kj}.
	\end{equation*}
	We thus desire $2\mathbf{\Sigma}(\mathbf{\Theta})_{\xi|x_p}\LAt\PS^{-1/2}$ to converge entrywise to the null matrix.
	We note that $\mathbf{\Sigma}(\mathbf{\Theta})_{\xi|x}\LAt\PS^{-1/2}$ is equivalent to $\mathbf{I}_m\mathbf{\Sigma}(\mathbf{\Theta})_{\xi|x}\LAt\PS^{-1/2}$, allowing us to express the individual elements of $\mathbf{\Sigma}(\mathbf{\Theta})_{\xi|x}\LAt\PS^{-1/2}$ in bilinear form:
	\begin{equation*}
		\Big[\mathbf{\Sigma}(\mathbf{\Theta})_{\xi|x}\LAt\PS^{-1/2}\Big]_{kj} = (\mathbf{I}_m)_{k,\bullet}\mathbf{\Sigma}(\mathbf{\Theta})_{\xi|x}\big(\PS^{-1/2}\LA\big)_{j,\bullet}^{\top},
	\end{equation*}
	where $(\mathbf{I}_m)_{k,\bullet}$ denotes the $k$th row of $\mathbf{I}_m$ and $(\PS^{-1/2}\LA)_{j,\bullet} \equiv (\LA_{\PS^{-1/2}})_{j,\bullet} $ denotes the $j$th row of $\LA_{\PS^{-1/2}}$.
	Then, by the bilinear inequality,
	\begin{equation*}
		\Bigg|\Big[\mathbf{\Sigma}(\mathbf{\Theta})_{\xi|x}\LA_{\PS^{-1/2}}^{\top}\Big]_{kj}\Bigg|\leq \big\|(\mathbf{I}_m)_{k,\bullet}\big\|_2\,\big\|\mathbf{\Sigma}(\mathbf{\Theta})_{\xi|x}\big\|_2\,\big\|(\LA_{\PS^{-1/2}})_{j,\bullet}\big\|_2.
	\end{equation*}
	Now, $\|(\mathbf{I}_m)_{k,\bullet}\|_2 = 1$ for all $k$ and, by previous developments, $\|(\LA_{\PS^{-1/2}})_{j,\bullet}\|_2$ implies a bounded $m$-sum for all $j$.
	Hence,
	\begin{align*}
		\lim\limits_{p\,\uparrow\,\infty} \Big\|2\mathbf{\Sigma}(\mathbf{\Theta})_{\xi|x_p}\LAt\PS^{-1/2}\Big\|_{\infty,\infty}
		&= \,\lim\limits_{p\,\uparrow\,\infty} \max_{k,j}	\Bigg|2\Big[\mathbf{\Sigma}(\mathbf{\Theta})_{\xi|x_p}\LA_{\PS^{-1/2}}^{\top}\Big]_{kj}\Bigg| \\
		&\leq \, 2\,\bigg(\lim\limits_{p\,\uparrow\,\infty} \Big\|\big(\PHi + \Gram\big)^{-1}\Big\|_2\bigg) \max_{j}\big\|(\LA_{\PS^{-1/2}})_{j,\bullet}\big\|_2\\
		&= \,0,
	\end{align*}
	as required.
	As a direct consequence, $\lim_{p\,\uparrow\,\infty} \kappa_{kj} = 0$ for all $k,j$.
	Sufficiency and necessity of $p\uparrow\infty$ follows directly from Theorem 1 of the MT.
\end{proof}

\subsubsection{Factor model collapse}
\label{SMSSSEC:collapse}
An interesting point is to be made if we ignore the puncture at $\infty$ and assume $p = \infty$. 
This collapses the factor model in the spectrally strong delocalization regime.
For example, at perfect conical contraction the hypercone (Section 2.5 of the MT) collapses onto the origin.
The following corollary represents the main issue.

\begin{corollary}[Factor model collapse under spectrally strong delocalization]
	\label{CORROL:Collapse}
	Let Theorem 1 of the MT hold.
	If we ignore the puncture at $\infty$ and assume $p = \infty$ then $\xih_p$ is no longer a random variable.
\end{corollary}

\begin{proof}[Proof of Corollary \ref{CORROL:Collapse}]
	At $p = \infty$ we have that $\xih = (\PHi + \Gram)^{-1}\LAt\PSi x$ no longer implies infinite sums with infinitesimally small elements as $(\PHi + \Gram)^{-1}$ is exactly $\boldsymbol{0}$. 
	Hence, all scorings are then necessarily always $0$.
\end{proof}

%%%%%%%%%%%%%%%%%%%%%%%%%%%%
%%%-Details Algorithm ------
%%%%%%%%%%%%%%%%%%%%%%%%%%%%
\subsection{Details for Section 3.3 of the Main Text}
\label{SMSSEC:AlgoExample}
Section \ref{SMSSSEC:Equivalence} gives additional detail on rotational equivalence in the formulation of the factor analytic correlation matrix.
Section \ref{SMSSSEC:Canonical} details how the canonical representation is  derived.
Section \ref{SMSSSEC:canonicality} discusses the various ways in which this representation can be understood to be `canonical'.
Section \ref{SMSSSEC:geneig} details how the canonical matrices are connected to a particular generalized eigenvalue problem.
The canonical representation gives a system or set of estimating equations.
Section \ref{SMSSSEC:equivariance} shows that these estimating equations are scale-equivariant. 
Section \ref{SMSSSEC:Rotation} explains the use of orthogonal and oblique simple structure rotations of the loadings matrix as obtained through canonical estimation in order to enhance the interpretation of the latent representation.
Lastly, Section \ref{SMSSSEC:AlgoExample} supports, with examples, the assertions made with respect to Algorithm 1 of the MT.

\subsubsection{Equivalent formulations of the model-implied correlation matrix}
\label{SMSSSEC:Equivalence}
\begin{sloppypar}
	Each oblique model representation has equivalent orthogonal representation and vice versa. 
	Here we explicate the equivalent representations for the model-implied correlation matrix.
	Using ready covariance algebra on the equivalent model formulation 	
	$x:=\LA\mathbf{K}^{-1}\HDot\HDot^{\top}\mathbf{K}\xi + \epsilon$ we obtain
	\begin{equation}
		\label{SMEQ:equivv}
		\mathbf{\Sigma}(\mathbf{\Theta})_{xx} = \LA \mathbf{K}^{-1} \HDot\HDot^{\top} \mathbf{K}\PH\mathbf{K}^{\top} \HDot\HDot^{\top} \mathbf{K}^{-\top} \LAt + \PS.
	\end{equation}
	We can then express (\ref{SMEQ:equivv}) in the original oblique formulation $\Com + \PS$ as $\mathbf{K}\PH\mathbf{K}^{\top} = \mathbf{K}\mathbf{K}^{-1}\mathbf{K}^{-\top}\mathbf{K}^{\top} = \mathbf{I}_m$, $\HDot\HDot^{\top} \mathbf{K}\PH\mathbf{K}^{\top} \HDot\HDot^{\top} = \mathbf{I}_m$ and $\mathbf{K}^{-1}\mathbf{K}^{-\top} = \PH$ or as $\mathbf{K}^{-1} \HDot\HDot^{\top} \mathbf{K} = \mathbf{K}^{\top} \HDot\HDot^{\top} \mathbf{K}^{-\top} = \mathbf{I}_m$.
	Equivalently, (\ref{SMEQ:equivv}) can then also be expressed as $(\LA \mathbf{K}^{-1})(\LA \mathbf{K}^{-1})^{\top} + \PS = \lambdaDot\lambdaDot^{\top} + \PS$, which is an orthogonal representation.
	The orthogonal representation may receive arbitrary orthogonal rotations such that (\ref{SMEQ:equivv}) can also be equivalently expressed as $\lambdaDot\lambdaDot^{\top} + \PS = (\lambdaDot\HDot)(\lambdaDot\HDot)^{\top} + \PS$, because $\HDot\HDot^{\top} = \mathbf{I}_m$.
	The latter equivalency can also be directly seen from (\ref{SMEQ:equivv}) als also $\HDot^{\top} \mathbf{K}\PH\mathbf{K}^{\top} \HDot = \mathbf{I}_m$.
\end{sloppypar}

\subsubsection{Canonical approach}
\label{SMSSSEC:Canonical}
We have noted that, if the factor model holds, 
\begin{equation*}
	\mathbf{\Sigma}_{xx} = 	\mathbf{\Sigma}(\mathbf{\Theta})_{xx} = \lambdaDot\lambdaDot^{\top} + \PS,
\end{equation*}
from which we obtain the implication 
\begin{equation*}
	\mathbf{\Sigma}_{xx} - \PS = \lambdaDot\lambdaDot^{\top}.
\end{equation*}
We may use this as the basis for finding an anchor representation and, also, for constructing an iterative estimation procedure.
The basis for doing so is noticing that the above implies, after performing whitening by pre- and postmultiplication with $\PS^{-1/2}$, that
\begin{equation*}
	\PS^{-1/2}\mathbf{\Sigma}_{xx}\PS^{-1/2} - \mathbf{I}_p = 
	\PS^{-1/2}\lambdaDot\lambdaDot^{\top}\PS^{-1/2},
\end{equation*}
which is symmetric, p.s.d., and of rank $m$.
Define $\mathbf{\Sigma}^{\PS}_{xx} \equiv \PS^{-1/2}\mathbf{\Sigma}_{xx}\PS^{-1/2}$ and let $\mathbf{V}\mathbf{E} \mathbf{V}^{\top}$ be its spectral decomposition, where $\mathbf{E} = \mathrm{diag}[\boldsymbol{e}(\mathbf{\Sigma}^{\PS}_{xx})]$ with $\boldsymbol{e}(\mathbf{\Sigma}^{\PS}_{xx})$ representing the vector of eigenvalues $e_1(\mathbf{\Sigma}^{\PS}_{xx}) \geq \ldots \geq e_p(\mathbf{\Sigma}^{\PS}_{xx})$.
We can then express
\begin{equation*}
	\mathbf{\Sigma}^{\PS}_{xx} - \mathbf{I}_p = 
	\mathbf{V}\mathbf{E} \mathbf{V}^{\top} - \mathbf{I}_p = \mathbf{V}\mathbf{E} \mathbf{V}^{\top} - \mathbf{V}\mathbf{V}^{\top}
	=
	\mathbf{V}\big(\mathbf{E} - \mathbf{I}_p\big)\mathbf{V}^{\top}.
\end{equation*}
Remember that $\lambdaDot$ is of rank $m$.
Hence, $e_1(\mathbf{\Sigma}^{\PS}_{xx}) \geq \ldots \geq e_m(\mathbf{\Sigma}^{\PS}_{xx}) > 1$ and $e_{j'}(\mathbf{\Sigma}^{\PS}_{xx}) = 1$ for $j' = m+1, \ldots, p$.
Let $\mathbf{E}_m = \mathrm{diag}[\boldsymbol{e}_m(\mathbf{\Sigma}^{\PS}_{xx})]$, with $\boldsymbol{e}_m(\mathbf{\Sigma}^{\PS}_{xx})$ the vector of eigenvalues $e_1(\mathbf{\Sigma}^{\PS}_{xx}) \geq \dots \geq e_m(\mathbf{\Sigma}^{\PS}_{xx})$ and let $\mathbf{V}_m \in \mathbb{R}^{p \times m}$ represent the corresponding matrix of eigenvectors.
The above then implies
\begin{equation*}
	\PS^{-1/2}\lambdaDot\lambdaDot^{\top}\PS^{-1/2} = 
	\mathbf{V}_m\big(\mathbf{E}_m - \mathbf{I}_m\big)\mathbf{V}_m^{\top}.
\end{equation*}
Using the square root factorization we then have
\begin{equation*}
	\PS^{-1/2}\lambdaDot\lambdaDot^{\top}\PS^{-1/2} = 
	\mathbf{V}_m\big(\mathbf{E}_m - \mathbf{I}_m\big)^{1/2}
	\big(\mathbf{E}_m - \mathbf{I}_m\big)^{1/2}
	\mathbf{V}_m^{\top}.
\end{equation*}
Solving for $\PS^{-1/2}\lambdaDot$ and undoing the whitening we get
\begin{equation}
	\label{SMEQ:LAcanon}
	\lambdaDot = \PS^{1/2}\mathbf{V}_m\big(\mathbf{E}_m - \mathbf{I}_m\big)^{1/2}.
\end{equation}
We then also immediately have a representation for $\PS$ by
\begin{equation}
	\label{SMEQ:PScanon}
	\PS = \big(\mathbf{\Sigma}_{xx} - \lambdaDot\lambdaDot^{\top}\big) \odot \mathbf{I}_p.
\end{equation}
Together, (\ref{SMEQ:LAcanon}) and (\ref{SMEQ:PScanon}) form the canonical representation system.
This system will give us a handle on high-dimensional asymptotics in the distribution-free setting as well as direct pointers for iterative estimation in the case of realized data.

\subsubsection{On the canonicality of the canonical approach}
\label{SMSSSEC:canonicality}
There are three reasons why we may call the representation system of Section 3.3 of the MT and Section \ref{SMSSSEC:Canonical} above, `canonical'.
First, it adheres to the canonical (rotational) identification criterion that $\lambdaDot^{\top}\PSi\lambdaDot$ be diagonal and ordered.
We may see so directly by using the canonical representations
\begin{equation*}
	\lambdaDot^{\top}\PSi\lambdaDot = 
	\big(\mathbf{E}_m - \mathbf{I}_m\big)^{1/2} 
	\mathbf{V}_m^{\top}\PS^{1/2}\PSi\PS^{1/2}\mathbf{V}_m
	\big(\mathbf{E}_m - \mathbf{I}_m\big)^{1/2},
\end{equation*}
as $\PS^{1/2}\PSi\PS^{1/2} = \mathbf{I}_p$, $\mathbf{V}_m^{\top}\mathbf{V}_m = \mathbf{I}_m$, and $(\mathbf{E}_m - \mathbf{I}_m)$ is diagonal and ordered.
Second, the representation is intricately related to the (generalized) canonical eigenvalue problem $\mathbf{\Sigma}_{xx}\bm{\mathrm{a}} = e\PS\bm{\mathrm{a}}$, which we will explicate in Section \ref{SMSSSEC:geneig}.
This also connects tightly to the next point.

Third, the representation has a direct connection to the canonical correlation approach.
This connection was explicated by Rao \cite{SMref_Rao1955} and Harris \cite{SMref_Harris62}.
The goal is to find the orientation in two sets of variables that maximizes the correlation between these sets.
For the orthogonal representation of the factor model we then consider the correlation matrix of $(x,\dot{\xi})$,
\begin{equation*}
	\begin{bmatrix}
		\lambdaDot\lambdaDot^{\top} + \PS  & \lambdaDot \\
		\lambdaDot^{\top}  & \mathbf{I}_m
	\end{bmatrix},
\end{equation*}
from which we may obtain the characteristic equation for canonical correlations $\rho$:
\begin{equation*}
	\big|\lambdaDot\lambdaDot^{\top} - \rho^2\big(\lambdaDot\lambdaDot^{\top} + \PS\big)\big| = 0.
\end{equation*}
By expansion we may write the foregoing as
\begin{equation*}
	\big|\big(1 - \rho^2\big)\lambdaDot\lambdaDot^{\top} - \rho^2\PS\big| = 0.
\end{equation*}
Multiplying the determinant by $1/(1 - \rho^2)$ does not affect roots and gives
\begin{equation*}
	\bigg|\lambdaDot\lambdaDot^{\top} - \frac{\rho^2}{1 - \rho^2}\PS\bigg| = 0.
\end{equation*}
Define $e \equiv 1/(1 - \rho^2)$ so that $e - 1 = \rho^2/(1 - \rho^2)$.
The preceding is then equivalent to
\begin{equation*}
	\big|\lambdaDot\lambdaDot^{\top} - (e - 1)\PS\big| = 0.
\end{equation*}
Using $\lambdaDot\lambdaDot^{\top} = \mathbf{\Sigma}_{xx} - \PS$ the foregoing can be equivalently stated as
\begin{equation*}
	\big|\mathbf{\Sigma}_{xx} - e\PS\big| = 0,
\end{equation*}
which is the characteristic equation corresponding to the generalized canonical eigenvalue problem (Section \ref{SMSSSEC:geneig}).
Hence, 
\begin{equation*}
	\big|\lambdaDot\lambdaDot^{\top} - \rho^2\mathbf{\Sigma}_{xx}\big| = 0
	\,\,\,\Longleftrightarrow\,\,\,
	\big|\mathbf{\Sigma}_{xx} - e\PS\big| = 0,
\end{equation*}
and the canonical correlations $\rho$ between $x$ and $\xi$ correspond to the canonical roots $e$ through $e = 1/(1 - \rho^2)$.
Naturally, $\rho^2 \in (0,1)$ and  $e \in (0,\infty)$.
Note that there are $m$ non-trivial roots and $m$ non-trivial canonical correlations.
Also note that the factor estimator stemming from canonical correlation and our best linear predictor in the canonical reflection are equivalent \citep[cf.][Section 7.6.3]{SMref_Kochbook}.

\subsubsection{A generalized canonical eigenvalue problem}
\label{SMSSSEC:geneig}
Let us use $(e,\bm{\mathrm{v}})$ as shorthand for an eigenpair of $\mathbf{\Sigma}^{\PS}_{xx}$.
Our canonical representation then hinges upon the non-trivial pairs $(e,\bm{\mathrm{v}})$ that are solutions to the standard (or canonical) eigenproblem
\begin{equation}
	\label{SMEQ:regeig}
	\mathbf{\Sigma}^{\PS}_{xx}\bm{\mathrm{v}} = e\bm{\mathrm{v}}.
\end{equation}
If we multiply by $\PS^{1/2}$ and set $\bm{\mathrm{a}} = \PS^{-1/2}\bm{\mathrm{v}}$ we have 
\begin{equation}
	\label{SMEQ:geneig}
	\mathbf{\Sigma}_{xx}\bm{\mathrm{a}} = e\PS\bm{\mathrm{a}},
\end{equation}
which is a generalized canonical eigenvalue problem.
Note that the generalized representation leaves the eigenvalues unchanged.
This can be seen by considering the following (well-known) fundamental implication:
\begin{equation*}
	\big|\mathbf{\Sigma}^{\PS}_{xx} - e\mathbf{I}_p\big| = 0 
	\,\,\,\Longleftrightarrow\,\,\,
	\exists\bm{\mathrm{v}}\neq \boldsymbol{0} ~:~ \mathbf{\Sigma}^{\PS}_{xx}\bm{\mathrm{v}} = e\bm{\mathrm{v}}.
\end{equation*}
Starting from the characteristic equation to (\ref{SMEQ:regeig}), we have
\begin{equation*}
	\big|\mathbf{\Sigma}^{\PS}_{xx} - e\mathbf{I}_p\big| = 
	\big|\PS^{-1/2}\big(\mathbf{\Sigma}_{xx} - e\PS\big)\PS^{-1/2}\big| =
	\big|\PSi\big|\big|\mathbf{\Sigma}_{xx} - e\PS\big| = 0.
\end{equation*}
As $\PS \succ 0$, $|\PSi| \neq 0$ and, hence, the roots hinge upon $|\mathbf{\Sigma}_{xx} - e\PS| = 0$, which is exactly the characteristic equation corresponding to the generalized eigenvalue problem (\ref{SMEQ:geneig}).
One could also say $|\mathbf{\Sigma}_{xx} - e\PS| = |\PS||\mathbf{\Sigma}^{\PS}_{xx} - e\mathbf{I}_p|$.
The eigenvalue problems then generate the same canonical representations for $\lambdaDot$ and $\PS$, resulting in the same canonical representation system.
The difference simply lies in a shift in metric. 
The standard canonical eigenproblem may be seen as the generalized canonical eigenproblem expressed in coordinates where $\PS$ has been whitened to the identity matrix. 
The generalized eigenproblem expresses the coordinates in the original variable space.
This has an interpretational advantage as (\ref{SMEQ:geneig}) has a solution defined by maximization of the \emph{generalized} Rayleigh-Ritz ratio \cite[][p.\ 176--180]{SMref_Horn1985}:
\begin{equation*}
	e = \max_{\bm{\mathrm{a}}\neq\boldsymbol{0}}\frac{\bm{\mathrm{a}}^{\top}\mathbf{\Sigma}_{xx}\bm{\mathrm{a}}}{\bm{\mathrm{a}}^{\top}\PS\bm{\mathrm{a}}}.
\end{equation*}
Hence, the canonical directions maximize the total variance relative to error variance.
The $m$ canonical directions thus find the bearings where the common variance (signal) is largest relative to the error variance (noise).
They are not `typical' directions and, as such, avoid the usual tendencies of projection geometry in high dimension (see MT).

\subsubsection{Scale-equivariance of the estimating equations}
\label{SMSSSEC:equivariance}
The estimating equations are scale-equivariant in the sense that initialization of Algorithm 1 of the MT with scaling $\mathbf{XC}$ of $\mathbf{X}$ will produce estimates $\mathbf{C}\hat{\lambdaDot}$ and $\mathbf{C}\hat{\PS}\mathbf{C}$.
This stems from the scale-invariance of the reduced second moment matrix from which the estimating equations are derived.
To see this note that, for any diagonal matrix $\mathbf{C} \in \mathbb{R}^{p\times p}$ such that $\mathbf{C} \succ 0$, we have (in the canonical representation) $\mathbf{C}x := \mathbf{C}\lambdaDot\dot{\xi} + \mathbf{C}\epsilon$ and $\mathbf{C}\mathbf{\Sigma}_{xx}\mathbf{C} = \mathbf{C}\lambdaDot\lambdaDot^{\top}\mathbf{C} + \mathbf{C}\PS\mathbf{C}$.
Using the developments in Section \ref{SMSSSEC:Canonical} we have that:
\begin{footnotesize}
	\begin{equation*}
		\PS^{-1/2}\mathbf{C}^{-1}\mathbf{C}\mathbf{\Sigma}_{xx}\mathbf{C}\mathbf{C}^{-1}	\PS^{-1/2} - \PS^{-1/2}\mathbf{C}^{-1}\mathbf{C}\PS
		\mathbf{C}\mathbf{C}^{-1}\PS^{-1/2} = 
		\PS^{-1/2}\mathbf{C}^{-1}\mathbf{C}\lambdaDot\lambdaDot^{\top}\mathbf{C}\mathbf{C}^{-1}	\PS^{-1/2}.
	\end{equation*}
\end{footnotesize}
Let $\lambdaDot^* \equiv \mathbf{C}\lambdaDot$.
Nullifying the scalings on the left-hand side it is immediate that
\begin{equation*}
	\PS^{-1/2}\mathbf{\Sigma}_{xx}\PS^{-1/2} - \mathbf{I}_p = 
	\PS^{-1/2}\mathbf{C}^{-1}\lambdaDot^*\lambdaDot^{*\top}\mathbf{C}^{-1}	\PS^{-1/2},
\end{equation*}
implying the scale-invariance of the (scaled) reduced second moment matrix.
Using the spectral decomposition we get
\begin{equation*}
	\mathbf{V}_m(\mathbf{E}_m - \mathbf{I}_m)\mathbf{V}_m^{\top} = 
	\PS^{-1/2}\mathbf{C}^{-1}\lambdaDot^*\lambdaDot^{*\top}\mathbf{C}^{-1}	\PS^{-1/2}.
\end{equation*}
Solving for $\PS^{-1/2}\mathbf{C}^{-1}\lambdaDot^*$ and undoing the whitening we obtain
\begin{equation*}
	\lambdaDot^* = \mathbf{C}\lambdaDot = \mathbf{C}\PS^{1/2}	\mathbf{V}_m(\mathbf{E}_m - \mathbf{I}_m)^{1/2}.
\end{equation*}
For the error variance matrix it is immediate that
\begin{equation*}
	\mathbf{C}\PS\mathbf{C} = \Big[\mathbf{C}\mathbf{\Sigma}_{xx}\mathbf{C} - \lambdaDot^*\lambdaDot^{*\top}\Big] \odot \mathbf{I}_p = \mathbf{C}\Big[\mathbf{\Sigma}_{xx} - \lambdaDot\lambdaDot^{\top}\Big]\mathbf{C} \odot \mathbf{I}_p,
\end{equation*}
resulting in the system of estimating equations 
\begin{align*}
	\mathbf{C}\lambdaDot &= \mathbf{C}\PS^{1/2}\mathbf{V}_m\big(\mathbf{E}_m - \mathbf{I}_m\big)^{1/2},\\
	\mathbf{C}\PS\mathbf{C} &= \mathbf{C}\Big[\mathbf{\Sigma}_{xx} - \lambdaDot\lambdaDot^{\top}\Big]\mathbf{C} \odot \mathbf{I}_p.
\end{align*}

\subsubsection{Rotational representation}
\label{SMSSSEC:Rotation}
We may use rotation to aid interpretation of the latent representation through the weight matrix $\LA$.
We recommend criteria for \emph{simple structure rotation}.
Such criteria seek rotations of $\LA$ that produce approximately sparse patterns in the sense of encouraging each feature to load strongly on one factor and near-zero on remaining factors. 
Criteria for both orthogonal and oblique rotation to simple structure exist.
We take interest in the Varimax criterion \citep{SM_Ref_Kaiser1958} for orthogonal and the Promax criterion \citep{SMref_Hendrickson1964} for oblique simple structure rotation as they are theoretically supported \citep{SMref_Rohe2023}, can be formulated scale-equivariantly, and are widely implemented. 

In our case we start with $\hat{\lambdaDot}$, the canonical estimate of the loadings matrix.
Define the communalities 
\begin{equation*}
	l^2_j \equiv \sum_{k = 1}^{m} \big(\hat{\lambdaDot}\big)_{jk}^2,
\end{equation*}
and let $\mathbf{L} = \mathrm{diag}[l_1, \ldots, l_p]$.
The Varimax criterion is then applied to the normalized loadings $\mathbf{L}^{-1}\hat{\lambdaDot}$ \citep{SM_Ref_Kaiser1958}:
\begin{equation}\label{SMEQ:Varimax}
	\HDot_V \,\equiv \,\argmax_{\HDot \in \mathrm{O}(m)}\,
	\frac{1}{p}\bigg\|\Big(\mathbf{L}^{-1}\hat{\lambdaDot}\HDot\Big)^{\odot2} - \boldsymbol{1}_p\bigg[\frac{1}{p}\boldsymbol{1}_p^{\top}\Big(\mathbf{L}^{-1}\hat{\lambdaDot}\HDot\Big)^{\odot2}\bigg]\bigg\|_{F}^2,
\end{equation}
where $\mathrm{O}(m) \equiv \{\HDot \in \mathbb{R}^{m\times m} \,|\, \HDot\HDot^{\top} = \HDot^{\top}\HDot = \mathbf{I}_m\}$ represents the orthogonal group.
The rotated loadings matrix is then obtained as $\hat{\lambdaDot}_V = \hat{\lambdaDot}\HDot_V$.
The rotation preserves distance and total variance while enhancing interpretability by encouraging an approximately sparse structure.
The normalization with $\mathbf{L}^{-1}$ prevents features with larger communalities from dominating the rotation.
It also provides scale-equivariance as $\HDot_V$ stemming from (\ref{SMEQ:Varimax}) equals, up to signed permutation, $\HDot_V$ stemming from (\ref{SMEQ:Varimax}) applied to $\mathbf{C}\hat{\lambdaDot}$ where $\mathbf{C} \succ 0$ is a diagonal scale matrix.
The rotated loadings would then only differ (up to signed permutation) by scaling with $\mathbf{C}$.
We say `up to signed permutation' as the Varimax criterion is not convex.
It is invariant to label-switching and polarity reversals in the columns of $\hat{\lambdaDot}$ and thus invariant over the Coxeter group of order $2^m(m!)$ \citep{SMref_Peeters2012}.
Optimization of (\ref{SMEQ:Varimax}) is usually based on a type of constrained approach that can be characterized as a form of block coordinate ascent.

The Varimax rotation can also be used to produce an oblique rather than orthogonal simple structure rotation. 
The basic idea of Promax is to find an oblique rotation to simple structure that is close to an orthogonal target structure. 
The Varimax loadings matrix is often used as the basis for target matrix $\mathbf{T}$:
\begin{equation*}
	(\mathbf{T})_{jk} = (\hat{\lambdaDot}_V)_{jk}\big|(\hat{\lambdaDot}_V)_{jk}\big|^{q - 1},
\end{equation*}
for all $j,k$, with $q > 1$.
A recommended choice is $q = 4$ \citep{SMref_Hendrickson1964}. 
It provides an elementwise power-transformation with preservation of sign.
We then have the following Procrustes criterion \citep{SMref_GDproc}:
\begin{equation}\label{SMEQ:Procrustes}
	\tilde{\mathbf{K}}_P \,\equiv \,\argmin_{\tilde{\mathbf{K}} \in \tilde{\mathcal{K}}}
	\big\|\mathbf{T} - \hat{\lambdaDot}_V\tilde{\mathbf{K}}\big\|_{F}^2,
\end{equation}
where 
$\tilde{\mathcal{K}} \equiv \{\tilde{\mathbf{K}} \in \mathbb{R}^{m\times m} \,|\, |\tilde{\mathbf{K}}|\neq 0 \wedge \tilde{\mathbf{K}}^{-1}\tilde{\mathbf{K}}^{-\top} = \hat{\PH}^{*}\}$, with $\hat{\PH}^{*}$ a proper factor covariance matrix.
We note that, for Promax, the set of admissable rotation matrices is not defined as a closed-form group, but is rather algorithmically characterized from the initial orthogonal rotation.
This is why we cannot search over $\mathcal{K}$ from Corollary 5 of the MT directly, but rather have a practical approach based on an initial Procrustes criterion on $\tilde{\mathcal{K}}$.
The Procrustes problem (\ref{SMEQ:Procrustes}) then has the explicit least-squares solution
\begin{equation*}
	\tilde{\mathbf{K}}_P = \hat{\lambdaDot}_V^{+}\mathbf{T} = 
	\big(\hat{\lambdaDot}_V^{\top}\hat{\lambdaDot}_V\big)^{-1}\hat{\lambdaDot}_V^{\top}\mathbf{T},
\end{equation*}
which is a least-squares map from $\hat{\lambdaDot}_V$ to its sharpened version $\mathbf{T}$.
Now, $|\tilde{\mathbf{K}}_P|\neq 0$ and $(\tilde{\mathbf{K}}_P^{\top}\tilde{\mathbf{K}}_P)^{\top} = \tilde{\mathbf{K}}_P^{\top}\tilde{\mathbf{K}}_P$.
Consequently, $(\tilde{\mathbf{K}}_P^{\top}\tilde{\mathbf{K}}_P)\succ 0$.
Hence, 
\begin{equation*}
	\hat{\PH}^{*} = (\tilde{\mathbf{K}}_P^{\top}\tilde{\mathbf{K}}_P)^{-1} = 
	\Big[\mathbf{T}^{\top}\hat{\lambdaDot}_V\big(\hat{\lambdaDot}_V^{\top}\hat{\lambdaDot}_V\big)^{-2}\hat{\lambdaDot}_V^{\top}\mathbf{T}\Big]^{-1},
\end{equation*}
is a valid covariance matrix stemming from self-regression after sharpening of $\hat{\lambdaDot}_V$.
Note that the sharpening through power-transformation is stimulating $\hat{\PH}^{*}$ to be non-diagonal.
Note also that the least-squares solution does not enforce unit variances.
Define $\tilde{\mathbf{K}}_s \equiv (\hat{\PH}^{*} \odot \mathbf{I}_m)^{1/2}$ and $\dot{\mathbf{K}}_P \equiv \tilde{\mathbf{K}}_P\tilde{\mathbf{K}}_s$.
Then 
\begin{equation*}
	\hat{\PH} = \big(\dot{\mathbf{K}}^{\top}_P\dot{\mathbf{K}}_P\big)^{-1},
\end{equation*}
and $\hat{\LA}_P = \hat{\lambdaDot}\dot{\mathbf{K}}_P$ is the Promax oblique simple structure loadings matrix.

\subsubsection{Illustrations}
\label{SMSSSEC:AlgoExample}
All data analyses were executed on a laptop equipped with a 12th generation Intel Core i7-1280P processor (1.80 GHz base clock, 14 cores), 32 GB of SO-DIMM RAM (3200 MT/s), and a 1TB NVMe SSD. 
No parallel processing was used. 
Computations were performed within the \textsf{R} 4.3.1 environment running on a 64-bit Windows 11 Enterprise (23H2) operating system.

\paragraph{Illustration 1: Omics data and computing time}
\label{SMpar:Time}
The computational complexity of Algorithm 1 from the MT is linear in $p$.
Moreover, we mentioned that the procedure is very suited for the situation $p \gg n - 1 > m$.
Here we illustrate that with omics data from Singh et al.\ \cite{SMref_Singh}. 
The data stem from oligonucleotide microarrays on $n = 52$ prostate tumor samples probing $p = 12,600$ genes.
We will consider the standardized data.
An estimate $\tilde{m}$ of the intrinsic dimension $m$ is based on thresholding using the Mar\v{c}enko-Pastur law.
The eigenvalues of an isotropic covariance matrix stemming from $n$ observations of $p$-dimensional centered variables follow the
Marchenko-Pastur probability density function. 
This distribution is strictly supported on the interval $\big[\sigma^2(1 - \sqrt{p/n})^2,\sigma^2(1 + \sqrt{p/n})^2\big]$ \cite{SMref_MP}. 
When $p > n$ this spectrum includes a point-mass at zero. 
This model is, in a sense, a null-model. 
Representing a factor model that carries no common variance (signal) and consists solely of (isotropic) error variance (noise). 
The decision rule is then to consider any eigenvalue that exceeds
$\sigma^2(1 + \sqrt{p/n})^2$ to represent signal and all eigenvalues
below this threshold to represent noise. 
As our data are scaled our estimate $\hat{\sigma}^2$ of $\sigma^2$ defers to $1$. 
The choice $\tilde{m}$ of $m$ is then made as:
\begin{equation*}
	\tilde{m} = \mathrm{card}\Big\{j ~|~ e_{j}\big(\hat{\mathbf{\Sigma}}_{xx}\big) >
	\big(1 + \sqrt{p/n}\big)^2\Big\}.
\end{equation*}
Note that we again can use a truncated SVD to swiftly determine $\tilde{m}$.
For these data $\tilde{m} = 8$.
The execution time of Algorithm 1 is then micro-benchmarked over $100$ runs using sub-millisecond accurate timing functions.
We obtain the following output:
\begin{small}
	\begin{verbatim}
		Unit: milliseconds
		min       lq     mean   median       uq      max neval
		444.3573 460.0827 478.2379 467.2038 476.8144 559.8694   100.
	\end{verbatim}
\end{small}
Hence, the median runtime for Algorithm 1 in this situation is approximately $467$ milliseconds.
Code for this illustration can be found in Appendix \ref{SMAPP:Code}.

\paragraph{Illustration 2: Image compression and determinacy metrics}
\label{SMpar:Compress}
For our second illustration we will consider image compression.
The compression of high-resolution digital images is of interest for reducing storage and transmission costs. 
We will consider RGB images: rectangular arrays of pixels that contain three color channels (Red, Green, and Blue).
Hence, an RGB image can be considered a tensor.
We will besee a simple model in which we consider the columns as representing our random variables and the rows as representing the observations.
Realizations of the random vector $\dot{x} \in \mathbb{R}^{p}$ then assign color to pixels on the basis of an overlay of the intensities of our three color channels:
\begin{equation*}
	\dot{x} = \dot{x}_{\mathrm{R}} + \dot{x}_{\mathrm{G}} + \dot{x}_{\mathrm{B}}.
\end{equation*}
We then consider each color channel $c \in \mathcal{C} \equiv \{\mathrm{R}, \mathrm{G}, \mathrm{B}\}$ to be subject to the canonical factor model:
\begin{equation*}
	\dot{x}_c := \boldsymbol{\mu}_c + \mathbf{C}_c\lambdaDot_c \dot{\xi}_c + \mathbf{C}_c\epsilon_c,
\end{equation*}
where $\boldsymbol{\mu}_c$ represents the mean color intensity for channel $c$ and $\mathbf{C}_c = (\mathbf{\Sigma}_{\dot{x}\dot{x}|c} \odot \mathbf{I}_p)^{1/2}$ is the diagonal matrix with the color scales for channel $c$.
Let each color-channel factor model be subject to the conditions and assumptions as laid out in Section 3.3 of the MT.
We make the following additional assumptions: (i) $\dot{\xi}_c \ci \dot{\xi}_{c'}, \forall c \neq c'$; (ii) $\epsilon_c \ci \epsilon_{c'}, \forall c \neq c'$; and (iii) $\dot{\xi}_c \ci \epsilon_{c'}, \forall c,c'$.
These additional assumptions imply $\dot{x}_c \ci \dot{x}_{c'}$ for all $c \neq c'$, such that
\begin{equation*}
	\mathbb{E}\Big[(\dot{x} - \boldsymbol{\mu})(\dot{x} - \boldsymbol{\mu})^{\top}\Big] = 
	\sum_{c \in \mathcal{C}} \mathbf{C}_c\big(\lambdaDot_c\lambdaDot_c^{\top} + \PS_c\big)\mathbf{C}_c,
\end{equation*}
where $\boldsymbol{\mu} = \boldsymbol{\mu}_{\mathrm{R}} + \boldsymbol{\mu}_{\mathrm{G}} + \boldsymbol{\mu}_{\mathrm{B}}$.
We can, if we let $\hat{\boldsymbol{\mu}}_c = \bar{\boldsymbol{\mathrm{x}}}_c$ and let $\mathbf{C}_c$ be represented by its sample counterpart $(\hat{\mathbf{\Sigma}}_{\dot{x}\dot{x}|c} \odot \mathbf{I}_p)^{1/2}$, then feed each mean-centered color-channel array to Algorithm 1 to produce estimates $(\hat{\mathbf{\Sigma}}_{\dot{x}\dot{x}|c} \odot \mathbf{I}_p)^{1/2}\hat{\lambdaDot}_c \equiv \hat{\lambdaDot}_c^{*}$ and $(\hat{\mathbf{\Sigma}}_{\dot{x}\dot{x}|c} \odot \mathbf{I}_p)^{1/2}\hat{\PS}_c(\hat{\mathbf{\Sigma}}_{\dot{x}\dot{x}|c} \odot \mathbf{I}_p)^{1/2} \equiv \hat{\PS}_c^{*}$ (see Section \ref{SMSSSEC:equivariance} on equivariance).
We may then find the channel-specific factor score matrices
$\doublehat{\mathbf{\Xi}}_c$ so that the image can be reconstructed as
\begin{equation*}
	\hat{\dot{\mathbf{X}}} = \sum_{c \in \mathcal{C}} 
	\boldsymbol{1}_n \bar{\boldsymbol{\mathrm{x}}}_c^{\top} + \doublehat{\mathbf{\Xi}}_c \big(\hat{\lambdaDot}_c^{*}\big)^{\top}.
\end{equation*}

Panel (a) of Figure \ref{SMFIG:Compress} contains a high-resolution image with $5,184 \times 3,456$ pixels and a digital size of $6.6$ MB.
Each color channel was subjected to the procedure above.
The estimated intrinsic dimension $\tilde{m}_c$ was based on retaining those factors for which $10e(\hat{\mathbf{\Sigma}}_{\dot{x}\dot{x}|c}) > 1$.
This expresses that, a priori, 10 percent of the total variance is considered noise.
It leads to retainment of higher numbers of factors and is appropriate for image reconstruction purposes.
Panel (b) of Figure \ref{SMFIG:Compress} then contains the compressed reconstructed image of size $.735$ MB.

Figures \ref{SMFIG:RedBand}--\ref{SMFIG:BlueBand} contain distributions for selections of features in the original color channels and factors in the projected space for the red, green, and blue channels respectively. 
The original features have strictly positive values.
The projected latent features are not constrained to be strictly positive.
The atypical canonical directions produce non-Gaussian distributions.
The latent scores have, even though the projections were carried out with only mean-centered observables, expectation $0$ and a standard deviation that approximates $1$, exemplifying the scale-invariance property of Proposition 1 of the MT.
The aforementioned approximation is a reflection of near-determinacy, as implied by the results of Section 3.3 of the MT.
We may also assess the degree of factor determinacy by other means, such as the Guttman bound. 
The minimum Guttman bound over all retained factors over all color bands is $.894$.
We can thus deem the projections to be determinate to a very acceptable degree.
Code for this illustration can be found in Appendix \ref{SMAPP:CodeImage}.

The idea is clear: we represent the high-dimensional data on a linear low-dimensional manifold and impose flexibility by considering random variables on that manifold that are, in some sense, maximally a-typical (non-Gaussian).
The example shows that negative values in the latent space and associated weights translate to positive values in the ambient space.

\begin{figure}[h!]
	\centering
	\subfloat[a][6.6 MB original]{\includegraphics[width=.95\textwidth]{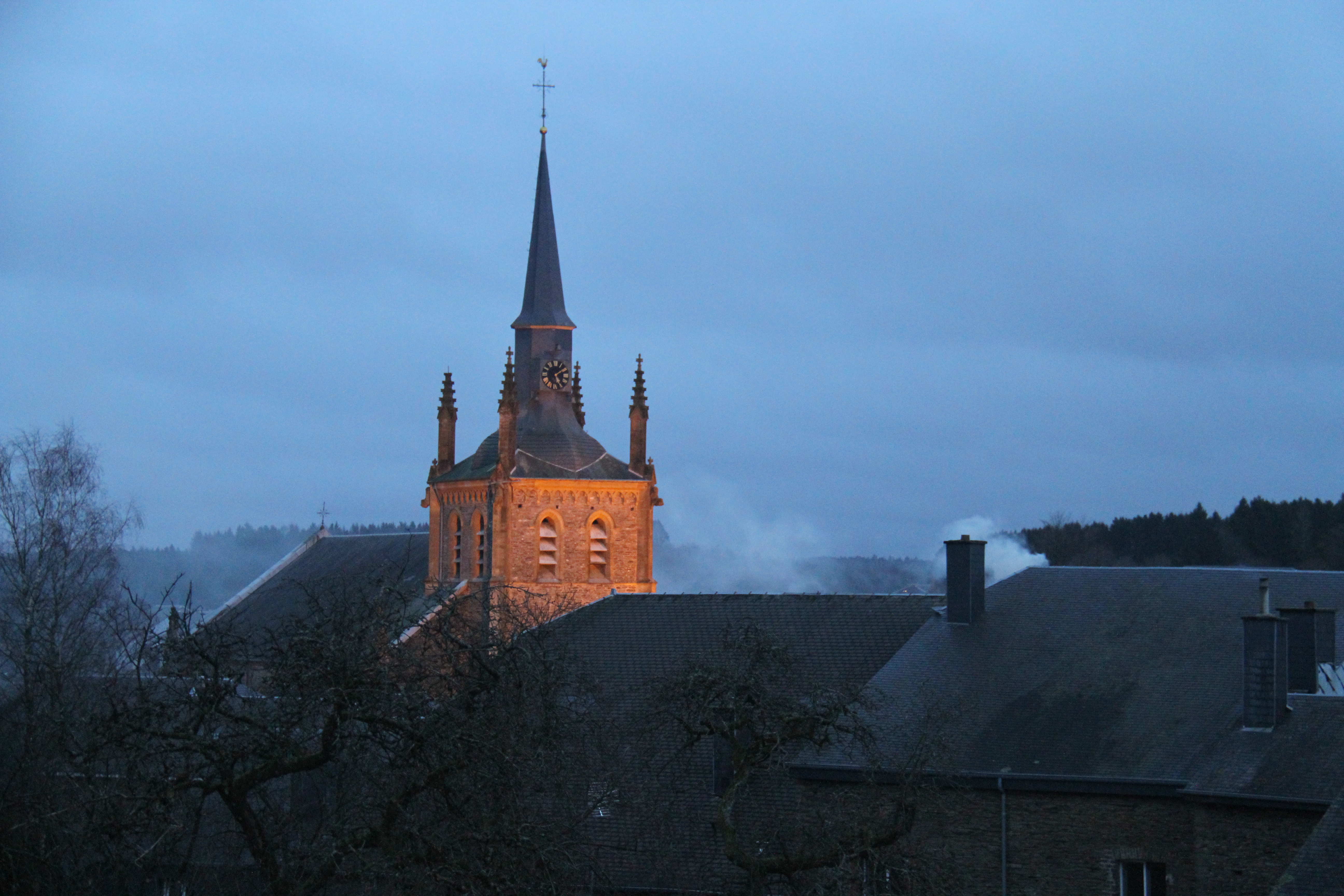} \label{fig:a}} \\
	\subfloat[b][0.735 MB compressed image]{\includegraphics[width=.95\textwidth]{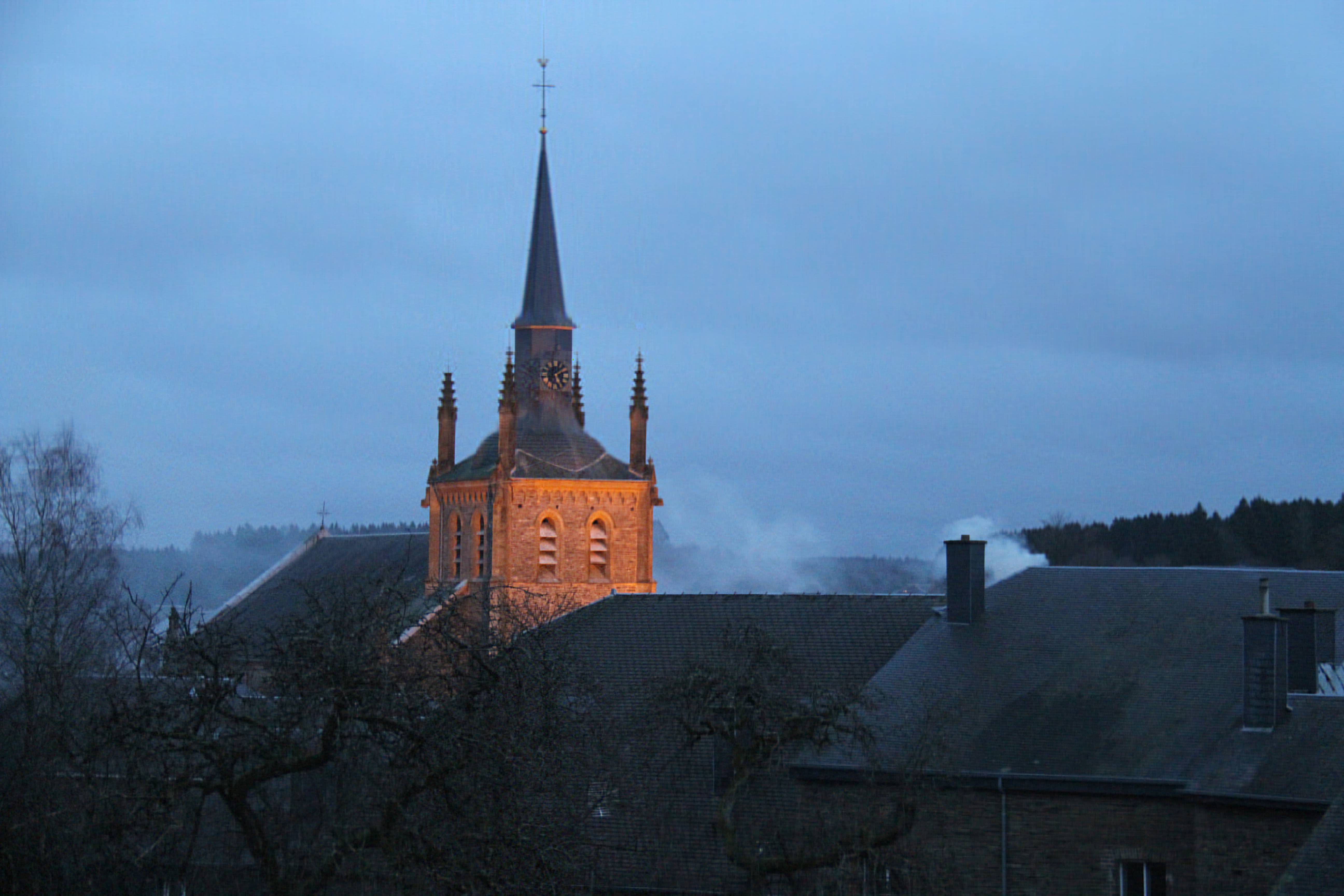} \label{fig:b}}
	\caption{
		Panel (a): Picture of the church of Sugny (Belgian Ardennes) at nightfall during winter. 
		Taken by the Author with a Canon EOS 7D. 
		Panel (b): Compressed version of the image obtained by compression of the individual color bands using Algorithm 1 of the MT.
	} 
	\label{SMFIG:Compress}
\end{figure}

\begin{figure}
	\centering
	\includegraphics[width=\textwidth]{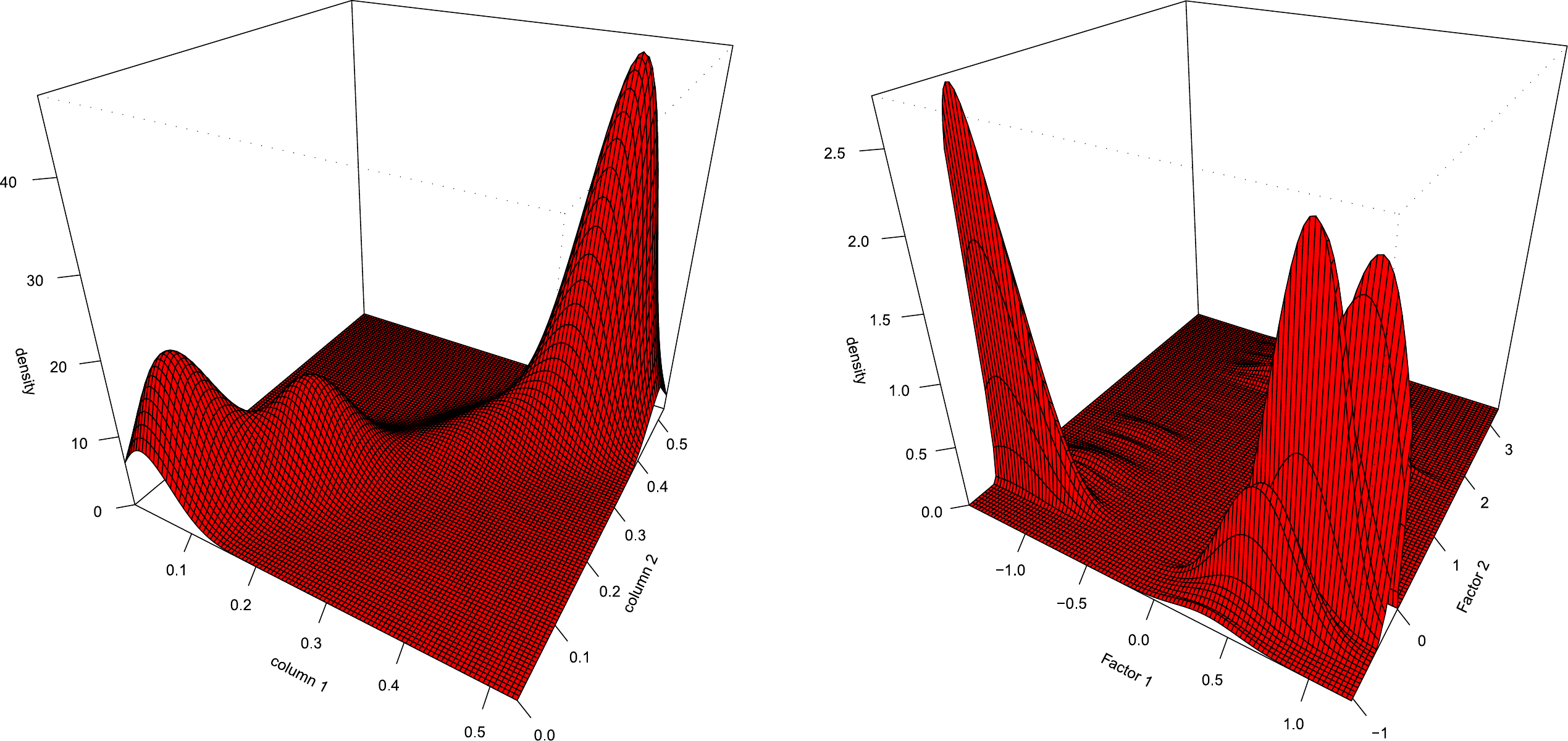}
	\caption{
		The left-hand panel contains the joint-distribution of the first two features in the red-channel column space of the uncompressed image.
		The right-hand panel contains the joint-distribution of the first two latent factors for the red channel.
	}
	\label{SMFIG:RedBand}
\end{figure}

\begin{figure}
	\centering
	\includegraphics[width=\textwidth]{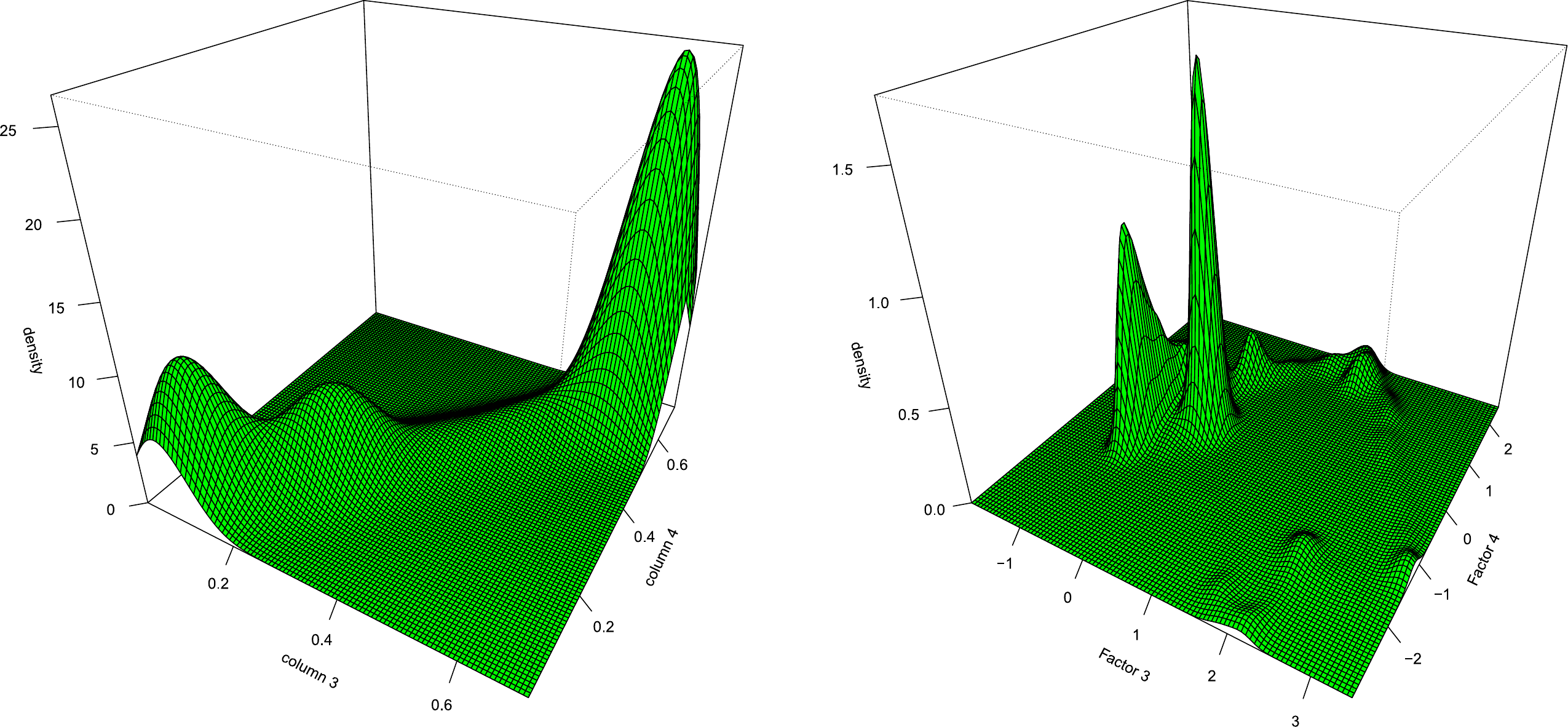}
	\caption{
		The left-hand panel contains the joint-distribution of the third and fourth features in the green-channel column space of the uncompressed image.
		The right-hand panel contains the joint-distribution of the third and fourth latent factors for the green channel.
	}
	\label{SMFIG:GreenBand}
\end{figure}

\begin{figure}
	\centering
	\includegraphics[width=\textwidth]{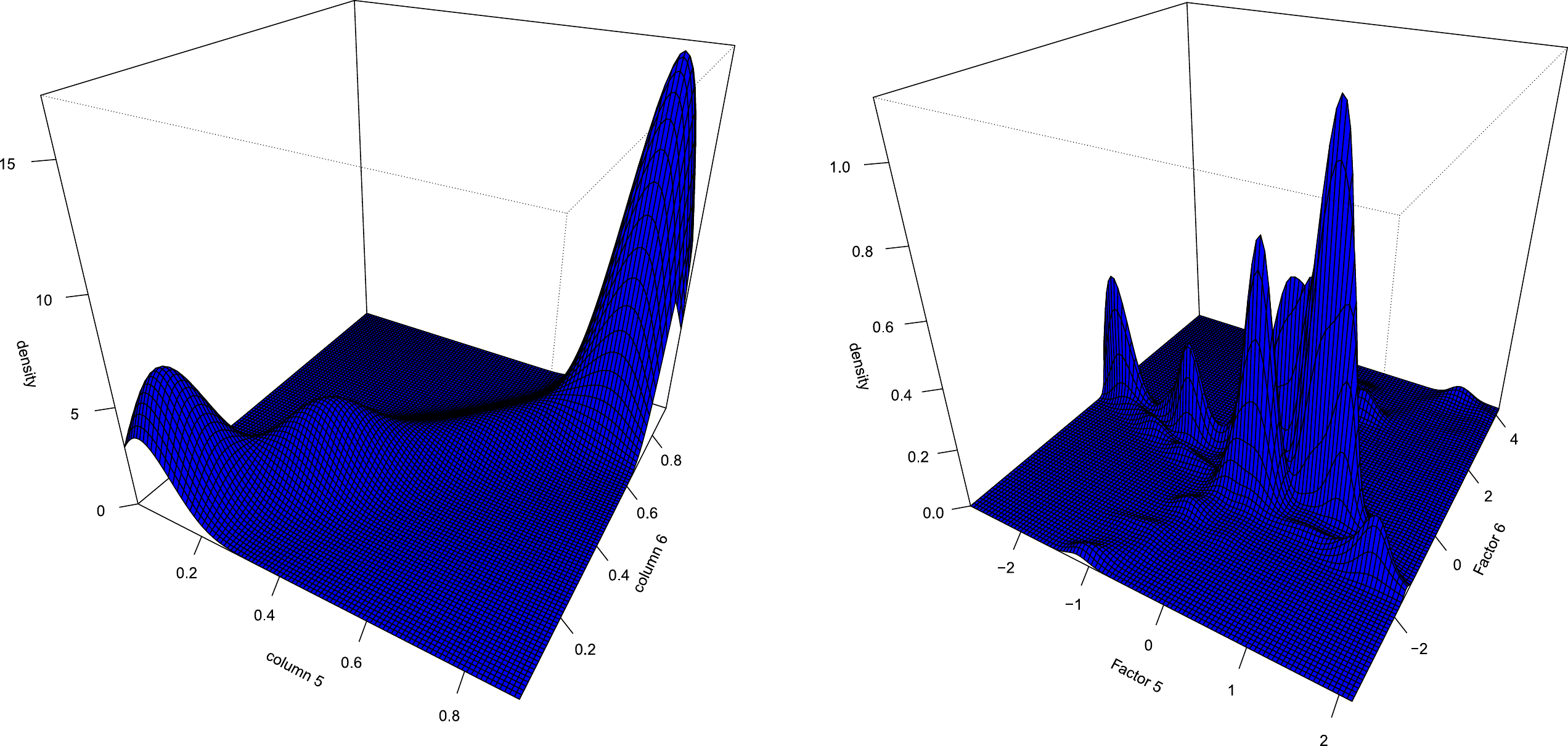}
	\caption{
		The left-hand panel contains the joint-distribution of the fifth and sixth features in the blue-channel column space of the uncompressed image.
		The right-hand panel contains the joint-distribution of the fifth and sixth latent factors for the blue channel.
	}
	\label{SMFIG:BlueBand}
\end{figure}
%--------------- Details 3 ----------------------------------------------
%------------------------------------------------------------------------

%------------------------------------------------------------------------
%--------------- Details 4 ----------------------------------------------
\section{Details for Section 4 of the Main Text}\label{SMSEC:Supp4}
%%%%%%%%%%%%%%%%%%%%%%%%%%%%
%%%-Details estimators -----
%%%%%%%%%%%%%%%%%%%%%%%%%%%%
\subsection{Details for Section 4.1 of the Main Text}
Here we show that other linear proxies of $\xi$ suffer from analogous indeterminacy problems and have higher MSE than $\xih$.
Perhaps the most popular estimator for $\xi$ is the Bartlett weighted least squares estimator.
Let $\xi_{\mathrm{B}} = \mathbf{B}^{\top} x$ with $ \mathbf{B} \in \mathbb{R}^{p \times m}$.
Bartlett \citep{Bartlett_SM} then proposed the following estimator:
\begin{align*}
	\xih_{\mathrm{B}} &\equiv \argmin_{\xi_{\mathrm{B}}} \,\tr
	\bigg\{\Big[\PS^{-1/2}\LA(\xi - \xi_{\mathrm{B}}) + \epsilon\Big]
	\Big[\PS^{-1/2}\LA(\xi - \xi_{\mathrm{B}}) + \epsilon\Big]^{\top}\bigg\}\\
	&= \big(\LAt\PSi\LA\big)^{-1}\LAt\PSi x.
\end{align*}
It can be shown that this is the best linear unbiased estimate.
We assume the same distributional behavior on $x$ as in the MT.
Then, $\mathbb{E}\big(\xih_{\mathrm{B}}\big) = \boldsymbol{0}$ and
\begin{align*}
	\mathbb{E}\big(\xih_{\mathrm{B}}\xih_{\mathrm{B}}^{\top}\big) &= 
	\big(\LAt\PSi\LA\big)^{-1}\LAt\PSi\mathbb{E}\big(xx^{\top}\big)\PSi\LA\big(\LAt\PSi\LA\big)^{-1} \\
	&= \PH + \big(\LAt\PSi\LA\big)^{-1},
\end{align*}
such that $\hat{\xi}_{\mathrm{B}} \sim\, _m\big[\boldsymbol{0},\PH + (\LAt\PSi\LA)^{-1}\big]$.
The MSE between $\xi$ and $\xih_{\mathrm{B}}$ can be found as
\begin{align*}
	\big\|\xi - \hat{\xi}_{\mathrm{B}}\big\|_{2}^{2} &= \tr\mathbb{E}\Big[\big(\xi - \hat{\xi}_{\mathrm{B}}\big)\big(\xi - \hat{\xi}_{\mathrm{B}}\big)^{\top}\Big] \\
	&= \tr\big(\Gram\big)^{-1},
\end{align*}
where we have used that $\mathbb{E}\big(\hat{\xi}_{\mathrm{B}}\xi^{\top}\big) = \mathbb{E}\big(\xi\hat{\xi}_{\mathrm{B}}^{\top}\big) = \PH$.
It is then immediate that the estimator and its associated MSE have the same limiting behavior as our preferred estimator $\xih$ due to
\begin{equation*}
	\lim\limits_{p\,\uparrow\,\infty} 
	\Big\|\big(\Gram\big)^{-1}\Big\|_2 = 0,
\end{equation*}
as a special case of Theorem 1 in the MT.
Irrespective of the limiting behaviors, the relative efficiency of $\xih$ is higher in the sense of (with the exception of the limit) 
\begin{equation*}
	\big\|\xi - \xih\big\|_{2}^{2} < \big\|\xi - \hat{\xi}_{\mathrm{B}}\big\|_{2}^{2},
\end{equation*}
as a direct consequence of the developments in the proof to Theorem 1 in the MT.
Similar results can be shown for other regression or least-squares type estimators, such as those by Horst \citep{ref_Horst_SM}, but will not be further pursued here. 

One could also understand the minimal MSE of $\xih$ in other ways.
Formally, $\mathbb{E}(\xi|x)$ is a measurable function $g_0(x)$ of $x$ that minimizes $\mathbb{E}\big[(\xi - g_0(x))(\xi - g_0(x))^{\top}\big]$ over all measurable functions $g$ \citep[][p.\ 155]{ref_Aad_SM}.
From another perspective one can also view $\mathbb{E}(\xi|x)$ as a Bayesian estimator with the general property of trading variance reduction for a little bias.

%%%%%%%%%%%%%%%%%%%%%%%%%%%%
%%%-Details Psychometrics --
%%%%%%%%%%%%%%%%%%%%%%%%%%%%
\subsection{Details for Section 4.2 of the Main Text}
Section \ref{SMSSSEC:additionalPsych} contains additional considerations on the implications of the findings of the MT for Psychometrics.
Section \ref{SMSSSEC:ConsiderHierar} details the limited representational power of the hierarchical linear model.

\subsubsection{Additional implications for Psychometrics}
\label{SMSSSEC:additionalPsych}

\paragraph{Appropriateness}
\label{SSSEC:Appropriate}
Factor analysis is often used as a panacea in the Social and Behavioral Sciences.
However, the preceding results suggest that, from the standpoint of indeterminacy, factor analysis is unlikely to be suited for the traditional psychometric setting, implying survey research with a very limited number of items (or indicators) per factor.
While especially in this context an estimate of $\{\LA,\PH,\PS\}$ is usually the endpoint of analysis, determinacy should at least be assessed to evaluate if usage of the factor model is warranted (see, for example, Indeterminacy Statement 5 of the MT. 
This work provides ample metrics for doing so.

\paragraph{Factor versus network models}
\label{SSSEC:FAvsNET}
There has been a recent call to move away from factor analysis in favor of (Gaussian) graphical modeling \citep[e.g.,][]{Nuijten2016, Borsboom2017}. 
Note, however, that Section 2.7 of the MT implies that the directed bipartite mixed graph corresponding to Model (1) of the MT is a graphical Markov object (with corresponding conditional independence graph) under $\mathfrak{m}$-separation. 
This contrasts with the habit of juxtaposing factor and (ordinary) graphical models.

Note that the focal quantity for graphical modeling $\mathbf{\Sigma}(\mathbf{\Theta})_{xx}^{-1}$, from the factor model standpoint, gives a 1-mode representation of a bipartite structure \citep[][Section 6.6]{Newman2010}.
It will converge, in a determinate situation, upon a (near-)diagonal structure.
Such structures signify (near) empty conditional independence graphs \citep[CIGs,][]{ref_ridgeP}.
Hence, an empty CIG indicates that the factor model may be a possible data-generating mechanism for the observable data. 
Conversely, not finding an empty CIG does not constitute a case against the factor model, especially in situations of relatively low or modest $p$, as this may rather imply that one has not reached a state of determinacy.

A more natural approach to graphical modeling, in case of factor analysis as a possible data generating mechanism, could be to obtain the correlated (oblique) factor projections and to proceed with (directed) graphical modeling on $\hat{\PH}^{-1}$.
This may serve as an approach to the structure after measurement \citep{Rosseel2024} take on structural equation modeling in (very) high dimension.
It may also be used in causal search (see Section \ref{SSSEC:Causal}).

\subsubsection{Representational power}
\label{SMSSSEC:ConsiderHierar}
Here we detail why the representational power of the hierarchical linear factor model is limited.
Say we have such a  structure with two generative layers:
\begin{align*}
	x     &:= \LA_1\xi_1 + \epsilon_1 \\
	\xi_1 &:= \LA_2\xi_2 + \epsilon_2.
\end{align*}
This means we have
\begin{equation*}
	x := \acute{\LA}\xi_2 + \acute{\epsilon},
\end{equation*}
with $\acute{\LA} = \LA_1\LA_2$ and $\acute{\epsilon} = \LA_1\epsilon_2 + \epsilon_1$.
The latter expression, from a neural network perspective, is simply an extended linear map.
This carries over to deeper hierarchies.
The usual way in which such models are fitted is to have correlated factors scores.
The correlations between these scores are the input for the next layer.
This process continues until the desired depth is achieved, the factor scores remain orthogonal even after oblique rotation, or when a single higher-order factor remains.
In a two-layer structure this would imply, if we use the best linear predictor as our factor proxy, that we are obtaining second-level factor scores through 
\begin{equation*}
	\mathbb{E}\Big(\xi_2 ~\big|~ \mathbb{E}(\xi_1 | x)\Big).
\end{equation*}
For a three-level structure we would be using 
\begin{equation*}
	\mathbb{E}\bigg[\xi_3 ~\bigg|~\mathbb{E}\Big(\xi_2 ~\big|~ \mathbb{E}(\xi_1 | x)\Big)\bigg],
\end{equation*}
and so on.
This implies (upward) propagation of indeterminacy (also see Sections 4.4.1 and 4.4.2 of the MT).

%%%%%%%%%%%%%%%%%%%%%%%%%%%%
%%%-Details Statistics -----
%%%%%%%%%%%%%%%%%%%%%%%%%%%%
\subsection{Details for Section 4.3 of the Main Text}
Section \ref{SMSSSEC:additionalSTATS} contains additional considerations on the implications of the findings of the MT for Statistics.
Section \ref{SMSSSEC:ConsiderNecess} further addresses the issue of suitability assessment.

\medskip
\subsubsection{Additional implications for Statistics}
\label{SMSSSEC:additionalSTATS}

\paragraph{Estimation}
We also mention two further implications in the realm of estimation that are not necessarily Bayesian in nature. 
First it is easy to take missingness into account as the projection hinges upon an estimate of $\mathbf{\Sigma}$.
One could thus circumvent missingness by simply basing this estimate on pairwise complete observations.
Second, in high dimension special amenities for estimation of the factor model under mixed or non-continuous metric \citep{Browne1984,Muthen1984} may be (next to computationally infeasible also) unnecessary due to the property of implicit regularization of the projection.  

\paragraph{Graphical modeling and causal search}
\label{SSSEC:Causal}
Directed and undirected graphical modeling have become mainstays in high-dimensional data analysis. 
The main challenge is reverse-engineering the most probable (causal) structure underlying the data.
This can be time-consuming.
For example, many (constraint or score based) algorithms for causal search have time-complexities exponential to the number of features (nodes) rendering them too intensive for large feature-spaces.
The time-complexity can often be reduced by assuming that the true underlying graph is (very) sparse.
However, there is accumulating evidence that many real-life systems are dense rather than sparse \citep{Clauset2008, Boyle2017, Bilgrau2020, Giannone2021}.
Moreover, even if one does obtain a sparse graphical representation on thousands of nodes it often implies `hairball' visualizations unconducive to spotting patterns or other aspects of interest \citep{Krzywinski2012}.
We propose to first project high-dimensional datasets onto its (obliquely rotated) latent meta-features. 
The resulting low dimension is more permitting of (integrative) causal search. 
Moreover, a meta-graph (i.e., a graph on the latent meta-features) will be more comprehensible than the graph obtained on the observables.

\subsubsection{Further details on suitability assessment}
\label{SMSSSEC:ConsiderNecess}
Here we sketch that $\hat{\mathbf{\Sigma}}^{+}$ being (near) diagonal is a necessity, but not sufficient to conclude the factor model as a data generating mechanism. First, consider that the conditions implying determinacy also imply that $\hat{\mathbf{\Sigma}}^{+} = \hat{\PS}^{-1}$. 
Hence, the condition is necessary.
To see that it is not sufficient we only have to consider a trivial matrix, such as $\varsigma\mathbf{I}_p$, where $\varsigma$ is any finite constant. 
Its inverse would also be diagonal.
However, we would not be able to extract a meaningful factor structure. 
We will explore these issues further elsewhere.

%%%%%%%%%%%%%%%%%%%%%%%%%%%%
%%%-Details AI -------------
%%%%%%%%%%%%%%%%%%%%%%%%%%%%
\subsection{Details for Section 4.4 of the Main Text}
Section \ref{SMSSSEC:additionalAI} contains an additional consideration on the implications of the findings of the MT for AI.
Section \ref{SMSSSEC:ConsiderArchitecture} visually clarifies the considerations in Section 4.4.2 of the MT.

\subsubsection{Relation to Smale's 18th problem and quantum connections}
\label{SMSSSEC:additionalAI}
If we assume that latent representation is an intrinsic characteristic of human-level intelligence \citep[cf.][]{Courville2006, Nayebi2023} and that representational determinacy should carry over to machine intelligence, then, from a connectionist perspective, latent machine representation should be based on expanding the observable stimulus and reinforcement space rather than a hierarchy of latent sources; ruling out deep generative architectures as a connectionist road to machine intelligence of the general type. 
This may be seen as another result on (the barriers of) deep learning in relation to Smale's 18th problem \citep{Smale1998, Colbrook2022}.
We also note, again without the intention to overstate, various concordances with quantum mechanics.
For example, $\xi'$ signifies a superposition between $\xih$ and $\tau'$ and Figure \ref{SMFIG:Cones} is congruent with the Bloch sphere \citep{Feynman1957}.
The overlapping distributional possibilities for factor scoring from the Bayesian perspective resemble quantum fidelity \citep{Nielsen2000}.
Moreover, factor determinacy is commensurable with a description of wave function collapse as a (Bayesian) measurement process \citep{Kiely2026}.
These concordances deserve further exploration to inform, from a connectionist perspective, the (im)possibilities of artificial general intelligence \citep[cf.][]{Penrose1994, King1996, Tegmark2000, Panitchayangkoon2010, Babcock2024}.

\subsubsection{Considerations regarding computational architectures}
\label{SMSSSEC:ConsiderArchitecture}
Here we visually clarify the considerations in Section 4.4.2 of the MT regarding computational architectures conducive to and avoiding propagation of latent factor indeterminacy.
Consider Figure \ref{SMFIG:Architecture}, visualizing two (unsupervised) deep neural architectures with $L = 3$ layers.
The left-hand side contains a deep generative architecture in terms of a hierarchy of factor models.  
It implies consecutive unsupervised learning of higher-order generative latents.
As each higher-order proxy is dependent on the preceding proxy and as each consecutive layer will be of lower dimension, there will be upward propagation of indeterminacy.  
The right-hand side contains a hybrid architecture in the sense of a combination of a generative (pre-training) layer followed by consecutive formative layers.
The formative latents are deterministic and do not necessarily have a lower dimension than their preceding layer.
As such, if we have a determinate proxy $\tilde{\xi}^{(1)}$ of $\xi^{(1)}$, there will be no upward propagation of indeterminacy.

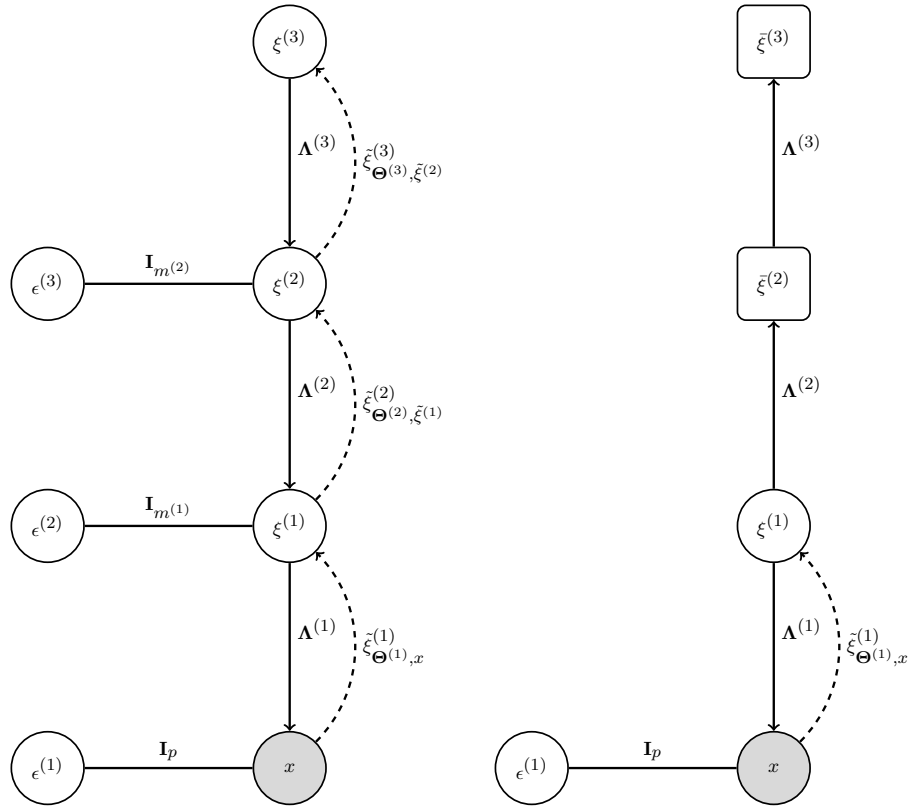
\begin{figure}[h!]
	\centering
	% Styles
	\tikzset{
		latent/.style={circle, thick, minimum size=1.2cm, draw=black!100, fill=white!100},
		formative/.style={rectangle, rounded corners, thick, minimum size=1.2cm, draw=black!100, fill=white!100},
		observed/.style={circle, thick, minimum size=1.2cm, draw=black!100, fill=gray!30}}
	\tikzstyle{plate} = [draw, rectangle, rounded corners, minimum size = 2.5cm]
	\def\bottom#1#2{\hbox{\vbox to #1{\vfill\hbox{#2}}}}
	\scalebox{0.8}{
		\begin{tikzpicture}
			% Nodes
			\node (L3) at(4,12) [latent]{$\xi^{(3)}$};
			\node (E3) at(0,8) [latent]{$\epsilon^{(3)}$};
			\node (L2) at(4,8) [latent]{$\xi^{(2)}$};
			\node (E2) at(0,4) [latent]{$\epsilon^{(2)}$};
			\node (L) at(4,4) [latent]{$\xi^{(1)}$};
			\node (E) at(0,0) [latent]{$\epsilon^{(1)}$};
			\node (X) at(4,0) [observed]{$x$};
			
			% Edges
			\draw [->, line width=1.1pt] (L) -- (X) node[midway,above right]{$\mathbf{\Lambda}^{(1)}$};
			\draw [->, line width=1.1pt] (L2) -- (L) node[midway,above right]{$\mathbf{\Lambda}^{(2)}$};
			\draw [->, line width=1.1pt] (L3) -- (L2) node[midway,above right]{$\mathbf{\Lambda}^{(3)}$};
			\draw [dashed,->, line width=1.1pt] (X) to [bend right = 45] node[midway, right]{$\tilde{\xi}^{(1)}_{\mathbf{\Theta}^{(1)},x}$} (L);
			\draw [dashed,->, line width=1.1pt] (L) to [bend right = 45] node[midway, right]{$\tilde{\xi}^{(2)}_{\mathbf{\Theta}^{(2)},\tilde{\xi}^{(1)}}$} (L2);
			\draw [dashed,->, line width=1.1pt] (L2) to [bend right = 45] node[midway, right]{$\tilde{\xi}^{(3)}_{\mathbf{\Theta}^{(3)},\tilde{\xi}^{(2)}}$} (L3);
			\draw [-, line width=1.1pt]  (E) -- (X) node[midway,above]{$\mathbf{I}_p$};
			\draw [-, line width=1.1pt]  (E2) -- (L) node[midway,above]{$\mathbf{I}_{m^{(1)}}$};
			\draw [-, line width=1.1pt]  (E3) -- (L2) node[midway,above]{$\mathbf{I}_{m^{(2)}}$};
			
			% Nodes
			\node (L32) at(12,12) [formative]{$\bar{\xi}^{(3)}$};
			\node (L22) at(12,8) [formative]{$\bar{\xi}^{(2)}$};
			\node (L2) at(12,4) [latent]{$\xi^{(1)}$};
			\node (E2) at(8,0) [latent]{$\epsilon^{(1)}$};
			\node (X2) at(12,0) [observed]{$x$};
			
			% Edges
			\draw [->, line width=1.1pt] (L2) -- (X2) node[midway,above right]{$\mathbf{\Lambda}^{(1)}$};
			\draw [->, line width=1.1pt] (L2) -- (L22) node[midway,above right]{$\mathbf{\Lambda}^{(2)}$};
			\draw [->, line width=1.1pt] (L22) -- (L32) node[midway,above right]{$\mathbf{\Lambda}^{(3)}$};
			\draw [dashed,->, line width=1.1pt] (X2) to [bend right = 45] node[midway, right]{$\tilde{\xi}^{(1)}_{\mathbf{\Theta}^{(1)},x}$} (L2);
			\draw [-, line width=1.1pt]  (E2) -- (X2) node[midway,above]{$\mathbf{I}_p$};
		\end{tikzpicture}
	}
	\caption{Schematic of deep generative (left-hand) versus deep hybrid (right-hand) architectures.
		Circular nodes represent random variables, with random latents represented by white nodes and observables in grey.
		Boxed nodes represent formative (non-random) latent variables $\bar{\xi}^{(l)}$.
		In this figure $\mathbf{\Theta}^{(l)}$ represents the collection of parameter matrices for retrieving the $l$th latent vector.
		The shorthand $\tilde{\xi}^{(l)}$ is used to represent $\tilde{\xi}^{(l)}_{\mathbf{\Theta}^{(l)},\tilde{\xi}^{(l-1)}}$ (to avoid notational clutter).}
	\label{SMFIG:Architecture}
\end{figure}

%--------------- Details 4 ----------------------------------------------
%------------------------------------------------------------------------

%%--------------------------------------------------------------------
%%--------------- Appendix -------------------------------------------
\begin{appendix}
	\renewcommand{\thesection}{S.\Alph{section}}
	%%%%%%%%%%%%%%%%%%%%%%%%%%%%
	%%%--- Code ex1 ------------
	%%%%%%%%%%%%%%%%%%%%%%%%%%%%
	\section{Code omics example}
	\label{SMAPP:Code}
	An implementation of Algorithm 1 of the MT as well as certain (in)determinacy metrics can be found in the software package \textsf{HAMMER}, available from the Comprehensive \textsf{R} Archive Network (CRAN) at \href{https://CRAN.R-project.org/package=HAMMER}
	{DOI:10.32614/CRAN.package.HAMMER}.
	The code below is based on \textsf{HAMMER} and can be used to reproduce the results from Section \ref{SMpar:Time}.
	The code also uses the \href{https://CRAN.R-project.org/package=microbenchmark}{\textsf{microbenchmark}} package for microbenchmarking executing times of Algorithm 1.
	The data that are used are publicly available and can be found in the \href{https://github.com/ramhiser/datamicroarray}{\textsf{datamicroarray}} package.
	
\begin{lstlisting}[language=R, caption={Code for reproducing the results from Section \ref{SMpar:Time}}]
##------------------------------------------------------------------------
## Libraries
##------------------------------------------------------------------------
library(HAMMER)
library(datamicroarray)
library(microbenchmark)
		
##------------------------------------------------------------------------
## Invoke data
##------------------------------------------------------------------------
## Keep only tumor samples and scale data
data('singh', package = 'datamicroarray')
Singh <- singh$x[singh$y == "Tumor",]
DAT   <- scale(Singh)
dim(Singh)
		
##------------------------------------------------------------------------
## Assessing intrinsic dimension
##------------------------------------------------------------------------
dim <- HAMMER.dimension(DAT, maxdim = 50)
		
##------------------------------------------------------------------------
## Run and time benchmark algorithm
##------------------------------------------------------------------------
TimeTest <- microbenchmark(HAMMER.estimate(DAT, m = dim),
                           setup = {set.seed(123)}, 
                           times = 100)
TimeTest
\end{lstlisting}

	\section{Code image compression example}
	\label{SMAPP:CodeImage}
	The code below also uses \textsf{HAMMER}.
	It reproduces the results from Section \ref{SMpar:Compress}.
	The code additionally uses the \href{https://CRAN.R-project.org/package=jpeg}{\textsf{jpeg}} package for reading and writing images stored in JPEG format.
	
	\begin{lstlisting}[language=R, caption={Code for reproducing the results from Section \ref{SMpar:Compress}}]
##------------------------------------------------------------------------
## Working directory
##------------------------------------------------------------------------
## Set to convenience
setwd(" ")

##------------------------------------------------------------------------
## Libraries
##------------------------------------------------------------------------
library(HAMMER)
library(jpeg)
library(MASS)

##------------------------------------------------------------------------
## Read image
##------------------------------------------------------------------------
im <- readJPEG("Sugny.jpg")

##------------------------------------------------------------------------
## Structure
##------------------------------------------------------------------------
## Three 3456x5184 matrices, each corresponding to an RGB color
str(im)

##------------------------------------------------------------------------
## Extract RGB color matrices
##------------------------------------------------------------------------
rimage <- (im[,,1])
gimage <- (im[,,2])
bimage <- (im[,,3])

##------------------------------------------------------------------------
## Mean centering
##------------------------------------------------------------------------
p      <- dim(rimage)[2]
n      <- dim(rimage)[1]
onevec <- as.matrix(rep(1, n))
Mr     <- onevec %*% colMeans(im[,,1])
Mg     <- onevec %*% colMeans(im[,,2])
Mb     <- onevec %*% colMeans(im[,,3])
rimage <- rimage - Mr
gimage <- gimage - Mg
bimage <- bimage - Mb

##------------------------------------------------------------------------
## Signal dimension evaluation
##------------------------------------------------------------------------
rdim <- HAMMER.dimension(rimage, maxdim = 600, 
                         method = "meansignal", factor = 10)
gdim <- HAMMER.dimension(gimage, maxdim = 600, 
                         method = "meansignal", factor = 10)
bdim <- HAMMER.dimension(bimage, maxdim = 600, 
                         method = "meansignal", factor = 10)

##------------------------------------------------------------------------
## Parameter learning
##------------------------------------------------------------------------
far <- HAMMER.estimate(rimage, m = rdim)
fag <- HAMMER.estimate(gimage, m = gdim)
fab <- HAMMER.estimate(bimage, m = bdim)

##------------------------------------------------------------------------
## Projection onto factor space
##------------------------------------------------------------------------
Fr <- HAMMER.score(far)
Fg <- HAMMER.score(fag)
Fb <- HAMMER.score(fab)

##------------------------------------------------------------------------
## Reconstruct compressed image
##------------------------------------------------------------------------
RED    <- (as.matrix(Fr) %*% t(far$Lambda) + Mr)
GREEN  <- (as.matrix(Fg) %*% t(fag$Lambda) + Mg)
BLUE   <- (as.matrix(Fb) %*% t(fab$Lambda) + Mb)
image  <- list(RED,GREEN,BLUE)
writeJPEG(simplify2array(image), "Sugny_compressed.jpg")

##------------------------------------------------------------------------
## Example distributions in color band
##------------------------------------------------------------------------
## These are Kernel densities 
den3d <- kde2d(im[,,1][,1], im[,,1][,2], n = 100)
persp(den3d, theta = 30, phi = 30, col = "red", box = TRUE, 
      ticktype = "detailed", zlab = "density",
      ylab = "column 2", xlab = "column 1")

den3d <- kde2d(im[,,2][,3], im[,,2][,4], n = 100)
persp(den3d, theta = 30, phi = 30, col = "green", box = TRUE, 
      ticktype = "detailed", zlab = "density",
      ylab = "column 4", xlab = "column 3")

den3d <- kde2d(im[,,3][,5], im[,,3][,6], n = 100)
persp(den3d, theta = 30, phi = 30, col = "blue", box = TRUE, 
      ticktype = "detailed", zlab = "density",
      ylab = "column 6", xlab = "column 5")

##------------------------------------------------------------------------
## Example distribution factor scores on color band
##------------------------------------------------------------------------
den3d <- kde2d(Fr[,1], Fr[,2], n = 100)
persp(den3d, theta = 30, phi = 30, col = "red", box = TRUE, 
      ticktype = "detailed", zlab = "density",
      ylab = "Factor 2", xlab = "Factor 1")

den3d <- kde2d(Fg[,3], Fg[,4], n = 100)
persp(den3d, theta = 30, phi = 30, col = "green", box = TRUE, 
      ticktype = "detailed", zlab = "density",
      ylab = "Factor 4", xlab = "Factor 3")

den3d <- kde2d(Fb[,5], Fb[,6], n = 100)
persp(den3d, theta = 30, phi = 30, col = "blue", box = TRUE, 
      ticktype = "detailed", zlab = "density",
      ylab = "Factor 6", xlab = "Factor 5")

##------------------------------------------------------------------------
## Simple assessment determinacy
##------------------------------------------------------------------------
rd <- HAMMER.determinacy(far)
rg <- HAMMER.determinacy(fag)
rb <- HAMMER.determinacy(fab)
gb <- c(rd$Imetrics[2,], rg$Imetrics[2,], rb$Imetrics[2,])
min(gb)
\end{lstlisting}
	
\end{appendix}
%%--------------- Appendix -------------------------------------------
%%--------------------------------------------------------------------

%--------------------------------------------------------------------
%--------------- Endnotes -------------------------------------------
%\begin{footnotesize}
\section*{Notes}
~\\
\printnotes*
%\end{footnotesize}
%--------------- Endnotes -------------------------------------------
%--------------------------------------------------------------------

%%--------------------------------------------------------------------
%%--------------- References -----------------------------------------
%\bibliographystyle{imsart-number}
%\bibliography{FAIn_refs_SM}

\begin{thebibliography}{107}
% BibTex style file: imsart-number.bst, 2017-11-03
% Default style options (sort=1,type=number).
% Used options (sort=1,type=number).

\bibitem{ref_RBM}
\begin{barticle}[author]
\bauthor{\bsnm{Ackley},~\bfnm{D.~H.}\binits{D.~H.}},
  \bauthor{\bsnm{Hinton},~\bfnm{G.~E.}\binits{G.~E.}} \AND
  \bauthor{\bsnm{Sejnowksi},~\bfnm{T.~J.}\binits{T.~J.}}
(\byear{1985}).
\btitle{A learning algorithm for {Boltzmann} machines}.
\bjournal{Cognitive Science}
\bvolume{9}
\bpages{147--169}.
\end{barticle}
\endbibitem

\bibitem{ref_AA2005}
\begin{bincollection}[author]
\bauthor{\bsnm{Aitkin},~\bfnm{M.}\binits{M.}} \AND
  \bauthor{\bsnm{Aitkin},~\bfnm{I.}\binits{I.}}
(\byear{2005}).
\btitle{Bayesian inference for factor scores}.
In \bbooktitle{{Contemporary Psychometrics: A Festschrift for Roderick P.\
  McDonald}}
(\beditor{\bfnm{A.}\binits{A.}~\bsnm{{Maydeu-Olivares}}} \AND
  \beditor{\bfnm{J.~J.}\binits{J.~J.}~\bsnm{{McArdle}}}, eds.)
\bpages{207--221}.
\bpublisher{{Mahwah New Jersey: Lawrence Erlbaum Associates}}.
\end{bincollection}
\endbibitem

\bibitem{ref_Alemi}
\begin{binproceedings}[author]
\bauthor{\bsnm{Alemi},~\bfnm{A.~A.}\binits{A.~A.}},
  \bauthor{\bsnm{Poole},~\bfnm{B.}\binits{B.}},
  \bauthor{\bsnm{Fischer},~\bfnm{I.}\binits{I.}},
  \bauthor{\bsnm{Dillon},~\bfnm{J.~V.}\binits{J.~V.}},
  \bauthor{\bsnm{Saurous},~\bfnm{R.~A.}\binits{R.~A.}} \AND
  \bauthor{\bsnm{Murphy},~\bfnm{K.}\binits{K.}}
(\byear{2018}).
\btitle{Fixing a broken {ELBO}}
In \bbooktitle{{Proceedings of the 35th International Conference on Machine
  Learning}}.
\end{binproceedings}
\endbibitem

\bibitem{ref_AmariDeep}
\begin{barticle}[author]
\bauthor{\bsnm{Amari},~\bfnm{S.}\binits{S.}}
(\byear{1967}).
\btitle{A theory of adaptive patterns classifier}.
\bjournal{IEEE Transactions on Electronic Computers}
\bvolume{EC-16}
\bpages{299--307}.
\end{barticle}
\endbibitem

\bibitem{ref_AndersonBible}
\begin{bbook}[author]
\bauthor{\bsnm{Anderson},~\bfnm{T.~W.}\binits{T.~W.}}
(\byear{2003}).
\btitle{{An Introduction to Multivariate Statistical Analysis}}.
\bpublisher{Hoboken, {NJ}: {J}ohn {W}iley \& {S}ons, {I}nc. (3rd ed.)}.
\end{bbook}
\endbibitem

\bibitem{ref_AndersonRubinClassic}
\begin{binproceedings}[author]
\bauthor{\bsnm{Anderson},~\bfnm{T.~W.}\binits{T.~W.}} \AND
  \bauthor{\bsnm{Rubin},~\bfnm{H.}\binits{H.}}
(\byear{1956}).
\btitle{Statistical Inference in Factor Analysis}.
In \bbooktitle{Proceedings of the Third Berkeley Symposium on Mathematical
  Statistics and Probability}
\bvolume{{5: Contributions to Econometrics, Industrial Research, and
  Psychometry}}
\bpages{111--150}.
\bpublisher{{Berkeley, CA: University of California Press}}.
\end{binproceedings}
\endbibitem

\bibitem{ref_IRLBM}
\begin{barticle}[author]
\bauthor{\bsnm{Baglama},~\bfnm{J.}\binits{J.}} \AND
  \bauthor{\bsnm{Reichel},~\bfnm{L.}\binits{L.}}
(\byear{2005}).
\btitle{Augmented implicitly restarted {Lanczos} bidiagonalization methods}.
\bjournal{SIAM Journal on Scientific Computing}
\bvolume{27}
\bpages{19--42}.
\end{barticle}
\endbibitem

\bibitem{ref_Bartholomew1981}
\begin{barticle}[author]
\bauthor{\bsnm{Bartholomew},~\bfnm{D.~J.}\binits{D.~J.}}
(\byear{1981}).
\btitle{Posterior analysis of the factor model}.
\bjournal{British Journal of Mathematical and Statistical Psychology}
\bvolume{34}
\bpages{93-99}.
\end{barticle}
\endbibitem

\bibitem{ref_BMW94}
\begin{bbook}[author]
\bauthor{\bsnm{Bekker},~\bfnm{P.~A.}\binits{P.~A.}},
  \bauthor{\bsnm{Merckens},~\bfnm{A.}\binits{A.}} \AND
  \bauthor{\bsnm{Wansbeek},~\bfnm{T.~J.}\binits{T.~J.}}
(\byear{1994}).
\btitle{Identification, Equivalent Models, and Computer Algebra}.
\bpublisher{Boston [etc.]: {A}cademic {P}ress}.
\end{bbook}
\endbibitem

\bibitem{ref_Bishop&Bishop}
\begin{bbook}[author]
\bauthor{\bsnm{Bishop},~\bfnm{C.~M.}\binits{C.~M.}} \AND
  \bauthor{\bsnm{Bishop},~\bfnm{H.}\binits{H.}}
(\byear{2024}).
\btitle{{Deep Learning: Foundations and Concepts}}.
\bpublisher{Springer}.
\end{bbook}
\endbibitem

\bibitem{ref_Bollen1989}
\begin{bbook}[author]
\bauthor{\bsnm{Bollen},~\bfnm{K.~A.}\binits{K.~A.}}
(\byear{1989}).
\btitle{{Structural Equations with Latent Variables}}.
\bpublisher{{N}ew {Y}ork [etc.]: {J}ohn {W}iley \& {S}ons}.
\end{bbook}
\endbibitem

\bibitem{ref_BorsOnto}
\begin{barticle}[author]
\bauthor{\bsnm{Borsboom},~\bfnm{D.}\binits{D.}},
  \bauthor{\bsnm{Mellenbergh},~\bfnm{G.~J.}\binits{G.~J.}} \AND
  \bauthor{\bsnm{{van Heerden}},~\bfnm{J.}\binits{J.}}
(\byear{2003}).
\btitle{The theoretical status of latent variables}.
\bjournal{Psychological Review}
\bvolume{110}
\bpages{203--219}.
\end{barticle}
\endbibitem

\bibitem{ref_Bowman}
\begin{binproceedings}[author]
\bauthor{\bsnm{Bowman},~\bfnm{S.~R.}\binits{S.~R.}},
  \bauthor{\bsnm{Vilnis},~\bfnm{L.}\binits{L.}},
  \bauthor{\bsnm{Vinyals},~\bfnm{O.}\binits{O.}},
  \bauthor{\bsnm{Dai},~\bfnm{A.~M.}\binits{A.~M.}},
  \bauthor{\bsnm{Jozefowicz},~\bfnm{R.}\binits{R.}} \AND
  \bauthor{\bsnm{Bengio},~\bfnm{S.}\binits{S.}}
(\byear{2016}).
\btitle{Generating sentences from a continuous space}
In \bbooktitle{SIGNLL Conference on Computational Natural Language Learning
  (CONLL), 2016}.
\end{binproceedings}
\endbibitem

\bibitem{ref_Cattell78}
\begin{bbook}[author]
\bauthor{\bsnm{Cattell},~\bfnm{R.~B.}\binits{R.~B.}}
(\byear{1978}).
\btitle{{The Scientific Use of Factor Analysis in Behavioral and Life
  Sciences}}.
\bpublisher{{New York: Plenum Press}}.
\end{bbook}
\endbibitem

\bibitem{ref_Chen}
\begin{barticle}[author]
\bauthor{\bsnm{Chen},~\bfnm{B.}\binits{B.}},
  \bauthor{\bsnm{Polatkan},~\bfnm{G.}\binits{G.}},
  \bauthor{\bsnm{Shapiro},~\bfnm{G.}\binits{G.}},
  \bauthor{\bsnm{Blei},~\bfnm{D.}\binits{D.}},
  \bauthor{\bsnm{Dunson},~\bfnm{D.}\binits{D.}} \AND
  \bauthor{\bsnm{Carin},~\bfnm{L.}\binits{L.}}
(\byear{2013}).
\btitle{Deep learning with hierarchical convolutional factor analysis}.
\bjournal{{IEEE Transactions on Pattern Analysis and Machine Intelligence}}
\bvolume{35}
\bpages{1887--1901}.
\end{barticle}
\endbibitem

\bibitem{ref_Cramer46}
\begin{bbook}[author]
\bauthor{\bsnm{Cram\'{e}r},~\bfnm{H.}\binits{H.}}
(\byear{1946}).
\btitle{{Mathematical Methods of Statistics}}.
\bpublisher{{Princeton: Princeton University Press}}.
\end{bbook}
\endbibitem

\bibitem{ref_BoltzFA}
\begin{bincollection}[author]
\bauthor{\bsnm{Cueto},~\bfnm{M.~A.}\binits{M.~A.}},
  \bauthor{\bsnm{Morton},~\bfnm{J.}\binits{J.}} \AND
  \bauthor{\bsnm{Sturmfels},~\bfnm{B.}\binits{B.}}
(\byear{2010}).
\btitle{Geometry of the restricted {Boltzmann machine}}.
In \bbooktitle{{Algebraic Methods in Statistics and Probability II}},
(\beditor{\bfnm{M.}\binits{M.}~\bsnm{Viana}} \AND
  \beditor{\bfnm{H.}\binits{H.}~\bsnm{Wynn}}, eds.).
\bseries{{Contemporary Mathematics}}
\bvolume{516}
\bpages{135--153}.
\bpublisher{{Providence: American Mathematical Society}}.
\end{bincollection}
\endbibitem

\bibitem{ref_Dai20}
\begin{binproceedings}[author]
\bauthor{\bsnm{Dai},~\bfnm{B.}\binits{B.}},
  \bauthor{\bsnm{Wang},~\bfnm{Z.}\binits{Z.}} \AND
  \bauthor{\bsnm{Wipf},~\bfnm{D.}\binits{D.}}
(\byear{2020}).
\btitle{{The usual suspects? Reassessing blame for VAE posterior collapse}}
In \bbooktitle{{Proceedings of the 37th International Conference on Machine
  Learning}}.
\end{binproceedings}
\endbibitem

\bibitem{ref_DaiDutta}
\begin{barticle}[author]
\bauthor{\bsnm{Dai},~\bfnm{F.}\binits{F.}},
  \bauthor{\bsnm{Dutta},~\bfnm{S.}\binits{S.}} \AND
  \bauthor{\bsnm{Maitra},~\bfnm{R.}\binits{R.}}
(\byear{2020}).
\btitle{A matrix-free likelihood method for exploratory factor analysis of
  high-dimensional {G}aussian data}.
\bjournal{Journal of Computational and Graphical Statistics}
\bvolume{29}
\bpages{675--680}.
\end{barticle}
\endbibitem

\bibitem{ref_Helm96}
\begin{barticle}[author]
\bauthor{\bsnm{Dayan},~\bfnm{P.}\binits{P.}} \AND
  \bauthor{\bsnm{Hinton},~\bfnm{G.~E.}\binits{G.~E.}}
(\byear{1996}).
\btitle{Varieties of {Helmholtz} machine}.
\bjournal{Neural Networks}
\bvolume{8}
\bpages{1385--1403}.
\end{barticle}
\endbibitem

\bibitem{ref_GramPino}
\begin{barticle}[author]
\bauthor{\bsnm{{del Pino}},~\bfnm{G.~E.}\binits{G.~E.}} \AND
  \bauthor{\bsnm{Galaz},~\bfnm{H.}\binits{H.}}
(\byear{1995}).
\btitle{Statistical applications of the inverse {Gram} matrix: {A}
  revisitation}.
\bjournal{Brazilian Journal of Probability and Statistics}
\bvolume{9}
\bpages{177--196}.
\end{barticle}
\endbibitem

\bibitem{ref_DelBeng}
\begin{binproceedings}[author]
\bauthor{\bsnm{Delalleau},~\bfnm{O.}\binits{O.}} \AND
  \bauthor{\bsnm{Bengio},~\bfnm{Y.}\binits{Y.}}
(\byear{2011}).
\btitle{{Shallow vs. deep sum-product networks}}
In \bbooktitle{{Advances in Neural Information Processing Systems 24 (NIPS
  2011)}}.
\end{binproceedings}
\endbibitem

\bibitem{ref_DFeffect}
\begin{barticle}[author]
\bauthor{\bsnm{Diaconis},~\bfnm{P.}\binits{P.}} \AND
  \bauthor{\bsnm{Freedman},~\bfnm{D.}\binits{D.}}
(\byear{1984}).
\btitle{Asymptotics of graphical projection pursuit}.
\bjournal{Annals of Statistics}
\bvolume{12}
\bpages{793--815}.
\end{barticle}
\endbibitem

\bibitem{ref_GEOGROU}
\begin{bbook}[author]
\bauthor{\bsnm{Dieudonn\'{e}},~\bfnm{J.}\binits{J.}}
(\byear{1963}).
\btitle{{La G\'{e}om\'{e}trie des Groupes Classiques}}.
\bpublisher{{Berlin: Springer-Verlag} (2nd ed.)}.
\end{bbook}
\endbibitem

\bibitem{ref_Donoho2000}
\begin{bmisc}[author]
\bauthor{\bsnm{Donoho},~\bfnm{D.~L.}\binits{D.~L.}}
(\byear{2000}).
\btitle{High-dimensional data analysis: The curses and blessings of
  dimensionality. \emph{Invited lecture at Mathematical Challenges of the 21st
  Century, AMS National Meeting, Los Angeles, CA, USA}}.
\end{bmisc}
\endbibitem

\bibitem{ref_Duncan1944}
\begin{barticle}[author]
\bauthor{\bsnm{Duncan},~\bfnm{W.~J.}\binits{W.~J.}}
(\byear{1944}).
\btitle{Some devices for the solution of large sets of simultaneous linear
  equations (with an appendix on the reciprocation of partitioned matrices)}.
\bjournal{London, Edinburgh, and Dublin Philosophical Magazine and Journal of
  Science: Series 7}
\bvolume{35}
\bpages{660--670}.
\end{barticle}
\endbibitem

\bibitem{ref_EBG78}
\begin{barticle}[author]
\bauthor{\bsnm{Elffers},~\bfnm{H.}\binits{H.}},
  \bauthor{\bsnm{Bethlehem},~\bfnm{J.}\binits{J.}} \AND
  \bauthor{\bsnm{Gill},~\bfnm{R.}\binits{R.}}
(\byear{1978}).
\btitle{Indeterminacy problems and the interpretation of factor analysis
  results}.
\bjournal{Statistica Neerlandica}
\bvolume{32}
\bpages{181--199}.
\end{barticle}
\endbibitem

\bibitem{ref_PreTrain}
\begin{barticle}[author]
\bauthor{\bsnm{Erhan},~\bfnm{D.}\binits{D.}},
  \bauthor{\bsnm{Bengio},~\bfnm{Y.}\binits{Y.}},
  \bauthor{\bsnm{Courville},~\bfnm{A.}\binits{A.}},
  \bauthor{\bsnm{Manzagol},~\bfnm{P.~A.}\binits{P.~A.}},
  \bauthor{\bsnm{Vincent},~\bfnm{P.}\binits{P.}} \AND
  \bauthor{\bsnm{Bengio},~\bfnm{S.}\binits{S.}}
(\byear{2010}).
\btitle{Why does unsupervised pre-training help deep learning?}
\bjournal{Journal of Machine Learning Research}
\bvolume{11}
\bpages{625--660}.
\end{barticle}
\endbibitem

\bibitem{ref_Manifold}
\begin{barticle}[author]
\bauthor{\bsnm{Fefferman},~\bfnm{C.}\binits{C.}},
  \bauthor{\bsnm{Mitter},~\bfnm{S.}\binits{S.}} \AND
  \bauthor{\bsnm{Narayanan},~\bfnm{H.}\binits{H.}}
(\byear{2016}).
\btitle{Testing the manifold hypothesis}.
\bjournal{Journal of the American Mathematical Society}
\bvolume{29}
\bpages{983--1049}.
\end{barticle}
\endbibitem

\bibitem{ref_GZ1980}
\begin{barticle}[author]
\bauthor{\bsnm{Geweke},~\bfnm{J.~F.}\binits{J.~F.}} \AND
  \bauthor{\bsnm{Singleton},~\bfnm{K.~J.}\binits{K.~J.}}
(\byear{1980}).
\btitle{Interpreting the likelihood ratio statistic in factor models when
  sample size is small}.
\bjournal{Journal of the American Statistical Association}
\bvolume{75}
\bpages{133--137}.
\end{barticle}
\endbibitem

\bibitem{DL16}
\begin{bbook}[author]
\bauthor{\bsnm{Goodfellow},~\bfnm{I.}\binits{I.}},
  \bauthor{\bsnm{Bengio},~\bfnm{Y.}\binits{Y.}} \AND
  \bauthor{\bsnm{Courville},~\bfnm{A.}\binits{A.}}
(\byear{2016}).
\btitle{{Deep Learning}}.
\bpublisher{{Cambridge, MT: MIT Press}}.
\end{bbook}
\endbibitem

\bibitem{ref_Garbon18}
\begin{barticle}[author]
\bauthor{\bsnm{Gorban},~\bfnm{A.~N.}\binits{A.~N.}} \AND
  \bauthor{\bsnm{Tyukin},~\bfnm{I.~Y.}\binits{I.~Y.}}
(\byear{2018}).
\btitle{Blessing of dimensionality: {Mathematical} foundations of the
  statistical physics of data}.
\bjournal{Philosophical Transactions of the Royal Society A}
\bvolume{376}
\bpages{20170237}.
\end{barticle}
\endbibitem

\bibitem{ref_Gut2005Main}
\begin{bbook}[author]
\bauthor{\bsnm{Gut},~\bfnm{A.}\binits{A.}}
(\byear{2005}).
\btitle{{Probability: A Graduate Course}}.
\bpublisher{{New York: Springer}}.
\end{bbook}
\endbibitem

\bibitem{ref_Guttman1955}
\begin{barticle}[author]
\bauthor{\bsnm{Guttman},~\bfnm{L.}\binits{L.}}
(\byear{1955}).
\btitle{The determinacy of factor score matrices with implications for five
  other basic problems of common factor theory}.
\bjournal{British Journal of Mathematical and Statistical Psychology}
\bvolume{8}
\bpages{65--81}.
\end{barticle}
\endbibitem

\bibitem{ref_HaigAbductive}
\begin{barticle}[author]
\bauthor{\bsnm{Haig},~\bfnm{B.~D.}\binits{B.~D.}}
(\byear{2005}).
\btitle{Exploratory factor analysis, theory generation, and scientific method}.
\bjournal{{Multivariate Behavioral Research}}
\bvolume{40}
\bpages{303--329}.
\end{barticle}
\endbibitem

\bibitem{ref_HLeffect}
\begin{barticle}[author]
\bauthor{\bsnm{Hall},~\bfnm{P.}\binits{P.}} \AND
  \bauthor{\bsnm{Li},~\bfnm{K.~C.}\binits{K.~C.}}
(\byear{1993}).
\btitle{On almost linearity of low-dimensional projections from
  high-dimensional data}.
\bjournal{Annals of Statistics}
\bvolume{21}
\bpages{867--889}.
\end{barticle}
\endbibitem

\bibitem{ref_Push}
\begin{barticle}[author]
\bauthor{\bsnm{Henderson},~\bfnm{H.~V.}\binits{H.~V.}} \AND
  \bauthor{\bsnm{Searle},~\bfnm{S.~R.}\binits{S.~R.}}
(\byear{1981}).
\btitle{On deriving the inverse of a sum of matrices}.
\bjournal{{SIAM Review}}
\bvolume{23}
\bpages{53--60}.
\end{barticle}
\endbibitem

\bibitem{ref_PROMAX}
\begin{barticle}[author]
\bauthor{\bsnm{Hendrickson},~\bfnm{A.~E.}\binits{A.~E.}} \AND
  \bauthor{\bsnm{White},~\bfnm{P.~O.}\binits{P.~O.}}
(\byear{1964}).
\btitle{{PROMAX: A quick method for rotation to oblique simple structure}}.
\bjournal{British Journal of Mathematical \& Statistical Psychology}
\bvolume{17}
\bpages{65--70}.
\end{barticle}
\endbibitem

\bibitem{ref_HG97}
\begin{barticle}[author]
\bauthor{\bsnm{Hinton},~\bfnm{G.~E.}\binits{G.~E.}} \AND
  \bauthor{\bsnm{Ghahramani},~\bfnm{Z.}\binits{Z.}}
(\byear{1997}).
\btitle{Generative models for discovering sparse distributed representations}.
\bjournal{Philosophical Transactions of the Royal Society of London, B}
\bvolume{352}
\bpages{1177--1190}.
\end{barticle}
\endbibitem

\bibitem{ref_Hinton2006}
\begin{barticle}[author]
\bauthor{\bsnm{Hinton},~\bfnm{G.~E.}\binits{G.~E.}} \AND
  \bauthor{\bsnm{Salakhutdinov},~\bfnm{R.~R.}\binits{R.~R.}}
(\byear{2006}).
\btitle{Reducing the dimensionality of data with neural networks}.
\bjournal{Science}
\bvolume{313}
\bpages{504--507}.
\end{barticle}
\endbibitem

\bibitem{ref_Yama2015}
\begin{barticle}[author]
\bauthor{\bsnm{Hirose},~\bfnm{K.}\binits{K.}} \AND
  \bauthor{\bsnm{Yamamoto},~\bfnm{M.}\binits{M.}}
(\byear{2015}).
\btitle{Sparse estimation via nonconcave penalized likelihood in factor
  analysis model}.
\bjournal{Statistics \& Computing}
\bvolume{25}
\bpages{863--875}.
\end{barticle}
\endbibitem

\bibitem{ref_HJ85}
\begin{bbook}[author]
\bauthor{\bsnm{Horn},~\bfnm{R.~A.}\binits{R.~A.}} \AND
  \bauthor{\bsnm{Johnson},~\bfnm{C.~A.}\binits{C.~A.}}
(\byear{1985}).
\btitle{{Matrix Analysis}}.
\bpublisher{Cambridge: Cambdridge University Press}.
\end{bbook}
\endbibitem

\bibitem{ref_Universal}
\begin{barticle}[author]
\bauthor{\bsnm{Hornik},~\bfnm{K.}\binits{K.}},
  \bauthor{\bsnm{Stinchcombe},~\bfnm{M.}\binits{M.}} \AND
  \bauthor{\bsnm{White},~\bfnm{H.}\binits{H.}}
(\byear{1989}).
\btitle{Multilayer feedforward networks are universal approximators}.
\bjournal{Neural Networks}
\bvolume{2}
\bpages{359--366}.
\end{barticle}
\endbibitem

\bibitem{ref_Howe55}
\begin{btechreport}[author]
\bauthor{\bsnm{Howe},~\bfnm{W.~G.}\binits{W.~G.}}
(\byear{1955}).
\btitle{Some contributions to factor analysis.}
\btype{Technical Report} No. \bnumber{ORNL-1919},
\bpublisher{{Oak Ridge National Laboratory}}.
\end{btechreport}
\endbibitem

\bibitem{ref_IvaDeep}
\begin{bbook}[author]
\bauthor{\bsnm{Ivakhnenko},~\bfnm{A.~G.}\binits{A.~G.}} \AND
  \bauthor{\bsnm{Lapa},~\bfnm{V.~G.}\binits{V.~G.}}
(\byear{1967}).
\btitle{{Cybernetics and Forecasting Techniques}}.
\bpublisher{{New York: American Elsevier Publishing co.}}
\end{bbook}
\endbibitem

\bibitem{ref_Jenn69}
\begin{barticle}[author]
\bauthor{\bsnm{Jennrich},~\bfnm{R.~I.}\binits{R.~I.}} \AND
  \bauthor{\bsnm{Robinson},~\bfnm{S.~M.}\binits{S.~M.}}
(\byear{1969}).
\btitle{A {Newton-Rhapson} algorithm for maximum likelihood factor analysis}.
\bjournal{Psychometrika}
\bvolume{34}
\bpages{111--123}.
\end{barticle}
\endbibitem

\bibitem{ref_Joreskogg67}
\begin{barticle}[author]
\bauthor{\bsnm{J\"{o}reskog},~\bfnm{K.~G.}\binits{K.~G.}}
(\byear{1967}).
\btitle{Some contributions to maximum likelihood factor analysis}.
\bjournal{Psychometrika}
\bvolume{32}
\bpages{443--482}.
\end{barticle}
\endbibitem

\bibitem{ref_Kainen1997}
\begin{bincollection}[author]
\bauthor{\bsnm{Kainen},~\bfnm{P.~C.}\binits{P.~C.}}
(\byear{1997}).
\btitle{Utilizing Geometric Anomalies of High Dimension: {When} Complexity
  Makes Computation Easier}.
In \bbooktitle{Computer Intensive Methods in Control and Signal Processing}
(\beditor{\bfnm{M.}\binits{M.}~\bsnm{K\'{a}rn\'{y}}} \AND
  \beditor{\bfnm{K.}\binits{K.}~\bsnm{Warwick}}, eds.)
\bchapter{18},
\bpages{283--294}.
\bpublisher{Boston, MA: Birkh\"{a}user}.
\end{bincollection}
\endbibitem

\bibitem{ref_KainenKurk}
\begin{barticle}[author]
\bauthor{\bsnm{Kainen},~\bfnm{P.~C.}\binits{P.~C.}} \AND
  \bauthor{\bsnm{K\r{u}rkov\'{a}},~\bfnm{V.}\binits{V.}}
(\byear{1993}).
\btitle{Quasiorthogonal dimension of {Euclidian} spaces}.
\bjournal{Applied Mathematics Letters}
\bvolume{6}
\bpages{7--10}.
\end{barticle}
\endbibitem

\bibitem{ref_VARIMAX}
\begin{barticle}[author]
\bauthor{\bsnm{Kaiser},~\bfnm{H.~F.}\binits{H.~F.}}
(\byear{1958}).
\btitle{{The Varimax criterion for analytic rotation in factor analysis}}.
\bjournal{Psychometrika}
\bvolume{23}
\bpages{187--200}.
\end{barticle}
\endbibitem

\bibitem{ref_KMO}
\begin{barticle}[author]
\bauthor{\bsnm{Kaiser},~\bfnm{H.~F.}\binits{H.~F.}} \AND
  \bauthor{\bsnm{Rice},~\bfnm{J.}\binits{J.}}
(\byear{1974}).
\btitle{Little jiffy, mark {IV}}.
\bjournal{Educational \& Psychological Measurement}
\bvolume{34}
\bpages{111--117}.
\end{barticle}
\endbibitem

\bibitem{ref_Kato76}
\begin{bbook}[author]
\bauthor{\bsnm{Kato},~\bfnm{T.}\binits{T.}}
(\byear{1976}).
\btitle{{Perturbation Theory for Linear Operators}}.
\bpublisher{{New-York: Springer-Verlag}}.
\end{bbook}
\endbibitem

\bibitem{ref_Kestelman1952}
\begin{barticle}[author]
\bauthor{\bsnm{Kestelman},~\bfnm{H.}\binits{H.}}
(\byear{1952}).
\btitle{The fundamental equation of factor analysis}.
\bjournal{British Journal of Mathematical and Statistical Psychology}
\bvolume{5}
\bpages{1--6}.
\end{barticle}
\endbibitem

\bibitem{ref_KingmaNeu}
\begin{binproceedings}[author]
\bauthor{\bsnm{Kingma},~\bfnm{D.}\binits{D.}},
  \bauthor{\bsnm{Salimans},~\bfnm{T.}\binits{T.}},
  \bauthor{\bsnm{Josefowicz},~\bfnm{R.}\binits{R.}},
  \bauthor{\bsnm{Chen},~\bfnm{X.}\binits{X.}},
  \bauthor{\bsnm{Sutskever},~\bfnm{I.}\binits{I.}} \AND
  \bauthor{\bsnm{Welling},~\bfnm{M}\binits{M.}}
(\byear{2016}).
\btitle{Improving variational autoencoders with inverse autoregressive flow}.
In \bbooktitle{{30th Annual Conference on Neural Information Processing
  Systems}}
(\beditor{\bfnm{D.}\binits{D.}~\bsnm{Lee}},
  \beditor{\bfnm{U}\binits{U.}~\bsnm{{von Luxburg}}},
  \beditor{\bfnm{R.}\binits{R.}~\bsnm{Garnett}},
  \beditor{\bfnm{M.}\binits{M.}~\bsnm{Sugiyama}} \AND
  \beditor{\bfnm{I.}\binits{I.}~\bsnm{Guyon}}, eds.).
\bseries{Advances in Neural Information Processing Systems}
\bvolume{29}.
\end{binproceedings}
\endbibitem

\bibitem{ref_KWel}
\begin{binproceedings}[author]
\bauthor{\bsnm{Kingma},~\bfnm{D.~P.}\binits{D.~P.}} \AND
  \bauthor{\bsnm{Welling},~\bfnm{M.}\binits{M.}}
(\byear{2014}).
\btitle{Auto-encoding variational {Bayes}}
In \bbooktitle{International Conference on Learning Representations (ICLR)
  2014}.
\end{binproceedings}
\endbibitem

\bibitem{ref_DeepLVMs}
\begin{barticle}[author]
\bauthor{\bsnm{Kingma},~\bfnm{D.~P.}\binits{D.~P.}} \AND
  \bauthor{\bsnm{Welling},~\bfnm{M.}\binits{M.}}
(\byear{2019}).
\btitle{An introduction to variational autoencoders}.
\bjournal{Foundations and Trends in Machine Learning}
\bvolume{12}
\bpages{307--392}.
\end{barticle}
\endbibitem

\bibitem{ref_CM}
\begin{barticle}[author]
\bauthor{\bsnm{Kong},~\bfnm{X.}\binits{X.}},
  \bauthor{\bsnm{Jiang},~\bfnm{X.}\binits{X.}},
  \bauthor{\bsnm{Zhang},~\bfnm{B.}\binits{B.}},
  \bauthor{\bsnm{Yuan},~\bfnm{J.}\binits{J.}} \AND
  \bauthor{\bsnm{Ge},~\bfnm{Z.}\binits{Z.}}
(\byear{2022}).
\btitle{{Latent variable models in the era of industrial big data: Extension
  and beyond}}.
\bjournal{Annual Reviews in Control}
\bvolume{54}
\bpages{167--199}.
\end{barticle}
\endbibitem

\bibitem{ref_Koster1999}
\begin{barticle}[author]
\bauthor{\bsnm{Koster},~\bfnm{J.~T.~A.}\binits{J.~T.~A.}}
(\byear{1999}).
\btitle{{On the validity of the Markov interpretation of path diagrams of
  Gaussian structural equations systems with correlated errors}}.
\bjournal{Scandinavian Journal of Statistics}
\bvolume{26}
\bpages{413--431}.
\end{barticle}
\endbibitem

\bibitem{ref_KrijnenHeywood}
\begin{barticle}[author]
\bauthor{\bsnm{Krijnen},~\bfnm{W.~P.}\binits{W.~P.}},
  \bauthor{\bsnm{Dijkstra},~\bfnm{T.~K.}\binits{T.~K.}} \AND
  \bauthor{\bsnm{Gill},~\bfnm{R.~D.}\binits{R.~D.}}
(\byear{1998}).
\btitle{Conditions for factor (in)determinacy in factor analysis}.
\bjournal{Psychometrika}
\bvolume{63}
\bpages{359--367}.
\end{barticle}
\endbibitem

\bibitem{ref_KuaNominal}
\begin{barticle}[author]
\bauthor{\bsnm{Kuo},~\bfnm{Z.~Y.}\binits{Z.~Y.}}
(\byear{1921}).
\btitle{Giving up instincts in psychology}.
\bjournal{Journal of Philosophy}
\bvolume{18}
\bpages{645--664}.
\end{barticle}
\endbibitem

\bibitem{ref_Lawley40}
\begin{barticle}[author]
\bauthor{\bsnm{Lawley},~\bfnm{D.~N.}\binits{D.~N.}}
(\byear{1940}).
\btitle{The estimation of factor loadings by the method of {Maximum
  Likelihood}}.
\bjournal{Proceedings of the Royal Society of Edinburgh}
\bvolume{60}
\bpages{64--82}.
\end{barticle}
\endbibitem

\bibitem{ref_Lucas}
\begin{binproceedings}[author]
\bauthor{\bsnm{Lucas},~\bfnm{J.}\binits{J.}},
  \bauthor{\bsnm{Tucker},~\bfnm{G.}\binits{G.}},
  \bauthor{\bsnm{Grosse},~\bfnm{R.}\binits{R.}} \AND
  \bauthor{\bsnm{Narauzi},~\bfnm{M.}\binits{M.}}
(\byear{2019}).
\btitle{Understanding posterior collapse in generative latent variable models}
In \bbooktitle{International Conference on Learning Representations (ICLR)
  2019}.
\end{binproceedings}
\endbibitem

\bibitem{ref_Maraun1996}
\begin{barticle}[author]
\bauthor{\bsnm{Maraun},~\bfnm{M.~D.}\binits{M.~D.}}
(\byear{1996}).
\btitle{Metaphor taken as math: {I}ndeterminancy in the factor analysis model}.
\bjournal{Multivariate Behavioral Research}
\bvolume{31}
\bpages{517--538}.
\end{barticle}
\endbibitem

\bibitem{ref_HUAident}
\begin{barticle}[author]
\bauthor{\bsnm{{McCrimmon}},~\bfnm{K.}\binits{K.}}
(\byear{1978}).
\btitle{Jordan algebras and their applications}.
\bjournal{Bulletin of the American Mathematical Society}
\bvolume{84}
\bpages{612--627}.
\end{barticle}
\endbibitem

\bibitem{ref_LSscores}
\begin{barticle}[author]
\bauthor{\bsnm{McDonald},~\bfnm{R.~P.}\binits{R.~P.}} \AND
  \bauthor{\bsnm{Burr},~\bfnm{E.~J.}\binits{E.~J.}}
(\byear{1967}).
\btitle{A comparison of four methods of constructing factor scores}.
\bjournal{Psychometrika}
\bvolume{32}
\bpages{381--401}.
\end{barticle}
\endbibitem

\bibitem{ref_MM_Indet}
\begin{barticle}[author]
\bauthor{\bsnm{{McDonald}},~\bfnm{R.~P.}\binits{R.~P.}} \AND
  \bauthor{\bsnm{Mulaik},~\bfnm{S.~A.}\binits{S.~A.}}
(\byear{1979}).
\btitle{Determinacy of common factors: {A nontechnical review}}.
\bjournal{Psychological Bulletin}
\bvolume{86}
\bpages{297--306}.
\end{barticle}
\endbibitem

\bibitem{ref_MW2020}
\begin{barticle}[author]
\bauthor{\bsnm{{McNeish}},~\bfnm{D.}\binits{D.}} \AND
  \bauthor{\bsnm{Wolf},~\bfnm{M.~G.}\binits{M.~G.}}
(\byear{2020}).
\btitle{Thinking twice about sum scores}.
\bjournal{Behavior Research Methods}
\bvolume{52}
\bpages{2287--2305}.
\end{barticle}
\endbibitem

\bibitem{Milman}
\begin{bbook}[author]
\bauthor{\bsnm{Milman},~\bfnm{V.~D.}\binits{V.~D.}} \AND
  \bauthor{\bsnm{Schechtman},~\bfnm{G.}\binits{G.}}
(\byear{1986}).
\btitle{{Asymptotic Theory of Finite Dimensional Normed Spaces}}.
\bseries{Lecture Notes in Mathematics}
\bvolume{1200}.
\bpublisher{{Berlin: Springer-Verlag}}.
\end{bbook}
\endbibitem

\bibitem{ref_Mulaik1976}
\begin{barticle}[author]
\bauthor{\bsnm{Mulaik},~\bfnm{S.~A.}\binits{S.~A.}}
(\byear{1976}).
\btitle{Comments on ``The measurement of factorial indeterminacy"}.
\bjournal{Psychometrika}
\bvolume{41}
\bpages{249--262}.
\end{barticle}
\endbibitem

\bibitem{ref_Mulaik05}
\begin{bincollection}[author]
\bauthor{\bsnm{Mulaik},~\bfnm{S.~A.}\binits{S.~A.}}
(\byear{2005}).
\btitle{Looking back on the indeterminacy controversies in factor analysis}.
In \bbooktitle{{Contemporary Psychometrics: A Festschrift for Roderick P.\
  McDonald}}
(\beditor{\bfnm{A.}\binits{A.}~\bsnm{{Maydeu-Olivares}}} \AND
  \beditor{\bfnm{J.~J.}\binits{J.~J.}~\bsnm{{McArdle}}}, eds.)
\bchapter{7},
\bpages{174--206}.
\bpublisher{{Mahwah New Jersey: Lawrence Erlbaum Associates}}.
\end{bincollection}
\endbibitem

\bibitem{ref_Mulaik2010}
\begin{bbook}[author]
\bauthor{\bsnm{Mulaik},~\bfnm{S.~A.}\binits{S.~A.}}
(\byear{2010}).
\btitle{{Foundations of Factor Analysis}}.
\bpublisher{{Boca Raton, FL: Chapman Hall/CRC} (2nd ed.)}.
\end{bbook}
\endbibitem

\bibitem{ref_Munch_FR}
\begin{barticle}[author]
\bauthor{\bsnm{M\"{u}nch},~\bfnm{M.~M.}\binits{M.~M.}}, \bauthor{\bsnm{{van de
  Wiel}},~\bfnm{M.~A.}\binits{M.~A.}}, \bauthor{\bsnm{{van der
  Vaart}},~\bfnm{A.~W.}\binits{A.~W.}} \AND
  \bauthor{\bsnm{Peeters},~\bfnm{C.~F.~W.}\binits{C.~F.~W.}}
(\byear{2022}).
\btitle{Semi-supervised empirical {Bayes} group-regularized factor regression}.
\bjournal{Biometrical Journal}
\bvolume{64}
\bpages{1289--1306}.
\end{barticle}
\endbibitem

\bibitem{ref_Nakajma}
\begin{barticle}[author]
\bauthor{\bsnm{Nakajima},~\bfnm{S.}\binits{S.}},
  \bauthor{\bsnm{Tomioka},~\bfnm{R.}\binits{R.}},
  \bauthor{\bsnm{Sugiyama},~\bfnm{M.}\binits{M.}} \AND
  \bauthor{\bsnm{Babacan},~\bfnm{S.~D.}\binits{S.~D.}}
(\byear{2015}).
\btitle{Condition for perfect dimensionality recovery by variational {Bayesian
  PCA}}.
\bjournal{Journal of Machine Learning Research}
\bvolume{16}
\bpages{3757--3811}.
\end{barticle}
\endbibitem

\bibitem{ref_Pearl}
\begin{bbook}[author]
\bauthor{\bsnm{Pearl},~\bfnm{J.}\binits{J.}}
(\byear{2009}).
\btitle{Causality: {Models, Reasoning, and Inference}}.
\bpublisher{{Cambridge: Cambridge University Press} (2nd ed.)}.
\end{bbook}
\endbibitem

\bibitem{PeetersThesis}
\begin{bphdthesis}[author]
\bauthor{\bsnm{Peeters},~\bfnm{C.~F.~W.}\binits{C.~F.~W.}}
(\byear{2012}).
\btitle{{Bayesian Exploratory and Confirmatory Factor Analysis: Perspectives on
  Constrained-Model Selection}},
\btype{PhD thesis},
\bpublisher{{Dept.\ of Methodology \& Statistics, Utrecht University}}.
\end{bphdthesis}
\endbibitem

\bibitem{ref_Peeters12}
\begin{barticle}[author]
\bauthor{\bsnm{Peeters},~\bfnm{C.~F.~W.}\binits{C.~F.~W.}}
(\byear{2012}).
\btitle{Rotational uniqueness conditions under oblique factor correlation
  metric}.
\bjournal{Psychometrika}
\bvolume{77}
\bpages{288--292}.
\end{barticle}
\endbibitem

\bibitem{ref_Piaggo1931}
\begin{barticle}[author]
\bauthor{\bsnm{Piaggio},~\bfnm{H.~T.~H.}\binits{H.~T.~H.}}
(\byear{1931}).
\btitle{The general factor in {Spearman's} theory of intelligence}.
\bjournal{Nature}
\bvolume{127}
\bpages{56--57}.
\end{barticle}
\endbibitem

\bibitem{ref_Piaggo1933}
\begin{barticle}[author]
\bauthor{\bsnm{Piaggio},~\bfnm{H.~T.~H.}\binits{H.~T.~H.}}
(\byear{1933}).
\btitle{Three sets of conditions necessary for the existence of a $g$ that is
  real and unique except in sign}.
\bjournal{British Journal of Psychology}
\bvolume{24}
\bpages{88--105}.
\end{barticle}
\endbibitem

\bibitem{ref_Rao55}
\begin{barticle}[author]
\bauthor{\bsnm{Rao},~\bfnm{C.~R.}\binits{C.~R.}}
(\byear{1955}).
\btitle{Estimation and tests of significance in factor analysis}.
\bjournal{Psychometrika}
\bvolume{20}
\bpages{93--111}.
\end{barticle}
\endbibitem

\bibitem{ref_RaoPC99}
\begin{barticle}[author]
\bauthor{\bsnm{Rao},~\bfnm{R.~P.}\binits{R.~P.}} \AND
  \bauthor{\bsnm{Bollard},~\bfnm{D.~H.}\binits{D.~H.}}
(\byear{1999}).
\btitle{Predictive coding in the visual cortex: {A} functional interpretation
  of some extra-classical receptive field effects}.
\bjournal{Nature Neuroscience}
\bvolume{2}
\bpages{79--87}.
\end{barticle}
\endbibitem

\bibitem{ref_Ravazi}
\begin{binproceedings}[author]
\bauthor{\bsnm{Ravazi},~\bfnm{A.}\binits{A.}}, \bauthor{\bsnm{{van den
  Oord}},~\bfnm{A.}\binits{A.}}, \bauthor{\bsnm{Poole},~\bfnm{B.}\binits{B.}}
  \AND \bauthor{\bsnm{Vinyals},~\bfnm{O.}\binits{O.}}
(\byear{2019}).
\btitle{Preventing posterior collapse with {delta-VAEs}}
In \bbooktitle{International Conference on Learning Representations (ICLR)
  2019}.
\end{binproceedings}
\endbibitem

\bibitem{ref_Reiersol}
\begin{barticle}[author]
\bauthor{\bsnm{Reiers{\o}l},~\bfnm{O.}\binits{O.}}
(\byear{1950}).
\btitle{On the identifiability of parameters in {Thurstone's} multiple factor
  analysis}.
\bjournal{Psychometrika}
\bvolume{15}
\bpages{121--149}.
\end{barticle}
\endbibitem

\bibitem{ref_Rich2003}
\begin{barticle}[author]
\bauthor{\bsnm{Richardson},~\bfnm{T.}\binits{T.}}
(\byear{2003}).
\btitle{Markov properties for acyclic directed mixed graphs}.
\bjournal{Scandinavian Journal of Statistics}
\bvolume{30}
\bpages{145--157}.
\end{barticle}
\endbibitem

\bibitem{ref_Robertson07}
\begin{barticle}[author]
\bauthor{\bsnm{Robertson},~\bfnm{D.}\binits{D.}} \AND
  \bauthor{\bsnm{Symons},~\bfnm{J.}\binits{J.}}
(\byear{2007}).
\btitle{Maximum likelihood factor analysis with rank-deficient sample
  covariance matrices}.
\bjournal{Journal of Multivariate Analysis}
\bvolume{98}
\bpages{813--828}.
\end{barticle}
\endbibitem

\bibitem{ref_AImodern}
\begin{bbook}[author]
\bauthor{\bsnm{Russell},~\bfnm{S.}\binits{S.}} \AND
  \bauthor{\bsnm{Norvig},~\bfnm{P.}\binits{P.}}
(\byear{2022}).
\btitle{Artificial Intelligence: A Modern Approach},
\bedition{4th} ed.
\bpublisher{UK: Pearson Education Limited (Global Edition)}.
\end{bbook}
\endbibitem

\bibitem{ref_SchmidtDeep}
\begin{barticle}[author]
\bauthor{\bsnm{Schmid},~\bfnm{J.}\binits{J.}} \AND
  \bauthor{\bsnm{Leiman},~\bfnm{J.~M.}\binits{J.~M.}}
(\byear{1957}).
\btitle{The development of hierarchical factor solutions}.
\bjournal{Psychometrika}
\bvolume{22}
\bpages{53--61}.
\end{barticle}
\endbibitem

\bibitem{ref_Schonemann1971}
\begin{barticle}[author]
\bauthor{\bsnm{Sch\"{o}nemann},~\bfnm{P.}\binits{P.}}
(\byear{1971}).
\btitle{The minimum average correlation between equivalent sets of uncorrelated
  factors}.
\bjournal{Psychometrika}
\bvolume{36}
\bpages{21--30}.
\end{barticle}
\endbibitem

\bibitem{ref_SW72}
\begin{barticle}[author]
\bauthor{\bsnm{Sch\"{o}nemann},~\bfnm{P.}\binits{P.}} \AND
  \bauthor{\bsnm{Wang},~\bfnm{M.~M.}\binits{M.~M.}}
(\byear{1972}).
\btitle{Some new results on factor indeterminacy}.
\bjournal{Psychometrika}
\bvolume{37}
\bpages{61--91}.
\end{barticle}
\endbibitem

\bibitem{ref_Serre}
\begin{bbook}[author]
\bauthor{\bsnm{Serre},~\bfnm{D.}\binits{D.}}
(\byear{2002}).
\btitle{{Matrices: Theory and Applications}}.
\bpublisher{New York: Springer}.
\end{bbook}
\endbibitem

\bibitem{ref_Spearman1904}
\begin{barticle}[author]
\bauthor{\bsnm{Spearman},~\bfnm{C.}\binits{C.}}
(\byear{1904}).
\btitle{``{G}eneral intelligence," objectively determined and measured.}
\bjournal{American Journal of Psychology}
\bvolume{15}
\bpages{201--293}.
\end{barticle}
\endbibitem

\bibitem{ref_SpearmanAbilities}
\begin{bbook}[author]
\bauthor{\bsnm{Spearman},~\bfnm{C.}\binits{C.}}
(\byear{1927}).
\btitle{The Abilities of Man: Their Nature and Measurement}.
\bpublisher{{New York: The Macmillan Company}}.
\end{bbook}
\endbibitem

\bibitem{ref_Steiger79}
\begin{barticle}[author]
\bauthor{\bsnm{Steiger},~\bfnm{J.~H.}\binits{J.~H.}}
(\byear{1979}).
\btitle{{Factor indeterminacy in the 1930s and 1970s: Some interesting
  parallels}}.
\bjournal{Psychometrika}
\bvolume{44}
\bpages{157--167}.
\end{barticle}
\endbibitem

\bibitem{ref_SS78}
\begin{bincollection}[author]
\bauthor{\bsnm{Steiger},~\bfnm{J.~H.}\binits{J.~H.}} \AND
  \bauthor{\bsnm{Sch\"{o}nemann},~\bfnm{P.~H.}\binits{P.~H.}}
(\byear{1978}).
\btitle{A history of factor indeterminacy}.
In \bbooktitle{Theory Construction and Data Analysis in the Behavioral
  Sciences}
(\beditor{\bfnm{S.}\binits{S.}~\bsnm{Shye}}, ed.)
\bchapter{5},
\bpages{136--178}.
\bpublisher{{San Francisco: Jossey-Bass}}.
\end{bincollection}
\endbibitem

\bibitem{ref_Stevens2002}
\begin{bbook}[author]
\bauthor{\bsnm{Stevens},~\bfnm{J.~P.}\binits{J.~P.}}
(\byear{2002}).
\btitle{Applied Multivariate Statistics for the Social Sciences}.
\bpublisher{{Mahwah New Jersey: Lawrence Erlbaum Associates} (4th ed.)}.
\end{bbook}
\endbibitem

\bibitem{ref_Sundberg16}
\begin{barticle}[author]
\bauthor{\bsnm{Sundberg},~\bfnm{R.}\binits{R.}} \AND
  \bauthor{\bsnm{Feldmann},~\bfnm{U.}\binits{U.}}
(\byear{2016}).
\btitle{Exploratory factor analysis--{P}arameter estimation and scores
  prediction with high-dimensional data}.
\bjournal{Journal of Multivariate Analysis}
\bvolume{148}
\bpages{49--59}.
\end{barticle}
\endbibitem

\bibitem{Tala}
\begin{barticle}[author]
\bauthor{\bsnm{Talagrand},~\bfnm{M.}\binits{M.}}
(\byear{1996}).
\btitle{A new look at independence}.
\bjournal{Annals of Probability}
\bvolume{21}
\bpages{1--34}.
\end{barticle}
\endbibitem

\bibitem{ref_Thomson1939}
\begin{bbook}[author]
\bauthor{\bsnm{Thomson},~\bfnm{G.}\binits{G.}}
(\byear{1939}).
\btitle{{The Factorial Analysis of Human Ability}}.
\bpublisher{{London: University of Londen Press}}.
\end{bbook}
\endbibitem

\bibitem{ref_SparseEFA}
\begin{barticle}[author]
\bauthor{\bsnm{Trendafilov},~\bfnm{N.~T.}\binits{N.~T.}},
  \bauthor{\bsnm{Fontanella},~\bfnm{S.}\binits{S.}} \AND
  \bauthor{\bsnm{Adachi},~\bfnm{K.}\binits{K.}}
(\byear{2017}).
\btitle{Sparse exploratory factor analysis}.
\bjournal{Psychometrika}
\bvolume{82}
\bpages{778--794}.
\end{barticle}
\endbibitem

\bibitem{ref_Udell}
\begin{barticle}[author]
\bauthor{\bsnm{Udell},~\bfnm{M.}\binits{M.}} \AND
  \bauthor{\bsnm{Townsend},~\bfnm{A.}\binits{A.}}
(\byear{2019}).
\btitle{Why are big data matrices approximately low rank?}
\bjournal{{SIAM Journal on Mathematics of Data Science}}
\bvolume{1}
\bpages{144--160}.
\end{barticle}
\endbibitem

\bibitem{ref_Faithfulness}
\begin{barticle}[author]
\bauthor{\bsnm{Uhler},~\bfnm{C.}\binits{C.}},
  \bauthor{\bsnm{Raskutti},~\bfnm{G.}\binits{G.}},
  \bauthor{\bsnm{B{\"u}hlmann},~\bfnm{P.}\binits{P.}} \AND
  \bauthor{\bsnm{Yu},~\bfnm{B.}\binits{B.}}
(\byear{2013}).
\btitle{Geometry of the faithfulness assumption in causal inference}.
\bjournal{Annals of Statistics}
\bvolume{41}
\bpages{436--463}.
\end{barticle}
\endbibitem

\bibitem{ref_Vittadini1989}
\begin{barticle}[author]
\bauthor{\bsnm{Vittadini},~\bfnm{G.}\binits{G.}}
(\byear{1989}).
\btitle{Indeterminacy Problems in the {Lisrel} Model}.
\bjournal{Multivariate Behavioral Research}
\bvolume{24}
\bpages{397--414}.
\end{barticle}
\endbibitem

\bibitem{ref_HanPC18}
\begin{binproceedings}[author]
\bauthor{\bsnm{Wen},~\bfnm{H.}\binits{H.}},
  \bauthor{\bsnm{Han},~\bfnm{K.}\binits{K.}},
  \bauthor{\bsnm{S},~\bfnm{J.}\binits{J.}},
  \bauthor{\bsnm{Shi},~\bfnm{Y.}\binits{Y.} \bsuffix{Zhang}},
  \bauthor{\bsnm{Culurciello},~\bfnm{E.}\binits{E.}} \AND
  \bauthor{\bsnm{Liu},~\bfnm{Z.}\binits{Z.}}
(\byear{2018}).
\btitle{Deep predictive coding network for object recognition}
In \bbooktitle{{Proceedings of the 35th International Conference on Machine
  Learning}}.
\end{binproceedings}
\endbibitem

\bibitem{ref_Wilson1928Science}
\begin{barticle}[author]
\bauthor{\bsnm{Wilson},~\bfnm{E.~B.}\binits{E.~B.}}
(\byear{1928}).
\btitle{{Review of `The abilities of Man, their Nature and Measurement' by C.
  Spearman}}.
\bjournal{Science}
\bvolume{67}
\bpages{244--248}.
\end{barticle}
\endbibitem

\bibitem{ref_Wilson1928PNAS}
\begin{barticle}[author]
\bauthor{\bsnm{Wilson},~\bfnm{E.~B.}\binits{E.~B.}}
(\byear{1928}).
\btitle{On hierarchical correlation systems}.
\bjournal{Proceedings of the National Academy of Sciences}
\bvolume{14}
\bpages{283--291}.
\end{barticle}
\endbibitem

\bibitem{ref_Woodbury}
\begin{btechreport}[author]
\bauthor{\bsnm{Woodbury},~\bfnm{M.~A.}\binits{M.~A.}}
(\byear{1950}).
\btitle{Inverting modified matrices.}
\btype{{Statistical Research Group Memorandum Report}} No. \bnumber{42},
\bpublisher{{Princeton University}}.
\end{btechreport}
\endbibitem

\bibitem{ref_FArepresent}
\begin{barticle}[author]
\bauthor{\bsnm{Xie},~\bfnm{J.}\binits{J.}},
  \bauthor{\bsnm{Gao},~\bfnm{R.}\binits{R.}},
  \bauthor{\bsnm{Nijkamp},~\bfnm{E.}\binits{E.}},
  \bauthor{\bsnm{Zhu},~\bfnm{S.~C.}\binits{S.~C.}} \AND
  \bauthor{\bsnm{Wu},~\bfnm{Y.~N.}\binits{Y.~N.}}
(\byear{2020}).
\btitle{Representation learning: {A} statistical perspective}.
\bjournal{Annual Review of Statistics and Its Applications}
\bvolume{7}
\bpages{11.1--11.33}.
\end{barticle}
\endbibitem

\bibitem{ref_Zhang_bilinear}
\begin{barticle}[author]
\bauthor{\bsnm{Zhang},~\bfnm{K.}\binits{K.}}
(\byear{2002}).
\btitle{A bilinear inequality}.
\bjournal{Journal of Mathematical Analysis and Applications}
\bvolume{271}
\bpages{288--296}.
\end{barticle}
\endbibitem

\end{thebibliography}


\begin{thebibliography}{73}
% BibTex style file: imsart-number.bst, 2017-11-03
% Default style options (sort=1,type=number).
% Used options (sort=1,type=number).

\bibitem{SMref_AndersonBible}
\begin{bbook}[author]
\bauthor{\bsnm{Anderson},~\bfnm{T.~W.}\binits{T.~W.}}
(\byear{2003}).
\btitle{{An Introduction to Multivariate Statistical Analysis}}.
\bpublisher{Hoboken, {NJ}: {J}ohn {W}iley \& {S}ons, {I}nc. (3rd ed.)}.
\end{bbook}
\endbibitem

\bibitem{Babcock2024}
\begin{barticle}[author]
\bauthor{\bsnm{Babcock},~\bfnm{N.~S.}\binits{N.~S.}},
  \bauthor{\bsnm{{Montes-Cabrera}},~\bfnm{G.}\binits{G.}},
  \bauthor{\bsnm{Oberhofer},~\bfnm{K.~E.}\binits{K.~E.}},
  \bauthor{\bsnm{Chergui},~\bfnm{M.}\binits{M.}},
  \bauthor{\bsnm{Celardo},~\bfnm{G.~L.}\binits{G.~L.}} \AND
  \bauthor{\bsnm{Kurian},~\bfnm{P.}\binits{P.}}
(\byear{2024}).
\btitle{Ultraviolet superradiance from mega-networks of {Tryptophan} in
  biological architectures}.
\bjournal{The Journal of Physical Chemistry B}
\bvolume{128}
\bpages{4035--4046}.
\end{barticle}
\endbibitem

\bibitem{Bartlett_SM}
\begin{barticle}[author]
\bauthor{\bsnm{Bartlett},~\bfnm{M.~S.}\binits{M.~S.}}
(\byear{1937}).
\btitle{The statistical conception of mental factors}.
\bjournal{British Journal of Philosophy}
\bvolume{28}
\bpages{97--104}.
\end{barticle}
\endbibitem

\bibitem{SMref_Bhatta}
\begin{barticle}[author]
\bauthor{\bsnm{Bhattacharyya},~\bfnm{A.}\binits{A.}}
(\byear{1943}).
\btitle{On a measure of divergence between two statistical populations defined
  by their probability distributions}.
\bjournal{{Bulletin of the Calcutta Mathematical Society}}
\bvolume{35}
\bpages{99--109}.
\end{barticle}
\endbibitem

\bibitem{Bilgrau2020}
\begin{barticle}[author]
\bauthor{\bsnm{Bilgrau},~\bfnm{A.~E.}\binits{A.~E.}},
  \bauthor{\bsnm{Peeters},~\bfnm{C.~F.~W.}\binits{C.~F.~W.}},
  \bauthor{\bsnm{Eriksen},~\bfnm{P.~S.}\binits{P.~S.}},
  \bauthor{\bsnm{Boegsted},~\bfnm{M.}\binits{M.}} \AND \bauthor{\bsnm{{van
  Wieringen}},~\bfnm{W.~N.}\binits{W.~N.}}
(\byear{2020}).
\btitle{Targeted fused ridge estimation of inverse covariance matrices from
  multiple high-dimensional data classes}.
\bjournal{Journal of Machine Learning Research}
\bvolume{21}
\bpages{1--52}.
\end{barticle}
\endbibitem

\bibitem{SMref_Bollen1989}
\begin{bbook}[author]
\bauthor{\bsnm{Bollen},~\bfnm{K.~A.}\binits{K.~A.}}
(\byear{1989}).
\btitle{{Structural Equations with Latent Variables}}.
\bpublisher{{N}ew {Y}ork [etc.]: {J}ohn {W}iley \& {S}ons, {I}nc.}
\end{bbook}
\endbibitem

\bibitem{Borsboom2017}
\begin{barticle}[author]
\bauthor{\bsnm{Borsboom},~\bfnm{D.}\binits{D.}}
(\byear{2017}).
\btitle{A network theory of mental disorders}.
\bjournal{World Psychiatry}
\bvolume{16}
\bpages{5--13}.
\end{barticle}
\endbibitem

\bibitem{Boyle2017}
\begin{barticle}[author]
\bauthor{\bsnm{Boyle},~\bfnm{E.~A.}\binits{E.~A.}},
  \bauthor{\bsnm{Li},~\bfnm{Y.~I.}\binits{Y.~I.}} \AND
  \bauthor{\bsnm{Pritchard},~\bfnm{J.~K.}\binits{J.~K.}}
(\byear{2017}).
\btitle{An expanded view of complex traits: {From} polygenic to omnigenic}.
\bjournal{Cell}
\bvolume{169}
\bpages{1177--1186}.
\end{barticle}
\endbibitem

\bibitem{Browne1984}
\begin{barticle}[author]
\bauthor{\bsnm{Browne},~\bfnm{M.~W.}\binits{M.~W.}}
(\byear{1984}).
\btitle{Asymptotically distribution-free methods for the analysis of covariance
  structures}.
\bjournal{British Journal of Mathematical \& Statistical Psychology}
\bvolume{37}
\bpages{62--83}.
\end{barticle}
\endbibitem

\bibitem{Clauset2008}
\begin{barticle}[author]
\bauthor{\bsnm{Clauset},~\bfnm{A.}\binits{A.}},
  \bauthor{\bsnm{Moore},~\bfnm{C.}\binits{C.}} \AND
  \bauthor{\bsnm{Newman},~\bfnm{M.~E.~J.}\binits{M.~E.~J.}}
(\byear{2008}).
\btitle{Hierarchical structure and the prediction of missing links in
  networks}.
\bjournal{Nature}
\bvolume{453}
\bpages{98--101}.
\end{barticle}
\endbibitem

\bibitem{Cohen_SM}
\begin{bbook}[author]
\bauthor{\bsnm{Cohen},~\bfnm{J.}\binits{J.}}
(\byear{1988}).
\btitle{{Statistical Power Analysis for the Behavioral Sciences}}.
\bpublisher{{Hillsdale, NJ: Erlbaum (2nd ed.)}}.
\end{bbook}
\endbibitem

\bibitem{Colbrook2022}
\begin{barticle}[author]
\bauthor{\bsnm{Colbrook},~\bfnm{M.~J.}\binits{M.~J.}},
  \bauthor{\bsnm{Antun},~\bfnm{V.}\binits{V.}} \AND
  \bauthor{\bsnm{Hansen},~\bfnm{A.~C.}\binits{A.~C.}}
(\byear{2022}).
\btitle{The difficulty of computing stable and accurate neural networks: {O}n
  the barriers of deep learning and {Smale's} 18th problem}.
\bjournal{Proceedings of the National Academy of Sciences of America}
\bvolume{119}
\bpages{e2107151119}.
\end{barticle}
\endbibitem

\bibitem{Courville2006}
\begin{barticle}[author]
\bauthor{\bsnm{Courville},~\bfnm{A.~C.}\binits{A.~C.}},
  \bauthor{\bsnm{Daw},~\bfnm{N.~D.}\binits{N.~D.}} \AND
  \bauthor{\bsnm{Touretzky},~\bfnm{D.~S.}\binits{D.~S.}}
(\byear{2006}).
\btitle{Bayesian theories of conditioning in a changing world}.
\bjournal{Trends in Cognitive Science}
\bvolume{10}
\bpages{294--300}.
\end{barticle}
\endbibitem

\bibitem{ref_Hilbert_SM}
\begin{bbook}[author]
\bauthor{\bsnm{Dieudonn\'{e}},~\bfnm{J.}\binits{J.}}
(\byear{1960}).
\btitle{{Foundations of Modern Analysis}}.
\bpublisher{{New York: Academic Press}}.
\end{bbook}
\endbibitem

\bibitem{ref_GEOGROU_SM}
\begin{bbook}[author]
\bauthor{\bsnm{Dieudonn\'{e}},~\bfnm{J.}\binits{J.}}
(\byear{1963}).
\btitle{{La G\'{e}om\'{e}trie des Groupes Classiques}}.
\bpublisher{{Berlin: Springer-Verlag} (2nd ed.)}.
\end{bbook}
\endbibitem

\bibitem{ref_Duncan1944_SM}
\begin{barticle}[author]
\bauthor{\bsnm{Duncan},~\bfnm{W.~J.}\binits{W.~J.}}
(\byear{1944}).
\btitle{Some devices for the solution of large sets of simultaneous linear
  equations (with an appendix on the reciprocation of partitioned matrices)}.
\bjournal{London, Edinburgh, and Dublin Philosophical Magazine and Journal of
  Science: Series 7}
\bvolume{35}
\bpages{660--670}.
\end{barticle}
\endbibitem

\bibitem{Feynman1957}
\begin{barticle}[author]
\bauthor{\bsnm{Feynman},~\bfnm{R.~P.}\binits{R.~P.}},
  \bauthor{\bsnm{Vernon},~\bfnm{F.~L.}\binits{F.~L.}} \AND
  \bauthor{\bsnm{Hellwarth},~\bfnm{R.~W.}\binits{R.~W.}}
(\byear{1957}).
\btitle{Geometrical representation of the {S}chr\"{o}dinger equation for
  solving {Maser} problems}.
\bjournal{Journal of Applied Physics}
\bvolume{28}
\bpages{49--52}.
\end{barticle}
\endbibitem

\bibitem{ref_Gautschi_SM}
\begin{bincollection}[author]
\bauthor{\bsnm{Gautschi},~\bfnm{W.}\binits{W.}}
(\byear{1972}).
\btitle{Error function and {Fresnel} integrals}.
In \bbooktitle{{Handbook of Mathematical Functions with Formulas, Graphs, and
  Mathematical Tables}}
(\beditor{\bfnm{M.}\binits{M.}~\bsnm{Abramowitz}} \AND
  \beditor{\bfnm{I.~A.}\binits{I.~A.}~\bsnm{Stegun}}, eds.)
\bchapter{7},
\bpages{295--309}.
\bpublisher{{New York: Wiley}}.
\end{bincollection}
\endbibitem

\bibitem{Giannone2021}
\begin{barticle}[author]
\bauthor{\bsnm{Giannone},~\bfnm{D.}\binits{D.}},
  \bauthor{\bsnm{Lenza},~\bfnm{M.}\binits{M.}} \AND
  \bauthor{\bsnm{Primiceri},~\bfnm{G.~E.}\binits{G.~E.}}
(\byear{2021}).
\btitle{Economic predictions with big data: {T}he illusion of sparsity}.
\bjournal{Econometrica}
\bvolume{89}
\bpages{2409--2437}.
\end{barticle}
\endbibitem

\bibitem{SMref_GDproc}
\begin{bbook}[author]
\bauthor{\bsnm{Gower},~\bfnm{J.~C.}\binits{J.~C.}} \AND
  \bauthor{\bsnm{Dijksterhuis},~\bfnm{G.~B.}\binits{G.~B.}}
(\byear{2004}).
\btitle{Procrustes Problems}.
\bpublisher{Oxford: Oxford University Press}.
\end{bbook}
\endbibitem

\bibitem{Gut_SM}
\begin{bbook}[author]
\bauthor{\bsnm{Gut},~\bfnm{A.}\binits{A.}}
(\byear{2005}).
\btitle{{Probability: A Graduate Course}}.
\bpublisher{{New York: Springer}}.
\end{bbook}
\endbibitem

\bibitem{ref_Guttman1955_SM}
\begin{barticle}[author]
\bauthor{\bsnm{Guttman},~\bfnm{L.}\binits{L.}}
(\byear{1955}).
\btitle{The determinacy of factor score matrices with implications for five
  other basic problems of common factor theory}.
\bjournal{British Journal of Mathematical and Statistical Psychology}
\bvolume{8}
\bpages{65--81}.
\end{barticle}
\endbibitem

\bibitem{SMref_Harris62}
\begin{barticle}[author]
\bauthor{\bsnm{Harris},~\bfnm{C.}\binits{C.}}
(\byear{1962}).
\btitle{Some {Rao-Guttman} relationships}.
\bjournal{Psychometrika}
\bvolume{27}
\bpages{247--263}.
\end{barticle}
\endbibitem

\bibitem{Hartwig_SM}
\begin{barticle}[author]
\bauthor{\bsnm{Hartwig},~\bfnm{R.~E.}\binits{R.~E.}}
(\byear{1986}).
\btitle{The reverse order law revisited}.
\bjournal{{Linear Algebra and its Applications}}
\bvolume{76}
\bpages{241--246}.
\end{barticle}
\endbibitem

\bibitem{SMref_Hellinger1909}
\begin{barticle}[author]
\bauthor{\bsnm{Hellinger},~\bfnm{E.}\binits{E.}}
(\byear{1909}).
\btitle{Neue begr\"{u}ndung der theorie quadratischer formen von
  unendlichvielen Ver\"{a}nderlichen}.
\bjournal{Journal f\"{u}r die reine und angewandte Mathematik}
\bvolume{136}
\bpages{210--271}.
\end{barticle}
\endbibitem

\bibitem{SMref_Hendrickson1964}
\begin{barticle}[author]
\bauthor{\bsnm{Hendrickson},~\bfnm{A.~E.}\binits{A.~E.}} \AND
  \bauthor{\bsnm{White},~\bfnm{P.~O.}\binits{P.~O.}}
(\byear{1964}).
\btitle{{PROMAX: A quick method for rotation to oblique simple structure}}.
\bjournal{British Journal of Mathematical \& Statistical Psychology}
\bvolume{17}
\bpages{65--70}.
\end{barticle}
\endbibitem

\bibitem{SMref_Horn1985}
\begin{bbook}[author]
\bauthor{\bsnm{Horn},~\bfnm{R.~A.}\binits{R.~A.}} \AND
  \bauthor{\bsnm{Johnson},~\bfnm{C.~A.}\binits{C.~A.}}
(\byear{1985}).
\btitle{{Matrix Analysis}}.
\bpublisher{Cambridge: Cambdridge University Press}.
\end{bbook}
\endbibitem

\bibitem{ref_Horst_SM}
\begin{bbook}[author]
\bauthor{\bsnm{Horst},~\bfnm{P.}\binits{P.}}
(\byear{1965}).
\btitle{{Factor Analysis of Data Matrices}}.
\bpublisher{{Holt, Rinehart and Winston}}.
\end{bbook}
\endbibitem

\bibitem{SMref_OLC}
\begin{barticle}[author]
\bauthor{\bsnm{Inman},~\bfnm{H.~F.}\binits{H.~F.}} \AND
  \bauthor{\bsnm{Bradley},~\bfnm{E.~L.}\binits{E.~L.}}
(\byear{1989}).
\btitle{The overlapping coefficient as a measure of agreement between
  probability distributions and point estimation of the overlap of two normal
  densities}.
\bjournal{{Communications in Statistics -- Theory and Methods}}
\bvolume{18}
\bpages{3851--3874}.
\end{barticle}
\endbibitem

\bibitem{SMref_Kailath}
\begin{barticle}[author]
\bauthor{\bsnm{Kailath},~\bfnm{T.}\binits{T.}}
(\byear{1967}).
\btitle{{The divergence and Bhattacharyya distance measures in signal
  selection}}.
\bjournal{{IEEE Transactions on Communication Technology}}
\bvolume{15}
\bpages{52--60}.
\end{barticle}
\endbibitem

\bibitem{SM_Ref_Kaiser1958}
\begin{barticle}[author]
\bauthor{\bsnm{Kaiser},~\bfnm{H.~F.}\binits{H.~F.}}
(\byear{1958}).
\btitle{{The Varimax criterion for analytic rotation in factor analysis}}.
\bjournal{Psychometrika}
\bvolume{23}
\bpages{187--200}.
\end{barticle}
\endbibitem

\bibitem{Kiely2026}
\begin{barticle}[author]
\bauthor{\bsnm{Kiely},~\bfnm{A.}\binits{A.}},
  \bauthor{\bsnm{Chisholm},~\bfnm{D.~A.}\binits{D.~A.}},
  \bauthor{\bsnm{Touil},~\bfnm{A.}\binits{A.}},
  \bauthor{\bsnm{Deffner},~\bfnm{S.}\binits{S.}},
  \bauthor{\bsnm{Landi},~\bfnm{G.}\binits{G.}} \AND
  \bauthor{\bsnm{Campbell},~\bfnm{S.}\binits{S.}}
(\byear{2026}).
\btitle{Metrological approach to the emergence of classical objectivity}.
\bjournal{Physical Review A}
\bvolume{113}
\bpages{022403}.
\end{barticle}
\endbibitem

\bibitem{King1996}
\begin{barticle}[author]
\bauthor{\bsnm{King},~\bfnm{D.}\binits{D.}}
(\byear{1996}).
\btitle{Is the human mind a {Turing} machine?}
\bjournal{Synthese}
\bvolume{108}
\bpages{379--389}.
\end{barticle}
\endbibitem

\bibitem{SMref_Kochbook}
\begin{bbook}[author]
\bauthor{\bsnm{Koch},~\bfnm{I.}\binits{I.}}
(\byear{2014}).
\btitle{{Analysis of Multivariate and High-Dimensional Data}}.
\bpublisher{Cambridge: Cambridge University Press}.
\end{bbook}
\endbibitem

\bibitem{ref_Koster1996_SM}
\begin{barticle}[author]
\bauthor{\bsnm{Koster},~\bfnm{J.~T.~A.}\binits{J.~T.~A.}}
(\byear{1996}).
\btitle{Markov properties of nonrecursive causal models}.
\bjournal{Annals of Statistics}
\bvolume{24}
\bpages{2148--2177}.
\end{barticle}
\endbibitem

\bibitem{ref_Koster1999_SM}
\begin{barticle}[author]
\bauthor{\bsnm{Koster},~\bfnm{J.~T.~A.}\binits{J.~T.~A.}}
(\byear{1999}).
\btitle{{On the validity of the Markov interpretation of path diagrams of
  Gaussian structural equations systems with correlated errors}}.
\bjournal{Scandinavian Journal of Statistics}
\bvolume{26}
\bpages{413--431}.
\end{barticle}
\endbibitem

\bibitem{Krzywinski2012}
\begin{barticle}[author]
\bauthor{\bsnm{Krzywinski},~\bfnm{M.}\binits{M.}},
  \bauthor{\bsnm{Birol},~\bfnm{I.}\binits{I.}},
  \bauthor{\bsnm{Jones},~\bfnm{S.~J.~M.}\binits{S.~J.~M.}} \AND
  \bauthor{\bsnm{Marra},~\bfnm{M.~A.}\binits{M.~A.}}
(\byear{2012}).
\btitle{Hive plots -- rational approach to visualizing networks}.
\bjournal{Briefings in Bioinformatics}
\bvolume{13}
\bpages{627--644}.
\end{barticle}
\endbibitem

\bibitem{ref_Moral_SM}
\begin{barticle}[author]
\bauthor{\bsnm{Lauritzen},~\bfnm{S.~L.}\binits{S.~L.}},
  \bauthor{\bsnm{Dawid},~\bfnm{A.~P.}\binits{A.~P.}},
  \bauthor{\bsnm{Larsen},~\bfnm{B.~N.}\binits{B.~N.}} \AND
  \bauthor{\bsnm{Leimer},~\bfnm{{H. -G. }}\binits{H.}}
(\byear{1990}).
\btitle{Independence properties of directed {M}arkov fields}.
\bjournal{Networks}
\bvolume{20}
\bpages{491--505}.
\end{barticle}
\endbibitem

\bibitem{ref_Liouville_SM}
\begin{barticle}[author]
\bauthor{\bsnm{Liouville},~\bfnm{J.}\binits{J.}}
(\byear{1833}).
\btitle{Note sur la d\'{e}termination des int\'{e}grales dont la valeur est
  alg\'{e}brique}.
\bjournal{Journal f\"{u}r die reine und angewandte Mathematik}
\bvolume{10}
\bpages{347--359}.
\end{barticle}
\endbibitem

\bibitem{ref_Luenberg_SM}
\begin{bbook}[author]
\bauthor{\bsnm{Luenberger},~\bfnm{D.~G.}\binits{D.~G.}}
(\byear{1969}).
\btitle{{Optimization by Vector Space Methods}}.
\bpublisher{{New York: John Wiley \& Sons}}.
\end{bbook}
\endbibitem

\bibitem{SMref_Maha30}
\begin{barticle}[author]
\bauthor{\bsnm{Mahalanobis},~\bfnm{P.~C.}\binits{P.~C.}}
(\byear{1930}).
\btitle{On tests and measures of group divergence}.
\bjournal{{Journal and Proceedings of the Asiatic Society of Bengal}}
\bvolume{26}
\bpages{541--588}.
\end{barticle}
\endbibitem

\bibitem{SMref_Maha36}
\begin{barticle}[author]
\bauthor{\bsnm{Mahalanobis},~\bfnm{P.~C.}\binits{P.~C.}}
(\byear{1936}).
\btitle{On the generalized distance in statistics}.
\bjournal{{Proceedings of the National Institute of Sciences India}}
\bvolume{2}
\bpages{49--55}.
\end{barticle}
\endbibitem

\bibitem{SMref_Maraun1996}
\begin{barticle}[author]
\bauthor{\bsnm{Maraun},~\bfnm{M.~D.}\binits{M.~D.}}
(\byear{1996}).
\btitle{Metaphor taken as math: {I}ndeterminancy in the factor analysis model}.
\bjournal{Multivariate Behavioral Research}
\bvolume{31}
\bpages{517--538}.
\end{barticle}
\endbibitem

\bibitem{SMref_MP}
\begin{barticle}[author]
\bauthor{\bsnm{Mar\v{c}enko},~\bfnm{V.~A.}\binits{V.~A.}} \AND
  \bauthor{\bsnm{Pastur},~\bfnm{L.~A.}\binits{L.~A.}}
(\byear{1967}).
\btitle{The distribution of eigenvalues in certain sets of random matrices}.
\bjournal{Matematicheskii Sbornik}
\bvolume{72}
\bpages{507--536}.
\end{barticle}
\endbibitem

\bibitem{SMref_McDon}
\begin{barticle}[author]
\bauthor{\bsnm{{McDonald}},~\bfnm{R.~P.}\binits{R.~P.}}
(\byear{1974}).
\btitle{The measurement of factor indeterminacy}.
\bjournal{Psychometrika}
\bvolume{39}
\bpages{203--222}.
\end{barticle}
\endbibitem

\bibitem{Minko_SM}
\begin{bbook}[author]
\bauthor{\bsnm{Minkowski},~\bfnm{H.}\binits{H.}}
(\byear{1910}).
\btitle{{Geometrie der Zahlen}}.
\bpublisher{{Leipzig: Teubner}}.
\end{bbook}
\endbibitem

\bibitem{Moore}
\begin{barticle}[author]
\bauthor{\bsnm{Moore},~\bfnm{E.~H.}\binits{E.~H.}}
(\byear{1920}).
\btitle{On the reciprocal of the general algebraic matrix}.
\bjournal{Bulletin of the American Mathematical Society}
\bvolume{26}
\bpages{394--395}.
\end{barticle}
\endbibitem

\bibitem{SMref_Mulaik1976}
\begin{barticle}[author]
\bauthor{\bsnm{Mulaik},~\bfnm{S.~A.}\binits{S.~A.}}
(\byear{1976}).
\btitle{Comments on ``The measurement of factorial indeterminacy"}.
\bjournal{Psychometrika}
\bvolume{41}
\bpages{249--262}.
\end{barticle}
\endbibitem

\bibitem{Muthen1984}
\begin{barticle}[author]
\bauthor{\bsnm{Muth\'{e}n},~\bfnm{B.~O.}\binits{B.~O.}}
(\byear{1984}).
\btitle{A general structural equation model with dichotomous, ordered
  categorical, and continuous latent variable indicators}.
\bjournal{Psychometrika}
\bvolume{49}
\bpages{115--132}.
\end{barticle}
\endbibitem

\bibitem{Nayebi2023}
\begin{binproceedings}[author]
\bauthor{\bsnm{Nayebi},~\bfnm{A.}\binits{A.}},
  \bauthor{\bsnm{Rajalingham},~\bfnm{R.}\binits{R.}},
  \bauthor{\bsnm{Jazayeri},~\bfnm{M.}\binits{M.}} \AND
  \bauthor{\bsnm{Yang},~\bfnm{G.~R.}\binits{G.~R.}}
(\byear{2023}).
\btitle{Neural foundations of mental simulation: {F}uture prediction of latent
  representations on dynamic scenes}
In \bbooktitle{{Advances in Neural Information Processing Systems 36 (NeurIPS
  2023)}}.
\end{binproceedings}
\endbibitem

\bibitem{Newman2010}
\begin{bbook}[author]
\bauthor{\bsnm{Newman},~\bfnm{M.~E.~J.}\binits{M.~E.~J.}}
(\byear{2010}).
\btitle{{Networks: An Introduction}}.
\bpublisher{{Oxford: Oxford University Press}}.
\end{bbook}
\endbibitem

\bibitem{Nielsen2000}
\begin{bbook}[author]
\bauthor{\bsnm{Nielsen},~\bfnm{M.~A.}\binits{M.~A.}} \AND
  \bauthor{\bsnm{Chuang},~\bfnm{I.~L.}\binits{I.~L.}}
(\byear{2000}).
\btitle{{Quantum Computation and Quantum Information}}.
\bpublisher{{Cambridge: Cambridge University Press}}.
\end{bbook}
\endbibitem

\bibitem{Nuijten2016}
\begin{barticle}[author]
\bauthor{\bsnm{Nuijten},~\bfnm{M.~B.}\binits{M.~B.}},
  \bauthor{\bsnm{Deserno},~\bfnm{M.~K.}\binits{M.~K.}},
  \bauthor{\bsnm{Cramer},~\bfnm{A.~O.~J.}\binits{A.~O.~J.}} \AND
  \bauthor{\bsnm{Borsboom},~\bfnm{D.}\binits{D.}}
(\byear{2016}).
\btitle{Mental disorders as complex networks: {A}n introduction and overview of
  a network approach to psychopathology}.
\bjournal{Clinical Neuropsychiatry}
\bvolume{13}
\bpages{68--76}.
\end{barticle}
\endbibitem

\bibitem{Panitchayangkoon2010}
\begin{barticle}[author]
\bauthor{\bsnm{Panitchayangkoon},~\bfnm{G.}\binits{G.}},
  \bauthor{\bsnm{Hayes},~\bfnm{D.}\binits{D.}},
  \bauthor{\bsnm{Fransted},~\bfnm{K.~A.}\binits{K.~A.}},
  \bauthor{\bsnm{Caram},~\bfnm{J.~R.}\binits{J.~R.}},
  \bauthor{\bsnm{Harel},~\bfnm{E.}\binits{E.}},
  \bauthor{\bsnm{Wen},~\bfnm{J.}\binits{J.}},
  \bauthor{\bsnm{Blankenship},~\bfnm{R.~E.}\binits{R.~E.}} \AND
  \bauthor{\bsnm{Engel},~\bfnm{G.~S.}\binits{G.~S.}}
(\byear{2010}).
\btitle{Long-lived quantum coherence in photosynthetic complexes at
  physiological temperature}.
\bjournal{Proceedings of the National Academy of Sciences of the United States
  of America}
\bvolume{107}
\bpages{12766--12770}.
\end{barticle}
\endbibitem

\bibitem{SMref_Pearl2009}
\begin{bbook}[author]
\bauthor{\bsnm{Pearl},~\bfnm{J.}\binits{J.}}
(\byear{2009}).
\btitle{Causality: {Models, Reasoning, and Inference}}.
\bpublisher{{Cambridge: Cambridge University Press} (2nd ed.)}.
\end{bbook}
\endbibitem

\bibitem{Peeters2012_SM}
\begin{barticle}[author]
\bauthor{\bsnm{Peeters},~\bfnm{C.~F.~W.}\binits{C.~F.~W.}}
(\byear{2012}).
\btitle{Rotational Uniqueness Conditions Under Oblique Factor Correlation
  Metric}.
\bjournal{Psychometrika}
\bvolume{77}
\bpages{288--292}.
\end{barticle}
\endbibitem

\bibitem{SMref_Peeters2012}
\begin{bphdthesis}[author]
\bauthor{\bsnm{Peeters},~\bfnm{C.~F.~W.}\binits{C.~F.~W.}}
(\byear{2012}).
\btitle{{Bayesian Exploratory and Confirmatory Factor Analysis: Perspectives on
  Constrained-Model Selection}},
\btype{PhD thesis},
\bpublisher{{Dept.\ of Methodology \& Statistics, Utrecht University}}.
\end{bphdthesis}
\endbibitem

\bibitem{Penrose}
\begin{barticle}[author]
\bauthor{\bsnm{Penrose},~\bfnm{R.}\binits{R.}}
(\byear{1955}).
\btitle{A generalized inverse for matrices}.
\bjournal{Mathematical Proceedings of the Cambridge Philosophical Society}
\bvolume{51}
\bpages{406--413}.
\end{barticle}
\endbibitem

\bibitem{Penrose1994}
\begin{bbook}[author]
\bauthor{\bsnm{Penrose},~\bfnm{R.}\binits{R.}}
(\byear{1994}).
\btitle{{Shadows of the Mind: A Search for the Missing Science of
  Consciousness}}.
\bpublisher{{Oxford: Oxford University Press}}.
\end{bbook}
\endbibitem

\bibitem{SMref_MTricks}
\begin{bbook}[author]
\bauthor{\bsnm{Puntanen},~\bfnm{S.}\binits{S.}},
  \bauthor{\bsnm{Styan},~\bfnm{G.~P.~H.}\binits{G.~P.~H.}} \AND
  \bauthor{\bsnm{Isotalo},~\bfnm{J.}\binits{J.}}
(\byear{2011}).
\btitle{{Matrix Tricks for Linear Statistical Models}}.
\bpublisher{{Berlin: Springer-Verlag}}.
\end{bbook}
\endbibitem

\bibitem{SMref_Rao1955}
\begin{barticle}[author]
\bauthor{\bsnm{Rao},~\bfnm{C.~R.}\binits{C.~R.}}
(\byear{1955}).
\btitle{Estimation and tests of significance in factor analysis}.
\bjournal{Psychometrika}
\bvolume{20}
\bpages{93--111}.
\end{barticle}
\endbibitem

\bibitem{SMref_Rich2003}
\begin{barticle}[author]
\bauthor{\bsnm{Richardson},~\bfnm{T.}\binits{T.}}
(\byear{2003}).
\btitle{Markov properties for acyclic directed mixed graphs}.
\bjournal{Scandinavian Journal of Statistics}
\bvolume{30}
\bpages{145--157}.
\end{barticle}
\endbibitem

\bibitem{SMref_Rohe2023}
\begin{barticle}[author]
\bauthor{\bsnm{Rohe},~\bfnm{K.}\binits{K.}} \AND
  \bauthor{\bsnm{Zheng},~\bfnm{M.}\binits{M.}}
(\byear{2023}).
\btitle{{Vintage factor analysis with Varimax performs statistical inference}}.
\bjournal{Journal of the Royal Statistical Society, Series B}
\bvolume{85}
\bpages{1037--1060}.
\end{barticle}
\endbibitem

\bibitem{Rosseel2024}
\begin{barticle}[author]
\bauthor{\bsnm{Rosseel},~\bfnm{Y.}\binits{Y.}} \AND
  \bauthor{\bsnm{Loh},~\bfnm{W.~W.}\binits{W.~W.}}
(\byear{2024}).
\btitle{A structural after measurement approach to structural equation
  modeling}.
\bjournal{Psychological Methods}
\bvolume{29}
\bpages{561--588}.
\end{barticle}
\endbibitem

\bibitem{SMref_Singh}
\begin{barticle}[author]
\bauthor{\bsnm{Singh},~\bfnm{D.}\binits{D.}},
  \bauthor{\bsnm{Febbo},~\bfnm{P.~G.}\binits{P.~G.}},
  \bauthor{\bsnm{Ross},~\bfnm{K.}\binits{K.}},
  \bauthor{\bsnm{Jackson},~\bfnm{D.~G.}\binits{D.~G.}},
  \bauthor{\bsnm{Manola},~\bfnm{J}\binits{J.}},
  \bauthor{\bsnm{Ladd},~\bfnm{C}\binits{C.}},
  \bauthor{\bsnm{Tamayo},~\bfnm{P.}\binits{P.}},
  \bauthor{\bsnm{Renshaw},~\bfnm{A.~A.}\binits{A.~A.}},
  \bauthor{\bsnm{{D'Amico}},~\bfnm{A.~V.}\binits{A.~V.}},
  \bauthor{\bsnm{Richie},~\bfnm{J.~P.}\binits{J.~P.}},
  \bauthor{\bsnm{Lander},~\bfnm{E.~S.}\binits{E.~S.}},
  \bauthor{\bsnm{Loda},~\bfnm{M.}\binits{M.}},
  \bauthor{\bsnm{Kantoff},~\bfnm{P.~W.}\binits{P.~W.}},
  \bauthor{\bsnm{Golub},~\bfnm{T.~R.}\binits{T.~R.}} \AND
  \bauthor{\bsnm{Sellers},~\bfnm{W.~R.}\binits{W.~R.}}
(\byear{2002}).
\btitle{Gene expression correlates of clinical prostate cancer behavior}.
\bjournal{Cancer Cell}
\bvolume{1}
\bpages{P203--209}.
\end{barticle}
\endbibitem

\bibitem{Smale1998}
\begin{barticle}[author]
\bauthor{\bsnm{Smale},~\bfnm{S.}\binits{S.}}
(\byear{1998}).
\btitle{Mathematical problems for the next century}.
\bjournal{Mathematical Intelligencer}
\bvolume{20}
\bpages{7--15}.
\end{barticle}
\endbibitem

\bibitem{SMref_SpirtesSMR98}
\begin{barticle}[author]
\bauthor{\bsnm{Spirtes},~\bfnm{P.}\binits{P.}},
  \bauthor{\bsnm{Richardson},~\bfnm{T.}\binits{T.}},
  \bauthor{\bsnm{Meek},~\bfnm{C.}\binits{C.}},
  \bauthor{\bsnm{Scheines},~\bfnm{R.}\binits{R.}} \AND
  \bauthor{\bsnm{Glymour},~\bfnm{C.}\binits{C.}}
(\byear{1998}).
\btitle{Using path diagrams as a structural equation modeling tool}.
\bjournal{{Sociological Methods \& Research}}
\bvolume{27}
\bpages{182--225}.
\end{barticle}
\endbibitem

\bibitem{Tegmark2000}
\begin{barticle}[author]
\bauthor{\bsnm{Tegmark},~\bfnm{E.}\binits{E.}}
(\byear{2000}).
\btitle{Importance of quantum decoherence in brain processes}.
\bjournal{{Physical Review E}}
\bvolume{61}
\bpages{4194--4206}.
\end{barticle}
\endbibitem

\bibitem{TBprobPCA}
\begin{barticle}[author]
\bauthor{\bsnm{Tipping},~\bfnm{M.~E.}\binits{M.~E.}} \AND
  \bauthor{\bsnm{Bishop},~\bfnm{C.~M.}\binits{C.~M.}}
(\byear{1999}).
\btitle{Probabilistic principal component analysis}.
\bjournal{Journal of the Royal Statistical Society, Series B}
\bvolume{61}
\bpages{611--622}.
\end{barticle}
\endbibitem

\bibitem{ref_Aad_SM}
\begin{bbook}[author]
\bauthor{\bsnm{{Van Der Vaart}},~\bfnm{A.~W.}\binits{A.~W.}}
(\byear{1998}).
\btitle{Asymptotic Statistics}.
\bpublisher{Cambridge: {Cambridge University Press}}.
\end{bbook}
\endbibitem

\bibitem{ref_ridgeP}
\begin{barticle}[author]
\bauthor{\bsnm{{van Wieringen}},~\bfnm{W.~N.}\binits{W.~N.}} \AND
  \bauthor{\bsnm{Peeters},~\bfnm{C.~F.~W.}\binits{C.~F.~W.}}
(\byear{2016}).
\btitle{Ridge estimation of inverse covariance matrices from high-dimensional
  data}.
\bjournal{Computational Statistics \& Data Analysis}
\bvolume{103}
\bpages{284--303}.
\end{barticle}
\endbibitem

\bibitem{ref_Woodbury_SM}
\begin{btechreport}[author]
\bauthor{\bsnm{Woodbury},~\bfnm{M.~A.}\binits{M.~A.}}
(\byear{1950}).
\btitle{Inverting modified matrices.}
\btype{{Statistical Research Group Memorandum Report}} No. \bnumber{42},
\bpublisher{{Princeton University}}.
\end{btechreport}
\endbibitem

\bibitem{Handbook_SM}
\begin{bincollection}[author]
\bauthor{\bsnm{Zucker},~\bfnm{R.}\binits{R.}}
(\byear{1972}).
\btitle{Elementary Trancendental Functions: Logarithmic, Exponential, Circular,
  and Hyperbolic Functions}.
In \bbooktitle{{Handbook of Mathematical Functions with Formulas, Graphs, and
  Mathematical Tables}}
(\beditor{\bfnm{M.}\binits{M.}~\bsnm{Abramowitz}} \AND
  \beditor{\bfnm{I.~A.}\binits{I.~A.}~\bsnm{Stegun}}, eds.)
\bchapter{4},
\bpages{65--94}.
\bpublisher{{New York: Wiley}}.
\end{bincollection}
\endbibitem

\end{thebibliography}
\putbib[FAInSM]
\end{bibunit}
%%--------------- References -----------------------------------------
%%--------------------------------------------------------------------

\end{document}